\PassOptionsToPackage{unicode}{hyperref}
\PassOptionsToPackage{hyphens}{url}
\PassOptionsToPackage{dvipsnames,svgnames,x11names}{xcolor}
\documentclass[12pt]{article}

\usepackage{amsmath,amsfonts,bm}

\def\1{\bm{1}}

\def\vone{{\bm{1}}}

\def\vz{{\bm{z}}}

\def\mZ{{\bm{Z}}}

\DeclareMathAlphabet{\mathsfit}{\encodingdefault}{\sfdefault}{m}{sl}
\SetMathAlphabet{\mathsfit}{bold}{\encodingdefault}{\sfdefault}{bx}{n}

\def\gA{{\mathcal{A}}}
\def\gB{{\mathcal{B}}}

\def\gD{{\mathcal{D}}}
\def\gE{{\mathcal{E}}}
\def\gF{{\mathcal{F}}}
\def\gG{{\mathcal{G}}}

\def\gI{{\mathcal{I}}}
\def\gJ{{\mathcal{J}}}

\def\gL{{\mathcal{L}}}
\def\gM{{\mathcal{M}}}

\def\gQ{{\mathcal{Q}}}
\def\gR{{\mathcal{R}}}
\def\gS{{\mathcal{S}}}

\def\gU{{\mathcal{U}}}

\def\gX{{\mathcal{X}}}
\def\gY{{\mathcal{Y}}}
\def\gZ{{\mathcal{Z}}}

\def\sA{{\mathbb{A}}}

\def\sD{{\mathbb{D}}}

\def\sG{{\mathbb{G}}}

\def\sP{{\mathbb{P}}}

\def\sR{{\mathbb{R}}}

\newcommand{\E}{\mathbb{E}}

\DeclareMathOperator*{\argmax}{arg\,max}
\DeclareMathOperator*{\argmin}{arg\,min}

\newcommand{\twonorm}[1]{\|#1\|}

\newcommand{\LRabs}[1]{\left|#1\right|}

\newcommand{\Eqmark}[2]{\stackrel{(#1)}{#2}}

\newcommand{\LRs}[1]{\left(#1\right)}
\newcommand{\LRm}[1]{\left[#1\right]}
\newcommand{\LRl}[1]{\left\{#1\right\}}

\def\ECROMS{\mathrm{E}\text{-}\mathrm{CROMS}}
\def\FCROMS{\mathrm{F}\text{-}\mathrm{CROMS}}
\def\GFCROMS{\mathrm{GF}\text{-}\mathrm{CROMS}}
\def\JCROMS{\mathrm{J}\text{-}\mathrm{CROMS}}

\usepackage{bbm}
\usepackage{amsmath,amssymb}
\usepackage{algorithm}
\usepackage{makecell}
\usepackage[table]{xcolor}
\usepackage{algpseudocodex}
\usepackage{diagbox}
\usepackage{caption}
\usepackage{amsthm}
\newtheorem{theorem}{Theorem}[section]
\newtheorem{lemma}{Lemma}[section]

\newtheorem{assumption}{Assumption}
\newtheorem{proposition}{Proposition}[section]
\newtheorem{remark}{Remark}[section]
\newtheorem{corollary}{Corollary}[section]
\newtheorem{definition}{Definition}

\usepackage{soul}

\usepackage{iftex}
\ifPDFTeX
  \usepackage[T1]{fontenc}
  \usepackage[utf8]{inputenc}
  \usepackage{textcomp} 
\else 
  \usepackage{unicode-math}
  \defaultfontfeatures{Scale=MatchLowercase}
  \defaultfontfeatures[\rmfamily]{Ligatures=TeX,Scale=1}
\fi
\usepackage{lmodern}
\ifPDFTeX\else  
\fi
\IfFileExists{upquote.sty}{\usepackage{upquote}}{}
\IfFileExists{microtype.sty}{
  \usepackage[]{microtype}
  \UseMicrotypeSet[protrusion]{basicmath} 
}{}
\makeatletter
\@ifundefined{KOMAClassName}{
  \IfFileExists{parskip.sty}{%
    \usepackage{parskip}
  }{
    \setlength{\parindent}{0pt}
    \setlength{\parskip}{6pt plus 2pt minus 1pt}}
}{
  \KOMAoptions{parskip=half}}
\makeatother
\usepackage{xcolor}
\makeatletter
\ifx\textbf\undefined\else
  \let\oldparagraph\textbf
  \renewcommand{\textbf}{
    \@ifstar
      \xxxParagraphStar
      \xxxParagraphNoStar
  }
  \newcommand{\xxxParagraphStar}[1]{\oldparagraph*{#1}\mbox{}}
  \newcommand{\xxxParagraphNoStar}[1]{\oldparagraph{#1}\mbox{}}
\fi
\ifx\subparagraph\undefined\else
  \let\oldsubparagraph\subparagraph
  \renewcommand{\subparagraph}{
    \@ifstar
      \xxxSubParagraphStar
      \xxxSubParagraphNoStar
  }
  \newcommand{\xxxSubParagraphStar}[1]{\oldsubparagraph*{#1}\mbox{}}
  \newcommand{\xxxSubParagraphNoStar}[1]{\oldsubparagraph{#1}\mbox{}}
\fi
\makeatother

\usepackage{longtable,booktabs,array}
\usepackage{calc} 
\usepackage{etoolbox}
\makeatletter
\patchcmd\longtable{\par}{\if@noskipsec\mbox{}\fi\par}{}{}
\makeatother
\IfFileExists{footnotehyper.sty}{\usepackage{footnotehyper}}{\usepackage{footnote}}
\makesavenoteenv{longtable}
\usepackage{graphicx}
\makeatletter
\def\maxwidth{\ifdim\Gin@nat@width>\linewidth\linewidth\else\Gin@nat@width\fi}
\def\maxheight{\ifdim\Gin@nat@height>\textheight\textheight\else\Gin@nat@height\fi}
\makeatother
\setkeys{Gin}{width=\maxwidth,height=\maxheight,keepaspectratio}
\makeatletter
\def\fps@figure{htbp}
\makeatother

\makeatletter
\@ifpackageloaded{caption}{}{\usepackage{caption}}
\AtBeginDocument{%
\ifdefined\contentsname
  \renewcommand*\contentsname{Table of contents}
\else
  \newcommand\contentsname{Table of contents}
\fi
\ifdefined\listfigurename
  \renewcommand*\listfigurename{List of Figures}
\else
  \newcommand\listfigurename{List of Figures}
\fi
\ifdefined\listtablename
  \renewcommand*\listtablename{List of Tables}
\else
  \newcommand\listtablename{List of Tables}
\fi
\ifdefined\figurename
  \renewcommand*\figurename{Figure}
\else
  \newcommand\figurename{Figure}
\fi
\ifdefined\tablename
  \renewcommand*\tablename{Table}
\else
  \newcommand\tablename{Table}
\fi
}
\@ifpackageloaded{float}{}{\usepackage{float}}
\floatstyle{ruled}
\@ifundefined{c@chapter}{\newfloat{codelisting}{h}{lop}}{\newfloat{codelisting}{h}{lop}[chapter]}
\floatname{codelisting}{Listing}

\makeatother
\makeatletter
\@ifpackageloaded{caption}{}{\usepackage{caption}}
\@ifpackageloaded{subcaption}{}{\usepackage{subcaption}}
\makeatother

\ifLuaTeX
  \usepackage{selnolig}  
\fi
\usepackage[]{natbib}
\usepackage{bookmark}
\usepackage{setspace}
\usepackage{booktabs}
\usepackage{multirow}

\IfFileExists{xurl.sty}{\usepackage{xurl}}{} 
\hypersetup{
  pdftitle={Title},
  pdfauthor={Author 1; Author 2},
  pdfkeywords={3 to 6 keywords, that do not appear in the title},
  colorlinks=true,
  linkcolor={blue},
  filecolor={Maroon},
  citecolor={Blue},
  urlcolor={Blue},
  pdfcreator={LaTeX via pandoc}}

\newcommand{\anon}{1}

\definecolor{HJblue}{RGB}{0,85,170}

\begin{document}

\def\spacingset#1{\renewcommand{\baselinestretch}%
{#1}\small\normalsize} \spacingset{1}

\makeatletter
\newcommand{\printfnsymbol}[1]{%
  \textsuperscript{\@fnsymbol{#1}}%
}
\makeatother

\if1\anon
{
  \title{\bf Optimal Model Selection for Conformalized Robust Optimization}

  \author{Yajie Bao$^1$\thanks{All authors are listed in alphabetical order.}, Yang Hu$^2$, Haojie Ren$^2$, Peng Zhao$^3$ and Changliang Zou$^1$\\
$^1$ {\normalsize School of Statistics and Data Science, Nankai University}\\
$^2$ {\normalsize School of Mathematical Sciences,  Shanghai Jiao Tong University} \\
$^3$ {\normalsize School of Mathematics and Statistics, Jiangsu Normal University}}
  \maketitle
} \fi

\if0\anon
{
  \bigskip
  \bigskip
  \bigskip
  \begin{center}
    {\LARGE\bf Optimal Model Selection for Conformalized Robust Optimization}
\end{center}
  \medskip
} \fi

\bigskip
\begin{abstract}
In decision-making under uncertainty, Contextual Robust Optimization (CRO) provides reliability by minimizing the worst-case decision loss over a prediction set. While recent advances use conformal prediction to construct prediction sets for machine learning models, the downstream decisions critically depend on model selection.
This paper introduces novel model selection frameworks for CRO that unify robustness control with decision risk minimization. 
{We first propose \emph{Conformalized Robust Optimization with Model Selection} (CROMS), a framework that selects the model to approximately minimize the averaged decision risk in CRO solutions. Given the target robustness level \(1-\alpha\), we present a computationally efficient algorithm called E-CROMS, which achieves asymptotic robustness control and decision optimality. To correct the control bias in finite samples, we further develop two algorithms: F-CROMS, which ensures a $1-\alpha$ robustness but requires searching the label space; and J-CROMS, which offers lower computational cost while achieving a $1-2\alpha$ robustness. Furthermore, we extend the CROMS framework to the \emph{individualized} setting, where model selection is performed by minimizing the conditional decision risk given the covariates of the test data. This framework advances conformal prediction methodology by enabling covariate-aware model selection.}
Numerical results demonstrate significant improvements in decision efficiency across diverse synthetic and real-world applications, outperforming baseline approaches.
\end{abstract}

\noindent%
{\it Keywords:} Conformal prediction; Contextual robust optimization; Empirical risk minimization; Individualized model selection; Uncertainty set
\vfill

\newpage
\spacingset{1.2} 

\section{Introduction}

In high-stakes domains like medical diagnosis or autonomous driving, traditional decision-making methods often focus on optimizing average-case outcomes, making them vulnerable to real-world uncertainties. In contrast, robust decision-making emphasizes resilience by design, ensuring that decisions remain effective even when actual conditions deviate from expectations. This adaptive stability is particularly crucial in fields like healthcare, agriculture, and climate modeling, where complex and uncertain environments require reliable solutions with a certain level of robustness.


Robust optimization \citep{ben2009robust} provides a principled framework for decision-making under uncertainty by optimizing against worst-case realizations within predefined uncertainty sets. 
However, this approach often leads to conservative or impractical solutions, as it does not dynamically incorporate observable covariates to modulate uncertainty. {Contextual robust optimization (CRO)  \citep{chenreddy2022data}} addresses this limitation by shifting the paradigm: instead of static uncertainty sets, CRO constructs data-driven, covariate-dependent prediction sets. This adaptation enables decisions to better reflect the specific situational contexts.

Formally, let $\phi(y,z)$ be a loss function regarding the decision $ z\in \gZ$ and the label $y\in \gY$, where $\gZ$ is the feasible set of decisions and $\gY$ is the label space. The label $Y$ will be predicted by the observed covariate $X \in \gX$. Assume $(X, Y)$ is drawn from an arbitrary distribution $P$, and let $\gU(X)$ be a prediction set for the unknown label $Y$. The CRO decision is obtained by solving the following minmax optimization problem:
\begin{equation}\label{eq:CRO}
    \begin{aligned}
        z(X)=\argmin_{z \in \gZ}&\max_{c\in \gU(X)} \phi(c, z).
    \end{aligned}
\end{equation}
The goal is to ensure that the decision $z(X)$ satisfies \emph{robustness} requirement, {meaning with probability $1-\alpha$, the true decision loss $\phi(Y, z(X))$ is smaller than the worst-case loss in the prediction set $\max_{c\in \gU(X)} \phi(c, z(X))$} (see Definition \ref{def:robustness}). To achieve this, the prediction set $\gU$ is required to have $1-\alpha$ level of marginal \emph{coverage}, i.e., $\sP\{Y\in \gU(X)\} \geq 1-\alpha$. Complex machine learning models, such as deep neural networks \citep{chenreddy2022data} and generative models \citep{patel2024conformal}, have been utilized to train the prediction set $\gU(X)$.
Although these models can provide informative sets for the unknown label, guaranteeing the coverage property remains challenging due to their ``black-box'' nature. Thus, a model-free uncertainty quantification tool is needed to ensure the robustness of decisions.

Recent works have applied the conformal prediction \citep{vovk2005algorithmic} to construct a valid prediction set for CRO problems \citep{johnstone2021conformal,sun2023predict}, leveraging its flexibility and validity for uncertainty quantification. Given any pre-trained machine learning model, conformal prediction constructs a valid prediction set using labeled data, ensuring the marginal coverage guarantee provided the training and test data are independent and identically distributed (i.i.d.) or exchangeable. This allows for a direct solution to problem~\eqref{eq:CRO} while satisfying the robustness of the decision in finite samples.
However, under exchangeability, the marginal coverage guarantee holds for any prediction model \citep{lei2018distribution}. Consequently, the practical performance---particularly the efficiency of downstream decisions---can vary dramatically with the model choice. Especially, it may yield overly conservative decisions when the model performs poorly on test data.



Consider, for example, in medical diagnosis systems, where multiple prediction models are trained on datasets from different hospitals, each with varying patient demographics and equipment. For a new patient, selecting an appropriate model before constructing the conformal prediction set is crucial for effective decision-making. The model selection problem in conformal prediction has gained significant attention in recent works \citep{yang2024selection,liang2024conformal}, which primarily aimed to select the model from a candidate set that minimizes the width of the prediction set while maintaining the validity of marginal coverage. However, minimizing the width is not directly relevant to the risk of the downstream decision, which is vital for practical applications. Additionally, existing works selected models from the viewpoint of average efficiency, failing to adapt models to specific decision contexts. This is particularly problematic in personalized adaptation like precision medicine \citep{mo2021learning}, where models selected based on average criteria might recommend the same treatment for all patients, ignoring individual variations in genetics, lifestyle, or comorbidities. 

\subsection{Our contributions}

In this paper, we develop a novel framework for model selection in the CRO problem with conformal prediction sets, aiming to optimize decision efficiency while guaranteeing robustness. Specifically, we consider a candidate model set $\{S_{\lambda}:\lambda\in \Lambda\}$, incorporating two typical scenarios: (i) $\Lambda$ is a finite index set corresponding to pre-trained models; (ii) $\Lambda \subset \sR^m$ constitutes a continuous parameter space for a class of models. Given labeled data $\{(X_i, Y_i)\}_{i=1}^n$ and test data $X_{n+1}$, our data-driven framework selects the optimal model $\hat{\lambda} \in \Lambda$ and produces a final decision $\hat{z}(X_{n+1})$ that simultaneously satisfies (asymptotic) robustness (see Definitions \ref{def:robustness}, \ref{def:individual_robustness}) and optimality (see Definitions \ref{def:efficiency}, \ref{def:individual_efficiency}). The selection criterion is based on minimizing the decision risk, defined as the expected decision loss for the true (unknown) label of test data. To approximate this risk, we first generate auxiliary decisions based on labeled data and compute their empirical decision losses. Model selection is then performed through \emph{empirical risk minimization} (ERM). Once the model index $\hat{\lambda}$ is selected, its corresponding conformal prediction set is incorporated into the CRO problem to make the final decision $\hat{z}(X_{n+1})$.
The main contributions of this work are as follows:

\begin{itemize}
    \item[(1)] We introduce \emph{Conformalized Robust Optimization with Model Selection} (CROMS), which unifies conformal prediction set construction with decision risk minimization. 
    We first propose ECROMS, a computationally efficient algorithm, and establish bounds in coverage error and excess decision risk under a general candidate model class.
    
    \item[(2)] To correct the coverage error in finite samples, we develop two improved algorithms: F-CROMS achieves $1-\alpha$ coverage and decision optimality via using augmented labeled data and searching over the label space to preserve the exchangeability. A theoretically justified grid-approximation procedure is developed to enable a computationally feasible implementation of F-CROMS for continuous label, without sacrificing coverage control; J-CROMS further reduces the computational cost in a leave-one-out fashion, and constructs the prediction set by the Jackknife+ technique \citep{barber2021predictive}, which achieves a $1-2\alpha$ coverage guarantee.

    \item[(3)] We extend the framework to \emph{Conformalized Robust Optimization with Individualized Model Selection} (CROiMS), which performs \emph{individualized} model selection by minimizing the conditional decision risk given the covariate of test data.
    We prove that CROiMS achieves asymptotic conditional coverage and decision optimality under mild nonparametric assumptions. To the best of our knowledge, this is the first study to introduce covariate-aware model selection in conformal prediction.

    \item[(4)] We conduct extensive numerical experiments on synthetic data, showing superior performance in enhancing decision efficiency and ensuring decision robustness across various settings. Additionally, our implementation on two real medical diagnosis datasets demonstrates that individualized model selection is important for achieving more precise and effective decisions tailored to different patients.
\end{itemize}

\subsection{Connections to existing works}


To improve the efficiency of decisions in the CRO problems, \citet{wang2023learning} and \citet{chenreddy2024end} proposed the \emph{end-to-end} approaches to directly train the uncertainty set by minimizing decision risk on historical data. However, these methods lack a finite-sample robustness guarantee. Their approaches are founded on the broader adoption of the end-to-end framework in predictive optimization \citep{donti2017task,elmachtoub2022smart}, which integrates model training with downstream optimization tasks. {Typically, \citet{yeh2024end} extended this framework to train the conformal uncertainty sets in CRO problems using a sample-splitting strategy: the first part of the labeled data is employed for model selection based on auxiliary decisions, while the second part is used to construct a split conformal prediction set for the final test decision.} While this approach ensures finite-sample coverage, the reduced sample size for constructing the prediction set may compromise decision efficiency to a large extent. {In addition, \citet{kiyani2025decision} studied the theoretically optimal prediction set in the CRO problem, which depends on the conditional distribution information. We emphasize two main differences between \citet{kiyani2025decision} and our work: first, the risk functions used to evaluate the efficiency of prediction sets are different; second, they aimed to use a machine learning model to learn the conditional distribution and then approximate the optimal prediction set, whereas we focus on selecting a model from a candidate set to minimize the downstream decision risk.}

In addition to the coverage property, the efficiency in the size of conformal prediction sets has also been extensively studied. \citet{lei2013distribution} and \citet{sadinle2019least} showed that the optimal prediction set with minimal size satisfying the marginal or conditional coverage is the level set of the conditional density of $Y$ given $X$. There is a line of works constructing prediction sets based on density estimators, see \citet{lei2013nonparametric}, \citet{lei2013distribution}, and \citet{izbicki2022cd}. Given a specific nonconformity score, the optimal size could be asymptotically achieved if the score estimator is consistent \citep{sesia2020comparison,lei2018distribution}. Several works \citep{bai2022efficient,kiyani2024length,braun2025minimum} considered directly minimizing the width of prediction sets by solving a constrained optimization problem. 
Notably, \citet{yang2024selection} and \citet{liang2024conformal} considered selecting the model that minimizes the set size while keeping a valid coverage.



This paper is organized as follows. Section \ref{sec:CROMS} provides a background on CRO problems under conformal prediction sets and introduces the CROMS framework aiming to minimize the decision risk, and then we propose a computationally efficient algorithm with asymptotic robustness control. Section \ref{sec:FCROMS} presents two viable algorithms to achieve the finite-sample robustness guarantee and asymptotic decision optimality. In Section \ref{sec:CROiMS}, we developed the individualized model selection framework CROiMS. In Sections \ref{sec:simulation} and \ref{sec:application}, we show the simulation results on synthetic data and applications on a real dataset, respectively.



\section{Conformalized Robust Optimization with Model Selection}\label{sec:CROMS}

\subsection{Warm-up: CRO with conformal prediction set}\label{sec:preliminaries}

As a starting point, it is useful to examine how the CRO problem with conformal prediction sets can be efficiently solved. Suppose the collected labeled dataset $\{(X_i,Y_i)\}_{i=1}^n$ and test data $X_{n+1}$ with unknown label $Y_{n+1}$ are i.i.d. Let $S:\gX\times\gY \to \sR$ be a pre-trained nonconformity score function. {The $(1-\alpha)$-level conformal prediction set} takes the form
\begin{align}\label{eq:conformal_set_single_model}
    \gU(X_{n+1}) = \LRl{c\in \gY: S(X_{n+1}, c) \leq \hat{q}},
\end{align}
where the calibration threshold $\hat{q} = Q_{(1-\alpha)(1+1/n)}(\{S(X_i,Y_i)\}_{i=1}^n)$ is the $(1-\alpha)(1+1/n)$ sample quantile. This set enjoys the finite-sample \emph{marginal coverage} property $\sP\{Y_{n+1}\in \gU(X_{n+1})\} \geq 1-\alpha$; see \citet{vovk2005algorithmic} and \citet{lei2018distribution}.

\paragraph*{Regression task.}
If $\gY = \sR^p$, there are two commonly used score functions \citep{johansson2017model,sun2023predict}: (1) \emph{Box score} $S(x,y) = \|(y-\hat{\mu}(x))/\hat{\sigma}(x)\|_{\infty}$, where $\hat{\mu}(\cdot), \hat{\sigma}(\cdot): \gX \to \sR^p$ are the mean and variance prediction models; (2) \emph{Ellipsoid score} $S(x, y) = \{(y- \hat{\mu}(x))^{\top}\hat{\Sigma}(x)^{-1}(y- \hat{\mu}(x))\}^{1/2}$, where $\hat{\mu}(\cdot):\gX \to \sR^p$ and $\hat{\Sigma}(\cdot): \gX \to \sR^{p\times p}$ are the estimators of conditional mean and covariance, respectively.
Since the conformal prediction sets are convex with both box and ellipsoid scores,
the CRO problem \eqref{eq:CRO} is tractable as long as the loss function $\phi(y,z)$ is concave in $y$ and convex in $z$. To have a direct intuition on the problem \eqref{eq:CRO} under conformal prediction set in \eqref{eq:conformal_set_single_model}, we consider the classical portfolio optimization application where $\phi(y,z) = -y^{\top}z$ with $\gZ = \{z\in [0,1]^p: \mathbf{1}^{\top}z=1\}$. Under the prediction set with box score, the problem \eqref{eq:CRO} is equivalent to $z(X_{n+1})=\argmin_{z\in \gZ} \{- \LRs{\hat{\mu}(X_{n+1}) - \hat{q} \hat{\sigma}(X_{n+1})}^{\top}z\}$.
Under the prediction set with ellipsoid score, the inner maximization has a closed form and the problem \eqref{eq:CRO} is equivalent to $z(X_{n+1})=\argmin_{z\in \gZ}\{\sqrt{\hat{q}}\sqrt{z^{\top}\hat{\Sigma}(X_{n+1})z} - \hat{\mu}(X_{n+1})^{\top}z\}$.
The problems are both convex and can be efficiently solved by well-studied methods \citep{Boyd_Vandenberghe_2004}.


\paragraph*{Classification task.}
For a discrete and finite label space $\gY$, the nonconformity function can be taken as $S(x,y) = 1 - \hat{f}^y(x)$, where $\hat{f}^y(x)$ is an estimator of $\sP(Y=y\mid X = x)$, such as the softmax output of a neural network. The prediction set is given by $\gU(X_{n+1}) = \{y\in \gY: S(X_{n+1}, y) \leq \hat{q}\}$. The decision space $\gZ$ is typically a finite set. The loss function can be represented by a matrix $M \in \sR^{|\gY| \times |\gZ|}$, where $\phi(y,z) = M_{y,z}$ for $y\in \gY$ and $z\in \gZ$, and the corresponding CRO problem becomes $z(X_{n+1}) = \argmin_{z\in \gZ} \max_{y\in \gU(X_{n+1})} M_{y,z}$, which can be easily solved among a finite set of possible solutions.

In the classical CRO problem, a foundational requirement is the marginal robustness of the decision given a prediction set. Here we state the robustness definition \citep{ben2009robust,sun2023predict} in the marginal notion.

\begin{definition}[Marginal robustness]\label{def:robustness}
    The prediction set $\gU(X_{n+1})$ satisfies $1-\alpha$ level of marginal robustness if $\sP\LRl{\phi(Y_{n+1}, z(X_{n+1})) \leq \max_{c\in \gU(X_{n+1})}\phi(c, z(X_{n+1}))} \geq 1-\alpha$.
\end{definition}

In the above definition, $\phi(Y_{n+1}, z(X_{n+1}))$ represents the ground truth decision loss on the test data, and $\max_{c\in \gU(X_{n+1})}\phi(c, z(X_{n+1}))$ denotes the observed worst-case loss under the prediction set $\gU(X_{n+1})$. As defined, the decision $z(X_{n+1})$ ensures robustness if we use the conformal prediction set in \eqref{eq:conformal_set_single_model} since the marginal robustness can be implied by the marginal coverage property.


\subsection{Oracle model selection to minimize decision risk}

Given a sequence of {pre-trained models} $\{S_{\lambda}:\lambda\in \Lambda\}$, we begin with the oracle model selection at the population level. For the data $(X,Y)\sim P$, we denote the $1-\alpha$ population quantile of score $S_{\lambda}(X,Y)$ as $q_{\lambda}^{o} = \inf\{q\in \sR: \sP\{S_{\lambda}(X,Y) \leq q\} \geq 1-\alpha\}$. Then define the oracle conformal prediction set of the candidate model $S_{\lambda}$ as $\gU_{\lambda}^{o}(X) = \LRl{c\in \gY: S_{\lambda}(X, c) \leq q_{\lambda}^{o}}$.

Plugging the prediction set $\gU_{\lambda}^{o}(X)$ into the CRO problem \eqref{eq:CRO} leads to the decision
\begin{align}\label{eq:oracle_decision}
    z_{\lambda}^{o}(X) = \argmin_{z\in \gZ}\max_{c\in \gU_{\lambda}^{o}(X)} \phi(c,z).
\end{align}
By the definition of $q_{\lambda}^{o}$, the marginal coverage $\sP\{Y\in \gU_{\lambda}^{o}(X)\} \geq 1-\alpha$ holds naturally, and thus $\gU_{\lambda}^{o}(X)$ satisfy the robustness requirement in Definition \ref{def:robustness}. To evaluate the efficiency of $z_{\lambda}^{o}(X)$, we introduce the {\emph{oracle decision risk}} of the model $S_{\lambda}$ as $\E[\phi(Y, z_{\lambda}^{o}(X))]$. Accordingly, the optimal model is the one that minimizes the downstream oracle decision risk,
\begin{align}\label{eq:oracle_optimal_model}
    \lambda^* = \argmin_{\lambda \in \Lambda}\E[\phi(Y, z_{\lambda}^{o}(X))].
\end{align}
From another perspective, the oracle model selection process discussed above can be regarded as a bilevel optimization problem \citep{dempe2002foundations}, where the lower-level problem \eqref{eq:oracle_decision} provides CRO solutions and the upper-level problem \eqref{eq:oracle_optimal_model} optimizes the efficiency of these solutions. 
By definition \eqref{eq:oracle_decision}, the decision $z_{\lambda^*}^{o}(X)$ achieves the minimum oracle decision risk among all models. We define the optimal efficiency of a data-driven decision $\hat{z}(X)$ as follows.

\begin{definition}[Asymptotic optimality]\label{def:efficiency}
    The decision $\hat{z}(X_{n+1})$ is asymptotically optimal if $\lim_{n\to \infty}\E[\phi(Y_{n+1},\hat{z}(X_{n+1}))] = v_{\Lambda}^*$, where $v_{\Lambda}^* = \E[\phi(Y_{n+1},z_{\lambda^*}^{o}(X_{n+1}))]$ is the minimum risk.
\end{definition}

In the following, we develop a data-driven framework named \textit{Conformalized Robust Optimization with Model Selection} (CROMS) {to perform optimal model selection by approximately solving problem \eqref{eq:oracle_optimal_model} while ensuring both marginal robustness (Definiton \ref{def:robustness}) and asymptotic optimality (Definition \ref{def:efficiency})} on the test data.

\subsection{E-CROMS: efficient selection with asymptotic optimality}\label{sec:ECROMS}



{We start with a computationally efficient model selection approach.} Recall that {the conformal prediction set} of the model $S_{\lambda}$ is given by $\gU_{\lambda}(\cdot) = \{c\in \gY: S_{\lambda}(\cdot, c) \leq \hat{q}_{\lambda}\},$ where $\hat{q}_{\lambda} = Q_{(1-\alpha)(1+1/n)}\LRs{ \LRl{S_{\lambda}(X_i,Y_i)}_{i=1}^n}$. {Based on the prediction set $\gU_\lambda$, we can obtain the decision for the test point $X_{n+1}$ as $z_{\lambda}(X_{n+1}) = \argmin_{z\in \gZ}\max_{c\in \gU_{\lambda}(X_{n+1})}\phi(c,z)$.}

{Due to the marginal coverage property of $\gU_\lambda(X_{n+1})$, the decision $z_\lambda(X_{n+1})$ satisfies $1-\alpha$ level of robustness for each $\lambda\in\Lambda$ under the i.i.d. assumption on data $\{(X_i,Y_i)\}_{i=1}^{n+1}$. Further}, since the sample quantile $\hat{q}_{\lambda}$ is a consistent estimator of the population quantile $q_{\lambda}^{o}$, thus the corresponding decision risk can approximate the oracle one in \eqref{eq:oracle_decision}, that is $\E[\phi(Y_{n+1},z_{\lambda}(X_{n+1}))] \approx \E[\phi(Y_{n+1},z_{\lambda}^{o}(X_{n+1}))]$.

To estimate the expectation $\E[\phi(Y_{n+1},z_{\lambda}(X_{n+1}))]$, it is natural to use the labeled data to compute the \emph{auxiliary decisions}, i.e., $z_{\lambda}(X_i) = \argmin_{z\in \gZ}\max_{c\in \gU_{\lambda}(X_{i})} \phi(c,z)$ for $i\in [n]$.
Since most CRO problems are convex as discussed in Section \ref{sec:preliminaries}, this step is computationally efficient, requiring solving $n$ convex problems. We regard those labeled decision losses $\{\phi(Y_i,z_{\lambda}(X_i))\}_{i=1}^n$ as nearly random ``copies'' of test decision loss $\phi(Y_{n+1},z_{\lambda}(X_{n+1}))$.
Then we perform the model selection through the following ERM problem
\begin{align}\label{eq:eff_CROMS_ERM}
    \hat{\lambda}_n = \argmin_{\lambda\in \Lambda} \frac{1}{n}\sum_{i=1}^n \phi(Y_i, z_{\lambda}(X_i)).
\end{align}
After that, we choose the corresponding conformal prediction set as the final prediction set, i.e., $\widehat{\gU}^{\ECROMS}(X_{n+1}) = \gU_{\hat{\lambda}_n}(X_{n+1})$. The \emph{final decision} is $\hat{z}^{\ECROMS}(X_{n+1}) = z_{\hat{\lambda}_n}(X_{n+1})$. We refer to this procedure as Efficient CROMS (E-CROMS) and summarize it in Algorithm \ref{alg:eff_CROMS}.

\begin{algorithm}[ht]
	\renewcommand{\algorithmicrequire}{\textbf{Input:}}
	\renewcommand{\algorithmicensure}{\textbf{Output:}}
	\caption{Efficient CROMS (E-CROMS)}\label{alg:eff_CROMS}
	\begin{algorithmic}
	\Require Pre-trained models $\{S_{\lambda}:\lambda\in \Lambda\}$, loss function $\phi$, labeled data $\{(X_{i},Y_i)\}_{i=1}^{n}$, test data $X_{n+1}$, robustness level $1-\alpha \in (0, 1)$.
        
        \For{$\lambda \in \Lambda$} \Comment{{Compute auxiliary decisions}}
        
        \State $\gU_{\lambda}(x) \gets \{c\in \gY: S_{\lambda}(x,c)\leq Q_{(1-\alpha)(1+1/n)}\LRs{ \LRl{S_{\lambda}(X_i,Y_i)}_{i=1}^n}\}$.

        \State $z_{\lambda}(X_i) \gets \argmin_{z\in \gZ}\max_{c\in \gU_{\lambda}(X_i)}\phi(c,z)$ for $i\in [n]$.


        \EndFor

        \State $\hat{\lambda}_n \gets \argmin_{\lambda \in \Lambda}\frac{1}{n}\sum_{i=1}^{n} \phi(Y_i, z_{\lambda}(X_i))$.  \Comment{{Select model via ERM}}

        \State $\widehat{\gU}^{\ECROMS}(X_{n+1}) \gets \{y\in \gY: S_{\hat{\lambda}_n}(X_{n+1}, y) \leq \hat{q}_{\hat{\lambda}_n}\}$. \Comment{{Construct prediction set}}

        \State $\hat{z}^{\ECROMS}(X_{n+1}) \gets \argmin_{z\in \gZ}\max_{c\in \widehat{\gU}^{\ECROMS}(X_{n+1})}\phi(c,z)$. \Comment{{Make the final decision}}
        
        \Ensure Prediction set $\widehat{\gU}^{\ECROMS}(X_{n+1})$ and decision $\hat{z}^{\ECROMS}(X_{n+1})$.
	\end{algorithmic}
\end{algorithm}

\subsubsection{Theoretical results of E-CROMS}\label{sec:theory_ECROMS}
For the ease of theoretical presentation, we denote the predcition set $\gU_{\lambda}(x;q) = \{y\in \gY: S_{\lambda}(x,y) \leq q\}$ and the CRO decision $z_{\lambda}(x; q) = \argmin_{z\in \gZ}\max_{c\in \gU_{\lambda}(x;q)} \phi(c,z)$ for model index $\lambda \in \Lambda$ and threshold $q \in \sR$.
Specifically, we introduce two function classes defined on the space $\gX \times \gY$: $\gF = \{\mathbbm{1}\{S_{\lambda}(x,y) > q\}: \lambda \in \Lambda, q\in \sR\}$ and $\gG = \{\phi(y, z_{\lambda}(x;q_{\lambda}^o)): \lambda \in \Lambda\}$, where $q_{\lambda}^o$ is the population quantile of $S_{\lambda}(X,Y)$.
Then we write their Rademacher complexities as $\mathfrak{R}_n(\gF) = \E\LRm{\sup_{f\in \gF} \left|\frac{1}{n}\sum_{i=1}^n \xi_i f(X_i,Y_i)\right|}$ and $\mathfrak{R}_n(\gG) = \E\LRm{\sup_{g\in \gG} \left|\frac{1}{n}\sum_{i=1}^n \xi_i g(X_i,Y_i)\right|}$, where $\{\xi_i\}_{i=1}^n$ are i.i.d. random variables taking $+1$ or $-1$ with equal probability.

{The first result is about the marginal coverage and robustness of E-CROMS.}


\begin{theorem}\label{thm:robustness_ECROMS}
    Suppose data $\{(X_i,Y_i)\}_{i=1}^{n+1}$ are i.i.d., E-CROMS satisfies $\sP\{Y_{n+1} \in \widehat{\gU}^{\ECROMS}(X_{n+1})\} \geq (1+n^{-1})(1-\alpha) - 2\mathfrak{R}_n(\gF)$. {Further, the decision of $\hat{z}^{\ECROMS}(X_{n+1})$ achieves the same level  of marginal robustness in Definition \ref{def:robustness}.}
\end{theorem}


The theorem above gives a non-asymptotic and distribution-free characterization for the marginal coverage of E-CROMS. 
Since the selected model index $\hat{\lambda}_n$ is not symmetric to the labeled and test data, the exchangeability between scores $\{S_{\hat{\lambda}_n}(X_i,Y_i)\}_{i=1}^{n+1}$ breaks, resulting in a coverage error for E-CROMS.

{Note that the coverage gap  $2\mathfrak{R}_n(\gF)$ comes from the size of the candidate set $\Lambda$. If $\{S_{\lambda}: \lambda \in \Lambda\}$ is a Vapnik–Chervonenkis (VC) model class with VC-dimension $\mathsf{v}(\gF)$, we can bound the gap by $\mathfrak{R}_n(\gF) = O\LRs{\sqrt{\mathsf{v}(\gF)/n}}$ \citep[Theorem 2.6.7 in][]{van1996weak}. In particular, if the index set $\Lambda$ is a finite set (i.e., $|\Lambda| < \infty$), we know $\mathsf{v}(\gF) \leq O(\log |\Lambda|)$, which recovers the bound of Theorem 1 in \citet{yang2024selection}. Differently, the latter focused on selecting the model to minimize the set width.}

Before analyzing the asymptotic efficiency, we introduce the following regular conditions on data distribution and loss function.

\begin{assumption}\label{assum:quantile_estimation}
    Let $f_{\lambda}(\cdot)$ and $F_{\lambda}(\cdot)$ be the density function and distribution function of nonconformity score $S_{\lambda}(X,Y)$. There exists a constant $\mu > 0$, $\inf_{s\in [F_{\lambda}^{-1}(1-\alpha-a_n), F_{\lambda}^{-1}(1-\alpha+a_n)]}f_{\lambda}(s) \geq \mu$ holds with $a_n = O(\sqrt{\log n/n}+\mathfrak{R}_n(\gF))$, where $F_{\lambda}^{-1}(\cdot)$ is the quantile function.
\end{assumption}

\begin{assumption}\label{assum:loss}
    There exists a constant $B > 0$, $\sup_{y\in \gY,z\in \gZ}|\phi(y,z)| \leq B$.
\end{assumption}

\begin{assumption}\label{assum:Lipschitz_solution_1}  
   There exists a constant $L > 0$, $\sup_{(x,y)\in\gX\times \gY}\LRabs{\phi(y,z_{\lambda}(x; q)) - \phi(y,z_{\lambda}(x;q_{\lambda}^o))} \leq L |q-q_{\lambda}^o|$ for any $\lambda \in \Lambda$ and {$|q-q_{\lambda}^o| \leq O\LRl{\mu^{-1}\LRs{\sqrt{\log n/n}+\mathfrak{R}_n(\gF)}}$}.
\end{assumption}

Assumption \ref{assum:quantile_estimation} ensures the consistency of quantile estimation, which is common in investigating the width efficiency of conformal prediction sets, e.g., \citet{lei2018distribution} and \citet{yang2024selection}. We impose Assumption \ref{assum:loss} for simplicity of concentration. 
Assumption \ref{assum:Lipschitz_solution_1} is key to safely approximating the oracle decision risk $\E[\phi(Y, z_{\lambda}^o(X))]$ of model $S_{\lambda}$.
This assumption is mild, as most robust optimization problems --- such as portfolio optimization with box or ellipsoid prediction sets --- are cone programming. Their solutions typically exhibit {locally Lipschitz continuous} to threshold $q$, see \citet{bolte2021nonsmooth} and \citet{wang2023learning}.


\begin{theorem}\label{thm:optimality_ECROMS} 
    {Suppose data $\{(X_i,Y_i)\}_{i=1}^{n+1}$ are i.i.d. and} Assumptions \ref{assum:quantile_estimation}-\ref{assum:Lipschitz_solution_1} hold, the decision risk of E-CROMS satisfies
    \begin{align}
        \left|\E[\phi(Y_{n+1}, \hat{z}^{\ECROMS}(X_{n+1}))] - v_{\Lambda}^*\right|&\leq O\LRl{\LRs{B + \frac{L}{\mu}}\sqrt{\frac{\log n}{n}} + \frac{L}{\mu}\mathfrak{R}_n(\gF) + \mathfrak{R}_n(\gG)}.
    \end{align}
\end{theorem}


{The upper bound of excess decision risk in Theorem \ref{thm:optimality_ECROMS} depends on two Rademacher complexities. To explicitly quantify the convergence rate of the decision risk, we further characterize the Rademacher complexity of $\mathfrak{R}_n(\gG)$ in the following proposition.
\begin{proposition}\label{pro:rademacher_G}
Under Assumption \ref{assum:loss}: (1) If $\Lambda$ is a finite set, then $\mathfrak{R}_n(\gG) \leq O\LRs{B\sqrt{\log (|\Lambda|) / n}}$; (2) If $\Lambda \subset \sR^m$ is a continuous set with a bounded radius $R$ (i.e., $\sup_{\lambda\in\Lambda}\|\lambda\| \leq R$), and satisfies $\sup_{(x,y)\in \gX \times \gY}\LRabs{\phi(y,z_{\lambda}(x; q_{\lambda}^o)) - \phi(y,z_{\lambda^{\prime}}(x;q_{\lambda^{\prime}}^o))} \leq L_{\Lambda} \|\lambda-\lambda^{\prime}\|$ for any $\|\lambda- \lambda^{\prime}\| \leq O(n^{-1})$, then  $\mathfrak{R}_n(\gG) \leq O\LRs{B \sqrt{m \log(nR)/n} + L_{\Lambda}/n}$.    
\end{proposition}

Proposition \ref{pro:rademacher_G} implies that for a finite model space, the decision risk converges to the optimal one provided that the cardinality grows sub-exponentially, i.e., $\log(|\Lambda|)=o(n)$. In the continuous case, the complexity is dominated by the model dimension $m$. For typical parametric models, such as linear models $S_{\lambda}=\|y-x^\top\lambda\|_{\infty}$ for $\Lambda \subset \sR^m$, we have $\mathfrak{R}_n(\gF)=O(\sqrt{m/n})$ and the term $\sqrt{m/n}$ governs the convergence. Thus, asymptotic robustness and decision optimality are guaranteed as long as $m=o(n)$.  }


{\begin{remark} The end-to-end (E2E) model training framework proposed by \citet{yeh2024end} maintains marginal robustness in finite samples via a sample splitting strategy. However, such splitting will degrade decision efficiency of the prediction set, as only part of the data is used in solving CRO problems to make final decisions. Despite its finite-sample margingal robustness, theoretical results in Appendix C.5 show that the E2E approach exhibits a larger deviation from the optimal value $v_{\Lambda}^*$ compared to our proposed methods. 
\end{remark}}

\section{CROMS with Finite-Sample Robustness Guarantee}\label{sec:FCROMS}

{In this section, we focus on introducing F-CROMS and J-CROMS to establish finite-sample robustness guarantees when the model space $\Lambda$ is finite. Theoretical results under continuous model space are deferred to Appendix C.2 and C.3.}


\subsection{F-CROMS: model selection by augmented data}\label{subsec:FCROMS}

The robustness gap in E-CROMS arises from the asymmetric dependence of the model selection on $\{(X_i,Y_i)\}_{i=1}^{n+1}$, as the test label $Y_{n+1}$ is unknown. We first consider adapting the full conformal technique \citep{vovk2005algorithmic,lei2018distribution} to address this issue {and name this procedure as Full CROMS (F-CROMS).}

Given a hypothesized value $y\in \gY$ intended to impute the test label $Y_{n+1}$, we now design a model selection process that is symmetric to the augmented dataset $\{(X_i,Y_i)\}_{i=1}^n \cup \{(X_{n+1}, y)\}$. First, for any $x\in \gX$, we define the prediction set of the model $S_{\lambda}$ with a symmetric threshold $\gU_{\lambda}^{y}(x) = \LRl{c\in \gY: S_{\lambda}(x, c) \leq Q_{1-\alpha}\LRs{\{S_{\lambda}(X_i,Y_i)\}_{i=1}^n\cup \{S_{\lambda}(X_{n+1},y)\}}}$.
Next, we introduce auxiliary decisions $z_{\lambda}^y(X_i) = \argmin_{z\in \gZ}\max_{c\in \gU_{\lambda}^{y}(X_i)}\phi(c,z)$ for $i\in [n+1]$.
Then, we select the model by solving the following augmented ERM problem:
\begin{align}\label{eq:hypothesized_avg_ERM}
    \hat{\lambda}^y = \argmin_{\lambda \in \Lambda}\frac{1}{n+1}\left\{\sum_{i=1}^n \phi(Y_i, z_{\lambda}^y(X_i)) + \phi(y, z_{\lambda}^y(X_{n+1}))\right\}.
\end{align}
By carefully checking the three steps above, we can see that the final model index $\hat{\lambda}^y$ is invariant to the permutation of the augmented data set. The step \eqref{eq:hypothesized_avg_ERM} can be recast as a recalibration of  \eqref{eq:eff_CROMS_ERM} in E-CROMS.
Having the selected model $\hat{\lambda}^y$ for each hypothesized label $y$, the \emph{final prediction set} is given by
\begin{align}\label{eq:FCROMS_set}
    \widehat{\gU}^{\FCROMS}(X_{n+1}) = \Big\{y\in \gY: S_{\hat{\lambda}^y}(X_{n+1}, y) \leq Q_{1-\alpha}\big(\{S_{\hat{\lambda}^y}(X_i,Y_i) \cup \{S_{\hat{\lambda}^y}(X_{n+1}, y)\}\}_{i=1}^n\big)\Big\}.
\end{align}
The \emph{final decision} is made by $\hat{z}^{\FCROMS}(X_{n+1}) = \argmin_{z\in \gZ}\max_{c\in \widehat{\gU}^{\FCROMS}(X_{n+1})}\phi(c,z)$. 
The detailed implementation of F-CROMS is deferred to Appendix B.1, where we also provide an acceleration approach by utilizing the relation between augmented quantile $\hat{q}_{\lambda}^y$ and labeled qauntile $\hat{q}_{\lambda}$.


To avoid sample splitting, the full conformal prediction uses a symmetric algorithm to train the model on the augmented dataset $\{(X_i,Y_i)\}_{i=1}^n \cup \{(X_{n+1},y)\}$ for $y\in \gY$. The training process usually aims to minimize the prediction error over a model class, e.g., least squares error and cross-entropy. In contrast, F-CROMS directly uses the downstream decision risk to select the model from a model class $\{S_{\lambda}:\lambda \in \Lambda\}$. Recently, \citet{liang2024conformal} also applied the full conformal prediction technique to maintain the finite-sample coverage after selecting the model based on the width or volume of the prediction set. However, optimizing the width of a conformal prediction set is not equivalent to improving decision efficiency. {It is also confirmed by the simulation results in Appendix G.5, where E-CROMS and F-CROMS achieve lower decision risk compared with the model selection methods in \citet{yang2024selection} and \citet{liang2024conformal}.}

\subsubsection{Theoretical results of F-CROMS}
The following theorem demonstrates that F-CROMS achieves finite-sample marginal coverage control, which is independent of any distributional assumptions.

\begin{theorem}\label{thm:FCROMS_coverage}
    Assume that $\{(X_i,Y_i)\}_{i=1}^{n+1}$ are i.i.d., then the prediction set of F-CROMS satisfies $\sP\{Y_{n+1} \in \widehat{\gU}^{\FCROMS}(X_{n+1})\}\geq 1-\alpha$.
\end{theorem}


Next, with the same assumptions needed for E-CROMS, we establish the upper bound of the excess decision risk of F-CROMS under a finite index set $\Lambda$. 

\begin{theorem}\label{thm:FCROMS_optimality}
Suppose there exists a positive sequence $\beta_n \geq O\LRl{(L/\mu + B)\sqrt{\log (n \vee |\Lambda|)/n}}$ such that $\E[\phi(Y, z_{\lambda}^{o}(X))] \geq \E[\phi(Y, z_{\lambda^*}^{o}(X))] + \beta_n$ for any $\lambda \neq \lambda^*$. 
Under the same assumptions of Theorem \ref{thm:optimality_ECROMS}, we have
    \begin{align}
        \left|\E\LRm{\phi\LRs{Y_{n+1}, \hat{z}^{\FCROMS}(X_{n+1})}} - v_{\Lambda}^*\right| \leq O\LRs{\frac{B}{n} + \frac{L}{\mu}\sqrt{\frac{\log n}{n}}}.\nonumber
    \end{align}
\end{theorem}

Theorem \ref{thm:FCROMS_optimality} imposes the minimum risk gap $\beta_n$ to establish model selection consistency in high probability, i.e., $\sP(\forall~ y\in\gY,\hat{\lambda}^y = \lambda^*) \geq 1-O(n^{-1})$. Under this event, $\widehat{\gU}^{\FCROMS}(X_{n+1})$ aligns with the conformal prediction set of the optimal model $S_{\lambda^*}$ in \eqref{eq:oracle_optimal_model}, that is $\gU_{\lambda^*}(X_{n+1}) = \left\{y\in \gY: S_{\lambda^*}(X_{n+1},y)\leq \hat{q}_{\lambda^*}\right\}$. Consequently, we can bound the difference between the decision risk of F-CROMS and the optimal value $v_{\Lambda}^*$.


\subsubsection{Grid-approximation of F-CROMS with optimality guarantee}\label{sec:GF-CROMS}

{Note that generating the prediction set for F-CROMS requires searching over the entire label space, which remains impossible for the regression task $\mathcal{Y}\subset\sR^p$.}
For those regression tasks with a bounded label space, we propose a grid-approximated version of F-CROMS based on the discretization technique introduced by \citet{chen2018discretized}. 
Let $\widetilde{\mathcal{Y}}$ be a set of uniformly spaced grid points over $\mathcal{Y}$ such that for any $y\in \gY$ there exists some $\tilde{y}\in \widetilde{\gY}$ such that $\|y - \tilde{y}\| \leq \epsilon_{\rm{grid}}$. Then we define a discretization mapping $\sD(y) = \argmin_{\tilde{y} \in \widetilde{\gY} } \| y - \tilde{y}\|$, assigning each $y \in \mathcal{Y}$ to its closest grid point in $\widetilde{\mathcal{Y}}$, see Figure \ref{fig:discretization_F_CROMS} (a) for the illustration.

\begin{figure}[ht] 
    \centering
    \begin{subfigure}[t]{0.48\textwidth}
        \centering
        \includegraphics[width=\linewidth]{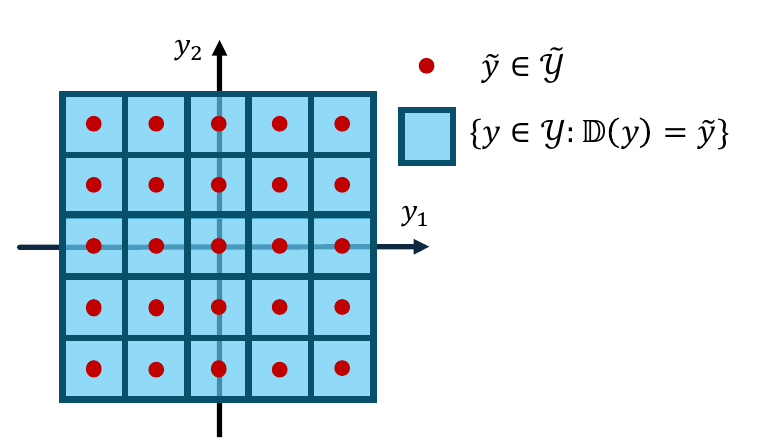}
        \caption{Discretization}
    \end{subfigure}
    \begin{subfigure}[t]{0.48\textwidth}
        \centering
        \includegraphics[width=\linewidth]{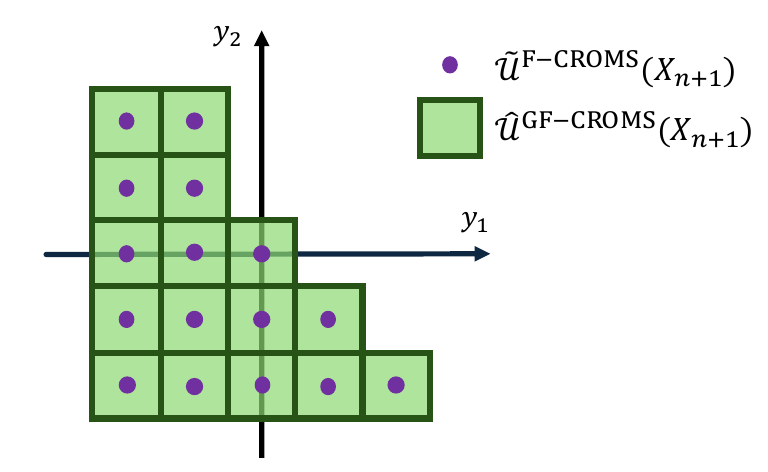}
        \caption{Final prediction set}
    \end{subfigure}
    \caption{Illustration for the grid-approximated F-CROMS with $\gY \subseteq \sR^2$. The red dots in panel (a) are grid points in $\widetilde{\mathcal{Y}}$, and the purple points in panel (b) are grid points in $\widetilde{\gU}^{\FCROMS}(X_{n+1})$, and the green area is the output prediction set $\widehat{\gU}^{\GFCROMS}(X_{n+1})$ in \eqref{eq:GFCROMS_set}.}
    \label{fig:discretization_F_CROMS}
\end{figure}

The subsequent implementation has three main steps. First, we apply the mapping $\sD$ to transform the labeled data $(X_i,Y_i)$ into the discretized version $(X_i, \sD(Y_i))$ for $i\in [n]$. Second, we call F-CROMS to output the prediction set $\widetilde{\mathcal{U}}^{\FCROMS}(X_{n+1}) \subset \widetilde{\mathcal{Y}}$ using the discretized dataset $\{(X_i, \sD(Y_i))\}_{i=1}^n$, which is computationally feasible since $\widetilde{\mathcal{Y}}$ contains finite points. Third, we output the following final prediction set via the inverse mapping,
\begin{align}\label{eq:GFCROMS_set}
    \widehat{\mathcal{U}}^{\GFCROMS}(X_{n+1}) = \LRl{y = \sD^{-1}(\tilde{y}): \tilde{y} \in \widetilde{\mathcal{U}}^{\FCROMS}(X_{n+1})}.
\end{align}
Then we make the decision by $\hat{z}^{\GFCROMS}(X_{n+1}) = \argmin_{z\in \gZ}\max_{c\in \widehat{\mathcal{U}}^{\GFCROMS}(X_{n+1})} \phi(c,z)$.
Since the prediction set \eqref{eq:GFCROMS_set} is a union of multiple convex sets (e.g., Figure \ref{fig:discretization_F_CROMS} (b)), we could adopt the gradient-based method in \citet{patel2024conformal} to facilitate solving the CRO problem, which has a polynomial time complexity. 
Since the discretized data points $\{(X_i, \sD(Y_i))\}_{i=1}^{n+1}$ are exchangeable, this approach retains a finite-sample marginal robustness guarantee. 

\begin{theorem}\label{thm:robustness_GFCROMS}
Assume that $\{(X_i,Y_i)\}_{i=1}^{n+1}$ are i.i.d., then grid-approximated F-CROMS satisfies $\sP\{Y_{n+1} \in \widehat{\gU}^{\GFCROMS}(X_{n+1})\}\geq 1-\alpha$ regardless of the choice of $\epsilon_{\rm{grid}}$.
\end{theorem}





{We now establish the decision optimality of the discretized predictor \eqref{eq:GFCROMS_set}. This analysis relies on two additional mild regularity conditions: {(i) the Lipschitz continuity of candidate models over the label, which can hold for the portfolio optimization problems with box or ellipsoid candidate models; (ii) the closeness of decision loss of the grid-approximated set and candidate set.} (formally stated as Assumptions C.1 and C.2 in Appendix C).

\begin{theorem}\label{thm:GFCROMS_optimality}
Suppose the assumptions of Theorem \ref{thm:FCROMS_optimality} hold. Under additional Assumptions C.1 and C.2, if the minimum risk gap in Theorem \ref{thm:optimality_ECROMS} satisfies $\beta_n \geq O\LRl{\epsilon_{\rm{grid}}+(L/\mu+B) \sqrt{\log (n \vee |\Lambda|)/n}}$, then we have 
    \begin{align*}
        \left|\E\LRm{\phi\LRs{Y_{n+1}, \hat{z}^{\GFCROMS}(X_{n+1})}} - v_{\Lambda}^*\right| \leq O\LRs{\epsilon_{\rm{grid}} + \frac{B}{n}+ \frac{L}{\mu} \sqrt{\frac{\log (n \vee |\Lambda|)}{n}}}.
    \end{align*}
\end{theorem}
Theorem \ref{thm:GFCROMS_optimality} yields a theoretical guide for the choice of $\epsilon_{\rm{grid}}$. To match the same convergence rate of the exact F-CROMS method in Theorem \ref{thm:FCROMS_optimality}, it suffices to set the grid size as $\epsilon_{\rm{grid}} \asymp n^{-1/2}$. 
Considering the case $\gY \subset \sR^p$ with bounded radius $R_{\gY}$, the total number of grid points in $\widetilde{\gY}$ should be $(2R_{\gY}/\epsilon_{\rm{grid}})^p = O(R_{\gY}^p n^{p/2})$. This polynomial complexity guarantees that the GF-CROMS maintains theoretical optimality while remaining feasibility.}

{Additionally, Appendix B.3 details a more computationally feasible procedure to construct an exact F-CROMS superset when $\phi(y, z_{\lambda}(x;q))$ exhibits piecewise monotonicity in $q$. This approach is specifically applied to portfolio optimization tasks using box or ellipsoid prediction sets discussed in Section \ref{sec:preliminaries}. However, we empirically found that this procedure can be quite conservative compared with the grid-approximated implementation when the sample size $n$ is small.
}

\subsection{J-CROMS: model selection by Jackknife+ method}
{Although F-CROMS provides strong guarantees via grid search, it may be computationally impractical for complex tasks. To obtain a more efficient alternative while preserving finite-sample validity, we introduce the J-CROMS framework. By leveraging the Jackknife+ method \citep{barber2021predictive}, J-CROMS eliminates the splitting bias of E-CROMS and offers a $1-2\alpha$ robustness guarantee with significantly lower computational complexity than the discretized F-CROMS in the previous subsection.}

Define the leave-one-out prediction set as $\gU_{\lambda}^{-i}(\cdot) = \{c\in \gY: S_{\lambda}(\cdot, c) \leq \hat{q}_{\lambda}^{-i}\}$, where $\hat{q}_{\lambda}^{-i} = Q_{(1-\alpha)(1+ (n-1)^{-1})}\LRs{\{S_{\lambda}(X_{\ell},Y_{\ell})\}_{\ell \in [n],\ell \neq i}}$.
J-CROMS performs model selection by $\hat{\lambda}^{-i} = \argmin_{\lambda \in \Lambda} \frac{1}{n-1}\sum_{\ell \in [n], \ell \neq i}\phi(Y_{\ell}, z_{\lambda}^{-i}(X_{\ell}))$, where $z_{\lambda}^{-i}(X_{\ell}) = \argmin_{z\in \gZ}\max_{c\in \gU_{\lambda}^{-i}(X_{\ell})} \phi(c,z)$. Then, the final prediction set is built as
\begin{align}\label{eq:JCROMS_set}
    \widehat{\gU}^{\text{J-CROMS}}(X_{n+1}) = \LRl{y\in \gY: \frac{\sum_{i=1}^n \mathbbm{1}\big\{ S_{\hat{\lambda}^{-i}}(X_{n+1}, y)\leq S_{\hat{\lambda}^{-i}}(X_{i}, Y_{i}) \big\}+1}{n+1} > \alpha},
\end{align}
and the final decision is 
\[
\hat{z}^{\text{J-CROMS}}(X_{n+1}) = \argmin_{z\in \gZ}\max_{c\in \widehat{\gU}^{\text{J-CROMS}}(X_{n+1})}\phi(c,z).
\]

\begin{theorem}\label{thm:JCROMS_robustness_finite_sample}
    If data $\{(X_i,Y_i)\}_{i=1}^{n+1}$ are i.i.d., the J-CROMS method satisfies $\sP\{Y_{n+1}\in \widehat{\gU}^{\JCROMS}(X_{n+1})\} \geq 1-2\alpha$.
\end{theorem}

\begin{theorem}\label{thm:JCROMS_optimality}
    Under the same conditions of Theorem \ref{thm:FCROMS_optimality}, the J-CROMS prediction set satisfies $\sP\{Y_{n+1}\in \widehat{\gU}^{\JCROMS}(X_{n+1})\} \geq 1-\alpha - O(n^{-1})$ and 
    \begin{align}
        \left|\E\LRm{\phi(Y, \hat{z}^{\JCROMS}(X))} - v_{\Lambda}^*\right| \leq O\LRl{\frac{L}{\mu} \sqrt{\frac{\log (n \vee |\Lambda|)}{n}}+\frac{B}{n}}.\nonumber
    \end{align} 
\end{theorem}

Theorem \ref{thm:JCROMS_robustness_finite_sample} shows that J-CROMS can guarantee a distribution-free $1-2\alpha$ level of robustness. Under certain stability conditions, \citet{barber2021predictive} proved that the Jackknife+ conformal prediction set can achieve $1-\alpha - o(1)$ marginal coverage. Essentially, the minimum risk gap condition in Theorem \ref{thm:FCROMS_optimality} ensures model selection stability, which leads to $1-\alpha - O(n^{-1})$ marginal robustness in Theorem \ref{thm:JCROMS_optimality} and  asymptotic optimality. Extensions to cross-validation are given in Appendix D.4. 

{For a general candidate model set, the J-CROMS prediction set $\widehat{\gU}^{\text{J-CROMS}}(X_{n+1})$ has no closed form and could be disconnected or nonconvex. 
When the candidate model set are box scores, that is $S_{\lambda}(x,y) = \|(y - \hat{\mu}_{\lambda}(x))/\hat{\sigma}_{\lambda}(x)\|_{\infty}$ for each $\lambda \in \Lambda$, $\hat{\mu}_{\lambda}(x)\in \sR^p$ and $\hat{\sigma}_{\lambda}(x)\in \sR^p$, we can construct a box-shaped superset. 
Let two endpoint vectors $c^{\rm{up}}$, $c^{\rm{lo}}\in\sR^p$ with components $c_{k}^{\text{up}} = Q_{(1-\alpha)(1+1/n)}\left(\LRl{\hat{\mu}_{\hat{\lambda}^{-i},k}(X_{n+1}) + \hat{\sigma}_{\hat{\lambda}^{-i},k}(X_{n+1})S_{\hat{\lambda}^{-i}}(X_i,Y_i)}_{i=1}^{n}\right)$ and $c_{k}^{\text{lo}} = -Q_{(1-\alpha)(1+1/n)}\left(\LRl{\hat{\sigma}_{\hat{\lambda}^{-i},k}(X_{n+1})S_{\hat{\lambda}^{-i}}(X_i,Y_i) - \hat{\mu}_{\hat{\lambda}^{-i},k}(X_{n+1})}_{i=1}^{n}\right).$
Then, a superset for \eqref{eq:JCROMS_set} is $\widehat{\gU}_{\rm{box}}^{\text{J-CROMS}}(X_{n+1}) = \LRl{y\in \sR^p: c^{\text{lo}} \leq y \leq c^{\text{up}}}$,
which satisfies $\widehat{\gU}^{\JCROMS}(X_{n+1}) \subseteq \widehat{\gU}_{\rm{box}}^{\JCROMS}(X_{n+1})$ and enables efficient CRO optimization. }

\subsection{Comparisons among CROMS methods}
    
To conclude this section, we summarize and compare the three algorithms in the CROMS framework in Table \ref{table:comparison}, including their computational complexity in terms of {in view of number of model selection steps and CRO optimaztions} and theoretical guarantees. 
E-CROMS exhibits the lowest computational complexity and is therefore well suitable for large sample sizes where the coverage error can be neglected. 
F-CROMS is the only method that provides a finite-sample, distribution-free \(1-\alpha\) marginal coverage guarantee, making it particularly recommended for classification tasks. 
J-CROMS serves as a practical trade-off between the two: it requires relatively low computational cost and achieves a \(1-2\alpha\) coverage.

\begin{table}[H]
    \begin{minipage}{\textwidth}
    \centering
    \renewcommand{\arraystretch}{1.2}
    \caption{Comparison of computational cost {in terms of number of model selection and solving CRO} on $n_{\rm{test}}$ test points and theoretical coverage guarantee.}\label{table:comparison} 
    \resizebox{\textwidth}{!}{
    \begin{tabular}{lcccc}
    \toprule
    Algorithm & No. of Model Selection & No. of Solving CRO & \makecell{Marginal Coverage\\ (Distribution-free)} &  Asymptotic Optimality \\
    \midrule
    E-CROMS & 1 & $n+n_{\rm{test}}$ & $1 - \alpha - O\LRs{\sqrt{\mathsf{v}/n}}$\footnote{Here $\mathsf{v}$ denotes the VC-dimension of candidate model set $\{S_{\lambda}:\lambda \in \Lambda\}$.}  & Theorem \ref{thm:optimality_ECROMS}\\ \vspace{0.2cm}
    F-CROMS &  \makecell{$n_{\rm{test}}(n+1)\cdot |\gY|$\\[2pt] \underline{or}\footnote{For regression task $\gY \subseteq \sR^p$, it refers to the complexity of grid-approximation with $\epsilon_{\rm{grid}}\asymp n^{-1/2}$.} $O(n_{\rm{test}}\cdot n^{p/2})$} & \makecell{$n_{\rm{test}}\cdot ((n+1)|\gY| + 1)$\\[2pt] \underline{or} $O(n_{\rm{test}}\cdot n^{p/2+1})$} & $1 - \alpha$  & Theorems \ref{thm:FCROMS_optimality}, \ref{thm:GFCROMS_optimality}\\ \vspace{0.2cm}
    J-CROMS & $n$ & $n(n-1) +n_{\rm{test}}$ & $1 - 2\alpha$ & Theorem \ref{thm:JCROMS_optimality} \\
    \bottomrule
    \end{tabular}}
    \end{minipage}
    \end{table}


\section{Individualized Optimal Model Selection}\label{sec:CROiMS}


In applications like precision medicine, the ideal model may vary greatly depending on the unique characteristics of each patient. This section introduces individualized model selection, extending the ideas in the CROMS framework to minimize conditional decision risk given the test data $X_{n+1}$. Instead of marginal robustness in Definition \ref{def:robustness}, individualized model selection needs to guarantee conditional robustness.

\begin{definition}[Asymptotic conditional robustness]\label{def:individual_robustness}
    The prediction set $\gU(X_{n+1})$ satisfies $1-\alpha$ level of asymptotic conditional robustness if it almost surely holds that
    \[\sP\LRl{\phi(Y_{n+1}, z(X_{n+1})) \leq \max_{c\in \gU(X_{n+1})}\phi(c, z(X_{n+1})) \mid X_{n+1}} \geq 1-\alpha + o(1).
    \]
\end{definition}

\subsection{Oracle model selection to minimize conditional decision risk}

For the pre-trained model $S_{\lambda}$, we denote the conditional quantile function as $q_{\lambda}^{co}(X) = \inf\{q\in \sR: \sP\{S_{\lambda}(X,Y) \leq q\mid X\}\}$. For data $(X, Y)\sim P$, the conditional oracle prediction set of model $S_{\lambda}$ is defined as $\gU_{\lambda}^{co}(X) = \{y\in \gY: S_{\lambda}(X, y) \leq q_{\lambda}^{co}(X)\}$,
which satisfies the conditional coverage property: $\sP\LRl{Y\in \gU_{\lambda}^{co}(X) \mid X} \geq 1-\alpha.$

Let $z_{\lambda}^{co}(X) = \argmin_{z\in \gZ}\max_{c\in \gU_{\lambda}^{co}(X)}$ be the decision enjoying the exact conditional robustness in Definition \ref{def:individual_robustness} without $o(1)$ term due to the conditional coverage.
To evaluate individual efficiency of decisions $\{z_{\lambda}^{co}(X): \lambda\in \Lambda\}$, we introduce the oracle conditional decision risk of model $S_{\lambda}$ as $\E[\phi(Y,z_{\lambda}^{co}(X))\mid X]$. 
Then, the index of the individually optimal model is defined as 
\begin{equation}\label{eq:oracle_individual_optimal_model}
    \begin{aligned}
    \lambda^*(X) = \argmin_{\lambda \in \Lambda}&\  \E\LRm{ \phi(Y, z_{\lambda}^{co}(X)) \mid X}.
    \end{aligned}
\end{equation}

Before proceeding further, it would be beneficial to discuss $\sP\LRl{Y\in \gU_{\lambda}^{co}(X) \mid X} \geq 1-\alpha$, which is also known as \emph{test-conditional coverage} in conformal prediction literature \citep{vovk2005algorithmic}.
As shown by \citet{vovk2012conditional} and \citet{lei2013distribution}, exact test-conditional coverage is impossible to achieve in a distribution-free regime, except for a noninformative trivial set $\gY$. Extensive research has been dedicated to constructing prediction sets that satisfy asymptotic or approximate conditional coverage, see Chapter 4 in \citet{angelopoulos2024theoretical}. Next, we define the asymptotic conditional optimality of one data-driven decision.

\begin{definition}[Asymptotic conditional optimality]\label{def:individual_efficiency}
    The decision $\hat{z}(X_{n+1})$ is asymptotically conditional optimal if $\lim_{n\to \infty}\E[\phi(Y_{n+1},\hat{z}(X_{n+1})) \mid X_{n+1}] = v_{\Lambda}^*(X_{n+1})$ almost surely, where $v_{\Lambda}^*(X_{n+1})=\E[\phi(Y_{n+1}, z_{\lambda^*(X_{n+1})}^{co}(X_{n+1}))\mid X_{n+1}]$ is the minimum conditional risk.
\end{definition}





\subsection{CROiMS: individualized model selection and robust decision}

To approximate the conditional decision risk in \eqref{eq:oracle_individual_optimal_model}, we first employ the kernel method to estimate the conditional quantile function $q_{\lambda}^{co}(X_{n+1})$. A similar localized strategy was used in \citet{guan2023localized} and \citet{hore2024conformal}. Equipped with a kernel function $H(\cdot,\cdot):\gX\times \gX \to \sR_{\geq 0}$, the conditional quantile function $q_{\lambda}^{co}(\cdot)$ can be estimated by the $(1-\alpha)$ weighted sample quantile $\hat{q}_{\lambda}(\cdot) = Q_{1-\alpha}\LRs{\{S_{\lambda}(X_i,Y_i)\}_{i=1}^n; \{w_i(\cdot)\}_{i=1}^n}$,
where $w_i(\cdot) = H(X_i,\cdot)/\sum_{j=1}^n H(X_j,\cdot)$ for $i\in [n]$ are weights. The localized conformal prediction (LCP) set of model $S_{\lambda}$ is defined as $\gU_{\lambda}^{\mathrm{LCP}}(X_{n+1}) = \LRl{c\in \gY: S_{\lambda}(X_{n+1},c) \leq \hat{q}_{\lambda}(X_{n+1})}$.
The induced decision is given by $z_{\lambda}^{\mathrm{LCP}}(X_{n+1}) = \argmin_{z\in \gZ}\max_{c\in \gU_{\lambda}^{\mathrm{LCP}}(X_{n+1})} \phi(c,z)$. Under mild nonparametric assumptions, $\hat{q}_{\lambda}(X_{n+1})$ is a consistent estimator to the conditional quantile $q_{\lambda}^{co}(X_{n+1})$, which indicates that $\phi(Y_{n+1}, z_{\lambda}^{\mathrm{LCP}}(X_{n+1})) \approx \phi(Y_{n+1}, z_{\lambda}^{co}(X_{n+1}))$. 

Next, we proceed with approximating the conditional decision risk $\E[\phi(Y_{n+1}, z_{\lambda}^{\mathrm{LCP}}(X_{n+1})) \mid X_{n+1}]$ through the labeled data. Denote $\gU_{\lambda}^{\mathrm{LCP}}(\cdot) = \{y\in\gY: S_{\lambda}(\cdot,y) \leq \hat{q}_{\lambda}(\cdot)\}$, and define auxiliary decisions as $z_{\lambda}^{\mathrm{LCP}}(X_i) = \argmin_{z\in \gZ}\max_{c\in \gU_{\lambda}^{\mathrm{LCP}}(X_i)}\phi(c,z)$ for $i\in [n]$. The \emph{individualized model selection} is conducted by solving the weighted ERM problem
\begin{align}
    \hat{\lambda}(X_{n+1}) = \argmin_{\lambda \in \Lambda}\sum_{i=1}^nw_i(X_{n+1})\cdot \phi(Y_i,z_{\lambda}^{\mathrm{LCP}}(X_i)).\nonumber
\end{align}
After obtaining the model index $\hat{\lambda}(X_{n+1})$, we output \emph{final individualized prediction set} as $\widehat{\gU}^{\mathrm{CROiMS}}(X_{n+1}) = \gU_{\hat{\lambda}(X_{n+1})}^{\mathrm{LCP}}(X_{n+1})$, and then \emph{final decision} is made by $\hat{z}^{\mathrm{CROiMS}}(X_{n+1}) = \argmin_{z\in \gZ}\max_{c\in \widehat{\gU}^{\mathrm{CROiMS}}(X_{n+1})} \phi(c,z)$.
The selection procedure above is referred to as \emph{Conformalized Robust Optimization with individualized Model Selection} (CROiMS), whose implementation is in Algorithm \ref{alg:CROiMS}.

\begin{algorithm}[ht]
	\renewcommand{\algorithmicrequire}{\textbf{Input:}}
	\renewcommand{\algorithmicensure}{\textbf{Output:}}
	\caption{CROiMS}
	\label{alg:CROiMS}
        \linespread{1.4}\selectfont
	\begin{algorithmic}
	\Require Pre-trained models $\{S_{\lambda}:\lambda\in \Lambda\}$, loss function $\phi$, labeled data $\{(X_{i},Y_i)\}_{i=1}^{n}$, test data $X_{n+1}$, kernel function $H$, robustness level $1-\alpha \in (0, 1)$.
        \For{$\lambda \in \Lambda$}\Comment{{Compute auxiliary decisions}}
        \State $\hat{q}_{\lambda}(\cdot) \gets Q_{1-\alpha}\LRs{\{S_{\lambda}(X_i,Y_i)\}_{i=1}^n; \{w_i(\cdot)\}_{i=1}^n}$ with $w_i(\cdot) = \frac{H(X_i,\cdot)}{\sum_{j=1}^n H(X_j,\cdot)}$.
        
        \State $\gU_{\lambda}^{\mathrm{LCP}}(\cdot) \gets \{c\in \gY: S_{\lambda}(\cdot,c)\leq \hat{q}_{\lambda}(\cdot)\}$.

        \State $z_{\lambda}^{\mathrm{LCP}}(X_i) \gets \argmin_{z\in \gZ}\max_{c\in \gU_{\lambda}^{\mathrm{LCP}}(X_i)}\phi(c,z)$ for $i\in [n]$.

        \EndFor

        \State $\hat{\lambda}(X_{n+1}) \gets \argmin_{\lambda \in \Lambda}\sum_{i=1}^{n} w_i(X_{n+1})\cdot \phi(Y_i, z_{\lambda}^{\mathrm{LCP}}(X_i))$. \Comment{{Select model via weighted ERM}}

        \State $\widehat{\gU}^{\mathrm{CROiMS}}(X_{n+1}) \gets \gU_{\hat{\lambda}(X_{n+1})}^{\mathrm{LCP}}(X_{n+1})$. \Comment{{Conditional prediction set for the selected model}}

        \State $\hat{z}^{\mathrm{CROiMS}}(X_{n+1}) \gets \argmin_{z\in \gZ}\max_{c\in \widehat{\gU}^{\mathrm{CROiMS}}(X_{n+1})} \phi(c,z)$. \Comment{{Make CRO decision}}
        
        \Ensure Prediction set $\widehat{\gU}^{\mathrm{CROiMS}}(X_{n+1})$ and decision $\hat{z}^{\mathrm{CROiMS}}(X_{n+1})$.
	\end{algorithmic}
\end{algorithm}


We further explore the conditional robustness and efficiency properties of CROiMS under a finite index set $\Lambda$. The results for a more general index set are deferred to Appendix E. For simplicity, we assume that $\gX \subseteq \sR^d$, and consider the kernel function $H(x_1,x_2) = \exp\LRs{-\frac{\|x_1-x_2\|^2}{h_n^2}}$. We start with introducing the basic distributional assumptions required to establish non-asymptotic conditional bounds.

\begin{assumption}\label{assum:cond_quantile}
    Let $p(x)$ be the density function of $X$. There exists some $\rho > 0$ such that $\inf_{x\in \gX} p(x) \geq \rho$. Let $f_{\lambda}(s|x)$ and $F_{\lambda}(s|x)$ be the conditional density function and distribution function of score $S_{\lambda}(X,Y)$ given $X=x$, respectively. There exist some constants $\bar{\tau}>0$, $\bar{\mu}>0$ such that for any $\lambda\in\Lambda$, $\sup_{s\in \sR}|F_{\lambda}(s|x) - F_{\lambda}(s| x^{\prime})| \leq \bar{\tau}\|x-x^{\prime}\|$ for any $x,x^{\prime}\in \gX$, and $\inf_{x\in \gX}f_{\lambda}(s|x) \geq \bar{\mu}$ for any $s\in [F_{\lambda}^{-1}(1-\alpha-b_n|x),F_{\lambda}^{-1}(1-\alpha+b_n|x)]$ with $b_n = O\LRl{\tau h_n\log(h_n^{-d})+\sqrt{\frac{\log (n \vee |\Lambda|)}{\rho nh_n^d}}}$, where $F_{\lambda}^{-1}(\cdot|x)$ is the conditional quantile function.
\end{assumption}

\begin{assumption}\label{assum:cond_expectation} 
    Suppose Assumption \ref{assum:Lipschitz_solution_1} holds by replacing $a_n$ with $b_n$ and $q_{\lambda}^o$ with $q_{\lambda}^{co}(x)$.
    Let $\Phi_{\lambda}(X,Y) = \phi(Y, z_{\lambda}(X; q_{\lambda}^{co}(X)))$.
    There exists a constant $\tau > 0$, for any $\lambda \in \Lambda$ and $x,x^{\prime}\in \gX$, $\left|\E[\Phi_{\lambda}(X,Y)\mid X=x] - \E[\Phi_{\lambda}(X,Y)\mid X=x^{\prime}]\right| \leq \tau \twonorm{x - x^{\prime}}$.
\end{assumption}

    

The conditions in Assumption \ref{assum:cond_quantile} are used to ensure the estimation consistency of conditional quantiles.
The smoothness condition for conditional risk in Assumption \ref{assum:cond_expectation} is common in the nonparametric estimation of conditional expectation \citep{wasserman2006all}. 

\begin{theorem}\label{thm:cond_coverage}
    Under Assumption \ref{assum:cond_quantile}, if $nh_n^d \to \infty$ and $h_n \to 0$, then CROiMS almost surely satisfies $1-\alpha - O\LRl{\sqrt{\frac{\log(n\vee |\Lambda|)}{\rho n h_n^d}} + \frac{\bar{\tau}}{\rho} h_n \log(h_n^{-d})}$ level of conditional robustness in Definition \ref{def:individual_robustness}. In addition, under Assumptions \ref{assum:loss}, \ref{assum:cond_quantile}, \ref{assum:cond_expectation}, it almost surely holds that
        \begin{align}
            &\left|\E\Big[\phi\LRs{Y_{n+1}, \hat{z}^{\mathrm{CROiMS}}(X_{n+1})} \mid X_{n+1}\Big] - v_{\Lambda}^*(X_{n+1})\right|\nonumber\\
            &\qquad\qquad\qquad\qquad\leq O\LRl{\LRs{\frac{L}{\bar{\mu}}+B}\sqrt{\frac{\log (n \vee |\Lambda|)}{\rho n h_n^d}} + \LRs{\frac{L }{\bar{\mu}}\frac{\bar{\tau}}{\rho} + \frac{\tau}{\rho}}h_n\log(h_n^{-d})}.\nonumber
        \end{align}
\end{theorem}


By comparing Theorems \ref{thm:cond_coverage} with the convergence rate of classical kernel estimation, we observe that individualized model selection introduces only an additional error factor of $\log|\Lambda|$, regarding the size of the candidate set. Assuming $|\Lambda| \leq O(n^{c})$ for constant $c>0$ and choosing the bandwidth $h_n \asymp n^{-\frac{1}{d+2}}$, CROiMS can achieve asymptotic conditional robustness and optimality with nearly optimal rate $O(n^{-\frac{1}{d+2}} \log n)$. In practice, we can set the bandwidth as $h_n = Cn^{-\frac{1}{d+2}}$ and tune the constant $C$ by the strategy in \citet{hore2024conformal}, ensuring that the effective sample size is greater than a specified value.



\section{Simulation Results}\label{sec:simulation}

In this section, we examine the numerical performance of the proposed methods. All the simulation results are based on 100 independent replications. For each replication, an independent draw of a labeled dataset $\gD_n = \{(X_i,Y_i)\}_{i=1}^n$ of size $n$ and a test dataset $\gD_{\rm test} = \{(X_{j},Y_{j})\}_{j=n+1}^{n+m}$ of size $m$ are generated. For simplicity, we let $\widehat{\gU}(X_{j})$ and $\hat{z}(X_{j})$ be the prediction set and decision returned by each method for test point $X_{j}$ for $j\in [m]$. {For classification tasks, F-CROMS returns the prediction set \eqref{eq:FCROMS_set}; for regression tasks $\gY \subset \sR^p$, F-CROMS is returns the grid-approximation set \eqref{eq:GFCROMS_set} with {the number of grid points $(A n^{1/2})^p$ and $A=3$}, and the sensitivity analysis on the constant $A$ is conducted in Appendix B.2.}

\subsection{Model selection in the averaged case} \label{Simulation: GCP}

In this case, we compare \texttt{E-CROMS}, \texttt{F-CROMS} and {\texttt{J-CROMS}} {with($\alpha/2$)/without($\alpha$)} against two baseline methods:
\begin{itemize}\parskip 4pt
     \item \texttt{Naive-CP}: The model $S_{\lambda}$ used in the CRO procedure is randomly selected from all pre-trained models $\{S_{\lambda}:\lambda \in \Lambda\}$, and then the conformal prediction set (\ref{eq:conformal_set_single_model}) is constructed for the randomly selected model.
     \item \texttt{E2E}: The labeled data $\{(X_i,Y_i)\}_{i=1}^n$ is split into $\gD_1$ and $\gD_2$, where $\gD_1$ is used to select a model via \texttt{E-CROMS}, and $\gD_2$ is used to construct the conformal prediction set (\ref{eq:conformal_set_single_model}) for the selected model. For annotations such as ``\texttt{E2E-0.75}'', it means that the dataset is split such that  $|\gD_1| = 0.75n$, $|\gD_2|=0.25n$. The implementation of \texttt{E2E} is outlined in Appendix C.10, which is adapted from the end-to-end approach in \citet{yeh2024end}. 
\end{itemize}
For evaluation, we compute the following metrics on the test data. Average loss is used to assess decision efficiency, with lower values indicating higher efficiency. Both marginal miscoverage and marginal misrobustness are expected to be less than or equal to $\alpha$.
\begin{itemize}\parskip 2pt
    \item[(1)] {\textit{Avg. Loss}} $=\frac{1}{m} \sum_{j = n+1}^{n+m} \phi(Y_{j},\hat{z}(X_{j}))$; 

    \item[(2)] {\textit{Marg. Misrob.}} $=\frac{1}{m}\sum_{j=n+1}^{n+m} \mathbbm{1}\left\{\phi(Y_{j}, \hat{z}(X_{j})) > \max_{c\in \widehat{\gU}(X_{j})}\phi(c,\hat{z}(X_{j}))\right\}$;

    \item[(3)] {\textit{Marg. Miscov.}} $=\frac{1}{m}\sum_{j=n+1}^{n+m} \mathbbm{1}\left\{Y_{j} \notin \widehat{\mathcal{U}}(X_{j})\right\}$.
\end{itemize} 

\subsubsection{Classification task}
We consider the classification task with $\mathcal{Y} =\{1,2,3, 4,5\}$ and $\mathcal{Z} = \{1,2,3,4,5\}$. The loss function is defined as a $|\gY|\times |\gZ|$ loss matrix $M$, that is $\phi(y,z) = M_{y,z}$ for $y\in \gY$ $z\in \gZ$.
The loss function carries practical clinical meaning in our framework as discussed in the real application of Section \ref{sec:application}: $\mathcal{Y}$ represents ordinal disease severity levels, and $\mathcal{Z}$ corresponds to available treatment options. The loss structure encodes two key clinical principles: (1) zero loss when the treatment perfectly matches disease severity (e.g., $\phi(1,1) = 0$), and (2) high loss for severe patients ($y=5$) when the treatment mismatched. 
This explicitly captures the dual dependence of clinical costs on both disease severity and treatment appropriateness.

The labeled and test data are i.i.d. generated by
$
\sP(Y=k|X) \propto \exp(-v_k(X)) \text{ for } k\in \gY,
$
where the first four coordinates of the covariate $X$ are \emph{categorical} variables taking value 0 or 1 with equal probability; the last three coordinates are standard normal random variables; and all coordinates are independent of each other. The functions $\{v_k(\cdot)\}_{k=1}^5$ have the following form: $v_k(X) = A_{k1} + A_{k2}X_1 + (A_{k3} + A_{k4}X_2)X_5 + (A_{k5} + A_{k6}X_3)X_6 + (A_{k7} + A_{k8}X_4)X_7$. The nonconformity score function is a greedy scoring rule tailored to the max-min policy from \citet{cortes2024decision}. The candidate model is defined as $S_{\lambda}(x,y) = \rho(x,y) + \lambda L(y)$ where $\lambda \in \Lambda$ is the score penalty parameter.

\begin{table}[H]
\centering
\small
\setlength{\tabcolsep}{4pt}
\caption{The evaluation metrics and running time (seconds) with the 95\% asymptotic standard error in parentheses under the classification task with $n=200$, $|\Lambda|=20$.}
\begin{tabular}{@{}clccccc@{}}
\toprule
$\alpha$ & \textbf{Method} & \textbf{Avg. Loss} & \textbf{Marg. Miscov.} & \textbf{Marg. Misrob.} & \textbf{Time} \\
\midrule
\multirow{8}{*}{0.10} 
 &Naive-CP & 4.032 (0.086) & 0.097 (0.007) & 0.037 (0.005) & 0.848 (0.027) \\
 &E2E-0.25 & 3.922 (0.096) & 0.100 (0.007) & 0.040 (0.005) & 1.843 (0.019) \\
 &E2E-0.50 & 3.860 (0.100) & 0.095 (0.008) & 0.043 (0.006) & 2.852 (0.019) \\
 &E2E-0.75 & 3.790 (0.103) & 0.100 (0.010) & 0.051 (0.008) & 3.892 (0.023) \\
 &E-CROMS & \textbf{3.470} (0.073) & 0.115 (0.007) & 0.059 (0.005) & 4.844 (0.026)\\
 &F-CROMS & \textbf{3.590} (0.086) & 0.095 (0.009) & 0.046 (0.006) & 45.004 (1.195)\\
 &J-CROMS($\alpha$) & \textbf{3.673} (0.100) & 0.091 (0.010) & 0.044 (0.006) & 95.211 (0.706)\\
 &J-CROMS($\alpha/2$) & 4.119 (0.068) & 0.045 (0.006) & 0.017 (0.003) & 94.671 (0.742)\\
\cmidrule(lr){1-6}
\multirow{8}{*}{0.20} 
 & Naive-CP & 2.938 (0.075) & 0.195 (0.009) & 0.138 (0.009) & 0.869 (0.029) \\
 & E2E-0.25 & 2.886 (0.081) & 0.195 (0.010) & 0.146 (0.011) & 1.859 (0.018) \\
 & E2E-0.50 & 2.974 (0.101) & 0.189 (0.012) & 0.136 (0.013) & 2.882 (0.019) \\
 & E2E-0.75 & 2.974 (0.114) & 0.197 (0.014) & 0.148 (0.016) & 3.932 (0.024) \\
 & E-CROMS & \textbf{2.613} (0.068) & 0.233 (0.010) & 0.194 (0.011) & 4.888 (0.029) \\
 & F-CROMS & \textbf{2.692} (0.073) & 0.189 (0.014) & 0.147 (0.014) & 34.534 (0.732) \\
 & J-CROMS($\alpha$) & \textbf{2.786} (0.086) & 0.188 (0.014) & 0.145 (0.014) & 95.970 (0.826) \\
 & J-CROMS($\alpha/2$) & 3.673 (0.100) & 0.091 (0.010) & 0.044 (0.006) & 95.554 (0.731) \\
\bottomrule
\end{tabular}
\label{tab: performance in average classification tasks}
\end{table}

Table \ref{tab: performance in average classification tasks} presents the evaluation metrics of the compared methods along with their corresponding running times. We observe that, except for \texttt{E-CROMS}, all methods consistently maintain marginal miscoverage close to the nominal level $\alpha$. And our proposed methods consistently achieve lower average loss, and both \texttt{F‑CROMS} and \texttt{J‑CROMS} empirically maintain valid coverage and robustness guarantees. Moreover, since the number of label categories is significantly smaller than the sample size, \texttt{F‑CROMS} runs faster than \texttt{J‑CROMS}, and also achieves better decision performance. Notably, the marginal misrobustness of all methods is much lower than $\alpha$, meaning the marginal coverage control in CRO is relatively conservative.

The simulation results in Figure \ref{fig:ACSSLP} illustrate the effect of labeled sample size $n$ and index set size $|\Lambda|$ on decision performance. Due to selection bias, \texttt{E-CROMS} yields miscoverage rates exceeding the nominal $\alpha$ when the sample size $n$ is small or the index set size $|\Lambda|$ is large, which aligns with the results in Theorem \ref{thm:optimality_ECROMS}. As expected, the three proposed CROMS methods outperform baselines in terms of averaged decision loss, as they effectively leverage all available data for efficient decisions. It is noteworthy that \texttt{F-CROMS} achieves a comparable loss to \texttt{E-CROMS} while still maintaining the finite-sample coverage. 

\begin{figure}[ht]
    \centering
    \begin{subfigure}[t]{\textwidth}
        \includegraphics[width=\textwidth]{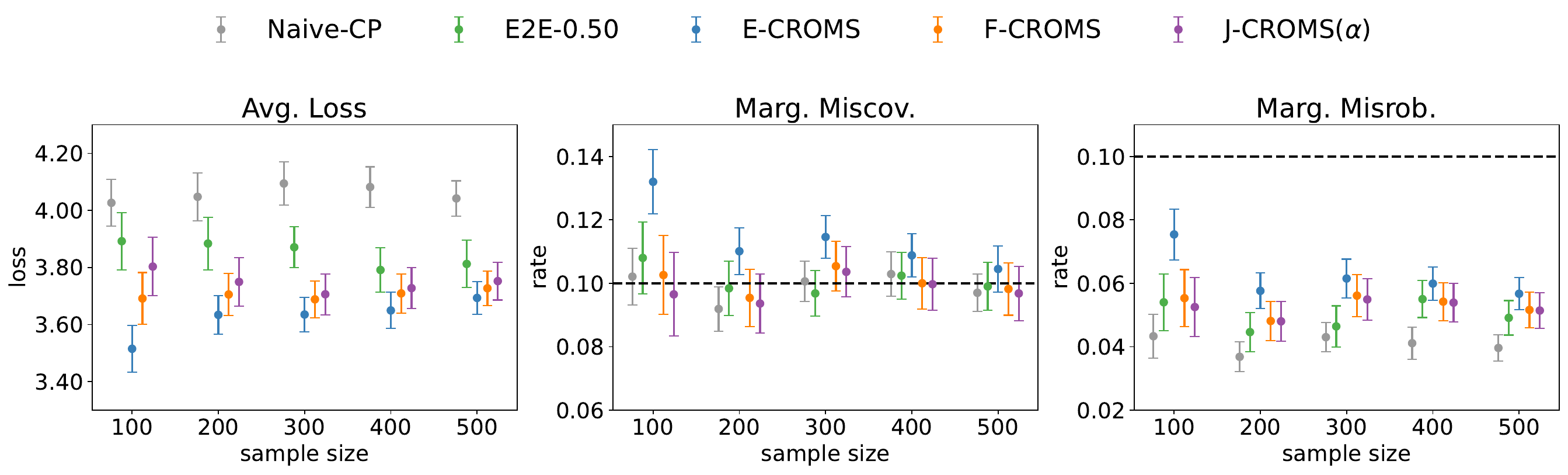}
        \caption{Varying sample size $n$ with $|\Lambda| = 10$.}
    \end{subfigure}

    \begin{subfigure}[t]{\textwidth}
        \includegraphics[width=\textwidth]{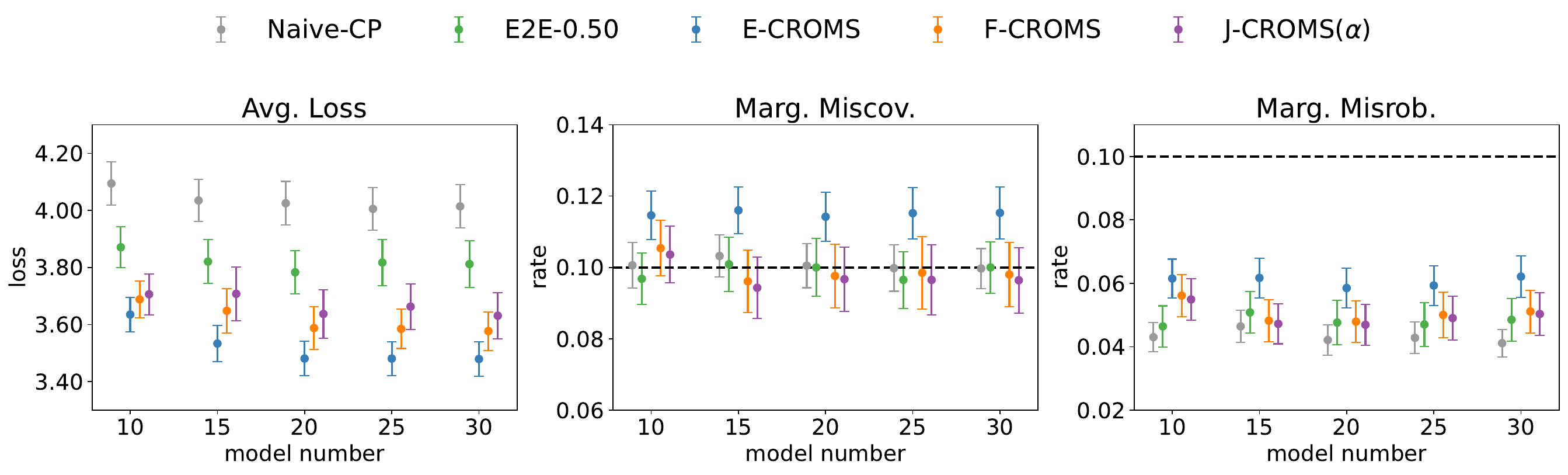}
        \caption{Varying candidate model numbers $|\Lambda|$ with $n = 300$.}
    \end{subfigure}

    \caption{The evaluation metrics with confidence intervals under the classification task. The nominal level is $\alpha = 0.1$.}
    \label{fig:ACSSLP}
\end{figure}

\subsubsection{Regression task}
In the regression task, we define the loss function as $\phi(y,z) = -y^{\top}z$, where $\mathcal{Y} = \mathbb{R}^2$ and $\mathcal{Z} = \{z \in [0,1]^2 : \|z\|_1 = 1, z \geq 0\}$. 
The labeled and test data are generated as follows:
$Y_1 = \sum_{k=1}^{50}\beta_{k1}X_k + \epsilon_1, Y_2 =\sum_{k=1}^{50}\beta_{k2}X_k + \epsilon_2$, where $\beta_{kl} = \mathbbm{1}\{(k+l) \text{ mod } 10 = 0\}$ for each $k\in[50], l\in[2]$. The features $\{X_k\}_{k=1}^{50}$ are i.i.d. truncated t-distribution (with degree 3), and the independent noises $\epsilon_1,\epsilon_2$ follow truncated normal distribution.


The candidate models are different box scores $S_{\lambda}(x,y) = \|\left(y-\hat{\mu}_{\lambda}(x)\right)/\hat{\sigma}_{\lambda}(x)\|_{\infty}$ for $\lambda \in \Lambda$, where $\hat{\mu}_{\lambda}$ and $\hat{\sigma}_{\lambda}$ are pre-trained mean and standard deviation functions, respectively. The candidate models are generated similarly to the procedure in \citet{liang2024conformal}. Specifically, $\hat{\mu}_{\lambda}, \hat{\sigma}_{\lambda}$ are obtained by first uniformly at random selecting 20\% features, then fitting the mean and standard deviation functions on the projected data, and finally embedding them back into the original 50-dimensional space. In other words, each $\lambda \in \Lambda$ corresponds to a distinct feature subset. The numerical results for ellipsoid candidate scores are deferred to Appendix G.1.3.

Table \ref{tab: performance in regression tasks} reports the evaluation metrics and running time of different methods. Our methods consistently achieve a lower average decision loss. In particular, the \texttt{J-CROMS} method has the best performance and requires significantly shorter computation time than \texttt{F-CROMS}. Even though \texttt{J-CROMS} has a distribution-free $1-2\alpha$ coverage by Theorem \ref{thm:JCROMS_robustness_finite_sample}, it empirically achieves $1-\alpha$ coverage in our simulation, which can be explained by the asymptotic results in Theorem \ref{thm:JCROMS_optimality}.
Because grid search introduces both approximation errors and computational overhead, \texttt{F‑CROMS} reaches a higher average loss comparable to \texttt{J‑CROMS}. However, it is important to emphasize that only \texttt{F‑CROMS} can provably guarantee $1-\alpha$ robustness in finite samples, irrespective of the underlying settings or data distribution.

\begin{table*}[ht]
\centering
\small
\setlength{\tabcolsep}{4pt}
\caption{The evaluation metrics and running time (seconds) with the 95\% asymptotic standard error in parentheses under the regression task with $n=150$, $|\Lambda|=25$.}
\begin{tabular}{@{}clccccc@{}}
\toprule
$\alpha$ & \textbf{Method} & \textbf{Avg. Loss} & \textbf{Marg. Miscov.} & \textbf{Marg. Misrob.} & \textbf{Time} \\
\midrule
\multirow{8}{*}{0.10} 
 & Naive-CP & -0.363 (0.072) & 0.103 (0.007) & 0.037 (0.004) & 0.670 (0.010) \\
 & E2E-0.25 & -0.583 (0.065) & 0.105 (0.008) & 0.035 (0.004) & 4.381 (0.012) \\
 & E2E-0.50 & -0.660 (0.064) & 0.091 (0.008) & 0.029 (0.004) & 7.732 (0.013) \\
 & E2E-0.75 & -0.682 (0.066) & 0.077 (0.010) & 0.023 (0.004) & 11.175 (0.034) \\
 & E-CROMS & \textbf{-0.732} (0.064) & 0.107 (0.008) & 0.032 (0.004) & 14.571 (0.103) \\
 & F-CROMS & \textbf{-0.714} (0.064) & 0.095 (0.008) & 0.023 (0.004) & 925.977 (41.851) \\
 & J-CROMS($\alpha$) & \textbf{-0.765} (0.061) & 0.089 (0.009) & 0.026 (0.004) & 80.561 (0.347) \\
 & J-CROMS($\alpha/2$) & -0.706 (0.061) & 0.038 (0.006) & 0.012 (0.003) & 80.633 (0.362) \\
\cmidrule(lr){1-6}
\multirow{8}{*}{0.20} 
 & Naive-CP & -0.393 (0.075) & 0.200 (0.010) & 0.068 (0.005) & 0.671 (0.007) \\
 & E2E-0.25 & -0.625 (0.069) & 0.204 (0.011) & 0.064 (0.005) & 4.395 (0.016) \\
 & E2E-0.50 & -0.744 (0.064) & 0.197 (0.011) & 0.059 (0.005) & 7.745 (0.014) \\
 & E2E-0.75 & -0.781 (0.065) & 0.177 (0.013) & 0.052 (0.005) & 11.181 (0.033) \\
 & E-CROMS & \textbf{-0.805} (0.061) & 0.210 (0.010) & 0.060 (0.004) & 14.516 (0.033) \\
 & F-CROMS & \textbf{-0.788} (0.061) & 0.206 (0.012) & 0.050 (0.005) & 443.787 (16.183) \\ 
 & J-CROMS($\alpha$) & \textbf{-0.823} (0.060) & 0.196 (0.011) & 0.054 (0.004) & 81.079 (0.331) \\
 & J-CROMS($\alpha/2$) & -0.765 (0.061) & 0.089 (0.009) & 0.026 (0.004) & 81.437 (0.345) \\
\bottomrule
\end{tabular}
\label{tab: performance in regression tasks}
\end{table*}


\subsection{Model selection in the individualized case}
In this section, we consider the individualized model selection setting studied in Section \ref{sec:CROiMS}. 
We compare the proposed \texttt{CROiMS} with the benchmarks mentioned in the previous subsection, as well as with \texttt{Naive-LCP}:
\begin{itemize}\parskip 4pt
     \item \texttt{Naive-LCP}: The model $S_{\lambda}$ is randomly selected from all pre-trained models $\{S_{\lambda}:\lambda \in \Lambda\}$, then the LCP set is constructed for the randomly selected model.
\end{itemize}
Notice that only \texttt{CROiMS} and \texttt{Naive-LCP} can enjoy asymptotic conditional robustness control, as defined in Definition \ref{def:individual_robustness}. For these two methods, the kernel function is chosen as $H(x,x^{\prime}) = \exp\LRs{-{\|x-x^{\prime}\|^2}/{h_n^2}}$, where the bandwidth is $h_n = C n^{-1/(d+2)}$ and matches the order of optimal choice in Theorems \ref{thm:cond_coverage}. Specifically, the constant $C$ is chosen such that $\hat{n}_{\text{eff}}(h_n) \geq 50$ when $n=200$, where $\hat{n}_{\text{eff}}(h_n)$ is the estimator of the effective sample size $n_{\text{eff}}(h_n) = n \cdot \frac{\mathbb{E}[\mathbb{E}[H(X,X^{\prime})\mid X]^2]}{\mathbb{E}[H(X,X^{\prime})^2]}$ and is computed in the pre-training dataset. 
To evaluate conditional coverage and robustness, we define $\gB$ as a set of balls in $\gX \subseteq \sR^d$, where each ball $B \in \mathcal{B}$ has its center randomly chosen and its radius set as the $10$-th (or $20$-th) percentile of the distances from the test dataset $\mathcal{D}_{\rm test}$ to the center. This ensures that each $B \in \mathcal{B}$ always contains $10\%$ (or $20\%$) of the test samples. To evaluate conditional decision loss, let $\mathcal{G}$ be a well-defined partition family of the covariate space $\mathcal{X}$, and we compute the average decision loss for each group $G \in \gG$. Denote $n_{B} = \sum_{j=n+1}^{n+m} \mathbbm{1}\{X_{j} \in B\}$ for $B \in \gB$ and $n_{G} = \sum_{j=n+1}^{n+m}\mathbbm{1}\{X_{j} \in G\}$ for $G\in \gG$. Let $W_j = \max_{c\in \widehat{\gU}(X_{j})}\phi(c,\hat{z}(X_{j}))$, we define:
\begin{itemize}\parskip 4pt
    \item[(4)] \textit{Worst Cond. Miscov.} $=\min_{B \in \mathcal{B}} \frac{1}{n_{B}} \sum_{j=n+1}^{n+m} \mathbbm{1}\left\{X_{j} \in B, Y_{j} \notin \widehat{\mathcal{U}}(X_{j})\right\}$.
    
    \item[(5)] \textit{Worst Cond. Misrob.} $=\min_{B \in \mathcal{B}} \frac{1}{n_{B}} \sum_{j=n+1}^{n+m} \mathbbm{1}\left\{X_{j} \in B, \phi(Y_{j}, \hat{z}(X_{j})) > W_j\right\}$.

    \item[(6)] \textit{Group Cond. Loss} $=\frac{1}{n_{G}} \sum_{j=n+1}^{n+m} \mathbbm{1}\{X_{j} \in G\} \phi(Y_{j},\hat{z}(X_{j}))$.
\end{itemize}


In this simulation, we consider the classification task with $\mathcal{Y} = \{1,2,3\}$ and $\gZ=\{1,2,3\}$. The loss function is defined as $\phi(z,y) = M_{z,y}$ for $z\in \gZ$ and $y\in \gY$.
The labeled data and test data are generated by
$
p(Y=k|X) \propto \exp(-\beta_k^{\top}X) \text{ for } k \in \{1,2,3\},
$
where $\beta_1^{\top}=(1,5,6),\ \beta_2^{\top} = (5,1,6),\ \beta_3^{\top} = (4,4,4)$ and $X = (X_1,X_2,X_3)^\top$ follows the multivariate
normal distribution $N(0, \Sigma)$. Other details of this simulation are given in Appendix G.2.
The candidate nonconformity score functions are $S_{\lambda}(x,y) = 1 - f_{\lambda}^{y}(x)$ for $\lambda \in \Lambda = \{1,2,3\}$, where $f_{\lambda}:\mathcal{X} \rightarrow [0,1]^{|\mathcal{Y}|}$ is the softmax layer of a classifier. Candidate models $\{S_\lambda:\lambda \in \Lambda\}$ are trained by the Gradient Boosting algorithm with the target variable $Y$ and various covariates including $(X_1,X_2)$, $(X_1,X_3)$ and $(X_2,X_3)$, respectively. 

We evaluate the decision performance of \texttt{CROiMS} and the baseline methods by varying the labeled sample size $n$. 
In Figure \ref{fig:ICFSM-loss} with a varying sample size, those methods aiming to control average decision risk, \texttt{E2E}, \texttt{E-CROMS}, and \texttt{F-CROMS}, fail to control conditional miscoverage and conditional misrobustness at the target level $\alpha$. In contrast, the conditional miscoverage and conditional misrobustness of both \texttt{Naive-LCP} and \texttt{CROiMS} gradually are close to $\alpha$ as the sample size $n$ increases. Moreover, \texttt{CROiMS} consistently outperforms the other methods in terms of average decision loss. In Figure \ref{fig:ICFSM-losscon}, \texttt{CROiMS} also achieves the lowest conditional loss across most regions $G \in \mathcal{G}$ among all methods. These results suggest that \texttt{CROiMS} performs better in individualized model selection.

\begin{figure}[H] 
\centering
\includegraphics[width=\textwidth]{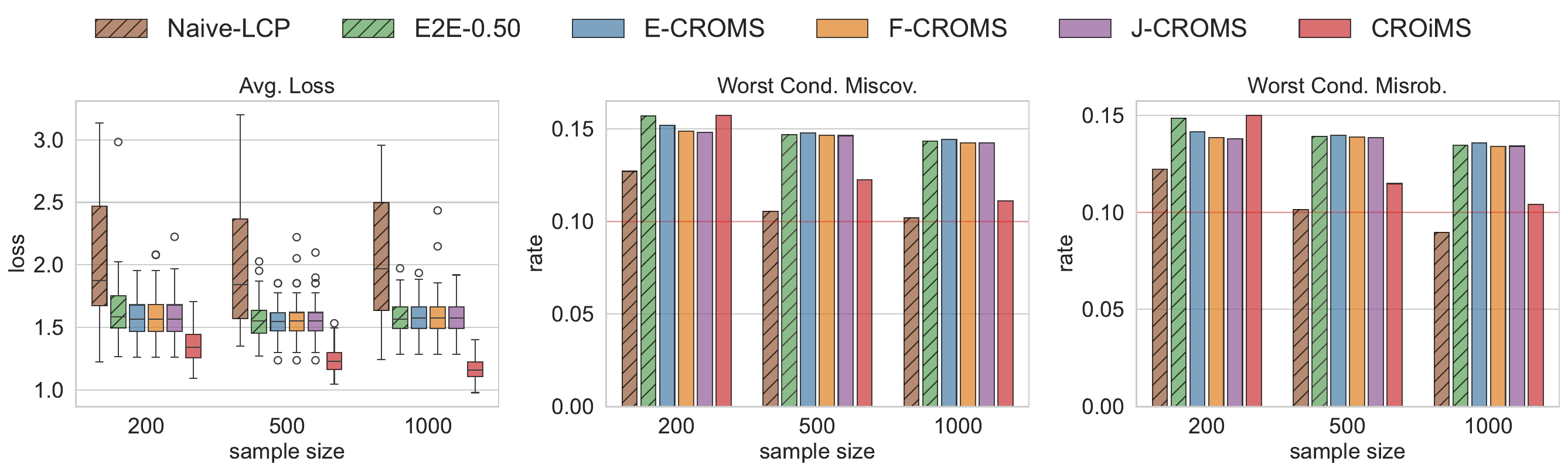}
\caption{The average loss, worst-case conditional miscoverage, and worst-case conditional misrobustness when varying sample size $n$ in the classification task, where candidate models are trained on different covariates, $|\Lambda| = 3$ and $\alpha = 0.1$.}
\label{fig:ICFSM-loss}
\end{figure}

\begin{figure}[H] 
\centering
\includegraphics[width=.88\textwidth]{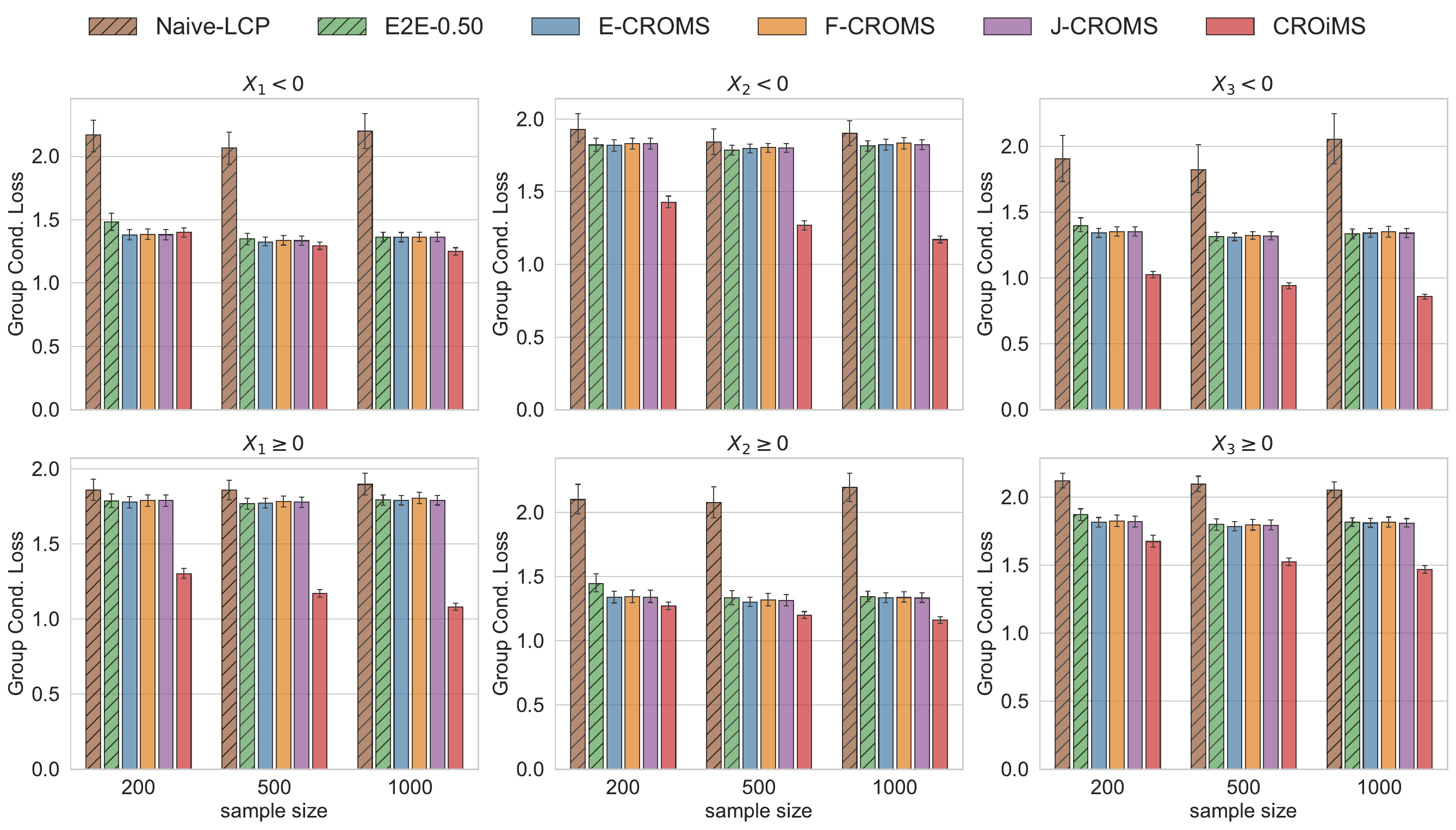}
\caption{The group conditional losses when varying sample size $n$ in the classification task, where candidate models are trained on different covariates, $|\Lambda| = 3$ and $\alpha = 0.1$.}
\label{fig:ICFSM-losscon}
\end{figure}


\section{Real Data Application}\label{sec:application}

The COVID-19 Radiography Database \citep{chowdhury2020can} comprises chest X-ray images (covariates) categorized into four classes (labels): Normal, Pneumonia, COVID-19, and Lung Opacity. To align with clinical priorities, we employ the loss matrix designed in \citet{kiyani2025decision}. 
We apply 8,240 images from this dataset for the experiment. Candidate models $\{S_{\lambda}: \lambda \in \Lambda\}$ with $|\Lambda|=4$ are trained on four randomly sampled datasets with size 1000, each with a distinct label distribution. In this experiment, the score function is $S_{\lambda}(x,y) = 1 - f_{\lambda}^{y}(x)$, where the classifier $f_{\lambda}:\mathcal{X} \rightarrow [0,1]^4$ is obtained with the convolutional neural network (CNN). In each replication, we randomly sample labeled and test data of size 300, respectively. 
For the similarity measurement of CROiMS, we consider the kernel function as $H(x,x^{\prime}) = \exp\left(-\|f_{\text{ex}}(x),f_{\text{ex}}(x^{\prime})\|^2/h^2\right)$, where $f_{\text{ex}}(x)$ is a pre-trained feature extractor that maps high-dimensional images $X$ ($3\times224\times 224$) to low-dimensional feature representations ($16\times 1$). Other details are deferred to Appendix G.3. 

To illustrate the robustness and efficiency, we examine the decision performance across two nominal misrobustness levels $\alpha=0.05$ and $\alpha=0.1$. In Figure \ref{fig:Medical diagnosis-SCP}, we compare Average Loss, Marginal Misrobustness, and Worst-case Conditional Misrobustness among different model selection approaches on COVID-19 dataset. We observe that the proposed methods \texttt{E-CROMS}, \texttt{F-CROMS}, and \texttt{CROiMS} achieve lower average loss compared to other benchmarks. Specifically, \texttt{CROiMS} results in the smallest average loss while maintaining control of worst-case conditional misrobustness. 
In Figure \ref{fig:Medical diagnosis-Part}, we present the group conditional loss across various prediction categories $\hat{Y} \in \{\text{``COVID-19''}, \text{``Lung Opacity''}, \text{``Normal''}, \text{``Pneumonia''}\}$. The results demonstrate that \texttt{CROiMS} consistently outperforms other methods by achieving lower losses across all groups. The significant improvements achieved by \texttt{CROiMS} highlight that individualized model selection based on decision efficiency could reduce the risk of the treatment strategy, aligning with the personalized medicine paradigm.

\begin{figure}[H] 
    \centering
    \includegraphics[width=1.0\textwidth]{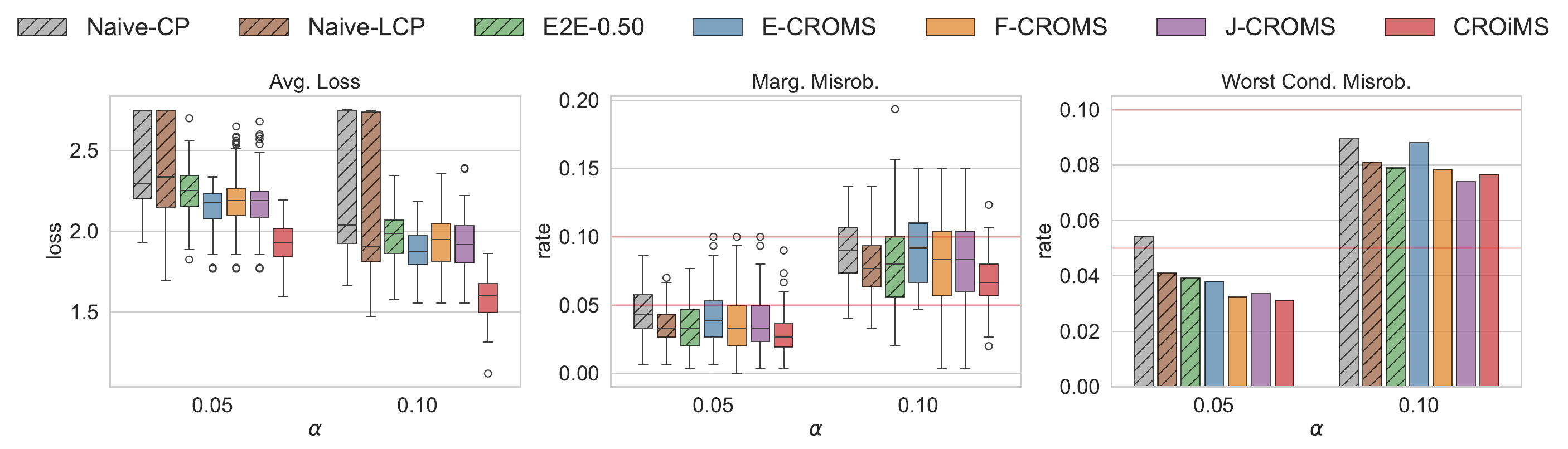}
    \caption{The average loss, marginal misrobustness, and worst-case conditional misrobustness on COVID-19 Radiography Database.}
    \label{fig:Medical diagnosis-SCP}
\end{figure}


\begin{figure}[H] 
    \centering
    \includegraphics[width=.88\textwidth]{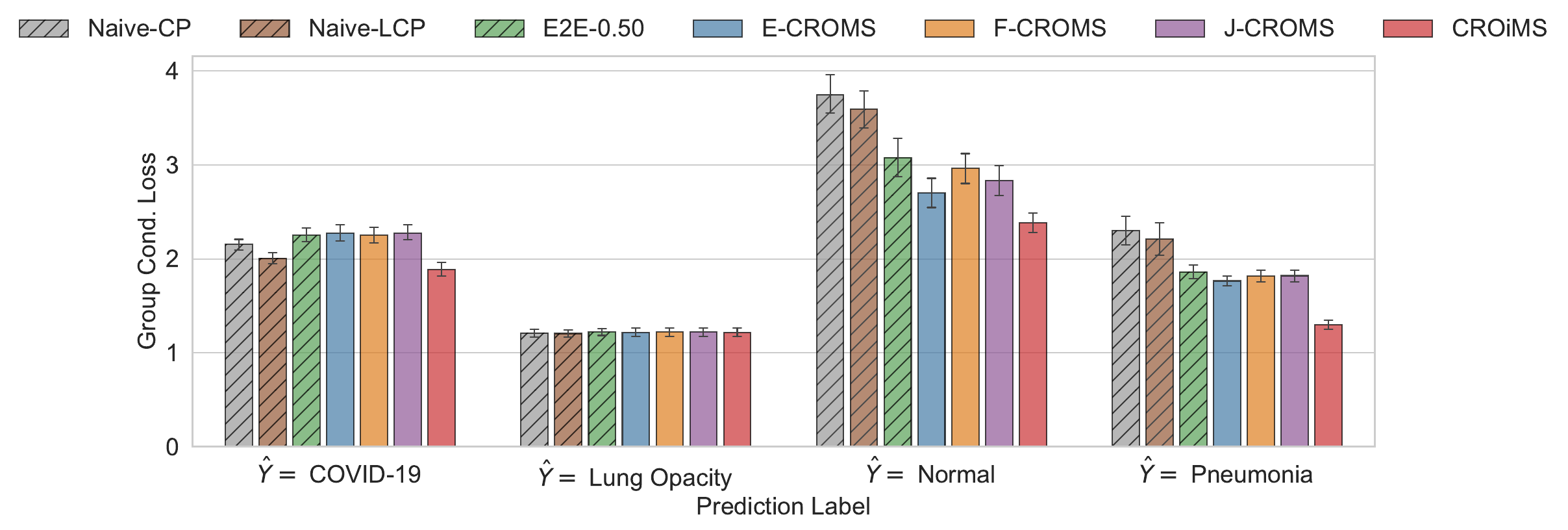}
    \caption{The group conditional loss on COVID-19 Radiography Database under $\alpha = 0.1$.}
    \label{fig:Medical diagnosis-Part}
\end{figure}



\section{Discussion and Future Work}


This paper explores optimal model selection for constructing a conformal prediction set in solving CRO problems. We propose two novel frameworks, CROMS and CROiMS, targeting at minimizing average and individual decision risks. Numerical results demonstrate substantial improvements in decision efficiency and robustness. Our model-free frameworks do not require data splitting, maximizing sample utilization in decision-making.
There are two promising future directions on this topic. 

(1) \textit{CROMS under a continuous model class.}
First, when $\Lambda$ is a continuous model space, the theoretical results in Theorems \ref{thm:robustness_ECROMS} and \ref{thm:optimality_ECROMS} show that E-CROMS retains asymptotic robustness and optimality, while F-CROMS satisfies finite-sample robustness and asymptotic optimality guarantees. For implementation, gradient-based methods can be employed to solve the ERM problems \eqref{eq:eff_CROMS_ERM} and \eqref{eq:hypothesized_avg_ERM} over the model index $\lambda$. Since most CRO problems are convex, we may adopt implicit differentiation techniques for convex programs, as proposed in \citet{bolte2021nonsmooth}, to compute the gradients of $\phi(Y_i, z_{\lambda}(X_i))$ with respect to $\lambda$. However, the corresponding ERM problems are generally nonconvex in $\lambda$, making it challenging to obtain a global minimizer in practice. Rather than relying on the exact \texttt{argmin} of the ERM formulation, a more practical analysis of robustness and efficiency should be grounded in a concrete algorithm, such as the bilevel optimization methods in \citet{ghadimi2018approximation}.


{(2) \textit{The discrepancy between coverage and robustness.}
The achieved robustness level may exceed the nominal level $1-\alpha$ because the coverage is sufficient but not a necessary condition for the robustness, which is also verified by our numerical result, e.g., Figure \ref{fig:ACSSLP}. Given the prediction set $\mathcal{U}(X)$ and the resulting decision $z(X)$, we define the ``robust region'' in the label space as $\gY_{\rm robust}(X) := \{y\in \gY: \phi(y, z(X)) \leq \max_{c\in \gU(X)}\phi(c, z(X))\}$. Clearly, $\mathcal{U}(X)$ is a subset of $\mathcal{Y}_{\mathrm{robust}}(X)$, and the discrepancy between marginal robustness and coverage can be quantified by the probability $\sP\{Y\in \gY_{\rm robust}(X)\setminus \gU(X)\}$. In the context of portfolio optimization, if we use the ellipsoidal prediction set, the corresponding robust region is a half-space in $\sR^p$. In Appendix A, we precisely analyze the gap $\mathbb{P}\{Y \in \mathcal{Y}_{\mathrm{robust}}(X) \setminus \mathcal{U}(X)\}$ for the portfolio optimization problem under a Gaussian data assumption. In such a setting, we can adjust the confidence level of the prediction set to achieve exact marginal robustness at level $1-\alpha$. However, when the data distribution is unknown, we could provide a procedure to construct a prediction set by directly controlling robustness in an asymptotic regime, but the finite-sample control becomes challenging and warrants further research.}






\phantomsection\label{supplementary-material}
\bigskip

\begin{center}

{\large\bf SUPPLEMENTARY MATERIAL}

\end{center}

The supplementary material contains the implementation details of F-CROMS, proofs of theoretical results, and deferred numerical settings and results.


\bibliography{references}

\begin{thebibliography}{53}
\newcommand{\enquote}[1]{``#1''}
\expandafter\ifx\csname natexlab\endcsname\relax\def\natexlab#1{#1}\fi

\bibitem[{Angelopoulos et~al.(2024)Angelopoulos, Barber, and
  Bates}]{angelopoulos2024theoretical}
Angelopoulos, A.~N., Barber, R.~F., and Bates, S. (2024), \enquote{Theoretical
  foundations of conformal prediction,} \textit{arXiv preprint
  arXiv:2411.11824}.

\bibitem[{Bai et~al.(2022)Bai, Mei, Wang, Zhou, and Xiong}]{bai2022efficient}
Bai, Y., Mei, S., Wang, H., Zhou, Y., and Xiong, C. (2022), \enquote{Efficient
  and differentiable conformal prediction with general function classes,} in
  \textit{International Conference on Learning Representations}.

\bibitem[{Bao et~al.(2024{\natexlab{a}})Bao, Huo, Ren, and Zou}]{bao2024cap}
Bao, Y., Huo, Y., Ren, H., and Zou, C. (2024{\natexlab{a}}), \enquote{{CAP}: a
  general algorithm for online selective conformal prediction with {FCR}
  control,} \textit{arXiv preprint arXiv:2403.07728}.

\bibitem[{Bao et~al.(2024{\natexlab{b}})Bao, Huo, Ren, and
  Zou}]{bao2024selective}
--- (2024{\natexlab{b}}), \enquote{Selective conformal inference with false
  coverage-statement rate control,} \textit{Biometrika}, 111, 727--742.

\bibitem[{Barber et~al.(2021)Barber, Candes, Ramdas, and
  Tibshirani}]{barber2021predictive}
Barber, R.~F., Candes, E.~J., Ramdas, A., and Tibshirani, R.~J. (2021),
  \enquote{Predictive inference with the jackknife+,} \textit{The Annals of
  Statistics}, 49, 486--507.

\bibitem[{Ben-Tal et~al.(2009)Ben-Tal, Ghaoui, and Nemirovski}]{ben2009robust}
Ben-Tal, A., Ghaoui, L., and Nemirovski, A. (2009), \textit{Robust
  Optimization}, Princeton Series in Applied Mathematics, Princeton University
  Press.

\bibitem[{Bian and Barber(2023)}]{bian2023training}
Bian, M. and Barber, R.~F. (2023), \enquote{Training-conditional coverage for
  distribution-free predictive inference,} \textit{Electronic Journal of
  Statistics}, 17, 2044--2066.

\bibitem[{Bolte et~al.(2021)Bolte, Le, Pauwels, and
  Silveti-Falls}]{bolte2021nonsmooth}
Bolte, J., Le, T., Pauwels, E., and Silveti-Falls, T. (2021),
  \enquote{Nonsmooth implicit differentiation for machine-learning and
  optimization,} \textit{Advances in Neural Information Processing Systems},
  34, 13537--13549.

\bibitem[{Boyd and Vandenberghe(2004)}]{Boyd_Vandenberghe_2004}
Boyd, S. and Vandenberghe, L. (2004), \textit{Convex Optimization}, Cambridge
  University Press.

\bibitem[{Braun et~al.(2025)Braun, Aolaritei, Jordan, and
  Bach}]{braun2025minimum}
Braun, S., Aolaritei, L., Jordan, M.~I., and Bach, F. (2025), \enquote{Minimum
  volume conformal sets for multivariate regression,} \textit{arXiv preprint
  arXiv:2503.19068}.

\bibitem[{Chen et~al.(2018)Chen, Chun, and Barber}]{chen2018discretized}
Chen, W., Chun, K.-J., and Barber, R.~F. (2018), \enquote{Discretized conformal
  prediction for efficient distribution-free inference,} \textit{Stat}, 7,
  e173.

\bibitem[{Chenreddy et~al.(2022)Chenreddy, Bandi, and
  Delage}]{chenreddy2022data}
Chenreddy, A., Bandi, N., and Delage, E. (2022), \enquote{Data-driven
  conditional robust optimization,} \textit{Advances in Neural Information
  Processing Systems}, 35, 9525--9537.

\bibitem[{Chenreddy and Delage(2024)}]{chenreddy2024end}
Chenreddy, A.~R. and Delage, E. (2024), \enquote{End-to-end conditional robust
  optimization,} in \textit{Proceedings of the Fortieth Conference on
  Uncertainty in Artificial Intelligence}, PMLR, vol. 244, pp. 736--748.

\bibitem[{Chowdhury et~al.(2020)Chowdhury, Rahman, Khandakar, Mazhar, Kadir,
  Mahbub, Islam, Khan, Iqbal, Al~Emadi, et~al.}]{chowdhury2020can}
Chowdhury, M.~E., Rahman, T., Khandakar, A., Mazhar, R., Kadir, M.~A., Mahbub,
  Z.~B., Islam, K.~R., Khan, M.~S., Iqbal, A., Al~Emadi, N., et~al. (2020),
  \enquote{Can {AI} help in screening viral and {COVID}-19 pneumonia?}
  \textit{Ieee Access}, 8, 132665--132676.

\bibitem[{Cortes-Gomez et~al.(2024)Cortes-Gomez, Pati{\~n}o, Byun, Wu, Horvitz,
  and Wilder}]{cortes2024decision}
Cortes-Gomez, S., Pati{\~n}o, C., Byun, Y., Wu, S., Horvitz, E., and Wilder, B.
  (2024), \enquote{Decision-focused uncertainty quantification,} \textit{arXiv
  preprint arXiv:2410.01767}.

\bibitem[{Dempe(2002)}]{dempe2002foundations}
Dempe, S. (2002), \textit{Foundations of Bilevel Programming}, Springer Science
  \& Business Media.

\bibitem[{Donti et~al.(2017)Donti, Amos, and Kolter}]{donti2017task}
Donti, P., Amos, B., and Kolter, J.~Z. (2017), \enquote{Task-based end-to-end
  model learning in stochastic optimization,} \textit{Advances in Neural
  Information Processing Systems}, 30, 5490 -- 5500.

\bibitem[{Elmachtoub and Grigas(2022)}]{elmachtoub2022smart}
Elmachtoub, A.~N. and Grigas, P. (2022), \enquote{Smart `predict, then
  optimize',} \textit{Management Science}, 68, 9--26.

\bibitem[{Ghadimi and Wang(2018)}]{ghadimi2018approximation}
Ghadimi, S. and Wang, M. (2018), \enquote{Approximation methods for bilevel
  programming,} \textit{arXiv preprint arXiv:1802.02246}.

\bibitem[{Guan(2023)}]{guan2023localized}
Guan, L. (2023), \enquote{Localized conformal prediction: A generalized
  inference framework for conformal prediction,} \textit{Biometrika}, 110,
  33--50.

\bibitem[{Hore and Barber(2024)}]{hore2024conformal}
Hore, R. and Barber, R.~F. (2024), \enquote{Conformal prediction with local
  weights: randomization enables robust guarantees,} \textit{Journal of the
  Royal Statistical Society Series B: Statistical Methodology}, 87, 549--578.

\bibitem[{Izbicki et~al.(2022)Izbicki, Shimizu, and Stern}]{izbicki2022cd}
Izbicki, R., Shimizu, G., and Stern, R.~B. (2022), \enquote{Cd-split and
  hpd-split: Efficient conformal regions in high dimensions,} \textit{Journal
  of Machine Learning Research}, 23, 1--32.

\bibitem[{Jin and Ren(2025)}]{jin2024confidence}
Jin, Y. and Ren, Z. (2025), \enquote{Confidence on the focal: conformal
  prediction with selection-conditional coverage,} \textit{Journal of the Royal
  Statistical Society Series B: Statistical Methodology}, qkaf016.

\bibitem[{Johansson et~al.(2017)Johansson, Linusson, L{\"o}fstr{\"o}m, and
  Bostr{\"o}m}]{johansson2017model}
Johansson, U., Linusson, H., L{\"o}fstr{\"o}m, T., and Bostr{\"o}m, H. (2017),
  \enquote{Model-agnostic nonconformity functions for conformal
  classification,} in \textit{2017 International Joint Conference on Neural
  Networks (IJCNN)}, IEEE, pp. 2072--2079.

\bibitem[{Johnstone and Cox(2021)}]{johnstone2021conformal}
Johnstone, C. and Cox, B. (2021), \enquote{Conformal uncertainty sets for
  robust optimization,} in \textit{Conformal and Probabilistic Prediction and
  Applications}, PMLR, pp. 72--90.

\bibitem[{Kaur et~al.(2025)Kaur, Jordan, and Alaa}]{kaur2025conformal}
Kaur, J.~N., Jordan, M.~I., and Alaa, A. (2025), \enquote{Conformal prediction
  sets with improved conditional coverage using trust scores,} \textit{arXiv
  preprint arXiv:2501.10139}.

\bibitem[{Kearns and Vazirani(1994)}]{kearns1994introduction}
Kearns, M.~J. and Vazirani, U. (1994), \textit{An introduction to computational
  learning theory}, MIT press.

\bibitem[{Kiyani et~al.(2025)Kiyani, Pappas, Roth, and
  Hassani}]{kiyani2025decision}
Kiyani, S., Pappas, G., Roth, A., and Hassani, H. (2025), \enquote{Decision
  theoretic foundations for conformal prediction: optimal uncertainty
  quantification for risk-averse agents,} \textit{arXiv preprint
  arXiv:2502.02561}.

\bibitem[{Kiyani et~al.(2024)Kiyani, Pappas, and Hassani}]{kiyani2024length}
Kiyani, S., Pappas, G.~J., and Hassani, H. (2024), \enquote{Length optimization
  in conformal prediction,} \textit{Advances in Neural Information Processing
  Systems}, 37, 99519--99563.

\bibitem[{Lei et~al.(2018)Lei, G’Sell, Rinaldo, Tibshirani, and
  Wasserman}]{lei2018distribution}
Lei, J., G’Sell, M., Rinaldo, A., Tibshirani, R.~J., and Wasserman, L.
  (2018), \enquote{Distribution-free predictive inference for regression,}
  \textit{Journal of the American Statistical Association}, 113, 1094--1111.

\bibitem[{Lei et~al.(2013)Lei, Robins, and Wasserman}]{lei2013distribution}
Lei, J., Robins, J., and Wasserman, L. (2013), \enquote{Distribution-free
  prediction sets,} \textit{Journal of the American Statistical Association},
  108, 278--287.

\bibitem[{Lei and Wasserman(2013)}]{lei2013nonparametric}
Lei, J. and Wasserman, L. (2013), \enquote{Distribution-free prediction bands
  for non-parametric regression,} \textit{Journal of the Royal Statistical
  Society Series B: Statistical Methodology}, 76, 71--96.

\bibitem[{Liang and Barber(2025)}]{liang2025algorithmic}
Liang, R. and Barber, R.~F. (2025), \enquote{Algorithmic stability implies
  training-conditional coverage for distribution-free prediction methods,}
  \textit{The Annals of Statistics}, 53, 1457--1482.

\bibitem[{Liang et~al.(2024)Liang, Zhu, and Barber}]{liang2024conformal}
Liang, R., Zhu, W., and Barber, R.~F. (2024), \enquote{Conformal prediction
  after efficiency-oriented model selection,} \textit{arXiv preprint
  arXiv:2408.07066}.

\bibitem[{Mo et~al.(2021)Mo, Qi, and Liu}]{mo2021learning}
Mo, W., Qi, Z., and Liu, Y. (2021), \enquote{Learning optimal distributionally
  robust individualized treatment rules,} \textit{Journal of the American
  Statistical Association}, 116, 659--674.

\bibitem[{Mulmuley(1994)}]{alma991023418299706532}
Mulmuley, K. (1994), \textit{Computational geometry : an introduction through
  randomized algorithms / Ketan Mulmuley.}, Englewood Cliffs, N.J:
  Prentice-Hall.

\bibitem[{Patel et~al.(2024)Patel, Rayan, and Tewari}]{patel2024conformal}
Patel, Y.~P., Rayan, S., and Tewari, A. (2024), \enquote{Conformal contextual
  robust optimization,} in \textit{International Conference on Artificial
  Intelligence and Statistics}, PMLR, pp. 2485--2493.

\bibitem[{Romano et~al.(2019)Romano, Patterson, and
  Candes}]{romano2019conformalized}
Romano, Y., Patterson, E., and Candes, E. (2019), \enquote{Conformalized
  quantile regression,} \textit{Advances in Neural Information Processing
  Systems}, 32, 3543 -- 3553.

\bibitem[{Sadinle et~al.(2019)Sadinle, Lei, and Wasserman}]{sadinle2019least}
Sadinle, M., Lei, J., and Wasserman, L. (2019), \enquote{Least ambiguous
  set-valued classifiers with bounded error levels,} \textit{Journal of the
  American Statistical Association}, 114, 223--234.

\bibitem[{Sesia and Cand{\`e}s(2020)}]{sesia2020comparison}
Sesia, M. and Cand{\`e}s, E.~J. (2020), \enquote{A comparison of some conformal
  quantile regression methods,} \textit{Stat}, 9, e261.

\bibitem[{Steinberger and Leeb(2023)}]{steinberger2023conditional}
Steinberger, L. and Leeb, H. (2023), \enquote{Conditional predictive inference
  for stable algorithms,} \textit{The Annals of Statistics}, 51, 290--311.

\bibitem[{Sun et~al.(2023)Sun, Liu, and Li}]{sun2023predict}
Sun, C., Liu, L., and Li, X. (2023), \enquote{Predict-then-calibrate: a new
  perspective of robust contextual {LP},} \textit{Advances in Neural
  Information Processing Systems}, 36, 17713--17741.

\bibitem[{Tschandl et~al.(2018)Tschandl, Rosendahl, and
  Kittler}]{tschandl2018ham10000}
Tschandl, P., Rosendahl, C., and Kittler, H. (2018), \enquote{The HAM10000
  dataset, a large collection of multi-source dermatoscopic images of common
  pigmented skin lesions,} \textit{Scientific Data}, 5, 180161.

\bibitem[{Van Der~Vaart and Wellner(1996)}]{van1996weak}
Van Der~Vaart, A.~W. and Wellner, J.~A. (1996), \textit{Weak convergence and
  empirical processes: with applications to statistics}, Springer.

\bibitem[{Vovk(2012)}]{vovk2012conditional}
Vovk, V. (2012), \enquote{Conditional validity of inductive conformal
  predictors,} in \textit{Proceedings of the Asian Conference on Machine
  Learning}, PMLR, vol.~25, pp. 475--490.

\bibitem[{Vovk(2015)}]{vovk2015cross}
--- (2015), \enquote{Cross-conformal predictors,} \textit{Annals of Mathematics
  and Artificial Intelligence}, 74, 9--28.

\bibitem[{Vovk et~al.(2005)Vovk, Gammerman, and Shafer}]{vovk2005algorithmic}
Vovk, V., Gammerman, A., and Shafer, G. (2005), \textit{Algorithmic Learning in
  A Random World}, vol.~29, Springer.

\bibitem[{Vovk et~al.(2018)Vovk, Nouretdinov, Manokhin, and
  Gammerman}]{vovk2018cross}
Vovk, V., Nouretdinov, I., Manokhin, V., and Gammerman, A. (2018),
  \enquote{Cross-conformal predictive distributions,} in \textit{conformal and
  probabilistic prediction and applications}, PMLR, pp. 37--51.

\bibitem[{Wang et~al.(2023)Wang, Becker, Van~Parys, and
  Stellato}]{wang2023learning}
Wang, I., Becker, C., Van~Parys, B., and Stellato, B. (2023), \enquote{Learning
  decision-focused uncertainty sets in robust optimization,} \textit{arXiv
  preprint arXiv:2305.19225}.

\bibitem[{Wasserman(2006)}]{wasserman2006all}
Wasserman, L. (2006), \textit{All of Nonparametric Statistics}, Springer
  Science \& Business Media.

\bibitem[{Wasserman(2020)}]{lecture9}
--- (2020), \enquote{Lecture 9: VC Dimension,} Statistical Learning Theory
  (36-705) course notes, Carnegie Mellon University.

\bibitem[{Yang and Kuchibhotla(2025)}]{yang2024selection}
Yang, Y. and Kuchibhotla, A.~K. (2025), \enquote{Selection and aggregation of
  conformal prediction sets,} \textit{Journal of the American Statistical
  Association}, 120, 435--447.

\bibitem[{Yeh et~al.(2024)Yeh, Christianson, Wu, Wierman, and Yue}]{yeh2024end}
Yeh, C., Christianson, N., Wu, A., Wierman, A., and Yue, Y. (2024),
  \enquote{End-to-end conformal calibration for optimization under
  uncertainty,} \textit{arXiv preprint arXiv:2409.20534}.

\end{thebibliography}

\newpage
\appendix
\allowdisplaybreaks
\numberwithin{equation}{section}
\numberwithin{figure}{section}
\numberwithin{assumption}{section}
\numberwithin{algorithm}{section}

\begin{center}
    {\LARGE\bf Supplementary Material for ``Optimal Model Selection for Conformalized Robust Optimization''}
\end{center}

\section*{Notations}
In Table \ref{table:notations}, we present the notations used throughout the main text and the appendix.

\begin{table}[ht]
    \centering
    \caption{Summary of notations.}
    \label{table:notations}
    \begin{center}
\resizebox{\linewidth}{!}{
\begin{tabular}{lll}
\toprule
Name & Definition & Comment\\
\midrule
$\mathcal{X}$ & Features space & $\mathcal{X} \subseteq \sR^d$ \\
$\mathcal{Y}$ & Label space & \\
$\gZ$ & Decision space &\\
\multicolumn{3}{@{}c@{}}{} \\
$\phi(y,z)$ & Loss function of decision $z$ on the label $y$& $|\phi(y,z)| \leq B$\\
$\gU_{\lambda}(x;q)$ & Conformal prediction set of score $S_{\lambda}$ with threshold $q$ & $\gU(x;q) = \{c\in \gY: S_{\lambda}(X,c) \leq q\}$\\
$z_{\lambda}(x;q)$ & CRO solution under prediction set $\gU_{\lambda}(x;q)$ & $z(x;q) = \argmin_{z\in \gZ} \max_{c\in \gU_{\lambda}(x;q)} \phi(c,z)$\\
\multicolumn{3}{@{}c@{}}{} \\
$F_{\lambda}(\cdot)$ & CDF of score $S_{\lambda}(X,Y)$ & $S_{\lambda}(X,Y)\sim F_{\lambda}$\\
$F_{\lambda}(\cdot|x)$ & Conditional CDF of score $S_{\lambda}(X,Y)$ given $X=x$ & $S_{\lambda}(X,Y)\mid X=x\sim F_{\lambda}(\cdot|x)$\\
$F_{\lambda}^{-1}(\cdot)$ & Quantile of score $S_{\lambda}(X,Y)$& $F_{\lambda}^{-1}(1-\alpha) = \inf\{q: F_{\lambda}(q) \geq 1-\alpha\}$\\
$F_{\lambda}^{-1}(\cdot|x)$ & Conditional quantile of score $S_{\lambda}(X,Y)$ given $X=x$& $F_{\lambda}^{-1}(1-\alpha|x) = \inf\{q: F_{\lambda}(q|x) \geq 1-\alpha\}$\\
\multicolumn{3}{@{}c@{}}{} \\
$q_{\lambda}^{o}$ & $(1-\alpha)$-th quantile of score $S_{\lambda}(X,Y)$ & $q_{\lambda}^{o} = F_{\lambda}^{-1}(1-\alpha)$\\
$Q_{1-\alpha}(\{s_i\}_{i=1}^n)$ & $(1-\alpha)$-th quantile of  $\frac{1}{n}\sum_{i=1}^n \delta_{s_i}$&\\
$\hat{q}_{\lambda}$ & $(1-\alpha)(1+n^{-1})$-th quantile of  $\frac{1}{n}\sum_{i=1}^n \delta_{S_{\lambda}(X_i,Y_i)}$ & $\hat{q}_{\lambda} = Q_{(1-\alpha)(1+n^{-1})}(\{S_{\lambda}(X_i,Y_i)\}_{i=1}^n)$\\
$\hat{q}_{\lambda}^y$ & $(1-\alpha)$-th quantile of $\frac{1}{n+1}\left(\sum_{i=1}^n \delta_{S_{\lambda}(X_i,Y_i)} + \delta_{S_{\lambda}(X_{n+1},y)}\right)$ & $\hat{q}_{\lambda}^y = Q_{1-\alpha}(\{S_{\lambda}(X_i,Y_i)\}_{i=1}^n \cup \{S_{\lambda}(X_{n+1},y)\})$\\
\multicolumn{3}{@{}c@{}}{} \\
$q_{\lambda}^{o}(x)$ & $(1-\alpha)$-th conditional quantile of $S_{\lambda}(X,Y)$ given $X=x$ & $q_{\lambda}^{o}(x) = F_{\lambda}^{-1}(1-\alpha|x)$\\
$Q_{1-\alpha}(\{s_i\}_{i=1}^n; \{w_i\}_{i=1}^n)$ & $(1-\alpha)$-th quantile of $\sum_{i=1}^n w_i \delta_{s_i}$ &\\
$w_i(x)$ & kernel weight function & $w_i(x) = H(X_i,x)/\sum_{j=1}^n H(X_j, x)$\\
$\hat{q}_{\lambda}(x)$ & $(1-\alpha)$-th quantile of  $\sum_{i=1}^n w_i(x)\delta_{S_{\lambda}(X_i,Y_i)}$& $\hat{q}_{\lambda}(x) = Q_{1-\alpha}(\{S_{\lambda}(X_i,Y_i)\}_{i=1}^n; \{w_i(x)\}_{i=1}^n)$\\
\bottomrule
\end{tabular}
}
\end{center}
\end{table}

\section{The gap between marginal coverage and robustness}\label{appen:gap_robustness_coverage}

\subsection{Quantify the gap by robust region}
 For the prediction set $\gU(X)$, we denote the corresponding CRO decision as $z(X) = \argmin_{z\in \gZ}\max_{c\in \gU(X)}\phi(c,z)$. Let us recall the definition of $(1-\alpha)$ marginal robustness and coverage:
\begin{align}
    &\sP\{Y \in \gU(X)\} \geq 1-\alpha, \tag{Coverage}\\
    &\sP\left\{\phi(Y, z(X)) \leq \max_{c\in \gU(X)} \phi(c,z(X))\right\} \geq 1-\alpha. \tag{Robustness}
\end{align}
If $\gU(X)$ covers the true label $Y$, the robustness event will be satisfied automatically by the definition. Hence, the marginal robustness is implied by the marginal coverage. However, the prediction set $\gU(X)$ is the subset of the \emph{robust region} $\gY_{\rm{robust}}(X) = \{y\in \gY: \phi(y, z(X)) \leq \max_{c\in \gU(X)} \phi(c,z(X))\}$. Since $\gU(X)\subseteq \gY_{\rm{robust}}(X)$, the discrepancy between the marginal robustness and coverage can be computed by
\begin{align}
    \Delta(\gU) := &\sP\left\{\phi(Y, z(X)) \leq \max_{c\in \gU(X)} \phi(c,z(X))\right\} - \sP\{Y \in \gU(X)\}\nonumber\\
    = &\sP\Big\{Y \in \gY_{\rm{robust}}(X)\setminus \gU(X)\Big\}.
\end{align}
The scale of the gap $\Delta(\gU) > 0$ largely depends on the data distribution, the structure of the prediction set, and the loss function.

In the portfolio optimization task, if we choose the prediction set $\gU(X) = \{y\in \sR^p: (y-\hat{\mu}(X))^{\top}\hat{\Sigma}^{-1}(X)(y-\hat{\mu}(X)) \leq \hat{q}\}$. Let $z(X) = \argmin_{z\in \gZ}\max_{c\in \gU(X)} -c^{\top}z$ be the CRO solution. By taking the dual of the inner maximization in the CRO problem, we have
\begin{align*}
    \max_{c\in \gU(X)} -c^{\top}z(X) = \sqrt{\hat{q}} \|\hat{\Sigma}^{1/2}(X) z(X)\|_2 - \hat{\mu}(X)^{\top}z(X).
\end{align*}
Hence, the robust region is given by
\begin{align*}
    \gY_{\rm robust}(X) &= \LRl{y\in \sR^p: -y^{\top}z(X) \leq \sqrt{\hat{q}}\cdot \|\hat{\Sigma}^{1/2}(X) z(X)\|_2 - \hat{\mu}(X)^{\top}z(X)}\\
    &= \LRl{y\in \sR^p: \frac{-(y - \hat{\mu}(X))^{\top}z(X)}{\|\hat{\Sigma}^{1/2}(X) z(X)\|_2} \leq \sqrt{\hat{q}}},
\end{align*}
which is a half-space in $\sR^p$. The next proposition characterizes the gap $\Delta(\gU)$ for the elliptical prediction set under the Gaussian distribution in the portfolio optimization task.

\begin{proposition}\label{pro:cover_robust_gap_Gaussian}
    Let $Y|X \sim N(\mu(X), \Sigma(X))$ with $Y \in \sR^p$ and $X \in \sR^d$ and assume $\Sigma(X)$ is of full rank. We consider the prediction set $\gU(X) = \{y\in \sR^p: (Y - \mu(X))^{\top} \Sigma^{-1}(X)(Y - \mu(X)) \leq \chi_{p, 1-\alpha}^2\}$, where $\chi_{p, 1-\alpha}^2$ is the $1-\alpha$ quantile of Chi-square distribution with degree of freedom $p$. Let $\phi(y,z) = -y^{\top}z$, then the marginal robustness satisfies that
    \begin{align}\label{eq:gap_Gaussian}
        \Delta(\gU) = \Phi(\sqrt{\chi_{p, 1-\alpha}^2}) - (1-\alpha),
    \end{align}
    where $\Phi(\cdot)$ is the c.d.f. of standard normal distribution.
\end{proposition}
\begin{proof}    
    By the definition of a robust region, we have
    \begin{align*}
        \sP\LRl{\phi(Y, z(X)) \leq \max_{c\in \gU(X)} \phi(c,z(X))}
        =& \sP\LRl{\frac{-z(X)^{\top}(Y - \mu(X))}{\|\Sigma^{1/2}(X) z(X)\|_2} \leq \sqrt{\chi_{p, 1-\alpha}^2}}.
    \end{align*}
    Then the conclusion follows from the fact $\frac{z(X)^{\top}(Y - \mu(X))}{\|\Sigma^{1/2}(X) z(X)\|_2}\mid X \sim N(0,1)$.
\end{proof}

In the case of Proposition \ref{pro:cover_robust_gap_Gaussian}, according to \eqref{eq:gap_Gaussian}, the gap $\Delta(\gU)$ is increasing as the dimension of label $p$ grows when $1-\alpha > 0.5$. Notice that, if we change the confidence level of prediction set from $1-\alpha$ to $1-\alpha-\Delta(\gU)$, then the final decision will satisfies the exact $1-\alpha$ level of robustness.
However, if the model is misspecified and the data distribution is unknown, finding the modification above is difficult.

\subsection{A direct approach for robustness control}
For the prediction set $\gU(X;q) = \{c\in \gY: s(X,c) \leq q\}$ with $q\in \sR$, we define its CRO decision as $z(X;q) = \argmin_{z\in \gZ}\max_{c\in \gU(X;q)} \phi(c,z)$. Given the data $(x,y)$, denote the robustness indicator as
\begin{align*}
    R(x,y;q) = \mathbbm{1}\LRl{\phi(y,z(x;q)) \leq \max_{c\in \gU(x;q)} \phi(c, z(x;q))}.
\end{align*}
We consider the following constrained optimization problem
\begin{align}\label{eq:CRC_threshold}
    \hat{q} = \min\LRl{q\in \sR: \frac{1}{n}\sum_{i=1}^n R(X_i,Y_i;q) \geq 1-\alpha.}
\end{align}
Then we make the decision for test point by $\hat{z}_{\rm{robust}}(X_{n+1}) = \argmin_{z\in \gZ}\max_{c\in \gU(X_{n+1};\hat{q})}$.
Define the function class $\gR = \{\mathbbm{1}\{\phi(y,z(x;q)) \leq \max_{c\in \gU(x;q)} \phi(c, z(x;q))\}: q\in \sR\}$.
Let $\{\xi_i\}_{i=1}^n$ are i.i.d. random variables taking $+1$ or $-1$ with equal probability. Denote the Rademacher complexity of $\gR$ as $\mathfrak{R}_n(\gR) = \E\LRm{\sup_{q \in \sR} n^{-1}\left|\sum_{i=1}^n\xi_i R(X_i,Y_i;q)]\right|}$.
\begin{theorem}
    Suppose $\{(X_i,Y_i)\}_{i=1}^{n+1}$ are i.i.d., under regular conditions, we have $\sP\LRl{R(X_{n+1},Y_{n+1};\hat{q}) = 1} \geq 1-\alpha-2\mathfrak{R}_n(\gR)$.
\end{theorem}

\begin{proof}
    Since $\hat{q}$ depends only on the labeled data, using the symmetrization technique,
    \begin{align*}
        \E\LRm{R(X_{n+1},Y_{n+1};\hat{q})} - (1-\alpha) &\geq \E\LRm{R(X_{n+1},Y_{n+1};\hat{q}) - \frac{1}{n}\sum_{i=1}^n R(X_i,Y_i;\hat{q})}\nonumber\\
        &\geq -  \E\LRm{\sup_{q\in \sR}\E\LRm{R(X_{n+1},Y_{n+1};q) - \frac{1}{n}\sum_{i=1}^n R(X_i,Y_i;q) \mid \gD_n}}\nonumber\\
        &= -  \E\LRm{\sup_{q\in \sR}\LRabs{\frac{1}{n}\sum_{i=1}^n R(X_i,Y_i;q) - \E\LRm{R(X_{i},Y_{i};q)}}}\nonumber\\
        &\leq -2 \mathfrak{R}_n(\gR),
    \end{align*}
    where the first inequality holds due to \eqref{eq:CRC_threshold}.
\end{proof}



\section{Implementation of F-CROMS and J-CROMS}


\subsection{Efficient implementation of F-CROMS}
The F-CROMS method builds upon the full conformal prediction framework by integrating the test point into the selection procedure. This preserves exchangeability and provides distribution-free robustness guarantees.
However, methods based on full conformal prediction generally computationally expensive. In what follows, we present an efficient implementation of F-CROMS that can substantially reduce unnecessary computations.

For each model $\lambda \in \Lambda$, we define the lower and upper quantiles as:
\begin{align}
\hat{q}_{\lambda}^{-} &= Q_{(1-\alpha)(1+1/n)-1/n}(\{S_{\lambda}(X_i,Y_i)\}_{i=1}^{n}),\quad \hat{q}_{\lambda} = Q_{(1-\alpha)(1+1/n)}(\{S_{\lambda}(X_i,Y_i)\}_{i=1}^{n}).\notag
\end{align}
By the property of sample quantile, it holds that
\begin{equation}\label{eq:qy_expression}
    \hat{q}_{\lambda}^y = \begin{cases}
        \hat{q}_{\lambda}^{-} & \text{if } S_{\lambda}(X_{n+1},y) \leq \hat{q}_{\lambda}^{-}\\
        \hat{q}_{\lambda} & \text{if } S_{\lambda}(X_{n+1},y) \geq \hat{q}_{\lambda}\\
        S_{\lambda}(X_{n+1}, y) & \text{if } \hat{q}_{\lambda}^{-}< S_{\lambda}(X_{n+1},y) < \hat{q}_{\lambda}.
    \end{cases}.
\end{equation}
The associated lower and upper losses are defined by $\mathcal{L}_{n}^{-}(\lambda) = \sum_{i=1}^{n}\phi(Y_i,z_{\lambda}(X_i; \hat{q}_{\lambda}^-))$ and $\mathcal{L}_{n}(\lambda) = \sum_{i=1}^{n}\phi(Y_i,z_{\lambda}(X_i;\hat{q}_{\lambda}))$.
By the definition of $\hat{\lambda}^y$, we know
\begin{equation*}
\hat{\lambda}^y = \argmin_{\lambda\in \Lambda} \LRl{\mathcal{L}_{n+1}(\lambda; y) = \sum_{i=1}^{n} \phi\left(Y_i, z_{\lambda}(X_i; \hat{q}_{\lambda}^y)\right) + \phi\left(y, z_{\lambda}(X_{n+1}; \hat{q}_{\lambda}^y)\right)}.
\end{equation*}
The objective function admits the following case-wise expression:
\begin{align*}
    \mathcal{L}_{n+1}(\lambda;y) = 
\begin{cases}
\gL_n^{-}(\lambda) + \phi\left(y,z_{\lambda}(X_{n+1};\hat{q}_{\lambda}^-)\right) & \text{ if } S_{\lambda}(X_{n+1},y) \leq \hat{q}_{\lambda}^{-} \\
\gL_n(\lambda) + \phi\left(y,z_{\lambda}(X_{n+1};\hat{q}_{\lambda})\right) & \text{ if } S_{\lambda}(X_{n+1},y)\geq \hat{q}_{\lambda} \\
\begin{matrix}
    &\sum_{i=1}^{n} \phi(Y_i, z_{\lambda}(X_i; S_{\lambda}(X_{n+1},y)))\\
    &+ \phi(y, z_{\lambda}(X_{n+1}; S_{\lambda}(X_{n+1},y)))
\end{matrix}
& \text{if } \hat{q}_{\lambda}^{-}< S_{\lambda}(X_{n+1},y) < \hat{q}_{\lambda}\\
\end{cases}.
\end{align*}
In practice, when the sample size $n$ is sufficiently large, the subset $\{y \in \mathcal{Y}: \hat{q}_{\lambda}^{-} < S_{\lambda}(X_{n+1},y) < \hat{q}_{\lambda}\}$
is typically small (see Lemma \ref{lemma:qy_diff_VC}), and the subset $\{y \in \mathcal{Y}: S_{\lambda}(X_{n+1},y) \leq \hat{q}_{\lambda}^{-} \text{ or } S_{\lambda}(X_{n+1}, y) \geq \hat{q}_{\lambda} \}$ constitutes a large portion of $\mathcal{Y}$. Therefore, for most labels $y\in \gY$,  the loss $\gL_{n+1}(\lambda;y)$  can be evaluated rapidly for all $\lambda \in \Lambda$. After that, the model index $\hat{\lambda}^{y}$ can be obtained and the inclusion $y\in \widehat{\gU}^{\FCROMS}(X_{n+1})$ can be determined. By storing and effectively reusing the precomputed values $\hat{q}_{\lambda}^{-},\ \hat{q}_{\lambda}$ and $\mathcal{L}_{n}^{-}(\lambda),\ \mathcal{L}_{n}(\lambda)$, which are independent of the hypothesized $y$, the overall computational cost is significantly reduced. The complete procedure is summarized in Algorithm \ref{alg:AOA}.

\begin{algorithm}[H]
    \caption{Efficient algorithm of F-CROMS (for both classification and regression)}
    \label{alg:AOA}
    \renewcommand{\algorithmicrequire}{\textbf{Input:}}
    \renewcommand{\algorithmicensure}{\textbf{Output:}}
    \linespread{1.2}\selectfont
    \begin{algorithmic}[1]
        \Require Pre-trained models $\{S_{\lambda}:\lambda\in \Lambda\}$, loss function $\phi$, test data $X_{n+1}$, labeled dataset $\mathcal{D}_n = \{(X_{i},Y_i)\}_{i=1}^{n}$, robustness level $1-\alpha \in (0, 1)$.
        

        \State Initialize $\widehat{\mathcal{U}}^{\FCROMS}(X_{n+1}) \gets \emptyset$.
        \For{$y \in \mathcal{Y}$}
            \For{$\lambda \in \Lambda$}
                \If{$S_{\lambda}(X_{n+1},y) \leq \hat{q}_{\lambda}^{-}$}
                    \State $ \hat{q}_{\lambda}^{y} \gets \hat{q}_{\lambda}^{-}$, $\mathcal{L}_{n+1}(\lambda;y)\gets  \gL_n^{-}(\lambda) + \phi\left(y,z_{\lambda}(X_{n+1};\hat{q}_{\lambda}^-)\right)$.
                \ElsIf{$\hat{q}_{\lambda} \leq S_{\lambda}(X_{n+1},y)$}
                    \State $\hat{q}_{\lambda}^{y} \leftarrow \hat{q}_{\lambda}$, $\mathcal{L}_{n+1}(\lambda;y) \gets  \gL_n(\lambda) + \phi\left(y,z_{\lambda}(X_{n+1};\hat{q}_{\lambda})\right)$.
                \ElsIf{$\hat{q}_{\lambda}^{-} < S_{\lambda}(X_{n+1},y) < \hat{q}_{\lambda}$}


                    \State $\mathcal{L}_{n+1}(\lambda;y) \gets \sum_{i=1}^{n} \phi(Y_i, z_{\lambda}(X_i; S_{\lambda}(X_{n+1},y))) + \phi(y, z_{\lambda}(X_{n+1}; S_{\lambda}(X_{n+1},y)))$.
                \EndIf
            \EndFor
            \State $\hat{\lambda}^y \leftarrow \argmin_{\lambda \in \Lambda} \mathcal{L}_{n+1}(\lambda,y)$.
            \If{$S_{\hat{\lambda}^{y}}(X_{n+1}, y) \leq \hat{q}_{\hat{\lambda}^{y}}^{y}$}
                \State $\widehat{\gU}^{\FCROMS}(X_{n+1}) \gets \widehat{\gU}^{\FCROMS}(X_{n+1}) \cup \{y\}$.
            \EndIf
        \EndFor
        
        \State Solve CRO problem $\hat{z}^{\FCROMS}(X_{n+1}) = \argmin_{z\in \gZ} \max_{c \in \widehat{\gU}^{\FCROMS}(X_{n+1})} \phi(c,z)$.

        \Ensure Decision $\hat{z}^{\FCROMS}(X_{n+1})$.
    \end{algorithmic}
\end{algorithm}

\subsection{Grid-approximated F-CROMS for regression tasks}\label{appen:Discretization}

The detailed implementation of GF-CROMS for regression is stated in Algorithm \ref{alg:Discretized full}. Given the spacing $\epsilon_{\rm{grid}}$, the grid points in the $j$-th dimension is $\widetilde{\gY}_j = \{-R_{\gY} + k\epsilon_{\rm{grid}}\}_{k=1}^{\lceil2R_{\gY}/\epsilon_{\rm{grid}}\rceil}$, where $R_{\gY}$ is the radius of the label space $\gY$. Then we can construct the discretized label space by $\widetilde{\gY} = \widetilde{\gY}_1\times\cdots\times \widetilde{\gY}_p$. The discretization mapping is defined by $\sD(y) = \argmin_{\tilde{y}\in \widetilde{\gY}} \|y - \tilde{y}\|$.

\begin{algorithm}[H]
	\renewcommand{\algorithmicrequire}{\textbf{Input:}}
	\renewcommand{\algorithmicensure}{\textbf{Output:}}
    \linespread{1.2}\selectfont
	\caption{GF-CROMS for regression tasks}
	\label{alg:Discretized full}
	\begin{algorithmic}[1]
	\Require Pre-trained models $\{S_{\lambda}:\lambda\in \Lambda\}$, loss function $\phi$, test data $X_{n+1}$, labeled dataset $\mathcal{D}_n = \{(X_{i},Y_i)\}_{i=1}^{n}$, grid $\widetilde{\gY}$, mapping $\sD: \gY \rightarrow \widetilde{\gY}$, robustness level $1-\alpha \in (0, 1)$.
        \State  Obtain a discretized labeled dataset $\widetilde{\mathcal{D}}_n =\{(X_i,\sD(Y_i))\}_{i=1}^{n}$.

        \State Call Algorithm 2 to construct $\widetilde{\mathcal{U}}^{\FCROMS}(X_{n+1}) \subset \widetilde{\mathcal{Y}}$ for $X_{n+1}$ based on $\widetilde{\mathcal{D}}_n$.

        \State The final prediction set $\LRl{y = \sD^{-1}(\tilde{y}): \tilde{y} \in \widetilde{\mathcal{U}}^{\FCROMS}(X_{n+1})}$. 
        
        \State Solve CRO problem $\hat{z}^{\GFCROMS}(X_{n+1}) = \argmin_{z\in \gZ} \max_{c \in \gU(X_{n+1})} \phi(c,z)$.
        
        \Ensure Decision $\hat{z}^{\GFCROMS}(X_{n+1})$.
	\end{algorithmic}
\end{algorithm}



\begin{figure}[H]
    \centering
    \includegraphics[width=\textwidth]{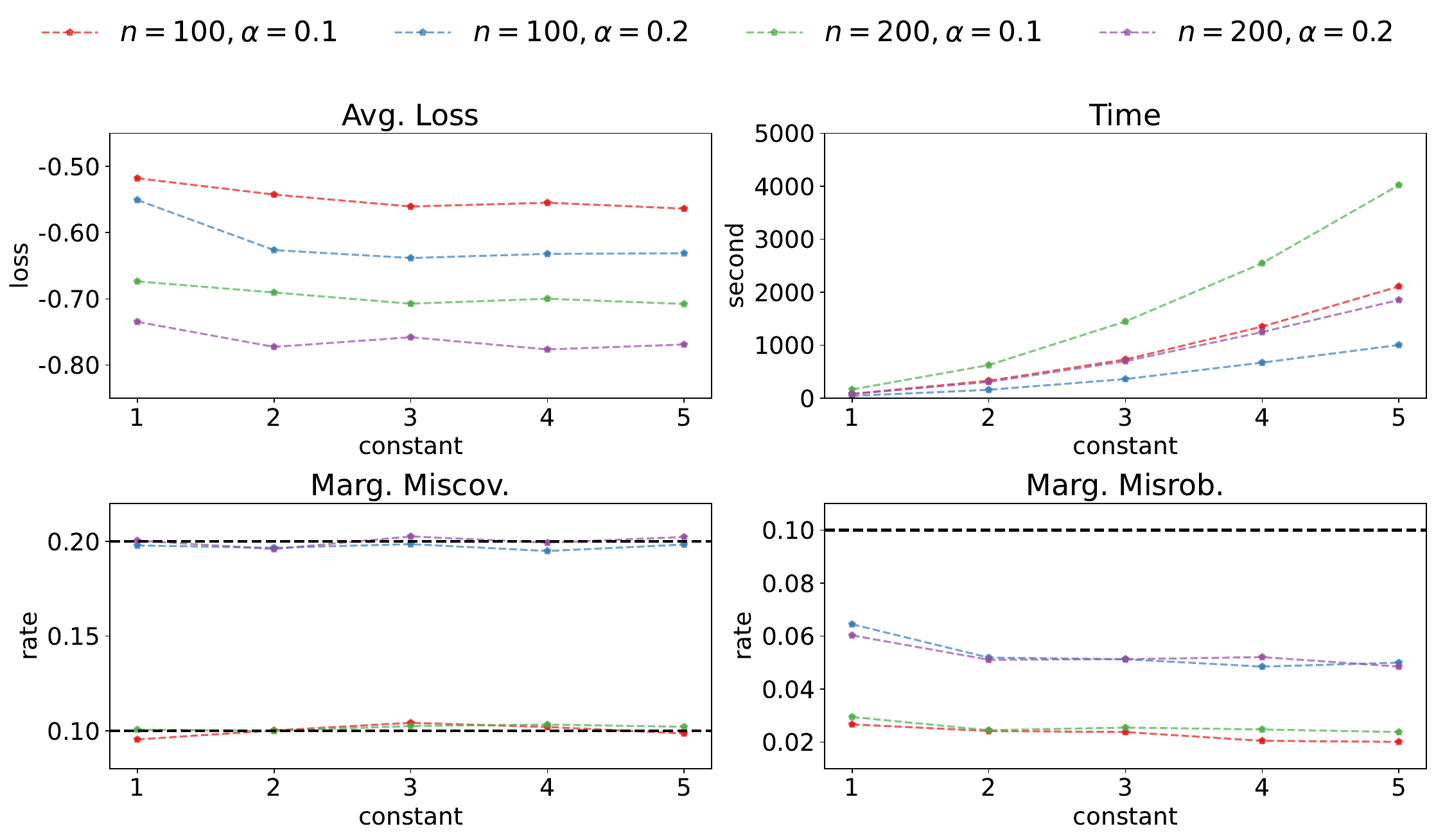}
    \caption{The performance of F-CROMS versus the grid size constant $c$ in each dimension. The simulation setting is consistent with that in Section 5.1.2, with $|\Lambda| = 30$. The time refers to the average single-run time over 100 independent repetitions.}
    \label{fig: Sensitivity}
\end{figure}

Next, we conduct a sensitivity analysis on the number of grid points for the two-dimensional label space, where the simulation setting is the same as Section 5.1.2. As suggested by the decision risk bound of GF-CROMS in Theorem 3.6, we discretize each dimension of the label $y\in \sR^2$ using $1/\epsilon_{\rm{grid}} = cn^{1/2}$ grid points. As demonstrated in Figure \ref{fig: Sensitivity}, the performance of F-CROMS is quite stable when the constant $c$ varies in all scenarios.


\subsection{A superset implementation method for F-CROMS}

In this section, we aim to construct a superset $\widehat{\gU}^{\text{SF-CROMS}}(X_{n+1})$ for the F-CROMS prediction set $\widehat{\gU}^{\FCROMS}(X_{n+1})$. Given the prediction set $\gU_{\lambda}(x; q) = \{y\in \sR^p: S_{\lambda}(x,y) \leq q\}$, we denote the CRO decision as $z_{\lambda}(x;q) = \argmin_{z\in \gZ} \max_{c\in \gU_{\lambda}(x; q)} \phi(c,z)$. The construction relies on the following assumption.

\begin{assumption}\label{assum:piecewise_monotone}
    Given a fixed $x\in \gX$ and $y\in \gY$, the loss value $\phi(y, z_{\lambda}(x;q))$ is a piecewise monotone function in $q\in \sR$ with a finite breakpoint set $\gQ_{\lambda}^b(x) = \{q_1^b(x) < q_2^b(x) < \ldots < q_K^b(x)\}$. It means that $\phi(y, z_{\lambda}(x;q))$ is monotone function for $q \in [q_k^b(x), q_{k+1}^b(x)]$.
\end{assumption}
Denote $\gY_{\lambda}^- = \{y\in \gY: S_{\lambda}(X_{n+1}, y) < \hat{q}_{\lambda}^-\}$, $\gY_{\lambda}^b = \LRl{y\in\gY: \hat{q}_{\lambda}^-\leq S_{\lambda}(X_{n+1}, y)\leq \hat{q}_{\lambda}}$ and $\gY_{\lambda}^+ = \{y\in \gY: S_{\lambda}(X_{n+1}, y) > \hat{q}_{\lambda}\}$.
For each $X_i$, by \eqref{eq:qy_expression}, we know
\begin{align}
    \phi(Y_i, z_{\lambda}(X_i;\hat{q}_{\lambda}^y)) = \begin{cases}
        \phi(Y_i, z_{\lambda}(X_i;\hat{q}_{\lambda}^{-})) & \text{if }y\in \gY_{\lambda}^-\\
        \phi(Y_i, z_{\lambda}(X_i;S_{\lambda}(X_{n+1}, y))) & \text{if }y\in \gY_{\lambda}^b\\
        \phi(Y_i, z_{\lambda}(X_i;\hat{q}_{\lambda})) & \text{if }y\in \gY_{\lambda}^+.
    \end{cases} 
\end{align}
Using Assumption \ref{assum:piecewise_monotone}, we can guarantee that
\begin{align}
    \sup_{y\in \gY}\phi(Y_i, z_{\lambda}(X_i;\hat{q}_{\lambda}^y)) &\leq \phi(Y_i, z_{\lambda}(X_i;\hat{q}_{\lambda}^{-})) \vee \phi(Y_i, z_{\lambda}(X_i;\hat{q}_{\lambda}))\nonumber\\
    &\qquad \vee \sup_{q \in \gQ_{\lambda}^b(X_i) \cap [\hat{q}_{\lambda}^-,\ \hat{q}_{\lambda}]}\phi(Y_i, z_{\lambda}(X_i;q))\nonumber\\
    &= \sup_{q\in \gQ_{\lambda}^b(X_i) \cap [\hat{q}_{\lambda}^-,\ \hat{q}_{\lambda}] \cup \{\hat{q}_{\lambda}^-,\hat{q}_{\lambda}\}} \phi(Y_i, z_{\lambda}(X_i;q))\nonumber\\
    &=: \varphi_i^+(\lambda),\nonumber\\
    \inf_{y\in \gY}\phi(Y_i, z_{\lambda}(X_i;\hat{q}_{\lambda}^y)) &\leq \phi(Y_i, z_{\lambda}(X_i;\hat{q}_{\lambda}^{-})) \wedge \phi(Y_i, z_{\lambda}(X_i;\hat{q}_{\lambda}))\nonumber\\
    &\qquad \wedge \inf_{q \in \gQ_{\lambda}^b(X_i) \cap [\hat{q}_{\lambda}^-,\ \hat{q}_{\lambda}]}\phi(Y_i, z_{\lambda}(X_i;q))\nonumber\\
    &=: \varphi_i^-(\lambda).\nonumber
\end{align}
We write $\gL_n^+(\lambda) = \sum_{i=1}^n \varphi_i^+(\lambda)$ and $\gL_n^-(\lambda) = \sum_{i=1}^n \varphi_i^-(\lambda)$ for each $\lambda \in \Lambda$, which can be computed based on labeled data $\{(X_i,Y_i)\}_{i=1}^{n}$. For the test point, we have
\begin{align*}
    \sup_{y\in\gY}\phi\LRs{y, z_{\lambda}(X_{n+1};\hat{q}_{\lambda}^y)} =& \sup_{y \in \gY_{\lambda}^-} \phi\LRs{y, z_{\lambda}(X_{n+1};\hat{q}_{\lambda}^-)} \vee \sup_{y \in \gY_{\lambda}^+} \phi\LRs{y, z_{\lambda}(X_{n+1};\hat{q}_{\lambda})}\\
    &\vee \sup_{y \in \gY_{\lambda}^b} \phi\LRs{y, z_{\lambda}(X_{n+1}; S_{\lambda}(X_{n+1}, y))}\\
    \leq & \sup_{y \in \gY_{\lambda}^-} \phi\LRs{y, z_{\lambda}(X_{n+1};\hat{q}_{\lambda}^-)} \vee \sup_{y \in \gY_{\lambda}^+} \phi\LRs{y, z_{\lambda}(X_{n+1};\hat{q}_{\lambda})}\\
    &\vee \sup_{y \in \gY_{\lambda}^b} \sup_{q\in (\gQ_{\lambda}^b(X_{n+1}) \cap [\hat{q}_{\lambda}^{-},\hat{q}_{\lambda}])\cup \{\hat{q}_{\lambda}^{-},\hat{q}_{\lambda}\}}\phi\LRs{y, z_{\lambda}(X_{n+1}; q)}\\
    =:& \widetilde{\varphi}_{n+1}^+(\lambda),\\
    \inf_{y\in\gY}\phi\LRs{y, z_{\lambda}(X_{n+1};\hat{q}_{\lambda}^y)} 
    \geq & \inf_{y \in \gY_{\lambda}^-} \phi\LRs{y, z_{\lambda}(X_{n+1};\hat{q}_{\lambda}^-)} \wedge \inf_{y \in \gY_{\lambda}^+} \phi\LRs{y, z_{\lambda}(X_{n+1};\hat{q}_{\lambda})}\\
    &\wedge \inf_{y \in \gY_{\lambda}^b} \inf_{q\in (\gQ_{\lambda}^b(X_{n+1}) \cap [\hat{q}_{\lambda}^{-},\hat{q}_{\lambda}])\cup \{\hat{q}_{\lambda}^{-},\hat{q}_{\lambda}\}}\phi\LRs{y, z_{\lambda}(X_{n+1}; q)}\\
    =:& \widetilde{\varphi}_{n+1}^-(\lambda).
\end{align*}
Let $\gL_n(\lambda;y) = \sum_{i=1}^n \phi(Y_i, z_{\lambda}(X_i;\hat{q}_{\lambda}^y))$.
According to the definition of $\hat{\lambda}^y$, we know
\begin{align*}
    \gL_n(\hat{\lambda}^y;y) + \phi\LRs{y, z_{\hat{\lambda}^y}(X_{n+1};\hat{q}_{\hat{\lambda}^y}^y)}
    \leq &\min_{\lambda\in \Lambda}\LRl{\gL_n(\lambda;y) + \phi\LRs{y, z_{\lambda}(X_{n+1};\hat{q}_{\lambda}^y)}}\nonumber\\
    \leq & \min_{\lambda\in \Lambda}\LRl{\sup_{y\in \gY}\gL_n(\lambda;y) + \sup_{y\in \gY}\phi\LRs{y, z_{\lambda}(X_{n+1};\hat{q}_{\lambda}^y)}} \nonumber\\
    \leq & \min_{\lambda\in \Lambda}\LRl{\gL_n^+(\lambda) + \widetilde{\varphi}_{n+1}^+(\lambda)}.
\end{align*}
In addition we also have the lower bound
\begin{align*}
    \gL_n(\hat{\lambda}^y;y) + \phi\LRs{y, z_{\hat{\lambda}^y}(X_{n+1};\hat{q}_{\hat{\lambda}^y}^y)} &\geq \gL_n^-(\hat{\lambda}^y) + \inf_{y\in \gY} \phi\LRs{y, z_{\hat{\lambda}^y}(X_{n+1};\hat{q}_{\hat{\lambda}^y}^y)}\\
    &\geq \gL_n^-(\hat{\lambda}^y) + \inf_{y\in \gY}\inf_{\lambda \in \Lambda} \phi\LRs{y, z_{\lambda}(X_{n+1};\hat{q}_{\lambda}^y)}\\
    &= \gL_n^-(\hat{\lambda}^y) + \inf_{\lambda \in \Lambda}\inf_{y\in \gY} \phi\LRs{y, z_{\lambda}(X_{n+1};\hat{q}_{\lambda}^y)}\\
    &\geq \gL_n^-(\hat{\lambda}^y) + \inf_{\lambda \in \Lambda} \widetilde{\varphi}_{n+1}^-(\lambda).
\end{align*}
Therefore, we can obtain the superset of model $\hat{\lambda}^{y}$ as
$$
\gM = \LRl{\lambda^{\prime}\in \Lambda: \gL_{n}^{-}(\lambda^{\prime}) \leq \inf_{\lambda\in \Lambda} \LRl{\gL_{n}^{+}(\lambda) + \widetilde{\varphi}_{n+1}^+(\lambda)} - \inf_{\lambda \in \Lambda}\widetilde{\varphi}_{n+1}^-(\lambda)}.
$$
Consequently, the superset of the F-CROMS prediction set is
$$
\widehat{\gU}^{\text{SF-CROMS}}(X_{n+1}) = \bigcup_{\lambda \in \gM} \left\{y\in \gY: S_{\lambda}(X_{n+1},y) \leq \hat{q}_{\lambda}\right\}.
$$
In fact, for any $y\in \widehat{\gU}^{\FCROMS}(X_{n+1})$, we know
\begin{align*}
    S_{\hat{\lambda}^y}(X_{n+1}, y) \leq \hat{q}_{\hat{\lambda}^y}^y & \Longleftrightarrow S_{\hat{\lambda}^y}(X_{n+1}, y) \leq \hat{q}_{\hat{\lambda}^y}\Longrightarrow \exists \lambda \in \gM,\ S_{\lambda}(X_{n+1}, y) \leq \hat{q}_{\lambda}.
\end{align*}
Hence, we can guarantee that $\widehat{\gU}^{\FCROMS}(X_{n+1})\subseteq \widehat{\gU}^{\text{SF-CROMS}}(X_{n+1})$.

In the following, we consider the portfolio optimization problem $\phi(y,z) = - y^{\top}z$ with $\gY = \sR^p$ and $\gZ = \{z\in [0,1]^p: \vone^{\top}z = 1\}$, and verify Assumption \ref{assum:piecewise_monotone} under the box scores and ellipsoid scores.

\subsubsection{Box candidate scores}\label{appen:breakpoint_box}
Under the box score $S_{\lambda}(x,y) = \|(y - \hat{\mu}_{\lambda}(x))/\hat{\sigma}(x)\|_{\infty}$, the CRO problem is equivalent to
\begin{align*}
    z(x;q)&=\argmin_{z\in \gZ} \LRl{- \LRs{\hat{\mu}(x) - q \hat{\sigma}(x)}^{\top}z} = e_{j(x;q)},
\end{align*}
where $j(x;q) = \argmax_{j\in [p]} \LRl{\hat{\mu}_j(x) - q \hat{\sigma}_j(x)}$ and $\{e_{j}\}_{j=1}^p$ are the standard basis vectors. Hence $z_{\lambda}(x;q)$ is a \textit{step function} with respect to the threshold $q$ for a fixed $x \in \gX$, and the breakpoints are
$$
\gQ_\lambda^b(x) = \left\{\frac{\hat{\mu}_{\lambda,k}(x) - \hat{\mu}_{\lambda,j}(x)}{\hat{\sigma}_{\lambda,j}(x) - \hat{\sigma}_{\lambda,k}(x)}: j,k \in [p], j\neq k\right\}.
$$
The decision at the breakpoint is given by $q_{\lambda}^b \in  \gQ_{\lambda}^b$ as $z_{\lambda}(x;q_{\lambda}^b) = \lim_{q\rightarrow (q_{\lambda}^b)^{+}}z_{\lambda}(x;q)$, which is the limit of $z_{\lambda}(x;q)$ as $q$ approaches $q_\lambda^b$ from the right. For each $X_i$, by \eqref{eq:qy_expression}, we know
\begin{align}\label{eq:z_qy_set}
    z_{\lambda}(X_i;\hat{q}^y) \in \LRl{z_{\lambda}(X_i;q): q \in (\gQ_{\lambda}^b(X_i) \cap [\hat{q}_{\lambda}^-, \hat{q}_{\lambda}]) \cup \{\hat{q}_{\lambda}^-, \hat{q}_{\lambda}\}}.
\end{align}
Let $\gQ_{\lambda}^{b} = \bigcup_{i=1}^n (\gQ_{\lambda}^b(X_i) \cap [\hat{q}_{\lambda}^-, \hat{q}_{\lambda}]) \cup \{\hat{q}_{\lambda}^-,\hat{q}_{\lambda}\}$, then we define better upper and lower loss by
\begin{align*}
    \gL_n^+(\lambda) &= \sup_{q\in \gQ_{\lambda}^{b}} \sum_{i=1}^n \phi\LRs{Y_i, z_{\lambda}(X_i; q)},\quad \gL_n^-(\lambda) = \inf_{q\in \gQ_{\lambda}^{b}} \sum_{i=1}^n \phi\LRs{Y_i, z_{\lambda}(X_i; q)}.
\end{align*}

\subsubsection{Ellipsoid candidate scores}
Under the ellipsoid score $S_{\lambda}(x,y) = (y - \hat{\mu}_{\lambda}(x))^{\top}\Sigma_{\lambda}(x)^{-1}(y - \hat{\mu}_{\lambda}(x))$, the CRO problem is equivalent to
\begin{align*}
    z(x;q)=\argmin_{z\in \sR^p}\LRl{\sqrt{q}\sqrt{z^{\top}\hat{\Sigma}(x)z} - \hat{\mu}(x)^{\top}z\quad \text{s.t.}\quad z\geq 0,\ \vone^{\top}z = 1}.
\end{align*}
The Lagrangian form is given by
\begin{align*}
    z(x;q)=\argmin_{z\in \sR^p}\LRl{\sqrt{q}\sqrt{z^{\top}\hat{\Sigma}z} - \hat{\mu}(x)^{\top}z - \eta^{\top}z + \gamma (\vone^{\top}z - 1)}.
\end{align*}
Let us suppress the dependence on $x$ for now. According to the KKT conditions, we have
\begin{align*}
    \sqrt{q}\frac{(\hat{\Sigma} z)_j}{\sqrt{z^{\top}\hat{\Sigma}z}} - \hat{\mu}_j - \eta_j + \gamma = 0,\ j=1,\ldots,p\\
    \eta_j \geq 0,\ z_j \geq 0,\ \eta_jz_j = 0,\ j=1,\ldots,p,\ \sum_{j=1}^p z_j = 1.
\end{align*}

\paragraph*{Monotone loss value for fixed active set.}
Denote the active set $A = \{j\in [p]: w_j > 0\}$, then we have
\begin{align*}
    \sqrt{q}\frac{\hat{\Sigma}_{AA} z_{A}}{w} - \hat{\mu}_A + \gamma \vone_A = 0,\\
    \sqrt{q}\frac{\hat{\Sigma}_{AA} z_{A^c}}{w} - \hat{\mu}_{A^c} + \gamma \vone_{A^c} \geq 0,\\
    \vone_{A}^{\top}z_A = 1,\ z^{\top}\hat{\Sigma}z = z_A^{\top}\hat{\Sigma}_{AA}z_A.
\end{align*}
Denote $t = \frac{\sqrt{z_A^{\top}\hat{\Sigma}_{AA}z_A}}{\sqrt{q}}$, then the active part of solution is
\begin{align}
    z_A = t\cdot \hat{\Sigma}_{AA}^{-1}\LRs{\hat{\mu}_A - \gamma \vone_A}.\label{eq:decision_gamma}
\end{align}
Since $\vone_{A}^{\top}z_A = 1$, we get $1 = t\cdot \vone_A^{\top} \hat{\Sigma}_{AA}^{-1} (\hat{\mu}_A - \gamma\cdot \vone_A)$, which implies that $\gamma = \frac{\vone_A^{\top}\hat{\Sigma}_{AA}^{-1} \hat{\mu}_A}{\vone_A^{\top}\hat{\Sigma}_{AA}^{-1} \vone_A} - \frac{1}{t\cdot \vone_A^{\top}\hat{\Sigma}_{AA}^{-1} \vone_A}$. In addition, plugging \eqref{eq:decision_gamma} into $z_A^{\top}\hat{\Sigma}_{AA}z_A$, we can get
\begin{align}\label{eq:decision_quadratic}
    z_A^{\top}\hat{\Sigma}_{AA}z_A &= t^2\cdot \LRs{\hat{\Sigma}_{AA}^{-1}\LRs{\hat{\mu}_A - \gamma \vone_A}}^{\top}\hat{\Sigma}_{AA}\LRs{\hat{\Sigma}_{AA}^{-1}\LRs{\hat{\mu}_A - \gamma \vone_A}}\nonumber\\
    &= t^2\cdot \LRs{\hat{\mu}_A - \gamma \vone_A}^{\top}\hat{\Sigma}_{AA}^{-1}\LRs{\hat{\mu}_A - \gamma \vone_A}.
\end{align}
We introduce the following notations,
\begin{align}\label{eq:active_parameters}
    \theta_A = \hat{\mu}_A^{\top} \hat{\Sigma}_{AA}^{-1} \hat{\mu}_A,\quad \beta_A = \hat{\mu}_A^{\top} \hat{\Sigma}_{AA}^{-1} \vone_A,\quad \zeta_A = \vone_A^{\top} \hat{\Sigma}_{AA}^{-1} \vone_A.
\end{align}
Plugging $\gamma = \frac{\beta_A}{\zeta_A} - \frac{1}{t \zeta_A}$ into \eqref{eq:decision_quadratic}, we have
\begin{align*}
    z_A^{\top}\hat{\Sigma}_{AA}z_A &= t^2 \LRs{\hat{\mu}_A - \LRs{\frac{\beta_A}{\zeta_A} - \frac{1}{t \zeta_A}}\cdot \vone_A}^{\top}\hat{\Sigma}_{AA}^{-1}\LRs{\hat{\mu}_A - \LRs{\frac{\beta_A}{\zeta_A} - \frac{1}{t \zeta_A}}\cdot \vone_A}\\
    &= \LRs{\theta_A - \frac{\beta_A^2}{\zeta_A}}t^2 + \frac{1}{\zeta_A}.
\end{align*}
Together with the definition of $t$, we have the equation $qt^2 = \LRs{\theta_A - \frac{\beta_A^2}{\zeta_A}}t^2 + \frac{1}{\zeta_A}$, leading to the root
\begin{align}
    t = \LRs{\zeta_A \LRs{q - \theta_A + \frac{\beta_A^2}{\zeta_A}}}^{-1/2},\quad \text{if }q > \theta_A - \frac{\beta_A^2}{\zeta_A}.
\end{align}
Notice that the active set $A$ appears only if $q > \theta_A - \frac{\beta_A^2}{\zeta_A}$.
Plugging it into \eqref{eq:decision_gamma}, we have
\begin{align}
    z_A(q) &= t\cdot \hat{\Sigma}_{AA}^{-1}\LRs{\hat{\mu}_A - \frac{\beta_A}{\zeta_A}\cdot \vone_A - \frac{1}{t \zeta_A}\cdot \vone_A}\nonumber\\
    &= -\frac{1}{\zeta_A} \hat{\Sigma}_{AA}^{-1} \vone_A + \LRs{\zeta_A \LRs{q - \theta_A + \frac{\beta_A^2}{\zeta_A}}}^{-1/2} \hat{\Sigma}_{AA}^{-1} \LRs{\hat{\mu}_A - \frac{\beta_A}{\zeta_A}\cdot \vone_A}.
\end{align}
Hence, given any fixed $y\in \sR^p$, the loss $-y^{\top}z_A(q)$ is a monotone function when the active set $A$ is fixed. Next, we derive the breakpoints where the active pattern changes.

\paragraph*{The breakpoints where the active set changes.}
The current active set $A$ changes when one of the following scenarios happens:
\begin{itemize}
    \item[(1)] $z_i(q) = 0$ for some $i\in A$, which means that
    \begin{align}\label{eq:active_change_out}
         -\frac{1}{\zeta_A} (\hat{\Sigma}_{AA}^{-1})_{\cdot i}\vone_A  + \LRs{\zeta_A \LRs{q - \theta_A + \frac{\beta_A^2}{\zeta_A}}}^{-1/2} (\hat{\Sigma}_{AA}^{-1})_{\cdot i}^{\top} \LRs{\hat{\mu}_A - \frac{\beta_A}{\zeta_A}\cdot \vone_A} = 0,\nonumber\\
         \Longrightarrow q = \frac{\LRs{(\hat{\Sigma}_{AA}^{-1})_{\cdot i}^{\top}(\hat{\mu}_A \zeta_A - \beta_A \vone_A)}^2}{\LRs{(\hat{\Sigma}_{AA}^{-1})_{\cdot i}^{\top} \vone_A}^2 \zeta_A} + \theta_A - \frac{\beta_A^2}{\zeta_A}.
    \end{align}

    \item[(2)] $\sqrt{q}\frac{(\hat{\Sigma} z)_j}{\sqrt{z^{\top}\hat{\Sigma}z}} - \hat{\mu}_j + \gamma = 0$ for some $j\notin A$, which means that
    \begin{align}\label{eq:active_change_in_0}
        \hat{\mu}_j - \sqrt{\zeta_A (q - \theta_A + \beta_A^2/\zeta_A)}\cdot u_j - v_j = \frac{\beta_A}{\zeta_A} - \frac{1}{\zeta_A\sqrt{\zeta_A (q - \theta_A + \beta_A^2/\zeta_A)}},
    \end{align}
    where $u_j = \hat{\Sigma}_{jA}\hat{\Sigma}_{AA}^{-1}\vone_A/\zeta_A$ and $v_j = \hat{\Sigma}_{jA}\hat{\Sigma}_{AA}^{-1}\LRs{\hat{\mu}_A - \beta_A/\zeta_A\cdot \vone_A}$. Let $w = \sqrt{\zeta_A (q - \theta_A + \beta_A^2/\zeta_A)}$, the equation \eqref{eq:active_change_in_0} is equivalent to
    \begin{align}\label{eq:active_change_in}
        &u_j\cdot w^2 - \LRs{\hat{\mu}_j - v_j - \frac{\beta_A}{\zeta_A}}w - \frac{1}{\zeta_A} = 0\nonumber\\
        &\Longrightarrow w = \frac{\LRs{\hat{\mu}_j - v_j - \frac{\beta_A}{\zeta_A}} + \sqrt{\LRs{\hat{\mu}_j - v_j - \frac{\beta_A}{\zeta_A}}^2 + 4u_j/\zeta_A}}{2u_j}\nonumber\\
        &\Longrightarrow q = \frac{1}{\zeta_A}\LRs{\frac{\LRs{\hat{\mu}_j - v_j - \frac{\beta_A}{\zeta_A}} + \sqrt{\LRs{\hat{\mu}_j - v_j - \frac{\beta_A}{\zeta_A}}^2 + 4u_j/\zeta_A}}{2u_j}}^2 + \theta_A - \frac{\beta_A^2}{\zeta_A}.
    \end{align}
\end{itemize}
Combing \eqref{eq:active_change_out} and \eqref{eq:active_change_in}, we conclude that the active set $A$ will change if
\begin{align}\label{eq:active_breakpoint}
    q&\geq \min_{i \in A} \LRl{\frac{\LRs{(\hat{\Sigma}_{AA}^{-1})_{\cdot i}^{\top}(\hat{\mu}_A \zeta_A - \beta_A \vone_A)}^2}{\LRs{(\hat{\Sigma}_{AA}^{-1})_{\cdot i}^{\top} \vone_A}^2 \zeta_A} + \theta_A - \frac{\beta_A^2}{\zeta_A}}\nonumber\\
    &\wedge \min_{j\notin A}\LRl{\frac{1}{\zeta_A}\LRs{\frac{\LRs{\hat{\mu}_j - v_j - \frac{\beta_A}{\zeta_A}}\sqrt{\LRs{\hat{\mu}_j - v_j - \frac{\beta_A}{\zeta_A}}^2 + 4u_j/\zeta_A}}{2u_j}}^2 + \theta_A - \frac{\beta_A^2}{\zeta_A}}.
\end{align}

Given the lower quantile $\hat{q}_{\lambda}^-$, we can determine the current active set $A$ by solving the CRO problem $z(x;\hat{q}_{\lambda}^-) = \argmin_{z\in \gZ}\max_{c\in \gU_{\lambda}(x;\hat{q}_{\lambda}^-)} \phi(c,z)$. After that, we compute the quantities in \eqref{eq:active_parameters} and further find the next breakpoint by \eqref{eq:active_breakpoint}. Then we update the active set and find the next breakpoint until it exceeds $\hat{q}_{\lambda}$.


\subsubsection{Simulation with box candidate scores}
In this section, we consider the regression task described in Section 5.1.2. We compare the GF-CROMS method (with the number of grid points set to $(1\cdot n^{1/2})^2$)) with the method introduced above, referred to as ``SF-CROMS''. Figure \ref{fig: Superset} illustrates the performance of each method as the size of the labeled sample varies.

\begin{figure}[H] 
\centering
\includegraphics[width=\textwidth]{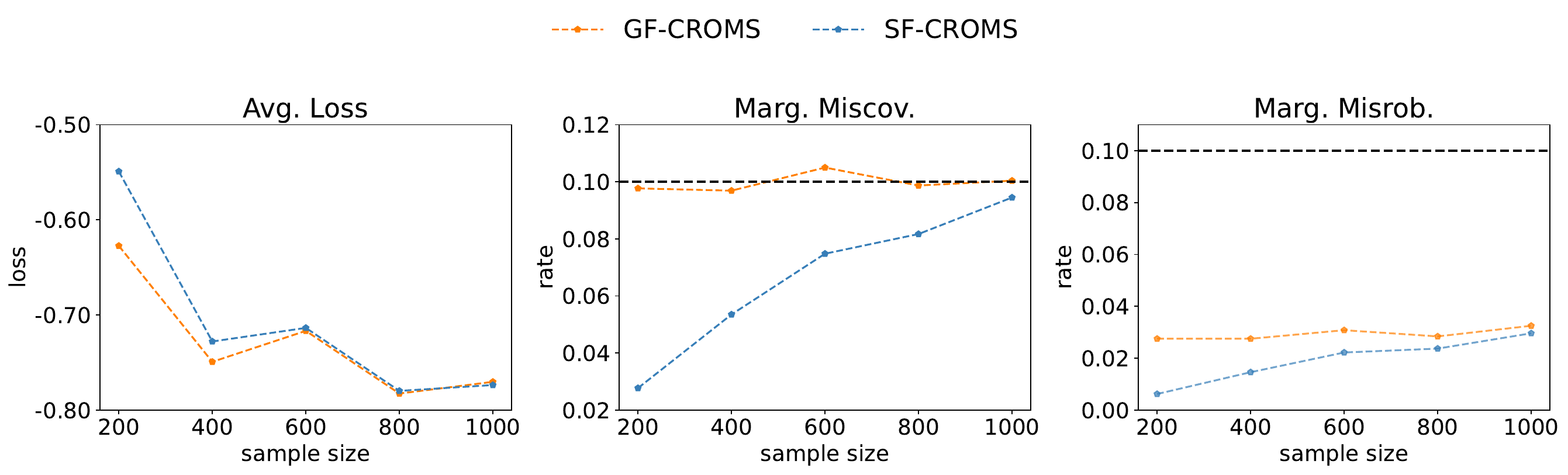}
\caption{The average loss, marginal coverage, and robustness in the regression task, with the sample size of labeled data points $n$ varied, $|\Lambda| = 25$ and $\alpha = 0.10$.}
\label{fig: Superset}
\end{figure}

Figure \ref{fig: Superset} illustrates that when the sample size is small, the superset of $\hat{\lambda}^{y}$ tends to contain multiple candidate models, which results in more conservative behavior of the SF-CROMS method. When the sample size increases, the optimal candidate model becomes relatively well determined; consequently, the average loss achieved by the SF-CROMS method aligns with that of the GF-CROMS method.

\section{Proofs for theoretical results of CROMS}

We introduce two function classes on $\gX \times \gY$: $\gF = \{\mathbbm{1}\{S_{\lambda}(x,y) > q\}: \lambda \in \Lambda, q\in \sR\}$ and $\gG = \{\phi(y, z_{\lambda}(x;q_{\lambda}^o)): \lambda \in \Lambda\}$. Then we define their Rademacher complexities as 
\begin{itemize}
    \item $\mathfrak{R}_n(\gF) = \E\LRm{\sup_{f\in \gF} \left|\frac{1}{n}\sum_{i=1}^n \xi_i f(X_i,Y_i)\right|}$,

    \item $\mathfrak{R}_n(\gG) = \E\LRm{\sup_{g\in \gG} \left|\frac{1}{n}\sum_{i=1}^n \xi_i g(X_i,Y_i)\right|}$,
\end{itemize}
where $\{\xi_i\}_{i=1}^n$ are i.i.d. random variables taking $+1$ or $-1$ with equal probability. In the following proofs, the constant $c>0$ represents a numerical constant, and is independent of any quantities in the assumptions. For simplicity, we do not distinguish the scale of $c$.

\subsection{Proofs of E-CROMS}
\subsubsection{Proof of Theorem 2.1}
\begin{proof}[Proof of Theorem 2.1]
    By the definition $\hat{q}_{\lambda} = Q_{(1-\alpha)(1+n^{-1})}\LRs{\{S_{\lambda}(X_i,Y_i)\}_{i=1}^n}$, we know 
        $$\frac{1}{n}\sum_{i=1}^{n}\mathbbm{1}\{S_{\hat{\lambda}_n}(X_i,Y_i) \leq \hat{q}_{\hat{\lambda}_n}\} = \frac{\lceil(1-\alpha)(n+1)\rceil}{n} \geq (1-\alpha)(1 + n^{-1}).$$ 
    Since both $\hat{\lambda}_n$ and $\hat{q}_{\lambda}$ depends only on $\gD_n = \{(X_i,Y_i)\}_{i=1}^n$, we have the following bound
    \begin{align}
          (1-\alpha)(1 + n^{-1}) &-\sP\LRl{S_{\hat{\lambda}_n} (X_{n+1},Y_{n+1}) \leq \hat{q}_{\hat{\lambda}_n} }\nonumber\\
         = & \E\LRm{(1-\alpha)(1 + n^{-1}) -\mathbbm{1}\LRl{S_{\hat{\lambda}_n} (X_{n+1},Y_{n+1}) \leq \hat{q}_{\hat{\lambda}_n}}}\nonumber\\
        \leq & \E\left[\frac{1}{n}\sum_{i=1}^{n}\mathbbm{1}\left\{S_{\hat{\lambda}_n}(X_i,Y_i) \leq \hat{q}_{\hat{\lambda}_n}\right\} -\mathbbm{1}\LRl{S_{\hat{\lambda}_n} (X_{n+1},Y_{n+1}) \leq \hat{q}_{\hat{\lambda}_n}}\right]\nonumber\\
        \leq & \E\left[\left|\frac{1}{n}\sum_{i=1}^{n}\mathbbm{1}\left\{S_{\hat{\lambda}_n}(X_i,Y_i) \leq \hat{q}_{\hat{\lambda}_n}\right\} - \sP\LRl{S_{\hat{\lambda}_n} (X_{n+1},Y_{n+1}) \leq \hat{q}_{\hat{\lambda}_n}\mid \gD_n}\right|\right]\nonumber\\
        \leq & \E\LRm{\sup_{\lambda \in \Lambda}\left|\frac{1}{n}\sum_{i=1}^n\mathbbm{1}\left\{S_{\lambda}(X_i,Y_i) \leq \hat{q}_{\lambda}\right\} - \sP\LRl{S_{\lambda} (X_{n+1},Y_{n+1}) \leq \hat{q}_{\lambda}\mid \gD_n}\right|}\nonumber\\
        \leq & \E\LRm{\sup_{\lambda\in \Lambda}\sup_{q\in \sR} \LRabs{\frac{1}{n}\sum_{i=1}^n\mathbbm{1}\left\{S_{\lambda}(X_i,Y_i) \leq q\right\} - \sP\LRs{S_{\lambda} (X_{n+1},Y_{n+1}) \leq q\mid \gD_n}}} \nonumber\\
        = & \E\LRm{\sup_{\lambda\in \Lambda}\sup_{q\in \sR} \LRabs{\frac{1}{n}\sum_{i=1}^n\mathbbm{1}\left\{S_{\lambda}(X_i,Y_i) \leq q\right\} - \sP\LRs{S_{\lambda} (X_{n+1},Y_{n+1}) \leq q}}}.\label{eq:uniform_marg_coverage}
    \end{align}
    Let $\xi_1,\ldots,\xi_n \stackrel{\text{i.i.d.}}{\sim} \mathrm{Unif}\{-1,+1\}$, by standard symmetrization technique, we have
    \begin{align}\label{eq:symmetrization}
        &\E\LRm{\sup_{\lambda\in \Lambda}\sup_{q\in \sR} \LRabs{\frac{1}{n}\sum_{i=1}^n\mathbbm{1}\left\{S_{\lambda}(X_i,Y_i) \leq q\right\} - \sP\LRs{S_{\lambda} (X_{n+1},Y_{n+1}) \leq q}}}\nonumber\\
        \leq & 2\E\LRm{\sup_{\lambda\in \Lambda, q\in \sR} \left|\frac{1}{n}\sum_{i=1}^n \xi_i \mathbbm{1}\left\{S_{\lambda}(X_i,Y_i) \leq q\right\}\right|} =  2\mathfrak{R}_n(\gF).
    \end{align}
    Together with \eqref{eq:uniform_marg_coverage}, we can prove the conclusion on the robustness.
\end{proof}

\subsubsection{Proof of Theorem 2.2}
\begin{lemma}\label{lemma:quantile_estimation_VC}
    Let $f_{\lambda}$ and $F_{\lambda}$ be the density function and distribution function of $S_{\lambda}(X,Y)$ respectively. For a large constant $c > 0$, if $f_{\lambda}(s) \geq \mu  > 0$ for any $s \in [F_{\lambda}^{-1}(1-\alpha-\epsilon_n), F_{\lambda}^{-1}(1-\alpha+\epsilon_n)]$ with $\epsilon_n = c\LRs{\sqrt{\frac{\log n}{n}} + \mathfrak{R}_n(\gF)} = o(1)$, then we have
    \begin{align}
    \sP\LRl{\sup_{\lambda\in \lambda}\LRabs{\hat{q}_{\lambda} - q_{\lambda}^{o}} \leq \frac{c}{\mu}\LRs{\sqrt{\frac{\log n}{n}} + \mathfrak{R}_n(\gF)}} \geq 1 - 3n^{-c}.\nonumber
    \end{align}
\end{lemma}

\begin{lemma}\label{lemma:uniform_concentration_loss}
    Under Assumption 2, then we have
    \begin{align*}
        \sP\LRl{\sup_{\lambda \in \Lambda}\LRabs{\frac{1}{n}\sum_{i=1}^n \phi(Y_i,z_{\lambda}(X_i;q_{\lambda}^o)) - \E\LRm{\phi(Y, z_{\lambda}(X;q_{\lambda}^o))}} \leq c\LRs{B\sqrt{\frac{\log n}{n}} + \mathfrak{R}_n(\gG)}} \geq 1 - n^{-c}.
    \end{align*}
\end{lemma}

\begin{proof}[Proof of Theorem 3.2]
    Recall the definitions:
    \begin{align}
        \hat{\lambda}_n &= \argmin_{\lambda\in \Lambda}\frac{1}{n}\sum_{i=1}^n \phi(Y_i,z_{\lambda}(X_i)) = \argmin_{\lambda\in \Lambda}\frac{1}{n}\sum_{i=1}^n \phi(Y_i,z_{\lambda}(X_i;\hat{q}_{\lambda})),\nonumber\\
        \lambda^* &= \argmin_{\lambda\in \Lambda}\E[\phi(Y, z_{\lambda}^{o}(X))] = \argmin_{\lambda\in \Lambda}\E[\phi(Y, z_{\lambda}(X;q_{\lambda}^{o}))].\nonumber
    \end{align}
    We define the event
    \begin{align*}
        \gE = \LRl{\sup_{\lambda \in \Lambda} |\hat{q}_{\lambda} - q_{\lambda}^{o}| \leq \frac{c}{\mu}\LRs{\sqrt{\frac{\log n}{n}} + \mathfrak{R}_n(\gF)}}.
    \end{align*}
    Under Assumption 3, using Lemma \ref{lemma:quantile_estimation_VC}, we have
    \begin{align}\label{eq:inermediate_loss_bound}
        &\E\left[\sup_{\lambda\in \Lambda}\left|\phi(Y, z_{\lambda}(X;q_{\lambda}^{o})) - \phi(Y, z_{\lambda}(X;\hat{q}_{\lambda}))\right|\right]\nonumber\\
        &\qquad\leq 2B\cdot \sP(\gE^c) + \E\LRm{\sup_{\lambda \in \Lambda}\mathbbm{1}_{\gE}\LRabs{\phi(Y, z_{\lambda}(X;q_{\lambda}^{o})) - \phi(Y, z_{\lambda}(X;\hat{q}_{\lambda}))}}\nonumber\\
        &\qquad\leq 2B\cdot n^{-c} + L  \E\LRm{\sup_{\lambda \in \Lambda}\mathbbm{1}_{\gE}|q_{\lambda}^{o} - \hat{q}_{\lambda}|}\nonumber\\
        &\qquad\leq 2Bn^{-c} + \frac{cL}{\mu}\LRs{\sqrt{\frac{\log n}{n}} + \mathfrak{R}_n(\gF)},
    \end{align}
    where the third inequality holds due to the definition of $\gE$. Since $\hat{\lambda}_n$, $\hat{q}_{\lambda}$ are independent of test data $(X_{n+1}, Y_{n+1})$, below we will write $(X,Y) \equiv (X_{n+1},Y_{n+1})$ for short.
    By the optimality of $\lambda^*$, we have
    \begin{align}\label{eq:inermediate_loss_bound_lower}
        \E[\phi(Y, z_{\hat{\lambda}_n}(X;\hat{q}_{\hat{\lambda}_n}))]& - \underbrace{\E\LRm{\phi(Y, z_{\lambda^*}(X;q^{o}_{\lambda^*}))}}_{v_{\Lambda}^*}\nonumber\\
        &\geq  \E[\phi(Y, z_{\hat{\lambda}_n}(X;\hat{q}_{\hat{\lambda}_n})) ] - \E\LRm{\phi(Y, z_{\hat{\lambda}_n}(X;q^{o}_{\hat{\lambda}_n}))}\nonumber\\
        &\geq  -2Bn^{-c} - \frac{cL}{\mu}\LRs{\sqrt{\frac{\log n}{n}} + \mathfrak{R}_n(\gF)}.
    \end{align}
    In addition, we also have the upper bound
    \begin{align}\label{eq:inermediate_loss_bound_upper}
        &\E[\phi(Y, z_{\hat{\lambda}_n}(X;\hat{q}_{\hat{\lambda}_n}))] - \underbrace{\E\LRm{\phi(Y, z_{\lambda^*}(X;q^{o}_{\lambda^*}))}}_{v_{\Lambda}^*}\nonumber\\
        &\leq \E[\phi(Y, z_{\hat{\lambda}_n}(X;\hat{q}_{\hat{\lambda}_n}))] - \E[\phi(Y, z_{\lambda^*}(X;\hat{q}_{\lambda^*}))] +\E[\phi(Y, z_{\lambda^*}(X;\hat{q}_{\lambda^*}))] - \E\LRm{\phi(Y, z_{\lambda^*}(X;q_{\lambda^*}^{o}))}\nonumber\\
        &\leq \E[\phi(Y, z_{\hat{\lambda}_n}(X;\hat{q}_{\hat{\lambda}_n}))] - \E[\phi(Y, z_{\lambda^*}(X;\hat{q}_{\lambda^*}))] + 2Bn^{-c} + \frac{cL}{\mu}\LRs{\sqrt{\frac{\log n}{n}} + \mathfrak{R}_n(\gF)}.
    \end{align}
    Next, using the optimality of $\hat{\lambda}_n$, we have
    \begin{align}
        &\E[\phi(Y, z_{\hat{\lambda}_n}(X;\hat{q}_{\hat{\lambda}_n})) \mid \gD_n] - \E\LRm{\phi(Y, z_{\lambda^*}(X;\hat{q}_{\lambda^*}))\mid \gD_n}\nonumber\\
        &\qquad = \E\LRm{\phi(Y, z_{\hat{\lambda}_n}(X;\hat{q}_{\hat{\lambda}_n})) \mid \gD_n} - \frac{1}{n}\sum_{i=1}^n \phi(Y_i,z_{\hat{\lambda}_n}(X_i;\hat{q}_{\hat{\lambda}_n})) \nonumber\\
        &\qquad + \underbrace{\frac{1}{n}\sum_{i=1}^n \phi(Y_i,z_{\hat{\lambda}_n}(X_i;\hat{q}_{\hat{\lambda}_n}))-\frac{1}{n}\sum_{i=1}^n \phi(Y_i,z_{\lambda^*}(X_i;\hat{q}_{\lambda^*}))}_{\leq 0}\nonumber\\
        &\qquad + \frac{1}{n}\sum_{i=1}^n \phi(Y_i,z_{\lambda^*}(X_i;\hat{q}_{\lambda^*})) - \E\LRm{\phi(Y, z_{\lambda^*}(X;\hat{q}_{\lambda^*}))\mid \gD_n}\nonumber\\
        &\qquad\leq 2\sup_{\lambda \in \Lambda}\LRabs{\frac{1}{n}\sum_{i=1}^n \phi(Y_i,z_{\lambda}(X_i;\hat{q}_{\lambda})) - \E\LRm{\phi(Y, z_{\lambda}(X;\hat{q}_{\lambda}))\mid \gD_n}}\nonumber\\
        &\qquad\leq 2\underbrace{\sup_{\lambda \in \Lambda}\mathbbm{1}_{\gE}\LRabs{\frac{1}{n}\sum_{i=1}^n \phi(Y_i,z_{\lambda}(X_i;q_{\lambda}^{o})) - \E\LRm{\phi(Y, z_{\lambda}(X;q_{\lambda}^{o}))}}}_{(\mathrm{I})}\nonumber\\
        &\qquad+ 2\underbrace{\sup_{\lambda \in \Lambda}\mathbbm{1}_{\gE} \LRabs{\frac{1}{n}\sum_{i=1}^n\LRl{\phi(Y_i,z_{\lambda}(X_i;\hat{q}_{\lambda}))-\phi(Y_i,z_{\lambda}(X_i;q_{\lambda}^{o}))}}}_{(\mathrm{II})}\nonumber\\
        &\qquad + 2\underbrace{\sup_{\lambda \in \Lambda}\mathbbm{1}_{\gE}\E\LRm{\LRabs{\phi(Y, z_{\lambda}(X;\hat{q}_{\lambda}))-\phi(Y, z_{\lambda}(X;q_{\lambda}^{o}))}\mid \gD_n}}_{(\mathrm{III})}\nonumber\\
        &\qquad + 2\underbrace{\sup_{\lambda \in \Lambda}\mathbbm{1}_{\gE^c}\LRabs{\frac{1}{n}\sum_{i=1}^n \phi(Y_i,z_{\lambda}(X_i;\hat{q}_{\lambda})) - \E\LRm{\phi(Y, z_{\lambda}(X;\hat{q}_{\lambda}))\mid \gD_n}}}_{(\mathrm{IV})},\label{eq:avg_loss_bound}
    \end{align}
    where the first inequality holds due to $\hat{\lambda}$ and $\hat{q}_{\lambda}$ are fixed given $\gD_n$. By Lemma \ref{lemma:uniform_concentration_loss}, we get
    \begin{align*}
        \sP\LRl{(\mathrm{I}) \leq c\LRs{B\sqrt{\frac{\log n}{n}} + \mathfrak{R}_n(\gG)}} \geq 1-n^{-c}.\nonumber
    \end{align*}
    Using Assumption 3 and the definition of $\gE$, we almost surely have
    \begin{align}
        \max\{(\mathrm{II}),(\mathrm{III})\} \leq \mathbbm{1}_{\gE} L \sup_{\lambda \in \Lambda}|\hat{q}_{\lambda} - q_{\lambda}^{o}| \leq \frac{cL}{\mu}\LRs{\sqrt{\frac{\log n}{n}} + \mathfrak{R}_n(\gF)}.\nonumber
    \end{align}
    In addition, by Lemma \ref{lemma:quantile_estimation_VC}, we also have $\sP\{(\mathrm{IV}) = 0\} \geq \sP(\gE) \geq 1-n^{-c}$. Substituting the bounds above into \eqref{eq:avg_loss_bound}, together with \eqref{eq:avg_loss_bound}, we can finish the proof.
\end{proof}

\subsection{Proofs of F-CROMS}
\subsubsection{Proof of Theorem 3.1}
\begin{proof}
    Define the virtually selected model as if $Y_{n+1}$ is known,
    \begin{align}
        \hat{\lambda} = \argmin_{\lambda \in \Lambda}\frac{1}{n+1}\sum_{i=1}^{n+1}\phi(Y_i, z_{\lambda}(X_i; \hat{Q}_{\lambda})),
    \end{align}
    where $\hat{Q}_{\lambda} = Q_{1-\alpha}\LRs{\{S_{\lambda}(X_i,Y_i)\}_{i=1}^{n+1}}$.
    Hence $\hat{\lambda}$ is symmetric to $\{(X_i,Y_i)\}_{i=1}^{n+1}$ because $\hat{Q}_{\lambda}$ is symmetric. By comparing the definitions of $\hat{Q}_{\lambda}$ and $\hat{\lambda}^y$, we know $\hat{Q}_{\lambda} \equiv \hat{q}_{\lambda}^{Y_{n+1}}$ and $\hat{\lambda} \equiv \hat{\lambda}^{Y_{n+1}}$.
    Hence, the coverage property follows from the full conformal prediction \citep{vovk2005algorithmic,lei2018distribution}, that is
    \begin{align}
        \sP\LRl{Y_{n+1} \in \widehat{\gU}^{\rm F-CROMS}(X_{n+1})} 
        &= \sP\LRl{S_{\hat{\lambda}^{Y_{n+1}}}(X_{n+1},Y_{n+1}) \leq Q_{1-\alpha}\LRs{\{S_{\hat{\lambda}^{Y_{n+1}}}(X_i,Y_i)\}_{i=1}^{n+1} }}\nonumber\\
        &= \sP\LRl{S_{\hat{\lambda}}(X_{n+1},Y_{n+1}) \leq Q_{1-\alpha}\LRs{\{S_{\hat{\lambda}}(X_i,Y_i)\}_{i=1}^{n+1} }}\nonumber\\
        &\geq 1-\alpha,\nonumber
    \end{align}
    where we also used the fact that $Q_{1-\alpha}\LRs{\{S_{\hat{\lambda}}(X_i,Y_i)\}_{i=1}^{n+1} }$ is symmetric to $\{(X_i,Y_i)\}_{i=1}^{n+1}$.
\end{proof}

\subsubsection{Proof of Theorem 3.2}
\begin{lemma}\label{lemma:qy_diff_VC}
    Under the same conditions of Lemma \ref{lemma:quantile_estimation_VC}, for a large constant $c > 0$, it holds that 
    \begin{align}
        \sP\LRl{\sup_{\lambda \in \Lambda,y\in \gY}\left|\hat{q}_{\lambda} - \hat{q}_{\lambda}^{y}\right| \leq \frac{c}{\mu} \LRs{\sqrt{\frac{\log n}{n}} + \mathfrak{R}_n(\gF)}}\geq n^{-c}.\nonumber
    \end{align}
\end{lemma}


\begin{proof}[Proof of Theorem 3.4]
Given any value $y\in \gY$, we define the hypothesized loss,
\begin{align}
    \gL_{n+1}(\lambda; y) &= \sum_{i=1}^{n}\phi(Y_i, z_{\lambda}(X_i; \hat{q}_{\lambda}^y)) + \phi(y, z_{\lambda}(X_{n+1}; \hat{q}_{\lambda}^y)),\nonumber
\end{align}
where $\gU_{\lambda}(X_i;\hat{q}_{\lambda}^y) = \{c\in \gY: S_{\lambda}(X_i,c) \leq \hat{q}_{\lambda}^y\}$ with $\hat{q}_{\lambda}^y = Q_{1-\alpha}\LRs{\{S_{\lambda}(X_i,Y_i)\}_{i=1}^n \cup \{S_{\lambda}(X_{n+1},y)\}}$.
\paragraph*{Step 1: model selection consistency for finite index set.}
We first show that for any $y\in \gY$, $\hat{\lambda}^y = \lambda^*$ holds with high probability. If there exists some $\lambda \in \Lambda$ and $\lambda \neq \lambda^*$ such that $\gL_{n+1}(\lambda;y) < \gL_{n+1}(\lambda^*;y)$, then we have
\begin{align}
    \phi(y,z_{\lambda^*}(X_{n+1};\hat{q}_{\lambda^*}^y)) - &\phi(y,z_{\lambda}(X_{n+1};\hat{q}_{\lambda}^y)) > \sum_{i=1}^n \phi(Y_i,z_{{\lambda}}(X_i;\hat{q}_{{\lambda}}^y)) - \sum_{i=1}^n \phi(Y_i,z_{\lambda^*}(X_i;\hat{q}_{\lambda^*}^y))\nonumber\\
    &= n\Big(\E[\phi(Y_{n+1},z_{\lambda}(X_{n+1};q_{\lambda}^{o}))] - \E[\phi(Y_{n+1},z_{\lambda^*}(X_{n+1};q_{\lambda^*}^{o}))]\Big)\nonumber\\
    &\qquad+ \underbrace{\sum_{i=1}^{n}\Big(\phi(Y_i,z_{{\lambda}}(X_i;\hat{q}_{{\lambda}}^y)) - \E[\phi(Y_{n+1},z_{\lambda}(X_{n+1};q_{\lambda}^{o}))]\Big)}_{\Delta_{\lambda}(y)}\nonumber\\
    &\qquad- \underbrace{\sum_{i=1}^{n}\Big(\phi(Y_i,z_{\lambda^*}(X_i;\hat{q}_{\lambda^*}^y)) - \E[\phi(Y_{n+1},z_{\lambda^*}(X_{n+1};q_{\lambda^*}^{o}))]\Big)}_{\Delta_{\lambda^*}(y)}\nonumber\\
    \geq& n\beta_n - |\Delta_{\lambda}(y)| - |\Delta_{\lambda^*}(y)|,\nonumber
\end{align}
where the last inequality holds due to the optimality gap condition in Theorem 3.5.
By the definition $\hat{\lambda}^y = \argmin_{\lambda^{\prime}\in \Lambda} \gL_{n+1}(\lambda^{\prime};y)$ and Assumption 2, we have
\begin{align}\label{eq:model_selection_event}
    \sP\LRl{\exists y \in \gY, \hat{\lambda}^y \neq \lambda^*} \leq  \sP\LRl{\frac{2}{n}\sup_{y\in \gY,\lambda \in \Lambda}\Delta_{\lambda}(y) \geq \beta_n -  \frac{2B}{n}}.
\end{align}
Using Assumption 3, for any $y\in \gY$ we have
    \begin{align}
        \max_{i\in [n]}\Big|\phi(Y_i,z_{{\lambda}}(X_i;\hat{q}_{{\lambda}}^y)) - \phi(Y_i,z_{{\lambda}}(X_i;q_{{\lambda}}^{o}))\Big| \leq L \cdot |\hat{q}_{\lambda}^y - q_{\lambda}^{o}|.\nonumber
    \end{align}
    Using Lemma \ref{lemma:uniform_concentration_loss} and \ref{lemma:qy_diff_VC}, with probability at least $1-2n^{-c}$ we have
    \begin{align}
        \frac{2}{n}\sup_{y\in \gY, \lambda \in \Lambda}\Delta_{\lambda}(y) &= 2\sup_{y\in\gY,\lambda \in \Lambda}\left|\frac{1}{n}\sum_{i=1}^{n}\Big(\phi(Y_i,z_{{\lambda}}(X_i;\hat{q}_{{\lambda}}^y)) - \E[\phi(Y_{n+1},z_{\lambda}(X_{n+1};q_{\lambda}^{o}))]\Big)\right|\nonumber\\
        &\leq 2\sup_{y\in \gY, \lambda \in \Lambda} \LRabs{\sum_{i=1}^{n}\phi(Y_i,z_{{\lambda}}(X_i;\hat{q}_{{\lambda}}^y)) - \phi(Y_i,z_{{\lambda}}(X_i;q_{{\lambda}}^{o}))}\nonumber\\
        &+ 2 \sup_{\lambda\in \lambda}\left|\sum_{i=1}^{n} \phi(Y_i, z_{\lambda}(X_i;q_{\lambda}^o)) - \E[\phi(Y_{n+1},z_{\lambda}(X_{n+1}; q_{\lambda}^o))]\right|\nonumber\\
        &\leq \frac{cL}{\mu} \LRs{\sqrt{\frac{\log n}{n}} + \mathfrak{R}_n(\gF)} + c\LRs{B\sqrt{\frac{\log n}{n}} + \mathfrak{R}_n(\gG)}\nonumber\\
        &\leq c\LRs{\frac{L}{\mu} + B}\sqrt{\frac{\log (n \vee |\Lambda|)}{n}},
    \end{align}
    where the last inequality holds due to Lemmas \ref{lemma:rademacher_finite_DKW} and \ref{lemma:rademacher_finite} when $\Lambda$ is a finite set.
    Recalling \eqref{eq:model_selection_event}, together with the assumption $\beta_n \geq O\LRs{(L/\mu + B)\sqrt{\log (n \vee |\Lambda|)} n^{-\gamma}}$ for $\gamma < 1/2$, we can show 
    \begin{equation}\label{eq:model_selection_event_prob}
        \sP\LRl{\forall y \in \gY, \hat{\lambda}^y = \lambda^*} > 1-2n^{-c}.
    \end{equation}
    
    \paragraph*{Step 2: optimality of decision.}
    Recall the definition of $\widehat{\gU}^{\FCROMS}$, we have
    \begin{align}
    \widehat{\gU}^{\FCROMS}(X_{n+1}) &= \LRl{y\in \gY: S_{\hat{\lambda}^y}(X_{n+1}, y) \leq \hat{q}_{\hat{\lambda}^y}^y}\nonumber\\
    &= \LRl{y\in \gY: S_{\hat{\lambda}^y}(X_{n+1}, y) \leq Q_{1-\alpha}\LRs{\{S_{\hat{\lambda}^y}(X_i,Y_i)\}_{i=1}^{n} \cup \{S_{\hat{\lambda}^y}(X_{n+1},y)\}}}\nonumber\\
    &= \LRl{y\in \gY: S_{\hat{\lambda}^y}(X_{n+1}, y) \leq Q_{(1-\alpha)(1+n^{-1})}\LRs{\{S_{\hat{\lambda}^y}(X_i,Y_i)\}_{i=1}^{n}}}\nonumber\\
    &= \LRl{y\in \gY: S_{\hat{\lambda}^y}(X_{n+1}, y) \leq \hat{q}_{\hat{\lambda}^y}},
    \end{align}
    where the second equality holds due to the inflation property of the sample quantile, see Lemma 2 in \citet{romano2019conformalized}.
    Under the event $\gA := \{\forall y\in \gY, \hat{\lambda}^y = \lambda^*\}$, by definitions, the final decision is equivalent to
    \begin{align}
        \hat{z}^{\FCROMS}(X_{n+1}) = z_{\lambda^*}(X_{n+1};\hat{q}_{\lambda^*}) = \argmin_{z\in \gZ}\max_{c\in \gU_{\lambda^*}(X_{n+1}; \hat{q}_{\lambda^*})} \phi(c,z).\nonumber
    \end{align}
    Recalling the definition of optimal decision:
    \begin{align}
        z_{\lambda^*}^o(X_{n+1}) = z(X_{n+1};q_{\lambda^*}^{o}) = \argmin_{z\in \gZ}\max_{c\in \gU_{\lambda^*}(X_{n+1}; q_{\lambda^*}^{o})} \phi(c,z).\nonumber
    \end{align}
    In addition, we define the event $\gE_{\lambda^*} = \LRl{|\hat{q}_{\lambda^*} - q_{\lambda^*}^{o}| \leq \mu ^{-1}\sqrt{c\log n/(2n)}}$. By \eqref{eq:model_selection_event_prob} and Lemma \ref{lemma:quantile_estimation_VC}, we can guarantee $\sP(\gE_{\lambda^*} \cap \gA) \geq 1 - 5n^{-c}$. Using Assumption 2, it follows that
    \begin{align}
        &\left|\E[\phi(Y_{n+1}, \hat{z}(X_{n+1}))] - \E[\phi(Y_{n+1}, z_{\lambda^*}^o(X_{n+1}))]\right|\nonumber\\
        &= \left|\E\LRm{\mathbbm{1}_{\gE_{\lambda^*} \cap \gA}\LRl{\phi(Y_{n+1}, z(X_{n+1};\hat{q}_{\lambda^*})) - \phi(Y_{n+1}, z(X_{n+1};q_{\lambda^*}))}}\right| + 2B\cdot \sP(\gA^c\cup \gE_{\lambda^*}^c)\nonumber\\
        &\leq \frac{L }{\mu }\sqrt{\frac{c\log n}{n}}  + 10B n^{-c}.\nonumber
    \end{align}
    We can prove the conclusion since $\E[\phi(Y_{n+1}, z_{\lambda^*}(X_{n+1}))] = v_{\Lambda}^*$.
\end{proof}


\subsubsection{Optimality of F-CROMS under continuous model class}

    \begin{assumption}\label{assum:set_smooth}
        For the general region $\gU(x) \subseteq \gY$, if  $\gU_{\lambda^*}(x; q_{\lambda^*}^o - e_n) \subseteq \gU(x) \subseteq \gU_{\lambda^*}(x; q_{\lambda^*}^o + e_n)$ with $e_n = o(1)$, then $\sup_{x,y}|\phi(y,z_{\gU}(x)) - \phi(y, z_{\lambda^*}(x;q_{\lambda^*}^o))| \leq \kappa e_n$ for some $\kappa > 0$, where $z_{\gU}(x) = \argmin_{z\in \gZ}\max_{c\in \gU(x)} \phi(c,z)$.
    \end{assumption}
    
    \begin{theorem}\label{thm:FCROMS_optimality_VC}
        For {a continuous index set $\Lambda$}, suppose that there exist two positive sequences $\beta_n$ and $\delta_n = o(1)$ such that $\E[\phi(Y,z_{\lambda}^o(X))] \geq \E[\phi(Y, z_{\lambda^*}^o(X))] + \beta_n$ holds for any $\|\lambda - \hat{\lambda}^y\| \leq \delta_n$ and $\sup_{(x,y)\in \gX \times \gY}|S_{\lambda}(x,y) - S_{\hat{\lambda}^y}(x,y)| \leq \bar{L}_{\Lambda}\delta_n$ if $\|\lambda - \hat{\lambda}^y\| \leq \delta_n$. Under Assumptions 1-3 and \ref{assum:set_smooth}, if $e_n = O\LRl{\frac{\mu}{L}\LRs{\mathfrak{R}_n(\gF)+\sqrt{\frac{\log n}{n}}} + \bar{L}_{\Lambda}\delta_n}$ in Assumption \ref{assum:set_smooth} and $\beta_n \geq O\LRl{\LRs{\frac{\mu}{L}+B} \sqrt{\frac{\log n}{n}} + \frac{\mu}{L}\mathfrak{R}_n(\gF) + \mathfrak{R}_n(\gG)}$, then we have
        \begin{align}
            \left|\E\LRm{\phi\LRs{Y_{n+1}, \hat{z}^{\FCROMS}(X_{n+1})}} - v_{\Lambda}^*\right| \leq O\LRl{\frac{\kappa L}{\mu} \LRs{\sqrt{\frac{\log n}{n}} + \mathfrak{R}_n(\gF)} + \kappa\bar{L}_{\Lambda}\delta_n + \frac{B}{n}}.\nonumber
        \end{align}
    \end{theorem}

Compared with Theorem 3.2, we replace the minimum risk gap condition with a ``continuous'' version in this theorem. It guarantees that, under a small perturbation $\beta_n$ in the oracle decision risk, the deviation between the selected model index $\hat{\lambda}^y$ and the optimal index $\lambda^*$ can be bounded by $\delta_n$. 
Such a condition ensures the model selection \emph{stability} of the F-CROMS method, that is $\|\hat{\lambda}^y - \lambda^*\| \leq \delta_n$ holds for any $y\in \gY$ with high probability. This type of stability is crucial for analyzing the theoretical properties of full conformal prediction methods, as discussed in \citet{bian2023training} and \citet{liang2025algorithmic}. Moreover, the required smoothness condition on $S_{\lambda}(x,y)$ around $\hat{\lambda}^y$ is naturally satisfied for the form $S_{\lambda}(x,y) = f(g_{\lambda}(x), y)$, provided that $f$ and $g$ are Lipschitz smooth.

\begin{proof}[Proof of Theorem \ref{thm:FCROMS_optimality_VC}]
    We use the same notations in the proof of Theorem 3.4.
    \paragraph*{Step 1: model selection consistency.}
    Due to the assumption
    \[
    \beta_n \geq O\LRl{\frac{L}{\mu} \LRs{\sqrt{\frac{\log n}{n}} + \mathfrak{R}_n(\gF)} + B\sqrt{\frac{\log n}{n}} + \mathfrak{R}_n(\gG)},
    \]
    using similar arguments in Step 1 of the proof of Theorem 3.5, we can show
    \begin{equation}\label{eq:model_selection_event_prob_VC}
        \sP\LRl{\forall y \in \gY, \|\hat{\lambda}^y - \lambda^*\| \leq \delta_n} > 1-2n^{-c}.
    \end{equation}
    \paragraph*{Step 2: optimality of decision.}
    Recall the definition of $\widehat{\gU}^{\FCROMS}$, we have
    \begin{align}\label{eq:U_full_expansion}
    \widehat{\gU}^{\FCROMS}(X_{n+1}) 
    &= \LRl{y\in \gY: S_{\hat{\lambda}^y}(X_{n+1}, y) \leq \hat{q}_{\hat{\lambda}^y}^y}\nonumber\\
    &= \LRl{y\in \gY: S_{\lambda^*}(X_{n+1}, y) \leq S_{\lambda^*}(X_{n+1}, y) -S_{\hat{\lambda}^y}(X_{n+1}, y) + \hat{q}_{\hat{\lambda}^y}^y}\nonumber\\
    &=: \LRl{y\in \gY: S_{\lambda^*}(X_{n+1}, y) \leq \widetilde{Q}^y}.
    \end{align}
    By the smoothness condition of $S_{\lambda}$, we have $\max_{i\in [n]}|S_{\lambda^*}(X_i,Y_i) - S_{\hat{\lambda}^y}(X_i,Y_i)| \leq \bar{L}_{\Lambda}\|\lambda^* - \hat{\lambda}^y\|$ and $|S_{\lambda^*}(X_{n+1}, y) - S_{\hat{\lambda}^y}(X_i,y)| \leq \bar{L}_{\Lambda}\|\lambda^* - \hat{\lambda}^y\|$. By Lemma \ref{lemma:quantile_uniform_diff}, we know $|\hat{q}_{\hat{\lambda}_y}^y - \hat{q}_{\lambda^*}^y| \leq \bar{L}_{\Lambda}\|\lambda^* - \hat{\lambda}^y\|.$
    Applying Lemma \ref{lemma:qy_diff_VC} and the relation \eqref{eq:model_selection_event_prob_VC}, with probability $1-5n^{-c}$,
    \begin{align}
        \sup_{y\in \gY}|\widetilde{Q}^y - q_{\lambda}^o| &\leq \sup_{y\in \gY}|\hat{q}_{\hat{\lambda}_y} - \hat{q}_{\lambda^*}^y| + \sup_{y\in \gY}|\hat{q}_{\lambda^*}^y - q_{\lambda^*}^o| + \sup_{y\in \gY} |S_{\lambda^*}(X_{n+1}, y) -S_{\hat{\lambda}^y}(X_{n+1}, y)|\nonumber\\
        &\leq \bar{L}_{\Lambda}\sup_{y\in \gY}\|\lambda^* - \hat{\lambda}^y\| + \frac{c}{\mu} \LRs{\sqrt{\frac{\log n}{n}} + \mathfrak{R}_n(\gF)} + \bar{L}_{\Lambda}\sup_{y\in \gY}\|\hat{\lambda}^y - \lambda^*\|\nonumber\\
        &\leq \underbrace{\frac{c}{\mu} \LRs{\sqrt{\frac{\log n}{n}} + \mathfrak{R}_n(\gF)} + 2\bar{L}_{\Lambda}\delta_n}_{e_n},\nonumber
    \end{align}
    where the second inequality holds due to Assumption \ref{assum:set_smooth}; and the last inequality holds due to Lemma \ref{lemma:quantile_estimation_VC} and \eqref{eq:model_selection_event_prob_VC}. With the same probability, by \eqref{eq:U_full_expansion}, we have
    \begin{align}
        \widehat{\gU}^{\FCROMS}(X_{n+1}) \subseteq \gU_{\lambda^*}(X_{n+1}; q_{\lambda^*}^o + e_n),\quad
        \gU_{\lambda^*}(X_{n+1}; q_{\lambda^*}^o - e_n) \subseteq \widehat{\gU}^{\FCROMS}(X_{n+1}).\nonumber
    \end{align}
    Using the additional assumption (ii), we can have
    \begin{align*}
        \left|\E\LRm{\phi\LRs{Y_{n+1}, z^{\FCROMS}(X_{n+1})}} - \E\LRm{\phi\LRs{Y_{n+1}, z_{\lambda^*}(X_{n+1})}}\right| \leq \kappa e_n + 10B n^{-c}.
    \end{align*}
    Then the conclusion follows from the definition of $e_n$.
    \end{proof}

\begin{lemma}\label{lemma:quantile_uniform_diff}
    Let $\{a_i\}_{i=1}^n$ and $\{b_i\}_{i=1}^n$ be two sequences without ties. Denote $a_{(1)} \leq \ldots \leq a_{(n)}$ and $b_{(1)} \leq \ldots \leq b_{(n)}$. If $\max_{i\in [n]}|a_i - b_i| \leq e$, then we have $|a_{(k)} - b_{(k)}| \leq e$ for any $k\in [n]$.
\end{lemma}

\begin{proof}
    We first notice that
    \begin{align*}
        \sum_{i=1}^n \mathbbm{1}\{a_i \leq b_{(k)} + e\} \geq \sum_{i=1}^n \mathbbm{1}\{b_i + e \leq b_{(k)} + e\} = k,
    \end{align*}
    which means that $a_{(k)} \leq b_{(k)} + e$. On the contrary side, we also have
    \begin{align*}
        \sum_{i=1}^n \mathbbm{1}\{a_i > b_{(k)} - e\} \geq \sum_{i=1}^n \mathbbm{1}\{b_i - e > b_{(k)} - e\} = n-k,
    \end{align*}
    which means that $a_{(k)} \geq  b_{(k)} - e$.
\end{proof}

\subsection{Proofs of grid-approximated F-CROMS}

\subsubsection{Proof of Theorem 3.3}
\begin{proof}
Let $\widetilde{\gD}_n = \{(X_i,\widetilde{Y_i})\}_{i=1}^{n+1}$ denote the entire discretized dataset, where $\widetilde{\gY}_i = \sD(Y_i)$. Since $\sD$ is a deterministic mapping, the discretized data $\{(X_i, \widetilde{Y_i})\}_{i=1}^{n+1}$ remain i.i.d. (or exchangeable). By the intermediate result established in the proof of Theorem 3.4, we have:
$$
\mathbb{P}\left\{ Y_{n+1} \in \widehat{\gU}^{\GFCROMS} (X_{n+1})\right\} = \mathbb{P}\left\{\widetilde{Y}_{n+1} \in \widetilde{\gU}^{\FCROMS}(X_{n+1})\right\} \geq 1-\alpha.
$$
where the first equality follows from the definition of the inverse mapping $\sD^{-1}$.
\end{proof}

\subsubsection{Proof of Theorem 3.4}

\begin{assumption}\label{assum:smooth_score}
    The score functions satisfy $\sup_{x\in \gX}|S_{\lambda}(x,y) - S_{\lambda}(x,y^{\prime})| \leq \bar{L}_{\gY}\|y - y^{\prime}\|$ for any $\lambda \in \Lambda$. The loss function satisfies $\sup_{z\in \gZ}|\phi(y,z) - \phi(y^{\prime},z)| \leq L_{\gY} \|y - y^{\prime}\|$.
\end{assumption}

\begin{proof}[Proof of Theorem 3.4]
    By the definition of the discretization mapping, we know $\|y - \sD(y)\| \leq \epsilon_{\rm{grid}}$ for any $y\in \gY$. 
    Given any value $\tilde{y}\in \widetilde{\gY}$, we define the hypothesized discretized loss as
    \begin{align}
        \widetilde{\gL}_{n+1}(\lambda; \tilde{y}) &= \frac{1}{n+1}\LRl{\sum_{i=1}^{n}\phi\LRs{\widetilde{Y_i}, z_{\lambda}(X_i; \tilde{q}_{\lambda}^{\tilde{y}})} + \phi\LRs{\tilde{y}, z_{\lambda}(X_{n+1}; \tilde{q}_{\lambda}^{\tilde{y}})}},\nonumber
    \end{align}
    where $\tilde{q}_{\lambda}^{\tilde{y}} = Q_{1-\alpha}\LRs{\{S_{\lambda}(X_i,\widetilde{Y_i})\}_{i=1}^n \cup \{S_{\lambda}(X_{n+1},\tilde{y})\}}$. According to assumption, it holds that $\max_{i\in [n]}|S_{\lambda}(X_i, \widetilde{Y_i}) - S_{\lambda}(X_i, Y_i)| \leq \bar{L}_{\gY}\epsilon_{\rm{grid}}$. By Lemma \ref{lemma:quantile_uniform_diff}, we have $\left|\tilde{q}_{\lambda}^{\tilde{y}} - \hat{q}_{\lambda}^{\sD^{-1}(\tilde{y})}\right| \leq \bar{L}_{\gY}\epsilon_{\rm{grid}}$. Using Lemma \ref{lemma:qy_diff_VC}, we further have
    \begin{align}\label{eq:discretized_quantile_diff}
        \sP\LRl{\sup_{\tilde{y}\in \widetilde{\gY}, \lambda \in \Lambda}\left|\tilde{q}_{\lambda}^{\tilde{y}} - q_{\lambda}^o\right| \leq \bar{L}_{\gY}\epsilon_{\rm{grid}} + \frac{c}{\mu} \LRs{\sqrt{\frac{\log n}{n}} + \mathfrak{R}_n(\gF)}} \geq 1-n^{-c}.
    \end{align}
    By Assumptions 3 and \ref{assum:smooth_score}, we have
    \begin{align}\label{eq:discretization_loss_error}
        &\sup_{\tilde{y}\in \widetilde{\gY},\lambda \in \Lambda}\max_{i\in [n]}\left|\phi\LRs{\widetilde{Y_i}, z_{\lambda}(X_i; \tilde{q}_{\lambda}^{\tilde{y}})} - \phi(Y_i, z_{\lambda}(X_i;q_{\lambda}^o))\right|\nonumber\\
        &\leq \sup_{\tilde{y}\in \widetilde{\gY},\lambda \in \Lambda}\max_{i\in [n]}\left|\phi\LRs{\widetilde{Y_i}, z_{\lambda}(X_i; \tilde{q}_{\lambda}^{\tilde{y}})} - \phi\LRs{Y_i, z_{\lambda}(X_i; \tilde{q}_{\lambda}^{\tilde{y}})}\right|\nonumber\\
        &\qquad + \sup_{\tilde{y}\in \widetilde{\gY},\lambda \in \Lambda}\max_{i\in [n]}\left|\phi\LRs{Y_i, z_{\lambda}(X_i; \tilde{q}_{\lambda}^{\tilde{y}})} - \phi(Y_i, z_{\lambda}(X_i;q_{\lambda}^o))\right|\nonumber\\
        &\leq L_{\gY}\epsilon_{\rm{grid}} + L\bar{L}_{\gY}\epsilon_{\rm{grid}} + \frac{cL}{\mu} \LRs{\sqrt{\frac{\log n}{n}} + \mathfrak{R}_n(\gF)}.
    \end{align}
    \paragraph*{Step 1: model selection consistency.}
    Notice that, $\tilde{\lambda}^{\tilde{y}}=\widetilde{\gL}_{n+1}(\lambda; \tilde{y})$.
    We first show that for any $\tilde{y}\in \gY$, $\tilde{\lambda}^{\tilde{y}} = \lambda^*$ holds with high probability. If there exists some $\lambda \in \Lambda$ and $\lambda \neq \lambda^*$ such that $\widetilde{\gL}_{n+1}(\lambda;\tilde{y}) < \widetilde{\gL}_{n+1}(\lambda^*;\tilde{y})$, then we have
    \begin{align}
        \phi(\tilde{y},z_{\lambda^*}(X_{n+1};\tilde{q}_{\lambda^*}^{\tilde{y}})) - &\phi(\tilde{y},z_{\lambda}(X_{n+1};\tilde{q}_{\lambda}^{\tilde{y}})) > \sum_{i=1}^n \phi(\widetilde{Y_i},z_{{\lambda}}(X_i;\tilde{q}_{{\lambda}}^{\tilde{y}})) - \sum_{i=1}^n \phi(\widetilde{Y_i},z_{\lambda^*}(X_i;\tilde{q}_{\lambda^*}^{\tilde{y}}))\nonumber\\
        &= n\Big(\E[\phi(Y_{n+1},z_{\lambda}(X_{n+1};q_{\lambda}^{o}))] - \E[\phi(Y_{n+1},z_{\lambda^*}(X_{n+1};q_{\lambda^*}^{o}))]\Big)\nonumber\\
        &\qquad+ \underbrace{\sum_{i=1}^{n}\Big(\phi(\widetilde{Y_i},z_{{\lambda}}(X_i;\tilde{q}_{{\lambda}}^{\tilde{y}})) - \E[\phi(Y_{n+1},z_{\lambda}(X_{n+1};q_{\lambda}^{o}))]\Big)}_{\widetilde{\Delta}_{\lambda}(\tilde{y})}\nonumber\\
        &\qquad- \underbrace{\sum_{i=1}^{n}\Big(\phi(\widetilde{Y_i},z_{\lambda^*}(X_i;\tilde{q}_{\lambda^*}^{\tilde{y}})) - \E[\phi(Y_{n+1},z_{\lambda^*}(X_{n+1};q_{\lambda^*}^{o}))]\Big)}_{\widetilde{\Delta}_{\lambda^*}(\tilde{y})}\nonumber\\
        &\geq n\beta_n - |\widetilde{\Delta}_{\lambda}(\tilde{y})| - |\widetilde{\Delta}_{\lambda^*}(\tilde{y})|,\nonumber
    \end{align}
    where the last inequality holds due to the minimum risk gap condition. Using \eqref{eq:discretization_loss_error}, we can have
    \begin{align}
        \frac{1}{n}\sup_{\lambda \in \Lambda,\tilde{y} \in \widetilde{\gY}} |\widetilde{\Delta}_{\lambda}(\tilde{y})| &\leq \sup_{\tilde{y}\in \widetilde{\gY}, \lambda \in \Lambda} \LRabs{\frac{1}{n}\sum_{i=1}^{n}\phi(\widetilde{Y_i},z_{{\lambda}}(X_i;\tilde{q}_{{\lambda}}^{\tilde{y}})) - \phi(Y_i,z_{{\lambda}}(X_i;q_{{\lambda}}^{o}))}\nonumber\\
        &\quad +  \sup_{\lambda\in \lambda}\left|\frac{1}{n}\sum_{i=1}^{n} \phi(Y_i, z_{\lambda}(X_i;q_{\lambda}^o)) - \E[\phi(Y_{n+1},z_{\lambda}(X_{n+1}; q_{\lambda}^o))]\right|\nonumber\\
        &\leq (L_{\gY}+L\bar{L}_{\gY})\epsilon_{\rm{grid}} + \frac{cL}{\mu} \LRs{\sqrt{\frac{\log n}{n}} + \mathfrak{R}_n(\gF)} + c\LRs{B\sqrt{\frac{\log n}{n}} + \mathfrak{R}_n(\gG)}\nonumber\\
        &\leq (L_{\gY}+L\bar{L}_{\gY})\epsilon_{\rm{grid}} + c\LRs{\frac{L}{\mu} + B}\sqrt{\frac{\log(n \vee |\Lambda|)}{n}},
    \end{align}
    where the last inequality holds due to Lemmas \ref{lemma:rademacher_finite_DKW} and \ref{lemma:rademacher_finite} when $\Lambda$ is a finite set.
    Recalling \eqref{eq:model_selection_event}, together with the assumption $\beta_n > O\LRs{(L_{\gY}+L\bar{L}_{\gY})\epsilon_{\rm{grid}}+(L/\mu + B) \sqrt{\frac{\log (n \vee |\Lambda|)}{n}}}$, we can show
    \begin{equation}\label{eq:grid_model_selection_event_prob}
        \sP\LRl{\forall \tilde{y} \in \widetilde{\gY}, \tilde{\lambda}^{\tilde{y}} = \lambda^*} > 1-2n^{-c}.
    \end{equation}
    
    \paragraph*{Step 2: optimality of decision.} By the relation \eqref{eq:grid_model_selection_event_prob}, with probability at least $1-2n^{-c}$, we have $\widetilde{\gU}^{\GFCROMS}(X_{n+1}) = \LRl{\tilde{y}\in\widetilde{\gY}: S_{\lambda^*}(X_{n+1}, \tilde{y}) \leq \tilde{q}_{\lambda^*}^{\tilde{y}}}.$
    For any $y\in \gU_{\lambda^*}(X_{n+1}; q_{\lambda^*}^o - e_n)$, using the relation \eqref{eq:discretized_quantile_diff}, we know that with probability at least $1-n^{-c}$,
    \begin{align}
        S_{\lambda^*}(X_{n+1},y) \leq q_{\lambda^*}^o - e_n &\Longrightarrow
        S_{\lambda^*}(X_{n+1}, \sD(y)) \leq q_{\lambda^*}^o - e_n + \bar{L}_{\gY}\epsilon_{\rm{grid}}\nonumber\\
        &\Longrightarrow
        S_{\lambda^*}(X_{n+1}, \sD(y)) \leq \tilde{q}_{\lambda^*}^{\sD(y)} - e_n + \bar{L}_{\gY}\epsilon_{\rm{grid}} + \frac{c}{\mu} \LRs{\sqrt{\frac{\log n}{n}} + \mathfrak{R}_n(\gF)}\nonumber\\
        &\Longrightarrow
        S_{\lambda^*}(X_{n+1}, \sD(y)) \leq \tilde{q}_{\lambda^*}^{\sD(y)},\nonumber
    \end{align}
    where we also used the assumption $e_n = O\LRs{\bar{L}_{\gY}\epsilon_{\rm{grid}} + \frac{1}{\mu} \LRs{\sqrt{\frac{\log n}{n}} + \mathfrak{R}_n(\gF)}}$. It means that $\sP\{\forall y \in \gU_{\lambda^*}(X_{n+1}; q_{\lambda^*}^o - e_n),\sD(y) \in \widetilde{\gU}^{\GFCROMS}(X_{n+1})\} \geq 1-3n^{-c}$.
    Then we have shown
    \begin{align}\label{eq:grid_FCROMS_right}
        \sP\LRl{\gU_{\lambda^*}(X_{n+1}; q_{\lambda^*}^o - e_n) \subseteq \widehat{\gU}^{\GFCROMS}(X_{n+1})} \geq 1-3n^{-c}.
    \end{align}
    For any $\tilde{y} \in \widetilde{\gU}^{\GFCROMS}(X_{n+1})$, using \eqref{eq:grid_model_selection_event_prob}, we have with probability at least $1-n^{-c}$,
    \begin{align}
        S_{\lambda^*}(X_{n+1}, \tilde{y}) \leq \tilde{q}_{\lambda^*}^{\tilde{y}}
        &\Longrightarrow S_{\lambda^*}(X_{n+1}, \tilde{y}) \leq q_{\lambda^*}^o + \frac{c}{\mu} \LRs{\sqrt{\frac{\log n}{n}} + \mathfrak{R}_n(\gF)}\nonumber\\
        &\Longrightarrow S_{\lambda^*}(X_{n+1}, \sD^{-1}(\tilde{y})) \leq q_{\lambda^*}^o + \bar{L}_{\gY}\epsilon_{\rm{grid}} + \frac{c}{\mu} \LRs{\sqrt{\frac{\log n}{n}} + \mathfrak{R}_n(\gF)}\nonumber\\
        &\Longrightarrow S_{\lambda^*}(X_{n+1}, \sD^{-1}(\tilde{y})) \leq q_{\lambda^*}^o + e_n.\nonumber
    \end{align}
    It implies that $\sP\{\forall \tilde{y} \in \widetilde{\gU}^{\GFCROMS}(X_{n+1}),\sD^{-1}(\tilde{y}) \in \gU_{\lambda^*}(X_{n+1}; q_{\lambda^*}^o + e_n)\} \geq 1-3n^{-c}$, which means that
    \begin{align}\label{eq:grid_FCROMS_left}
        \sP\LRl{\widehat{\gU}^{\GFCROMS}(X_{n+1}) \subseteq \gU_{\lambda^*}(X_{n+1}; q_{\lambda^*}^o + e_n)} \geq 1-3n^{-c}.
    \end{align}
    Combining \eqref{eq:grid_FCROMS_right} and \eqref{eq:grid_FCROMS_left}, and using Assumption \ref{assum:set_smooth}, we can have
    \begin{align}
        \sP\LRl{\left|\phi(Y_{n+1}, \hat{z}^{\GFCROMS}(X_{n+1})) - \phi(Y_{n+1}, z_{\lambda^*}(X_{n+1}))\right| \leq \kappa e_n} \geq 1-6n^{-c}.\nonumber
    \end{align}
    Then we can show the conclusion by noticing that $|\phi(y,z)| \leq B$.
\end{proof}

\subsubsection{Continuous index set}
\begin{theorem}
    Under the assumptions of Theorem 3.2 and Assumptions \ref{assum:set_smooth} and \ref{assum:smooth_score}, we additionally assume: $\sup_{(x,y)\in \gX \times \gY}|S_{\lambda}(x,y) - S_{\hat{\lambda}^y}(x,y)| \leq \bar{L}_{\Lambda}\delta_n$ if $\|\lambda - \hat{\lambda}^y\| \leq \delta_n$.
    If $e_n = O\LRl{\bar{L}_{\gY}\epsilon_{\rm{grid}} + \frac{c}{\mu} \LRs{\sqrt{\frac{\log n}{n}} + \mathfrak{R}_n(\gF)} + 2\bar{L}_{\Lambda} \delta_n}$ and
    \begin{align}\label{eq:minimum_gap_GF_continuous}
        \beta_n \geq (L_{\gY}+L\bar{L}_{\gY})\epsilon_{\rm{grid}} + \frac{cL}{\mu} \LRs{\sqrt{\frac{\log n}{n}} + \mathfrak{R}_n(\gF)} + c\LRs{B\sqrt{\frac{\log n}{n}} + \mathfrak{R}_n(\gG)},
    \end{align}
    then we have
        \begin{align}
            \E\LRm{\phi\LRs{Y_{n+1}, \hat{z}^{\GFCROMS}(X_{n+1})}} \leq v_{\Lambda}^* + O\LRl{\frac{\kappa L}{\mu} \LRs{\sqrt{\frac{\log n}{n}} + \mathfrak{R}_n(\gF)} + \kappa\bar{L}_{\Lambda}\delta_n + \frac{B}{n}}.\nonumber
        \end{align}
\end{theorem}
\begin{proof}
    We use the same notations in the proof of Theorem 3.4.
    \paragraph*{Step 1: model selection consistency.}
    Due to the assumption \eqref{eq:minimum_gap_GF_continuous}
    using similar arguments in Step 1 of the proof of Theorem 3.7, we can show
    \begin{equation}\label{eq:model_selection_event_prob_VC_GF}
        \sP\LRl{\forall y \in \gY, \|\tilde{\lambda}^y - \lambda^*\| \leq \delta_n} > 1-2n^{-c}.
    \end{equation}
    \paragraph*{Step 2: optimality of decision.}
    For any $y\in \gU_{\lambda^*}(X_{n+1}; q_{\lambda^*}^o - e_n)$, we know that with probability at least $1-2n^{-c}$,
    \begin{align}
        &S_{\lambda^*}(X_{n+1},y) \leq q_{\lambda^*}^o - e_n\nonumber\\
        &\Longrightarrow
        S_{\lambda^*}(X_{n+1}, \sD(y)) \leq q_{\lambda^*}^o - e_n + \bar{L}_{\gY}\epsilon_{\rm{grid}}\nonumber\\
        &\Longrightarrow S_{\tilde{\lambda}^{\sD(y)}}(X_{n+1}, \sD(y)) \leq q_{\lambda^*}^o - e_n + \bar{L}_{\gY}\epsilon_{\rm{grid}} + S_{\tilde{\lambda}^{\sD(y)}}(X_{n+1}, \sD(y)) - S_{\lambda^*}(X_{n+1}, \sD(y))\nonumber\\
        &\Longrightarrow
        S_{\tilde{\lambda}^{\sD(y)}}(X_{n+1}, \sD(y)) \leq \tilde{q}_{\tilde{\lambda}^{\sD(y)}}^{\sD(y)} - e_n + \bar{L}_{\gY}\epsilon_{\rm{grid}} + \frac{c}{\mu} \LRs{\sqrt{\frac{\log n}{n}} + \mathfrak{R}_n(\gF)} + 2\bar{L}_{\Lambda} \delta_n\nonumber\\
        &\Longrightarrow
        S_{\tilde{\lambda}^{\sD(y)}}(X_{n+1}, \sD(y)) \leq \tilde{q}_{\tilde{\lambda}^{\sD(y)}}^{\sD(y)},\nonumber
    \end{align}
    where we used the relations \eqref{eq:discretized_quantile_diff}, \eqref{eq:model_selection_event_prob_VC_GF}, additional assumption and Lemma \ref{lemma:quantile_uniform_diff}; and the assumption $e_n = O\LRs{\bar{L}_{\gY}\epsilon_{\rm{grid}} + \frac{c}{\mu} \LRs{\sqrt{\frac{\log n}{n}} + \mathfrak{R}_n(\gF)} + 2\bar{L}_{\Lambda} \delta_n}$. Then it implies that $\sD(y) \in \widetilde{\gU}^{\GFCROMS}(X_{n+1})$ with probability at least $1-2n^{-c}$. Hence we have
    \begin{align}
        \sP\LRl{\gU_{\lambda^*}(X_{n+1}; q_{\lambda^*}^o - e_n) \subseteq \widetilde{\gU}^{\GFCROMS}(X_{n+1})} \geq 1-2n^{-c}.\nonumber
    \end{align}
    On the contrary side, for any $\tilde{y} \in \widetilde{\gU}^{\GFCROMS}(X_{n+1})$, with probability at least $1-2n^{-c}$,
    \begin{align}
        &S_{\tilde{\lambda}^{\tilde{y}}}(X_{n+1}, \tilde{y}) \leq \tilde{q}_{\tilde{\lambda}^{\tilde{y}}}^{\tilde{y}}\nonumber\\
        &\Longrightarrow S_{\lambda^*}(X_{n+1}, \tilde{y}) \leq \tilde{q}_{\tilde{\lambda}^{\tilde{y}}}^{\tilde{y}} + \bar{L}_{\Lambda} \delta_n \nonumber\\
        &\Longrightarrow S_{\lambda^*}(X_{n+1}, \sD^{-1}(\tilde{y})) \leq \tilde{q}_{\tilde{\lambda}^{\tilde{y}}}^{\tilde{y}} + \bar{L}_{\Lambda} \delta_n + \bar{L}_{\gY} \epsilon_{\rm{grid}} \nonumber\\
        &\Longrightarrow S_{\lambda^*}(X_{n+1}, \sD^{-1}(\tilde{y})) \leq q_{\lambda^*}^o +\tilde{q}_{\tilde{\lambda}^{\tilde{y}}}^{\tilde{y}} - \tilde{q}_{\lambda^*}^{\tilde{y}} + \tilde{q}_{\lambda^*}^{\tilde{y}} - q_{\lambda^*}^o + \bar{L}_{\Lambda} \delta_n + \bar{L}_{\gY} \epsilon_{\rm{grid}} \nonumber\\
        &\Longrightarrow S_{\lambda^*}(X_{n+1}, \sD^{-1}(\tilde{y})) \leq q_{\lambda^*}^o + \bar{L}_{\gY}\epsilon_{\rm{grid}} + \frac{c}{\mu} \LRs{\sqrt{\frac{\log n}{n}} + \mathfrak{R}_n(\gF)} + 2\bar{L}_{\Lambda} \delta_n\nonumber\\
        &\Longrightarrow S_{\lambda^*}(X_{n+1}, \sD^{-1}(\tilde{y})) \leq q_{\lambda^*}^o + e_n.\nonumber
    \end{align}
    It means that
    \begin{align}
        \sP\LRl{\widetilde{\gU}^{\GFCROMS}(X_{n+1}) \subseteq \gU_{\lambda^*}(X_{n+1}; q_{\lambda^*}^o - e_n)} \geq 1-2n^{-c}.\nonumber
    \end{align}
    Using Assumption \ref{assum:set_smooth}, we can have
    \begin{align*}
        \left|\E\LRm{\phi\LRs{Y_{n+1}, z^{\GFCROMS}(X_{n+1})}} - \E\LRm{\phi\LRs{Y_{n+1}, z_{\lambda^*}(X_{n+1})}}\right| \leq \kappa e_n + 10B n^{-c}.
    \end{align*}
    Then the conclusion follows from the definition of $e_n$.
    \end{proof}

\subsection{Bounds for Rademacher complexities}
Proposition 3.1 can be proved by Lemma \ref{lemma:rademacher_VC}. Proposition 3.2 can be proved by Lemmas \ref{lemma:rademacher_finite} and \ref{lemma:rademacher_lipschitz}.

\begin{lemma}\label{lemma:rademacher_VC}
    If the model class $\{S_{\lambda}(x,y): \lambda \in \Lambda\}$ is VC-class with VC-dimension $\mathsf{v}$, then we have $\mathfrak{R}_n(\gF) \leq c\sqrt{\frac{\mathsf{v}}{n}}$.
\end{lemma}
\begin{proof}
    By Chapter 3.3 of \citet{kearns1994introduction}, we know
    the VC dimension of function class $\gF$ is $\mathsf{v} + 1$. 
    By Theorem 2.6.7 in \citet{van1996weak}, for any probability measure $Q$, there exists a universal constant $c$ such that
    \begin{align}\label{eq:covering_number_VC}
        N(\epsilon, \gF, L_2(Q)) \leq c (\mathsf{v}+1)(16e)^{\mathsf{v}}(1/\epsilon)^{2\mathsf{v}}.
    \end{align}
    For any $f \in \gF$ (i.e., $f(x,y) = \mathbbm{1}\left\{S_{\lambda}(X_i,Y_i) \leq q\right\}$ for some $\lambda \in \Lambda$ and $q\in \sR$), we define
    \begin{align*}
        \sG_n(f) = \frac{1}{\sqrt{n}}\sum_{i=1}^n \xi_i f(X_i,Y_i).
    \end{align*}
    For any $f,g\in \gF$, we define the metric
    \begin{align}
        \|\sG_n(f) - \sG_n(g)\|_{P_n}^2 &= \frac{1}{n}\sup_{\substack{\lambda,\lambda^{'} \in \Lambda,\\ q,q^{'}\in \sR}}\sum_{i=1}^n \LRs{\mathbbm{1}\{S_{\lambda}(X_i,Y_i) \leq q\} - \mathbbm{1}\{S_{\lambda^{\prime}}(X_i,Y_i) \leq q^{\prime}\}}.\nonumber
    \end{align}
    It holds that $\sup_{f,g\in \gF} \|\sG_n(f) - \sG_n(g)\|_{P_n} \leq 1$.
    Taking expectation only on $\xi_1,\ldots,\xi_n$ and using Dudley’s entropy integral bound, we have
    \begin{align*}
        \E\LRm{\sup_{f\in \gF} |\sG_n(f)| \mid \gD_n} \leq c\int_{0}^{1} \sqrt{\log N(\epsilon,\gF,\|\cdot\|_{P_n})} d\epsilon \leq c\sqrt{\mathsf{v}+1},
    \end{align*}
    where we used the upper bound \eqref{eq:covering_number_VC}. Taking full expectation, we can finish the proof.
\end{proof}

\begin{lemma}\label{lemma:rademacher_finite_DKW}
    If $\Lambda$ is a finite index set, i.e., $|\Lambda| < \infty$, then we have $\mathfrak{R}_n(\gF) \leq O(\sqrt{\log (|\Lambda|)/n})$.
\end{lemma}
\begin{proof}
    When $\Lambda$ is a finite set, the class $\gF_{\lambda} = \{\mathbbm{1}\{S_{\lambda}(x,y) \leq q\}: q\in \sR\}$ has VC-dimension 1. Then $\gF = \cup_{\lambda \in \Lambda} \gF_{\lambda}$ has VC-dimension $2+ \log_2 (|\Lambda|)$. Dudley’s entropy integral shows that $\mathfrak{R}_n(\gF) \leq O(\sqrt{\log (|\Lambda|)/n})$.
\end{proof}

\begin{lemma}\label{lemma:rademacher_finite}
    If $\Lambda$ is a finite index set, i.e., $|\Lambda| < \infty$, then we have $\mathfrak{R}_n(\gG) \leq 2B \sqrt{\frac{\log (2|\Lambda|)}{n}}$.
\end{lemma}
\begin{proof}
    We first show the conclusion for a general finite function class $\gG$, which is based on the proof in \citet{lecture9}.
    For convenience, let us augment the class $\widetilde{\gG} = \gG \cup -\gG$. Then it holds that
    \begin{align*}
        \mathfrak{R}_n(\gG) \leq \E\LRm{\sup_{f\in \widetilde{\gG}} \frac{1}{n}\sum_{i=1}^n \xi_i f(X_i,Y_i)}.
    \end{align*}
    It implies that for $t\geq 0$, if $\sup_{f\in \gG}|f| \leq b$ we have
    \begin{align*}
        \exp\LRl{t\mathfrak{R}_n(\gG)} &\leq \exp\LRl{t \E\LRm{\sup_{f\in \widetilde{\gG}} \frac{1}{n}\sum_{i=1}^n \xi_i f(X_i,Y_i)}}\\
        &\leq \E\LRm{\exp\LRl{t \sup_{f\in \widetilde{\gG}} \frac{1}{n}\sum_{i=1}^n \xi_i f(X_i,Y_i)}}\\
        &\leq \sum_{f\in \widetilde{\gG}}\prod_{i=1}^n\E\LRm{\exp\LRl{t \frac{1}{n} \xi_i f(X_i,Y_i)}}\\
        &= 2|\gG|\cdot \exp\LRs{\frac{4t^2 b^2}{n}},
    \end{align*}
    where the last inequality follows the proof of Hoeffding’s inequality. It follows that $\mathfrak{R}_n(\gG) \leq \frac{\log(2|\gG|)}{t} + \frac{4tb^2}{n}$.
    Choosing $t = \sqrt{\frac{n \log(2|\gG|)}{4b^2}}$, we can prove that $\mathfrak{R}_n(\gG) \leq 2b\sqrt{\frac{\log(2|\gG|)}{n}}$. Then the conclusion follows from $|\gG| = |\Lambda|$.
\end{proof}


\begin{lemma}\label{lemma:rademacher_lipschitz}
    For the continuous index set $\Lambda \subseteq \sR^m$ with bounded radius $R$, if there exists a constant $L_{\Lambda} > 0$ such that  $\sup_{x\in \gX,y\in \gY}\LRabs{\phi(y,z_{\lambda}(x; q_{\lambda}^o)) - \phi(y,z_{\lambda^{\prime}}(x;q_{\lambda^{\prime}}^o))} \leq L_{\Lambda} \|\lambda-\lambda^{\prime}\|$ for any $\|\lambda - \lambda^{\prime}\| \leq O(n^{-1})$, then we have $\mathfrak{R}_n(\gG) \leq O\LRs{B\sqrt{\frac{m \log(R n)}{n}} + \frac{L_{\Lambda}}{n}}$.
\end{lemma}

\begin{proof}
    Let $\{\lambda_{\ell}\}_{\ell=1}^{N_{\epsilon}}$ be an $\epsilon$-covering of $\Lambda \subset \sR^m$ under Euclidean norm $\|\cdot\|$, where $\epsilon \leq n^{-1}$. It holds that $N_{\epsilon} \leq O\{(2R/\epsilon)^m\}$.
    Then for any $\lambda \in \Lambda$, there exists some $\lambda_{\ell}$ such that $\|\lambda - \lambda_{\ell}\| \leq \epsilon$ and $\|\lambda - \lambda_{\ell^{'}}\| > \epsilon$ for $\ell^{'}\neq \ell$. It follows that
    \begin{align}
        \mathfrak{R}_n(\gG) &= \E\LRm{\sup_{\lambda \in \Lambda}\mathbbm{1}_{\gE}\LRabs{\frac{1}{n}\sum_{i=1}^n \xi_i\phi(Y_i,z_{\lambda}(X_i;q_{\lambda}^{o}))}}\nonumber\\
        &\leq \E\LRm{\sup_{\lambda \in \Lambda}\sum_{\ell \in [N_{\epsilon}]}\mathbbm{1}\LRl{\|\lambda - \lambda_{\ell}\| \leq \epsilon} \left|\frac{1}{n}\sum_{i=1}^n \xi_i \LRm{\phi(Y_i,z_{\lambda_{\ell}}(X_i;q_{\lambda_{\ell}}^o)) - \phi(Y_i,z_{\lambda}(X_i;q_{\lambda}^o))}\right|}\nonumber\\
        &\qquad + \E\LRm{\sum_{\ell \in [N_{\epsilon}]}\mathbbm{1}\LRl{\|\lambda - \lambda_{\ell}\| \leq \epsilon}\LRabs{\frac{1}{n}\sum_{i=1}^n\xi_i \phi(Y_i,z_{\lambda_{\ell}}(X_i;q_{\lambda_{\ell}}^o))}}\nonumber\\
        &\leq L_{\Lambda} \epsilon + \E\LRm{\max_{\ell \in [N_{\epsilon}]}\LRabs{\frac{1}{n}\sum_{i=1}^n \xi_i \phi(Y_i,z_{\lambda_{\ell}}(X_i;q_{\lambda_{\ell}}^o))}},\nonumber
    \end{align}
    where the last inequality holds due to locally Lipschitz continuity on $\lambda$ and $\epsilon \leq n^{-1}$. By the same proof of Lemma \ref{lemma:rademacher_finite} with finite index set $[N_{\epsilon}]$, we know
    \begin{align*}
        \E\LRm{\max_{\ell \in [N_{\epsilon}]}\LRabs{\frac{1}{n}\sum_{i=1}^n \xi_i \phi(Y_i,z_{\lambda_{\ell}}(X_i;q_{\lambda_{\ell}}^o))}} \leq 2B \sqrt{\frac{\log (2N_{\epsilon})}{n}}.
    \end{align*}
    Taking $\epsilon = n^{-1}$, together with $N_{\epsilon} \leq O\{(R/\epsilon)^m\}$, we can prove the conclusion.
\end{proof}

\subsection{Proofs of main lemmas}

\subsubsection{Proof of Lemma \ref{lemma:quantile_estimation_VC}}

\begin{proof}
    Denote $M_n = \sup_{\lambda\in \Lambda, q\in \sR} \left|\frac{1}{n}\sum_{i=1}^{n}\mathbbm{1}\LRl{S_{\lambda}(X_i,Y_i) \leq q} - \sP(S_{\lambda}(X_i,Y_i) \leq q)\right|$.
    Applying McDiarmid’s inequality, we have
    \begin{align*}
        \sP\LRs{M_n  > \E[M_n] + t } \leq \exp\LRs{-\frac{nt^2}{2}}.
    \end{align*}
    Same as \eqref{eq:symmetrization}, we have $\E[M_n] \leq 2 \mathfrak{R}_n(\gF)$. Fixing $\beta \in (0,1)$ and
    taking $\epsilon_n = \sqrt{\frac{2c\log n}{n}}+2 \mathfrak{R}_n(\gF)$, it follows from the definition of sample quantile that
    \begin{align}
        &\sP\LRl{\exists\lambda\in \lambda,\ Q_{\beta}\LRs{\{S_{\lambda}(X_i,Y_i)\}_{i=1}^{n}} > F_{\lambda}^{-1}(\beta + \epsilon_n)}\nonumber\\
        &\leq \sP\LRl{\exists\lambda\in \lambda,\ \frac{1}{n}\sum_{i=1}^{n}\mathbbm{1}\LRl{S_{\lambda}(X_i,Y_i) \leq F_{\lambda}^{-1}(\beta + \epsilon_n)} < \beta}\nonumber\\
        &\leq \sP\LRl{\sup_{\lambda\in\Lambda,q\in \sR}\left|\frac{1}{n}\sum_{i=1}^{n}\mathbbm{1}\LRl{S_{\lambda}(X_i,Y_i) \leq q} - \sP\LRl{S_{\lambda}(X_i,Y_i) \leq q}\right| > \epsilon_n}\nonumber\\
        &\leq 2n^{-c},\nonumber
    \end{align}
    Similarly, we can also show that $\sP\LRl{\exists\lambda\in \lambda,Q_{\beta}\LRs{\{S_{\lambda}(X_i,Y_i)\}_{i=1}^{n}} < F_{\lambda}^{-1}(\beta -\epsilon_n)} \leq 2n^{-c}.$
    It means that
    \begin{align}
        \sP\LRl{\forall \lambda\in \lambda, F_{\lambda}^{-1}(\beta-\epsilon_n) \leq Q_{\beta}\LRs{\{S_{\lambda}(X_i,Y_i)\}_{i=1}^{n}}\leq F_{\lambda}^{-1}(\beta+\epsilon_n)} \geq 1 - 4n^{-c}.\nonumber
    \end{align}
    In addition, since the density $f_{\lambda}(s)$ is lower bounded by $\mu $ for $s \in [F_{\lambda}^{-1}(1-\alpha-\epsilon_n), F_{\lambda}^{-1}(1-\alpha+\epsilon_n)]$, we have $F_{\lambda}^{-1}(1-\alpha+\epsilon_n) - q_{\lambda}^{o} = F_{\lambda}^{-1}(1-\alpha+\epsilon_n) - F_{\lambda}^{-1}(1-\alpha) \leq \frac{\epsilon_n}{\mu }$. Combining the results above, we can prove the conclusion.
\end{proof}

\subsubsection{Proof of Lemma \ref{lemma:uniform_concentration_loss}}
\begin{proof}
    Denote $G_n = \sup_{\lambda \in \Lambda,q\in \sR}\LRabs{\frac{1}{n}\sum_{i=1}^n \phi(Y_i,z_{\lambda}(X_i;q)) - \E\LRm{\phi(Y, z_{\lambda}(X;q))}}$. Since $\sup_{\lambda\in \Lambda}|\phi(Y_i,z_{\lambda}(X_i;q_{\lambda}^{o}))| \leq B$, using McDiarmid’s inequality gives
    \begin{align*}
        \sP\LRs{G_n  > \E[G_n] + t } \leq \exp\LRs{-\frac{nt^2}{2B^2}}.
    \end{align*}
    Using the symmetrization again, we have $\E[G_n] \leq 2 \mathfrak{R}_n(\gG)$ with $\gG = \{\phi(y,z_{\lambda}(x;q)): \lambda \in \Lambda, q\in \gQ\}$. Taking $t = B\sqrt{\frac{2c\log n}{n}}$, we can prove the conclusion.
\end{proof}

\subsubsection{Proof of Lemma \ref{lemma:qy_diff_VC}}
\begin{proof}
    By the proof of Lemma \ref{lemma:quantile_estimation_VC}, we have
    \begin{align*}
        \sP\LRl{\sup_{\lambda \in \Lambda} \left|\hat{q}_{\lambda} - F_{\lambda}^{-1}((1-\alpha)(1+n^{-1}))\right| > \frac{c}{\mu}\LRs{\sqrt{\frac{\log n}{n}} + \mathfrak{R}_n(\gF)}} \leq 2n^{-c},
    \end{align*}
    and
    \begin{align*}
        \sP\LRl{\sup_{\lambda \in \Lambda} \left|\hat{q}_{\lambda}^- - F_{\lambda}^{-1}((1-\alpha)(1+n^{-1}) -n^{-1})\right| > \frac{c}{\mu}\LRs{\sqrt{\frac{\log n}{n}} + \mathfrak{R}_n(\gF)}} \leq 2n^{-c},
    \end{align*}
    Using the lower bounded density condition in Assumption 1, for any $\lambda \in \Lambda$ we have
    \begin{align*}
        F_{\lambda}^{-1}((1-\alpha)(1+n^{-1}) +n^{-1}) -  \frac{1}{\mu} \frac{2-\alpha}{n} &\leq F_{\lambda}^{-1}(1-\alpha)\\
        &\leq F_{\lambda}^{-1}((1-\alpha)(1+n^{-1}) - n^{-1}) +  \frac{1}{\mu} \frac{\alpha}{n}.
    \end{align*}
    Together with the fact $\sup_{y\in \gY}|\hat{q}_{\lambda}^y - \hat{q}_{\lambda}| \leq \hat{q}_{\lambda}-\hat{q}_{\lambda}^-$, using Lemma \ref{lemma:quantile_estimation_VC}, we can prove the conclusion.
\end{proof}

\subsection{E2E method with sample splitting}\label{appen:split_CROMS}
Suppose the labeled data is split into two parts:
$\gD_n^{(1)} = \{(X_i,Y_i)\}_{i=1}^{n_0}$ and $\gD_n^{(2)} = \{(X_i,Y_i)\}_{i=n_0+1}^n$. Define the calibration threshold $\hat{q}_{\lambda}^{(1)} = Q_{(1-\alpha)(1+n_0^{-1})}(\{S_{\lambda}(X_i,Y_i)\}_{i=1}^{n_0})$ and respective prediction set $\gU_{\lambda}^{(1)}(\cdot) = \{y\in \gY:S_{\lambda}(\cdot,y) \leq \hat{q}_{\lambda}^{(1)}\}$. The auxiliary decisions are $z_{\lambda}^{(1)}(X_i) = \argmin_{z\in \gZ}\max_{c\in \gU_{\lambda}^{(1)}(X_i)} \phi(c,z)$ for $i\in [n_0]$. Then we perform the model selection via
\begin{align}
    \hat{\lambda}^{(1)} = \argmin_{\lambda \in \Lambda} \frac{1}{n_0}\sum_{i=1}^{n_0} \phi(Y_i,z_{\lambda}^{(1)}(X_i)).\nonumber
\end{align}
After that, we construct the conformal prediction set through the dataset $\gD_n^{(2)}$ as
\begin{align}
    \widehat{\gU}^{(2)}(X_{n+1}) = \LRl{y\in \gY: S_{\hat{\lambda}^{(1)}}(X_{n+1}, y) \leq Q_{1-\alpha}\LRs{\{S_{\hat{\lambda}^{(1)}}(X_i,Y_i)\}_{i=n_0+1}^n\cup \{\infty\}}}.\nonumber
\end{align}
The respective decision is given by $\hat{z}^{(2)}(X_{n+1}) = \argmin_{z\in \gZ}\max_{c\in \widehat{\gU}^{(2)}(X_{n+1})}\phi(c,z)$. This procedure is equivalent to the E2E method in \citet{yeh2024end} for finite $\Lambda$. According to the split conformal prediction theory \citep{vovk2005algorithmic,lei2018distribution}, $\widehat{\gU}^{(2)}(X_{n+1})$ enjoys finite-sample coverage property, hence the decision $\hat{z}^{(2)}(X_{n+1})$ achieves the $1-\alpha$ level of marginal robustness in Definition 1. 
We define the Rademacher complexities as 
\begin{itemize}
    \item $\mathfrak{R}_{n_0}(\gF) = \E\LRm{\sup_{f\in \gF} \left|\frac{1}{n_0}\sum_{i=1}^{n_0} \xi_i f(X_i,Y_i)\right|}$,
    
    \item $\mathfrak{R}_{n-n_0}(\gF) = \E\LRm{\sup_{f\in \gF} \left|\frac{1}{n-n_0}\sum_{i=n_0+1}^{n} \xi_i f(X_i,Y_i)\right|}$,

    \item $\mathfrak{R}_{n_0}(\gG) = \E\LRm{\sup_{g\in \gG} \left|\frac{1}{n_0}\sum_{i=1}^{n_0} \xi_i g(X_i,Y_i)\right|}$,
    
    \item $\mathfrak{R}_{n-n_0}(\gG) = \E\LRm{\sup_{g\in \gG} \left|\frac{1}{n-n_0}\sum_{i=n_0+1}^{n} \xi_i g(X_i,Y_i)\right|}$,
\end{itemize}
 where $\{\xi_i\}_{i=1}^n$ are i.i.d. random variables taking $+1$ or $-1$ with equal probability. The next theorem characterizes the decision loss of $\hat{z}^{(2)}(X_{n+1})$.

\begin{theorem}\label{thm:split_CROMS_loss}
    Under the same conditions of Theorem 3.2, we have
    \begin{align}
        &\E[\phi(Y_{n+1}, \hat{z}^{(2)}(X_{n+1}))] - v_{\Lambda}^*\nonumber\\
        &\leq O\LRl{\LRs{B+\frac{L }{\mu }}\sqrt{\frac{\log n}{n_0}} + \frac{L }{\mu }\sqrt{\frac{\log n}{n-n_0}} + \frac{L}{\mu}\LRs{\mathfrak{R}_{n_0}(\gF) + \mathfrak{R}_{n-n_0}(\gF)} + \mathfrak{R}_{n_0}(\gG)}.\nonumber
    \end{align}
\end{theorem}

\begin{proof}
    From the proof of Theorem 3.2, we know
    \begin{align}\label{eq:split_1_decision_loss}
        &\E\left[\phi(Y_{n+1}, z_{\hat{\lambda}^{(1)}}^{(1)}(X_{n+1}))\right] - v_{\Lambda}^*\nonumber\\
        \leq & O\LRl{\frac{L}{\mu}\LRs{\sqrt{\frac{\log n}{n_0}} + \mathfrak{R}_{n_0}(\gF)} + B\sqrt{\frac{\log n}{n_0}} + \mathfrak{R}_{n_0}(\gG)}.
    \end{align}
    Here we write $z_{\hat{\lambda}^{(1)}}(X_{n+1};\hat{q}_{\hat{\lambda}^{(1)}}^{(1)}) \equiv z_{\hat{\lambda}^{(1)}}^{(1)}(X_{n+1})$ and $z_{\hat{\lambda}^{(1)}}(X_{n+1};\hat{q}_{\hat{\lambda}^{(1)}}^{(2)}) \equiv \hat{z}^{(2)}(X_{n+1})$, where 
    $$\hat{q}_{\lambda}^{(2)} = Q_{(1-\alpha)(n-n_0+1)^{-1}}\LRs{\{S_{\lambda}(X_i,Y_i)\}_{i=n_0+1}^n},\quad \forall \lambda \in \Lambda.$$
    It follows from Assumption 3, we have
    \begin{align}
        &\left|\phi\LRs{Y_{n+1}, z_{\hat{\lambda}^{(1)}}(X_{n+1};\hat{q}_{\hat{\lambda}^{(1)}}^{(1)})} - \phi\LRs{Y_{n+1}, z_{\hat{\lambda}^{(1)}}(X_{n+1};\hat{q}_{\hat{\lambda}^{(1)}}^{(2)})}\right|\nonumber\\
        &\leq L  \LRabs{\hat{q}_{\hat{\lambda}^{(1)}}^{(1)} - \hat{q}_{\hat{\lambda}^{(1)}}^{(2)}} \leq L \sup_{\lambda \in \Lambda}\LRs{|\hat{q}_{\lambda}^{(1)} - q_{\lambda}^{o}| + |\hat{q}_{\lambda}^{(2)} - q_{\lambda}^{o}|}.\nonumber 
    \end{align}
    Invoking Lemma \ref{lemma:quantile_estimation_VC}, with probability at least $1-2n^{-c}$, we can guarantee that
    \begin{align}
        \sup_{\lambda \in \Lambda}\LRs{|\hat{q}_{\lambda}^{(1)} - q_{\lambda}^{o}| + |\hat{q}_{\lambda}^{(2)} - q_{\lambda}^{o}|} \leq \frac{c}{\mu}\LRs{\sqrt{\frac{\log n}{n_0}} + \sqrt{\frac{\log n}{n-n_0}} + \mathfrak{R}_{n_0}(\gF) + \mathfrak{R}_{n-n_0}(\gF)}.\nonumber
    \end{align}
    Hence we can show
    \begin{align}
        &\E\LRm{\left|\phi\LRs{Y_{n+1}, z_{\hat{\lambda}^{(1)}}(X_{n+1};\hat{q}_{\hat{\lambda}^{(1)}}^{(1)})} - \phi\LRs{Y_{n+1}, z_{\hat{\lambda}^{(1)}}(X_{n+1};\hat{q}_{\hat{\lambda}^{(1)}}^{(2)})}\right|}\nonumber\\
        \leq & O\LRl{\frac{L }{\mu }\LRs{\sqrt{\frac{\log n}{n_0}} + \sqrt{\frac{\log n}{n-n_0}} + \mathfrak{R}_{n_0}(\gF) + \mathfrak{R}_{n-n_0}(\gF)} + \frac{B}{n}}.\nonumber
    \end{align}
    Together with \eqref{eq:split_1_decision_loss}, we can prove the conclusion.
\end{proof}

\section{Jackknife+ and CV+ CROMS}

For $i\in [n]$, J-CROMS performs the leave-one-out (LOO) model selection by
\begin{align*}
    \hat{\lambda}^{-i} = \argmin_{\lambda \in \Lambda} \LRl{\mathcal{L}_{-i}^{\rm{LOO}}(\lambda)= \sum_{\ell =1,\ell \neq i}^n \phi\LRs{Y_{\ell}, z_{\lambda}(X_i; \hat{q}_{\lambda}^{-i})}},
\end{align*}
where $\hat{q}_{\lambda}^{-i} = Q_{(1-\alpha)(1+(n-1)^{-1})}\LRs{\{S_{\lambda}(X_{\ell},Y_{\ell})\}_{\ell =1,\ell \neq i}^n }$.
The final prediction set is given by
\begin{align}\label{eq:JCROMS_set_appen}
    \widehat{\gU}^{\text{J-CROMS}}(X_{n+1}) = \LRl{y\in \gY: \frac{\sum_{i=1}^n \mathbbm{1}\big\{ S_{\hat{\lambda}^{-i}}(X_{n+1}, y)\leq S_{\hat{\lambda}^{-i}}(X_{i}, Y_{i}) \big\}+1}{n+1} > \alpha}.
\end{align}
If the candidate models are box scores $S_{\lambda}(x,y) = \|(y - \hat{\mu}_{\lambda}(x))/\hat{\sigma}_{\lambda}(x)\|_{\infty}$ for $\lambda \in \Lambda$, a superset for J-CROMS prediction set \eqref{eq:JCROMS_set_appen} is given in Section 3.3, which is also a box area in $\gY$. Next, we consider the case where the candidate models are ellipsoid scores $S_{\lambda}(x,y) = (y - \hat{\mu}_{\lambda}(x))^{\top}\hat{\Sigma}_{\lambda}^{-1}(x)(y - \hat{\mu}_{\lambda}(x)) =: \|\hat{\Sigma}_{\lambda}^{-1/2}(x)(y - \hat{\mu}_{\lambda}(x))\|$ for $\lambda \in \Lambda$. Let $r_i = S_{\hat{\lambda}^{-i}}(X_{i}, Y_{i})$ and $O_i = \hat{\Sigma}_{\hat{\lambda}^{-i}}^{-1/2}(X_i)\hat{\mu}_{\hat{\lambda}^{-i}}(X_i)$ for $i\in [n]$, then notice that
\begin{align}
    S_{\hat{\lambda}^{-i}}(X_{n+1}, y)\leq r_i \Longleftrightarrow \|\hat{\Sigma}_{\hat{\lambda}^{-i}}^{-1/2}(X_i) y - O_i\| \leq r_i.\nonumber
\end{align}
Then we can equivalently write the J-CROMS prediction set \eqref{eq:JCROMS_set_appen} as
\begin{align}\label{eq:JCROMS_and_depth_set}
    \widehat{\gU}^{\text{J-CROMS}}(X_{n+1}) = \bigcup_{\substack{\gJ \subseteq [n],\\|\gJ| = \lceil (n+1)\alpha -1\rceil}}\bigcap_{i\in \gJ} \left\{y\in \gY: \|\hat{\Sigma}_{\hat{\lambda}^{-i}}^{-1/2}(X_i) y - O_i\| \leq r_i\right\}.
\end{align}
The right-hand side of \eqref{eq:JCROMS_and_depth_set} is the region of space $\gY = \sR^p$ covered by at least $L=\lceil (n+1)\alpha -1\rceil$ of the $n$ ellipsoids, which is the $L$-level problem in computational geometry \citep{alma991023418299706532}. However, since the intersection region of convex areas can be nonconvex and disconnected, directly solving the downstream CRO problem may be non-tractable.

\subsection{Efficient implementation of J-CROMS}
Similar to Algorithm \ref{alg:AOA}, we can develop an efficient implementation for the J-CROMS method. For each model $\lambda \in \Lambda$, we define the upper and lower quantile as:
\begin{align}
\hat{q}_{\lambda}^+ = Q_{1-\alpha+1/n}(\{S_{\lambda}(X_i,Y_i)\}_{i=1}^{n}),\quad \hat{q}_{\lambda}^- = Q_{1-\alpha}(\{S_{\lambda}(X_i,Y_i)\}_{i=1}^{n}).\nonumber
\end{align}
For a given model $\lambda$ and $i\in [n]$, the LOO lower and upper loss are computed by
\begin{equation*}
\mathcal{L}_{-i}^-(\lambda) = \frac{1}{n-1} \sum_{\ell=1, \ell \neq i}^{n} \phi\left(Y_{\ell}, z_{\lambda}(X_i; \hat{q}_{\lambda}^-)\right),\quad
\mathcal{L}_{-i}^+(\lambda) = \frac{1}{n-1} \sum_{\ell=1, \ell \neq i}^{n} \phi\left(Y_{\ell}, z_{\lambda}(X_i;\hat{q}_{\lambda})\right).
\end{equation*}
Then the LOO loss admits the following case-wise expression:
\begin{align*}
    \mathcal{L}_{-i}^{\rm{LOO}}(\lambda) = 
\begin{cases}
\mathcal{L}_{-i}^-(\lambda) & \text{ if } S_{\lambda}(X_{i},Y_i) \geq \hat{q}_{\lambda}^- \\[5pt]
\mathcal{L}_{-i}^+(\lambda) & \text{ if } S_{\lambda}(X_{i},Y_i) < \hat{q}_{\lambda}^-.
\end{cases}.
\end{align*}
Using this expression, we can accelerate the J-CROMS method.

\subsection{Proof of Theorem 3.8}
\begin{proof}
    The proof is similar to the proof of Theorem 1 in \citet{barber2021predictive}. We first introduce the leave-two-out definitions: for $i,j\in [n+1]$,
    \begin{align}
        \hat{\lambda}^{-(i,j)} = \argmin_{\lambda \in \Lambda} \frac{1}{n-1}\sum_{\ell =1,\ell \neq i,j}^{n+1} \phi(Y_{\ell}, z_{\lambda}^{-(i,j)}(X_{\ell})),\nonumber
    \end{align}
    where $z_{\lambda}^{-(i,j)}(X_{\ell}) = \argmin_{z\in \gZ}\max_{c\in \gU_{\lambda}^{-(i,j)}(X_{\ell})} \phi(c,z)$ and $\gU_{\lambda}^{-(i,j)}(X_{\ell}) = \{c\in \gY: S_{\lambda}(X_{\ell}, c) \leq Q_{(1-\alpha)(1+1/n)}\LRs{\{S_{\lambda}(X_l,Y_l)\}_{l=1,l\neq i,j}^{n+1}}\}$. Then we know that $\hat{\lambda}^{-(i,n+1)} \equiv \hat{\lambda}^{-i}$ and $\hat{\lambda}^{-(n+1,j)} \equiv \hat{\lambda}^{-j}$ for any $i,j\in [n+1]$. In addition, we also know that $\hat{\lambda}^{-(i,j)}$ is symmetric to the data points $\{(X_{\ell},Y_{\ell})\}_{\ell=1,\ell \neq i,j}^{n+1}$.

    \paragraph*{Coverage of J-CROMS prediction set.}
    Let $A_{i,j} = \mathbbm{1}\{S_{\hat{\lambda}^{-(i,j)}}(X_{i}, Y_{i}) > S_{\hat{\lambda}^{-(i,j)}}(X_{j}, Y_{j})\}$ for any $i,j\in [n+1]$, and $A_{i,i} = 0$ for any $i\in [n+1]$. Define the set of strange points:
    \begin{align*}
        \gS = \LRl{i\in [n+1]: \sum_{j=1}^{n+1} A_{i,j} \geq (n+1)(1-\alpha)}.
    \end{align*}
    According to \citet{barber2021predictive}, we know $|\gS| \leq 2\alpha(n+1)$, which means that $\sum_{i=1}^{n+1} \mathbbm{1}\{i\in \gS\} \leq 2\alpha (n+1)$. In addition, using the exchangeability, we have
    \begin{align*}
        \sP\LRs{n+1 \in \gS} = \frac{1}{n+1}\sum_{i=1}^{n+1}\sP(i\in \gS) = \frac{1}{n+1}\E\LRm{\sum_{i=1}^{n+1} \mathbbm{1}\{i\in \gS\}} \leq 2\alpha.
    \end{align*}
    Recall that $\sum_{j=1}^{n+1} A_{n+1,j} = \sum_{j=1}^{n} \mathbbm{1}\{S_{\hat{\lambda}^{-j}}(X_{n+1}, Y_{n+1}) > S_{\hat{\lambda}^{-j}}(X_{j}, Y_{j})\}$ and
    \begin{align}
        \widehat{\gU}^{\text{J-CROMS}}(X_{n+1}) &= \LRl{y\in \gY: \frac{\sum_{i=1}^n \mathbbm{1}\big\{S_{\hat{\lambda}^{-i}}(X_{n+1}, y) \leq S_{\hat{\lambda}^{-i}}(X_{i}, Y_{i}) \big\}+1}{n+1} > \alpha}\nonumber\\
        &= \LRl{y\in \gY: n+1 - \sum_{j=1}^{n+1} A_{n+1,j} \leq (n+1)\alpha}.\nonumber
    \end{align}
    Hence we have showed that $\sP\{Y_{n+1}\notin \widehat{\gU}^{\text{J-CROMS}}(X_{n+1})\} = \sP\LRs{n+1 \in \gS} \leq 2\alpha$.

    \paragraph*{Coverage of the box J-CROMS prediction set.}
    Recall that the J-CROMS prediction set for box candidate scores is $\widehat{\gU}_{\rm{box}}^{\text{J-CROMS}}(X_{n+1}) = \LRl{y\in \sR^p: c^{\text{lo}} \leq y \leq c^{\text{up}}}$, where for $k\in [p]$,
    \begin{align*}
        c_{k}^{\text{up}} &= Q_{(1-\alpha)(1+1/n)}\left(\LRl{\hat{\mu}_{\hat{\lambda}^{-i},k}(X_{n+1}) + \hat{\sigma}_{\hat{\lambda}^{-i},k}(X_{n+1}) S_{\hat{\lambda}^{-i}}(X_i,Y_i)}_{i=1}^{n}\right),\nonumber\\
        c_{k}^{\text{lo}} &= -Q_{(1-\alpha)(1+1/n)}\left(\LRl{\hat{\sigma}_{\hat{\lambda}^{-i},k}(X_{n+1})S_{\hat{\lambda}^{-i}}(X_i,Y_i) - \hat{\mu}_{\hat{\lambda}^{-i},k}(X_{n+1})}_{i=1}^{n}\right).
    \end{align*}
    It implies that for any $y\in \gY$, we have the following relation
    \begin{align*}
        &y\notin \widehat{\gU}_{\rm{box}}^{\text{J-CROMS}}(X_{n+1})\\
        \Longleftrightarrow &\ 
        \bigcup_{k\in [p]} \left\{y_k > c_k^{\rm{up}}\right\} \cup \left\{y_k < c_k^{\rm{lo}}\right\}\nonumber\\
        \Longleftrightarrow&\bigcup_{k\in [p]}\left\{ y_k > Q_{(1-\alpha)(1+1/n)}\left(\LRl{\hat{\mu}_{\hat{\lambda}^{-i},k}(X_{n+1}) + \hat{\sigma}_{\hat{\lambda}^{-i},k}(X_{n+1})S_{\hat{\lambda}^{-i}}(X_i,Y_i)}_{i=1}^{n}\right)\right\}\nonumber\\
        &\ \cup \left\{ y_k  < -Q_{(1-\alpha)(1+1/n)}\left(\LRl{\hat{\sigma}_{\hat{\lambda}^{-i},k}(X_{n+1})S_{\hat{\lambda}^{-i}}(X_i,Y_i)-\hat{\mu}_{\hat{\lambda}^{-i},k}(X_{n+1})}_{i=1}^{n}\right)\right\}\nonumber\\
        \Longleftrightarrow&\  \bigcup_{k\in [p]}\LRl{\sum_{i=1}^n \mathbbm{1}\LRl{\frac{y_k - \hat{\mu}_{\hat{\lambda}^{-i},k}(X_{n+1})}{\hat{\sigma}_{\hat{\lambda}^{-i},k}(X_{n+1})} > S_{\hat{\lambda}^{-i}}(X_i,Y_i)}> (1-\alpha)(n+1)}\nonumber\\
        &\quad \cup \LRl{\sum_{i=1}^n \mathbbm{1}\LRl{\frac{y_k - \hat{\mu}_{\hat{\lambda}^{-i},k}(X_{n+1})}{\hat{\sigma}_{\hat{\lambda}^{-i},k}(X_{n+1})} < -S_{\hat{\lambda}^{-i}}(X_i,Y_i)} > (1-\alpha)(n+1)}\nonumber\\
        \Longrightarrow&\  \LRl{\sum_{i=1}^n \mathbbm{1}\{S_{\hat{\lambda}^{-i}}(X_{n+1},y) \leq S_{\hat{\lambda}^{-i}}(X_i,Y_i)\} + 1 \leq \alpha(n+1)}\nonumber\\
        \Longleftrightarrow&\  y \notin \widehat{\gU}^{\text{J-CROMS}}(X_{n+1}).
    \end{align*}
    Hence we know $\widehat{\gU}^{\text{J-CROMS}}(X_{n+1}) \subseteq \widehat{\gU}_{\rm{box}}^{\text{J-CROMS}}(X_{n+1})$ and $\sP\LRl{Y_{n+1} \in \widehat{\gU}_{\rm{box}}^{\text{J-CROMS}}(X_{n+1})} \geq \sP\LRl{Y_{n+1} \in \widehat{\gU}^{\text{J-CROMS}}(X_{n+1})} \geq 1-2\alpha.$
\end{proof}

\subsection{Proof of Theorem 3.9}\label{proof:J-CROMS-cond-coverage}

\begin{proof}
    Recall that the LOO selected model index is
    \begin{align}
        \hat{\lambda}_{-i} = \argmin_{\lambda \in \Lambda} \frac{1}{n-1}\sum_{\ell \in [n], \ell \neq i}\phi(Y_i, z_{\lambda}(X_i)).\nonumber
    \end{align}
    According to the proof of Theorem 3.2 (by replacing $\gD_n$ with $\gD_{-i}$), we can show that for any $i\in [n]$,
    \begin{align}
        &\left|\E\LRm{\phi(Y, z_{\hat{\lambda}_{-i}}(X; \hat{q}_{\hat{\lambda}_{-i}}^{-i}))} - \E\LRm{\phi(Y, z_{\lambda^*}(X; q_{\lambda^*}^o))}\right|\nonumber\\
        &\qquad\qquad\leq O\LRl{\LRs{B + \frac{L}{\mu}}\sqrt{\frac{\log n}{n}} + \frac{L}{\mu}\mathfrak{R}_n(\gF) + \mathfrak{R}_n(\gG)}.\nonumber
    \end{align}
    In addition, by Lemma \ref{lemma:quantile_estimation_VC} Assumption 3, we also have
    \begin{align}
        &\left|\E\LRm{\phi(Y, z_{\hat{\lambda}_{-i}}(X; \hat{q}_{\hat{\lambda}_{-i}}^{-i}))} - \E\LRm{\phi(Y, z_{\hat{\lambda}_{-i}}(X; q_{\hat{\lambda}_{-i}}^o))}\right| \nonumber\\
        &\qquad\qquad\leq L |\hat{q}_{\hat{\lambda}_{-i}}^{-i} - q_{\hat{\lambda}_{-i}}^o|\leq O\LRl{\frac{L}{\mu}\LRs{\sqrt{\frac{\log n}{n}} + \mathfrak{R}_n(\gF)}}.\nonumber
    \end{align}
    Combining the two relations above, we can show
    \begin{align}
        &\left|\E\LRm{\phi(Y, z_{\hat{\lambda}_{-i}}^o(X))} - \E\LRm{\phi(Y, z_{\lambda^*}^o(X))}\right| \nonumber\\
        &\qquad\qquad= 
        \left|\E\LRm{\phi(Y, z_{\hat{\lambda}_{-i}}(X; q_{\hat{\lambda}_{-i}}^o))} - \E\LRm{\phi(Y, z_{\lambda^*}(X; q_{\lambda^*}^o))}\right|\nonumber\\
        &\qquad\qquad\leq O\LRl{\LRs{B + \frac{L}{\mu}}\sqrt{\frac{\log n}{n}} + \frac{L}{\mu}\mathfrak{R}_n(\gF) + \mathfrak{R}_n(\gG)}.\nonumber
    \end{align}
    \paragraph*{Finite index set.}
    If the index set $\Lambda$ is finite, using the minimum risk gap condition in Theorem 3.4, we can have
    \begin{align}
        \sP\LRl{\forall i\in [n], \hat{\lambda}^{-i} = \lambda^*} \geq 1-n^{-c}.\nonumber
    \end{align}
    Under this event, we can write the prediction set of J-CROMS as
    \begin{align}\label{eq:JCROMS_set_star}
        \widehat{\gU}^{\JCROMS}(X_{n+1}) &=  \LRl{y\in \gY: \frac{\sum_{i=1}^n \mathbbm{1}\big\{S_{\hat{\lambda}^{-i}}(X_{n+1}, y)\leq S_{\hat{\lambda}^{-i}}(X_{i}, Y_{i}) \big\}+1}{n+1} > \alpha}\nonumber\\
        &=\LRl{y\in \gY:\frac{\sum_{i=1}^n \mathbbm{1}\big\{S_{\lambda^*}(X_{n+1}, y)\leq S_{\lambda^*}(X_{i}, Y_{i}) \big\}+1}{n+1} > \alpha}\nonumber\\
        &= \Big\{y\in \gY: S_{\lambda^*}(X_{n+1}, y) \leq Q_{(1-\alpha)(1+1/n)}\LRs{\{S_{\lambda^*}(X_i,Y_i)\}_{i=1}^n}\Big\}\nonumber\\
        &= \Big\{y\in \gY: S_{\lambda^*}(X_{n+1}, y) \leq \hat{q}_{\lambda^*}\Big\}.
    \end{align}
    By Lemma \ref{lemma:quantile_estimation_VC}, it holds that
    \begin{align}
        \sP\LRl{\left|\hat{q}_{\lambda^*} - q_{\lambda^*}^o\right| \leq \frac{c}{\mu}\LRs{\sqrt{\frac{\log n}{n}} + \mathfrak{R}_n(\gF)}} \geq 1-n^{-c}.\nonumber
    \end{align}
    Using Assumption 3, we can guarantee
    \begin{align}
        \left|\E\LRm{\phi(Y_{n+1}, \hat{z}^{\JCROMS}(X_{n+1}))} - v_{\Lambda}^*\right| \leq O\LRl{\frac{L}{\mu}\LRs{\sqrt{\frac{\log n}{n}} + \mathfrak{R}_n(\gF)}+\frac{B}{n}}.\nonumber
    \end{align}
    In addition, by \eqref{eq:JCROMS_set_star}, we also have
    \begin{align}
        \sP\LRl{Y_{n+1} \notin \widehat{\gU}^{\JCROMS}(X_{n+1})} &\leq n^{-c} + \sP\LRl{S_{\lambda^*}(X_{n+1}, Y_{n+1}) > Q_{1-\alpha}\LRs{\{S_{\lambda^*}(X_i,Y_i)\}_{i=1}^n}}\nonumber\\
        &= n^{-c} + \sP\LRl{S_{\lambda^*}(X_{n+1}, Y_{n+1}) > Q_{(1-\alpha)(1-n^{-1})}\LRs{\{S_{\lambda^*}(X_i,Y_i)\}_{i=1}^{n+1}}}\nonumber\\
        &\leq n^{-c} + \alpha + \frac{1-\alpha}{n},\nonumber
    \end{align}
    where the equality holds due to the inflation property of the sample quantile, see Lemma 2 in \citet{romano2019conformalized}.

    \paragraph*{Continuous index set.}
If the index set $\Lambda$ is continuous, using the minimum risk gap condition in Theorem 3.5, we can have
\begin{align}
    \sP\LRl{\forall i\in [n], \|\hat{\lambda}^{-i} - \lambda^*\| \leq \delta_n} \geq 1-n^{-c}.\nonumber
\end{align}
Under this event, it follows that
\begin{align*}
    \max_{i\in [n]}\left|S_{\hat{\lambda}^{-i}}(X_i,Y_i) - S_{\lambda^*}(X_i,Y_i)\right|&\leq \bar{L}_{\Lambda}\cdot \max_{i\in [n]}\|\hat{\lambda}^{-i} - \lambda^*\| \leq \bar{L}_{\Lambda}\delta_n,\\
    \sup_{y\in \gY}\max_{i\in [n]}\left|S_{\hat{\lambda}^{-i}}(X_{n+1},y) - S_{\lambda^*}(X_{n+1},y)\right| &\leq \bar{L}_{\Lambda}\cdot \max_{i\in [n]}\|\hat{\lambda}^{-i} - \lambda^*\| \leq \bar{L}_{\Lambda}\delta_n.
\end{align*}
Then we have
\begin{align}\label{eq:quantile_two_diff}
    \sup_{y\in \gY}\max_{i\in [n]}\left|S_{\hat{\lambda}^{-i}}(X_{i}, Y_{i})-S_{\hat{\lambda}^{-i}}(X_{n+1}, y) - \LRs{S_{\lambda^*}(X_{i}, Y_{i})-S_{\lambda^*}(X_{n+1}, y)}\right| \leq \bar{L}_{\Lambda}\delta_n.
\end{align}
For any $y\in \gU_{\lambda^*}(X_{n+1}; q_{\lambda^*}^o - e_n)$ with $e_n = O\LRs{\frac{1}{\mu} \LRs{\sqrt{\frac{\log n}{n}} + \mathfrak{R}_n(\gF) + \frac{1}{n+1}} + \bar{L}_{\Lambda}\delta_n}$, with probability at least $1-n^{-c}$, it holds that
\begin{align*}
    &S_{\lambda^*}(X_{n+1},y) \leq q_{\lambda^*}^o - e_n \\
    \stackrel{(i)}{\Longrightarrow}&\quad S_{\lambda^*}(X_{n+1},y)\leq Q_{1-\alpha}\LRs{\{S_{\lambda^*}(X_{i}, Y_{i})\}_{i=1}^n \cup \{S_{\lambda^*}(X_{n+1}, y)\}} + \bar{L}\delta_n\\
    \Longrightarrow&\quad 0 \leq Q_{1-\alpha}\LRs{\{S_{\lambda^*}(X_{i}, Y_{i})-S_{\lambda^*}(X_{n+1}, y)\}_{i=1}^n \cup \{0\}} + \bar{L}\delta_n\\
    \stackrel{(ii)}{\Longrightarrow}&\quad 0 \leq Q_{1-\alpha}\LRs{\{S_{\hat{\lambda}^{-i}}(X_{i}, Y_{i})-S_{\hat{\lambda}^{-i}}(X_{n+1}, y)\}_{i=1}^n \cup \{0\}}\\
    \Longleftrightarrow &\quad 0 \leq Q_{(1-\alpha)(1+1/n)}\LRs{\{S_{\hat{\lambda}^{-i}}(X_{i}, Y_{i})-S_{\hat{\lambda}^{-i}}(X_{n+1}, y)\}_{i=1}^n}\\
    \Longrightarrow &\quad \sum_{i=1}^n \mathbbm{1}\{S_{\hat{\lambda}^{-i}}(X_{n+1}, y) \leq S_{\hat{\lambda}^{-i}}(X_{i}, Y_{i})\} > n - (1-\alpha)(n+1)\\
    \stackrel{(iii)}{\Longleftrightarrow}&\quad y\in \widehat{\gU}^{\JCROMS}(X_{n+1}).
\end{align*}
where $(i)$ holds due to Lemma \ref{lemma:quantile_estimation_VC}; $(ii)$ follows from Lemma \ref{lemma:quantile_uniform_diff} and \eqref{eq:quantile_two_diff}; and $(iii)$ follows from $\frac{1}{n+1}\LRs{\sum_{i=1}^n \mathbbm{1}\{S_{\hat{\lambda}^{-i}}(X_{n+1}, y)\leq S_{\hat{\lambda}^{-i}}(X_{i}, Y_{i})\} + 1} > \alpha$ and the definition of J-CROMS. Hence we have
\begin{align*}
    \sP\LRl{Y_{n+1} \in \widehat{\gU}^{\JCROMS}(X_{n+1})} &\geq \sP\LRl{Y_{n+1}\in \gU_{\lambda^*}(X_{n+1}; q_{\lambda^*}^o - e_n)} - O(n^{-c}).
\end{align*}
If the density function $f_{\lambda}(\cdot)$ of $S_{\lambda}(X,Y)$ satisfies $\sup_{s\in \sR}f_{\lambda}(s) \leq \mu^+$, we can have $\sP\LRl{Y_{n+1} \in \widehat{\gU}^{\JCROMS}(X_{n+1})} \geq 1-\alpha - \mu^+e_n - O(n^{-c})$.
\end{proof}

\subsection{CV+ CROMS}\label{appen:CV+}
In this subsection, we adapt the CV+ method proposed in \citet{barber2021predictive} to our framework, which also enjoys a distribution-free and finite-sample $1-2\alpha-\frac{1-K/n}{K+1}$ coverage guarantee like the J-CROMS method. Here $K$ denotes the folds of cross-validation, hence $K = n$ reduces to the J-CROMS method.

Suppose we split the labeled data into $K$ disjoint subsets $\gI_1,\ldots,\gI_K$ with equal size $n/K$. Denote $\gU_{\lambda}^{-\gI_k}(\cdot) = \LRl{y\in \gY: S_{\lambda}(\cdot, y) \leq \hat{q}_{\lambda}^{-\gI_k}}$, where the threshold $\hat{q}_{\lambda}^{-\gI_k} = Q_{(1-\alpha)(1+(n-n/K)^{-1})}\LRs{\{S_{\lambda}(X_i,Y_i)\}_{i\in [n]\setminus\gI_k}}$. Then we define the leave-$\gI_k$-out model index by
\begin{align}\label{eq:CV_ERM}
    \hat{\lambda}^{-\gI_k} = \argmin_{\lambda \in \Lambda} \frac{1}{n-n/K}\sum_{i\in [n]\setminus\gI_k} \phi\LRs{Y_i, z_{\lambda}^{-\gI_k}(X_i)},
\end{align}
where $z_{\lambda}^{-\gI_k}(X_i) = \argmin_{z\in \gZ}\max_{c\in \gU_{\lambda}^{-\gI_k}(X_i)} \phi(c,z)$. The CV-CROMS prediction set is defined as
\begin{align}
    \widehat{\gU}^{\text{CV-CROMS}}(X_{n+1}) = \LRl{y\in \gY: \frac{\sum_{i=1}^n \mathbbm{1}\{S_{\hat{\lambda}^{-\gI_{k(i)}}}(X_{n+1}, y)\leq S_{\hat{\lambda}^{-\gI_{k(i)}}}(X_i,Y_i)\} +1}{n+1} > \alpha},\nonumber
\end{align}
where $k(i) = \{m\in [K]: i\in \gI_{m}\}$. In fact, the prediction set above is a special case of the cross-conformal prediction method in \citet{vovk2015cross} and \citet{vovk2018cross} without randomization.
If the candidate models are all box scores with $S_{\lambda}(x,y) = \|(y-\hat{\mu}_{\lambda}(x))/\hat{\sigma}_{\lambda}(x)\|_{\infty}$ for all $\lambda \in \Lambda$, we have a closed form $\widehat{\gU}_{\text{box}}^{\text{CV-CROMS}}(X_{n+1}) = \LRl{y\in \sR^p: \bar{c}^{\text{lo}} \leq y \leq \bar{c}^{\text{up}}}$,
where for $k\in [p]$,
\begin{align*}
    \bar{c}_{k}^{\text{up}} &= Q_{(1-\alpha)(1+1/n)}\left(\LRl{\hat{\mu}_{\hat{\lambda}^{-\gI_{k(i)}},k}(X_{n+1}) + \hat{\sigma}_{\hat{\lambda}^{-i},k}(X_{n+1})S_{\hat{\lambda}^{-\gI_{k(i)}}}(X_i,Y_i)}_{i=1}^{n}\right),\nonumber\\
    \bar{c}_{k}^{\text{lo}} &= -Q_{(1-\alpha)(1+1/n)}\left(\LRl{\hat{\sigma}_{\hat{\lambda}^{-i},k}(X_{n+1})S_{\hat{\lambda}^{-\gI_{k(i)}}}(X_i,Y_i) - \hat{\mu}_{\hat{\lambda}^{-\gI_{k(i)}},k}(X_{n+1})}_{i=1}^{n}\right).
\end{align*}
Notably, if $p=1$ (one-dimensional label), the box score is identical to the absolute residual score, and the prediction set $\widehat{\gU}_{\text{box}}^{\text{CV-CROMS}}(X_{n+1})$ recovers the Jackknife+ method in \citet{barber2021predictive}.

\begin{theorem}
    Suppose data $\{(X_i,Y_i)\}_{i=1}^{n+1}$ are i.i.d., we have
    \begin{align*}
        \sP\LRl{Y_{n+1} \in \widehat{\gU}^{\rm{CV-CROMS}}(X_{n+1})} \geq 1-2\alpha-\frac{1-K/n}{K+1}.
    \end{align*}
\end{theorem}
\begin{proof}
    The proof is essentially the same as that of Theorem 3.8. In addition to $\gI_k$ for $m\in[K]$, we define $\gI_{K+1} = \{n+1,\ldots,n+(n/K)\}$.
    Let $A_{i,j} = \mathbbm{1}\{S_{\hat{\lambda}^{-\gI_{k(i)}}}(X_{i}, Y_{i}) > S_{\hat{\lambda}^{-\gI_{k(j)}}}(X_{j}, Y_{j})\}$ for any $i,j\in [n+n/K]$ such that $k(i) \neq k(j)$, and $A_{i,i} = 0$ for any $k(i) = k(j)$. Define the set of strange points:
    \begin{align*}
        \gS = \LRl{i\in [n+n/K]: \sum_{j=1}^{n+1} A_{i,j} \geq (n+1)(1-\alpha)}.
    \end{align*}
    According to the proof of Theorem 4 in \citet{barber2021predictive}, we know $|\gS| \leq 2\alpha(n+n/K) + (1-2\alpha)(K-1)-1$, which means that $\sum_{i=1}^{n+n/K} \mathbbm{1}\{i\in \gS\} \leq 2\alpha(n+n/K) + (1-2\alpha)(K-1)-1$. In addition, using the exchangeability, we have
    \begin{align*}
        \sP\LRs{n+1 \in \gS} = \frac{\sum_{i=1}^{n+n/K}\sP(i\in \gS)}{n+n/K} = \frac{\E\LRm{\sum_{i=1}^{n+n/K} \mathbbm{1}\{i\in \gS\}}}{n+n/K} \leq 2\alpha + \frac{1-K/n}{K+1}.
    \end{align*}
    Since $m(n+1) = \cdots = m(n+n/K)$ and $A_{i,j} = 0$ if $k(i) = k(j)$, we have
    \begin{align*}
        \sum_{j=1}^{n+n/K} A_{n+1,j} &= \sum_{j=1}^{n+n/K} \mathbbm{1}\{S_{\hat{\lambda}^{-\gI_{k(j)}}}(X_{n+1}, Y_{n+1}) > S_{\hat{\lambda}^{-\gI_{k(j)}}}(X_{j}, Y_{j})\} \nonumber\\
        &= \sum_{j=1}^{n} \mathbbm{1}\{S_{\hat{\lambda}^{-\gI_{k(j)}}}(X_{n+1}, Y_{n+1}) > S_{\hat{\lambda}^{-\gI_{k(j)}}}(X_{j}, Y_{j})\}.
    \end{align*}
    By the definition of $\widehat{\gU}^{\text{CV-CROMS}}(X_{n+1})$, we also have
    \begin{align}
        \widehat{\gU}^{\text{CV-CROMS}}(X_{n+1}) &= \LRl{y\in \gY: \frac{\sum_{i=1}^n \mathbbm{1}\{S_{\hat{\lambda}^{-\gI_{k(i)}}}(X_i,Y_i) \leq S_{\hat{\lambda}^{-\gI_{k(i)}}}(X_{n+1}, y)\}+1}{n+1} > \alpha}\nonumber\\
        &= \LRl{y\in \gY: (n+1) - \sum_{j=1}^{n+n/K} A_{n+1,j} > (n+1)\alpha}.\nonumber
    \end{align}
    Hence we have showed that $\sP\{Y_{n+1}\notin \widehat{\gU}^{\text{J-CROMS}}(X_{n+1})\} = \sP\LRs{n+1 \in \gS} \leq 2\alpha + \frac{1-K/n}{K+1}$.
\end{proof}

\subsection{Standard leave-one-out method}
In this subsection, we consider using the standard leave-out-out (Jackknife) approach in \citet{barber2021predictive} and \citet{steinberger2023conditional} to perform the model selection and construct the prediction set. Here we follow the notations in Section 3.3.

Given the selected model indexes $\hat{\lambda}^{-i}$ for $i\in [n]$ in Section 3.3 and $\hat{\lambda}_{n}$ in Section 3.1, the standard leave-one-out (LOO) prediction set is
\begin{align}\label{eq:LOO-set}
    \widehat{\gU}^{\text{LOO}}(X_{n+1}) = \LRl{y\in \gY: S_{\hat{\lambda}_n}(X_{n+1}, y) \leq Q_{(1-\alpha)(1+1/n)}\LRs{\{S_{\hat{\lambda}^{-i}}(X_i,Y_i)\}_{i=1}^n}}.
\end{align}
However, as shown by Theorem 2 in \citet{barber2021predictive}, the prediction set $\widehat{\gU}^{\text{LOO}}(X_{n+1})$ does not have the distribution-free coverage guarantee anymore. And establishing the coverage bound requires the model stability conditions, see Theorem 5 in \citet{barber2021predictive}. The next theorem shows the coverage bound of the LOO method under the same conditions of Theorem 3.9, and the proof is the same as that in Section \ref{proof:J-CROMS-cond-coverage}.

\begin{theorem}
    For a finite index set $\Lambda$, under the same conditions of Theorem 3.4, the LOO prediction set in \eqref{eq:LOO-set} satisfies $1-\alpha - O(n^{-1})$ level of marginal robustness and the decision risk satisfies that $\left|\E\LRm{\phi(Y, \hat{z}^{\JCROMS}(X))} - v_{\Lambda}^*\right| \leq O\LRl{\frac{L}{\mu} \sqrt{\frac{\log (n \vee |\Lambda|)}{n}}+\frac{B}{n}}$.
\end{theorem}

\subsection{Numerical comparison results}

Table \ref{table:classification_performance_appen} and \ref{table:regression_performance_appen} present the performance of the compared methods along with their corresponding running times in the classification task of Section 5.1.1 and the regression task of Section 5.1.2. Both {F‑CROMS} and {J‑CROMS} empirically maintain valid coverage and robustness guarantees. In the classification task, {F‑CROMS} runs faster than {J‑CROMS} since $|\gY|$ is significantly smaller than the sample size $n$. In the regression task, J-CROMS and CV-CROMS have the best decision performance, while requiring significantly less computational time than F-CROMS.

    \begin{table}[H]
    \centering
    \small
    \setlength{\tabcolsep}{4pt}
    \caption{The evaluation metrics and running time (seconds) on 100 test points with 95\% asymptotic standard errors in parentheses under the classification task in Section 5.1.1. The scenario is $n=200$, $|\Lambda|=20$. }
    \label{table:classification_performance_appen}
    \resizebox{\textwidth}{!}{
    \begin{tabular}{@{}clccccc@{}}
    \toprule
    $\alpha$ & \textbf{Method} & \textbf{Avg. Loss} & \textbf{Marg. Miscov.} & \textbf{Marg. Misrob.} & \textbf{Time} \\
    \midrule
    \multirow{6}{*}{0.10}
     &LOO & 3.778 (0.122) & 0.110 (0.012) & 0.008 (0.060) & 17.694 (0.285) \\
     &E-CROMS & \textbf{3.646} (0.090) & {0.128} (0.007) & 0.069 (0.006) & 4.579 (0.027) \\
     &F-CROMS & \textbf{3.759} (0.098) & 0.100 (0.010) & 0.051 (0.006) & {43.197} (0.994) \\
     &J-CROMS & \textbf{3.830} (0.106) & 0.094 (0.010) & 0.048 (0.007) & {90.956} (0.680) \\
     &CV-CROMS($K=5$) & 3.956 (0.095) & 0.052 (0.006) & 0.025 (0.004) & 103.471 (1.884) \\
     &CV-CROMS($K=10$) & 3.928 (0.098) & 0.062 (0.007) & 0.029 (0.005) & 124.830 (2.185) \\
    \midrule
    \multirow{6}{*}{0.20}
     &LOO & 2.855 (0.125) & 0.206 (0.013) & 0.168 (0.014) & 17.555 (0.272) \\
     & E-CROMS & \textbf{2.629} (0.067) & {0.227} (0.009) & 0.193 (0.010) & 4.593 (0.039) \\
     & F-CROMS & \textbf{2.713} (0.081) & 0.197 (0.013) & 0.160 (0.013) & {32.313} (0.630) \\
     & J-CROMS & 2.794 (0.101) & 0.194 (0.012) & 0.157 (0.013) & {93.099} (1.159) \\
     &CV-CROMS($K=5$) & \textbf{2.771} (0.075) & 0.136 (0.011) & 0.101 (0.010) & 103.269 (1.932) \\
     &CV-CROMS($K=10$) & 2.792 (0.073) & 0.149 (0.011) & 0.114 (0.012) & 124.445 (2.241) \\
    \arrayrulecolor{black}\bottomrule
    \end{tabular}}
    \end{table}

    \begin{table}[H]
    \centering
    \small
    \setlength{\tabcolsep}{4pt}
    \caption{The evaluation metrics and running time (seconds) on 100 test points with 95\% asymptotic standard errors under the regression task in Section 5.1.2. The scenario is $n=150$, $|\Lambda|=25$.}
    \label{table:regression_performance_appen}
    \resizebox{\textwidth}{!}{
    \begin{tabular}{@{}clccccc@{}}
    \toprule
    $\alpha$ & \textbf{Method} & \textbf{Avg. Loss} & \textbf{Marg. Miscov.} & \textbf{Marg. Misrob.} & \textbf{Time} \\
    \midrule
    \multirow{6}{*}{0.10} 
    & LOO & -0.744 (0.066) & 0.095 (0.007) & 0.029 (0.005) & 51.663 (0.229) \\
    & E-CROMS & {-0.749} (0.067) & {0.100} (0.007) & 0.031 (0.003) & 14.147 (0.023) \\
    & F-CROMS & {-0.742} (0.067) & 0.097 (0.007) & 0.022 (0.003) & {2386.55} (110.139) \\
    & J-CROMS($\alpha$) & \textbf{-0.774} (0.067) & 0.087 (0.008) & 0.025 (0.003) & {79.686} (0.259) \\
    & J-CROMS($\alpha/2$) & -0.724 (0.063) & 0.037 (0.005) & 0.011 (0.002) & 79.889 (0.283) \\
    & CV-CROMS($K = 5$) & \textbf{-0.860} (0.067) & 0.059 (0.009) & 0.016 (0.003) & 95.909 (0.359) \\
    & CV-CROMS($K = 10$) & \textbf{-0.833} (0.066) & 0.068 (0.008) & 0.018 (0.003) & 162.701 (0.719) \\
    \cmidrule(lr){1-6}
    \multirow{6}{*}{0.20} 
    & LOO & -0.800 (0.064) & 0.199 (0.011) & 0.060 (0.005) & 52.248 (0.329) \\
    & E-CROMS & {-0.801} (0.064) & {0.210} (0.010) & 0.065 (0.005) & 14.266 (0.045) \\
    & F-CROMS & {-0.790} (0.066) & 0.198 (0.010) & 0.052 (0.005) & {1173.70} (50.149) \\
    & J-CROMS($\alpha$) & \textbf{-0.816} (0.066) & 0.188 (0.012) & 0.056 (0.005) & {80.159} (0.259) \\
    & J-CROMS($\alpha/2$) & -0.774 (0.067) & 0.087 (0.008) & 0.025 (0.003) & 80.378 (0.407) \\
    & CV-CROMS($K = 5$) & \textbf{-0.914} (0.063) & 0.136 (0.013) & 0.034 (0.005) & 96.853 (0.549) \\
    & CV-CROMS($K = 10$) & \textbf{-0.884} (0.063) & 0.152 (0.013) & 0.042 (0.006) & 164.063 (1.013) \\
    \bottomrule
    \end{tabular}}
    \end{table}

\section{Theoretical results of CROiMS}
In this section, we define $\mathfrak{B}_n(X_{n+1}) = \sum_{j=1}^n H(X_j,X_{n+1})$ and the function classes: $\gF = \{\mathbbm{1}\{S_{\lambda}(x,y) \leq q\}: \lambda \in \Lambda, q\in \sR\}$ and $\gG^{\prime} = \{\phi(y, z_{\lambda}(x, q_{\lambda}^o(x))): \lambda \in \Lambda\}$. Let $\gX_{n+1} = \{X_i\}_{i=1}^{n+1}$, then we define the following two weighted Rademacher complexities:
\begin{itemize}
    \item $\widehat{\mathfrak{R}}_n(\gF) = \E\LRm{\sup_{f\in \gF} \left|\sum_{i=1}^n w_i(X_{n+1}) \xi_i f(X_i, Y_i)\right| \mid \gX_{n+1}}$;

    \item $\widehat{\mathfrak{R}}_n(\gG^{\prime}) = \E\LRm{\sup_{g\in \gG^{\prime}} \left|\sum_{i=1}^n w_i(X_{n+1}) \xi_i g(X_i, Y_i)\right| \mid \gX_{n+1}}$.
\end{itemize}

\subsection{Main lemmas}
\begin{lemma}\label{lemma:Bn_lower_bound}
    Let $\mathfrak{B}_n(X_j) = \sum_{i=1}^n H(X_i,X_j)$ for $j\in [n+1]$. Under Assumption 4, for any $c > 1$, if $\rho V n h_n^d > 2 \sqrt{\frac{2c\log n}{n}}$ where $V$ is the volume of unit ball in $\sR^d$, we have 
    $$\sP\LRl{\mathfrak{B}_n(X_{n+1}) > \frac{\rho V n h_n^d}{2e} \mid X_{n+1}} \geq 1 - n^{-c}.$$
    And for $j\in [n]$, we have
    \begin{align*}
        \sP\LRl{\mathfrak{B}_n(X_{j}) > 1+\frac{\rho V (n-1) h_n^d}{2e} \mid X_{j}} \geq 1 - (n-1)^{-c}.
    \end{align*}
\end{lemma}

\begin{lemma}\label{lemma:concentration_cover_VC}
    Under Assumption 4, if the conditional distribution  $S_{\lambda}(X_i,Y_i) \mid X_i$ is continuous, it holds that for any $j\in [n+1]$,
    \begin{align}
        \sP\LRl{M_{\gF} > 2 \widehat{\mathfrak{R}}_n(\gF) + \sqrt{\frac{c\log n}{\mathfrak{B}_n(X_j)}} \mid \gX_{n+1}} \leq 2n^{-c},\nonumber
    \end{align}
    where $M_{\gF} = \sup_{\lambda\in \Lambda, q\in \sR}\left|\sum_{i=1}^n w_i(X_j) \LRs{\mathbbm{1}\{S_{\lambda}(X_i,Y_i) \leq q\} - \sP\LRl{S_{\lambda}(X_i,Y_i) \leq q \mid X_i}}\right|$.
\end{lemma}

\begin{lemma}\label{lemma:cond_quantile_estimation_VC}
    Under Assumption 4, we have for any $j\in [n+1]$,
    \begin{align}
        \sP\LRl{\sup_{\lambda\in \Lambda} |\hat{q}_{\lambda}(X_{j}) - q_{\lambda}^{o}(X_{j})| \leq \frac{b_n}{\bar{\mu}} \mid \gX_{n+1}}\geq 1-n^{-c},\nonumber
    \end{align}
    where $b_n = 2\widehat{\mathfrak{R}}_n(\gF) + \bar{\tau}\LRs{ e h_n \log(h_n^{-d}) + \frac{n h_n\log(h_n^{-d})h_n^d}{\mathfrak{B}_n(X_{j})}} + \sqrt{\frac{c\log n}{\mathfrak{B}_n(X_{j})}}$.
\end{lemma}

\begin{lemma}\label{lemma:uniform_weight_concentration_VC}
    Under the conditions of Theorem 5.1, we have
    \begin{align*}
        \sP\LRl{M_{\gG^{\prime}} > 2 \widehat{\mathfrak{R}}_n(\gG^{\prime}) + \sqrt{\frac{c\log n}{\mathfrak{B}_n(X_j)}} \mid \gX_{n+1}} \leq n^{-c},
    \end{align*}
    where $M_{\gG^{\prime}} = \sup_{\lambda\in \Lambda, q\in \sR}\left|\sum_{i=1}^n w_i(X_{n+1})\Big(\Phi_{\lambda}(X_i,Y_i) - \E[\Phi_{\lambda}(X_i,Y_i)\mid X_i]\Big)\right|$.
\end{lemma}

\begin{lemma}\label{lemma:uniform_cover_bias}
    Under Assumption 4, conditioning on $\{X_i\}_{i=1}^{n+1}$ we have for any $j\in [n+1]$
    \begin{align}
        &\sup_{\lambda\in \Lambda}\sup_{q\in \sR}\left|\sum_{i=1}^n w_i(X_{j}) \Big[\sP(S_{\lambda}(X_i,Y_i)\leq q \mid X_i) - \sP(S_{\lambda}(X_{j},Y_{j})\leq q \mid X_{j})\Big]\right|\nonumber\\
        &\qquad \leq \LRs{e h_n \log(h_n^{-d}) + \frac{n h_n\log(h_n^{-d})h_n^d}{\mathfrak{B}_n(X_{j})}}\cdot \bar{\tau}.\nonumber
    \end{align}
\end{lemma}

\begin{lemma}\label{lemma:uniform_error_bound}
    Under Assumption 5, conditioning on $\{X_i\}_{i=1}^{n+1}$, we have
        \begin{align}
        &\sup_{\lambda\in \Lambda}\left|\sum_{i=1}^n w_i(X_{n+1})\Big( \E[\Phi_{\lambda}(X_i,Y_i)\mid X_i] - \E[\Phi_{\lambda}(X_{n+1},Y_{n+1})\mid X_{n+1}]\Big)\right|\nonumber\\
        &\qquad \leq  \LRs{e h_n \log(h_n^{-d}) + \frac{n h_n\log(h_n^{-d})h_n^d}{\mathfrak{B}_n(X_{n+1})}}\cdot \tau.\nonumber
        \end{align}
\end{lemma}

\subsection{Proof of conditional robustness}

\begin{proof}
    By the definition of LCP set, we know 
    \begin{align}
        \mathbbm{1}\LRl{Y_{n+1}\in \widehat{\gU}^{\mathrm{CROiMS}}(X_{n+1})} 
        = \mathbbm{1}\LRl{S_{\hat{\lambda}(X_{n+1})}(X_{n+1},Y_{n+1})\leq \hat{q}_{\hat{\lambda}(X_{n+1})}(X_{n+1})}.\nonumber
    \end{align}
    According to the definition of weighted quantile, we know
    \begin{align}
        \sum_{i=1}^{n}w_i(X_{n+1}) \mathbbm{1}\LRl{S_{\hat{\lambda}(X_{n+1})}(X_i,Y_i) \leq \hat{q}_{\hat{\lambda}(X_{n+1})}(X_{n+1})} \geq 1-\alpha.\nonumber
    \end{align}
    Combining the two relations above, conditioning on the labeled data $\gD_n = \{(X_i,Y_i)\}_{i=1}^n$ and test data $X_{n+1}$, we get
    \begin{align}
        &1-\alpha - \sP\LRl{Y_{n+1}\in \widehat{\gU}^{\mathrm{CROiMS}}(X_{n+1})\mid \gD_n,X_{n+1}}\nonumber\\
        &\leq \sum_{i=1}^{n}w_i(X_{n+1})\left[\mathbbm{1}\LRl{S_{\hat{\lambda}(X_{n+1})}(X_i,Y_i) \leq \hat{q}_{\hat{\lambda}(X_{n+1})}(X_{n+1})}\right.\nonumber\\
        &\qquad\qquad\qquad \left. - \sP\LRl{S_{\hat{\lambda}(X_{n+1})}(X_{n+1},Y_{n+1})\leq \hat{q}_{\hat{\lambda}(X_{n+1})}(X_{n+1})\mid \gD_n, X_{n+1}}\right]\nonumber\\
        &\leq \sup_{\lambda\in \Lambda}\sum_{i=1}^{n}w_i(X_{n+1})\LRm{\mathbbm{1}\LRl{S_{\lambda}(X_i,Y_i) \leq \hat{q}_{\lambda}(X_{n+1})} - \sP\LRl{S_{\lambda}(X_{n+1},Y_{n+1})\leq \hat{q}_{\lambda}(X_{n+1})\mid \gD_n,X_{n+1}}}\nonumber\\
        &\leq \sup_{\lambda\in \Lambda,q\in \sR}\sum_{i=1}^{n}w_i(X_{n+1})\LRm{\mathbbm{1}\LRl{S_{\lambda}(X_i,Y_i) \leq q} - \sP\LRl{S_{\lambda}(X_{n+1},Y_{n+1})\leq q\mid \gD_n,X_{n+1}}}\nonumber\\
        &\leq \underbrace{\sup_{\lambda\in \Lambda,q\in \sR}\left|\sum_{i=1}^{n}w_i(X_{n+1})\LRm{\mathbbm{1}\LRl{S_{\lambda}(X_i,Y_i) \leq q} - \sP\LRl{S_{\lambda}(X_i,Y_i) \leq q\mid X_i}}\right|}_{(\mathrm{I})}\nonumber\\
        &+\underbrace{\sup_{\lambda\in \Lambda,q\in \sR}\left|\sum_{i=1}^{n}w_i(X_{n+1})\LRm{\sP\LRl{S_{\lambda}(X_i,Y_i) \leq q\mid X_i}-\sP\LRl{S_{\lambda}(X_{n+1},Y_{n+1}) \leq q\mid X_{n+1}}}\right|}_{(\mathrm{II})},\label{eq:cond_cover_error}
    \end{align}
    where the second and third inequalities hold due to both $\hat{\lambda}(X_{n+1})$ and $\hat{q}_{\lambda}(X_{n+1})$ are determined by $\gD_n$ and $X_{n+1}$.
    For the term $(\mathrm{I})$, applying the symmetrization technique and recalling the definition of $\gF$, we have
    \begin{align}
        \E[(\mathrm{I}) \mid \gX_{n+1}] &\leq 2\E\LRm{\sup_{\lambda\in \Lambda, q\in \sR} \left|\sum_{i=1}^n w_i(X_{n+1}) \xi_i \mathbbm{1}\LRl{S_{\lambda}(X_i,Y_i) \leq q}\right| \mid \gX_{n+1}}\nonumber\\
        &= 2\E\LRm{\sup_{f\in \gF} \left|\sum_{i=1}^n w_i(X_{n+1}) \xi_i f(X_i, Y_i)\right| \mid \gX_{n+1}}\nonumber\\
        &= 2 \widehat{\mathfrak{R}}_n(\gF).\nonumber
    \end{align}
    Applying Lemma \ref{lemma:uniform_cover_bias}, we have: if $\rho V n h_n^d > 2 \sqrt{\frac{2c\log n}{n}}$,
    \begin{align}
        \sP\LRl{(\mathrm{II}) \leq \bar{\tau} e h_n \log(h_n^{-d}) + \frac{2e\bar{\tau} h_n\log(h_n^{-d})}{\rho V}\mid X_{n+1}} \geq 1 - n^{-c}.\nonumber
    \end{align}
    Plugging two concentration results into \eqref{eq:cond_cover_error}, and taking expectation over the randomness from $\gD_n$, we can guarantee
    \begin{align}
        \sP\LRl{Y_{n+1}\in \widehat{\gU}(X_{n+1})\mid X_{n+1}} & \geq 1-\alpha - \E\LRm{(\mathrm{I}) + (\mathrm{II}) \mid X_{n+1}}\nonumber\\
        &\geq 1-\alpha - 2\E\LRm{\widehat{\mathfrak{R}}_n(\gF) \mid X_{n+1}}\nonumber\\
        &\qquad - O\LRs{\bar{\tau} h_n \log(h_n^{-d}) + \frac{\bar{\tau} h_n\log(h_n^{-d})}{\rho V}}.\nonumber
    \end{align}
    Combining the bounds above, we can finish the proof.
\end{proof}

\subsection{Proof of conditional optimality}

\begin{proof}
    Given $(x,y)$, recall the definitions $\widehat{\Phi}_{\lambda}(x, y) = \phi(y, z_{\lambda}(x; \hat{q}_{\lambda}(x)))$ and $\Phi_{\lambda}(x, y) = \phi(y, z_{\lambda}(x; q_{\lambda}^{o}(x)))$.
    Let $b_n = 2\widehat{\mathfrak{R}}_n(\gF) + \bar{\tau}\LRs{ e h_n \log(h_n^{-d}) + \frac{n h_n\log(h_n^{-d})h_n^d}{\mathfrak{B}_n(X_{j})}} + \sqrt{\frac{c\log n}{\mathfrak{B}_n(X_{j})}}$ and define the event
    \begin{align*}
        \widehat{\gE} = \LRl{\sup_{\lambda \in \Lambda} |\hat{q}_{\lambda}(X_{n+1}) - q_{\lambda}^o(X_{n+1})| \leq \frac{b_n}{\bar{\mu}}}.
    \end{align*}
    According to Lemma \ref{lemma:cond_quantile_estimation_VC}, we know $\sP(\widehat{\gE} \mid X_{n+1}) \geq 1 - n^{-c}$. Under the event $\widehat{\gE}$, by Assumption 5, we have
    \begin{align}
        \sup_{\lambda \in \Lambda}\max_{i\in [n+1]}\left|\widehat{\Phi}_{\lambda}(X_i,Y_i) - \Phi_{\lambda}(X_i, Y_i)\right|
        & = \sup_{\lambda \in \Lambda}\max_{i\in [n+1]}\left|\phi(Y_{i}, z_{\lambda}(X_{i}; \hat{q}_{\lambda}(X_{i}))) - \phi(Y_{i}, z_{\lambda}(X_{i}; q_{\lambda}^{o}(X_{i})))\right| \nonumber\\
        & \leq L \sup_{\lambda \in \Lambda} \max_{i\in [n+1]}|\hat{q}_{\lambda}(X_{i}) - q_{\lambda}^{o}(X_{i})| \leq \frac{L  b_n}{\bar{\mu}}.\label{eq:intermediate_cond_expectation}
    \end{align}
    By optimality of $\lambda^*(X_{n+1})$, we also have the lower bound
    \begin{align}\label{eq:intermediate_cond_expectation_lower}
        &\E[\widehat{\Phi}_{\hat{\lambda}(X_{n+1})}(X_{n+1},Y_{n+1}) \mid X_{n+1}] - v_{\Lambda}^*(X_{n+1}) \nonumber\\
        & = \E[\widehat{\Phi}_{\hat{\lambda}(X_{n+1})}(X_{n+1},Y_{n+1}) \mid X_{n+1}] - \E[\Phi_{\lambda^*(X_{n+1})}(X_{n+1},Y_{n+1})\mid X_{n+1}]\nonumber\\
        & = \E[\widehat{\Phi}_{\hat{\lambda}(X_{n+1})}(X_{n+1},Y_{n+1}) \mid X_{n+1}] - \E[\Phi_{\hat{\lambda}(X_{n+1})}(X_{n+1},Y_{n+1})\mid X_{n+1}]\nonumber\\
        &\qquad + \underbrace{\E[\Phi_{\hat{\lambda}(X_{n+1})}(X_{n+1},Y_{n+1})\mid X_{n+1}] - \E[\Phi_{\lambda^*(X_{n+1})}(X_{n+1},Y_{n+1})\mid X_{n+1}]}_{\geq 0}\nonumber\\
        & \geq - \E\LRm{\left|\widehat{\Phi}_{\hat{\lambda}(X_{n+1})}(X_{n+1},Y_{n+1}) -\Phi_{\hat{\lambda}(X_{n+1})}(X_{n+1},Y_{n+1})\right| \mid X_{n+1}} \nonumber\\
        & \geq - \E\LRm{\E\LRm{(\mathbbm{1}_{\widehat{\gE}} + \mathbbm{1}_{\widehat{\gE}^c})\sup_{\lambda \in \Lambda}\left|\widehat{\Phi}_{\lambda}(X_{n+1},Y_{n+1}) -\Phi_{\lambda}(X_{n+1},Y_{n+1})\right| \mid \gD_n, X_{n+1}}\mid X_{n+1}} \nonumber\\
        & \geq -\frac{L  b_n}{\bar{\mu}} - \frac{2B}{n^{c}}.
    \end{align}
    Then using the optimality of $\hat{\lambda}(X_{n+1})$, we have
    \begin{align}\label{eq:intermediate_cond_expectation_upper}
        &\E\LRm{\widehat{\Phi}_{\hat{\lambda}(X_{n+1})}(X_{n+1},Y_{n+1}) \mid \gD_n, X_{n+1}} - v_{\Lambda}^*(X_{n+1})\nonumber\\
        = & \underbrace{\E\LRm{\widehat{\Phi}_{\hat{\lambda}(X_{n+1})}(X_{n+1},Y_{n+1})\mid \gD_n, X_{n+1}} - \sum_{i=1}^n w_i(X_{n+1}) \widehat{\Phi}_{\hat{\lambda}(X_{n+1})}(X_i,Y_i)}_{(\mathrm{I})}\nonumber\\
        +& \underbrace{\sum_{i=1}^n w_i(X_{n+1}) \widehat{\Phi}_{\hat{\lambda}(X_{n+1})}(X_i,Y_i) - \sum_{i=1}^n w_i(X_{n+1}) \widehat{\Phi}_{\lambda^*(X_{n+1})}(X_i,Y_i)}_{\leq 0}\nonumber\\
        +& \underbrace{\sum_{i=1}^n w_i(X_{n+1}) \left[\widehat{\Phi}_{\lambda^*(X_{n+1})}(X_i,Y_i) - \Phi_{\lambda^*(X_{n+1})}(X_i,Y_i)\right]}_{(\mathrm{II})}\nonumber\\
        +& \underbrace{\sum_{i=1}^n w_i(X_{n+1}) \Phi_{\lambda^*(X_{n+1})}(X_i,Y_i) - \E[\Phi_{\lambda^*(X_{n+1})}(X_{n+1},Y_{n+1})\mid X_{n+1}]}_{(\mathrm{III})}\nonumber\\
        \leq & (\mathrm{I}) + (\mathrm{II}) + (\mathrm{III}).
    \end{align}
    For the first term, using \eqref{eq:intermediate_cond_expectation}, we have
    \begin{align}\label{eq:error_I}
        \E[(\mathrm{I}) &\mid \gD_n, X_{n+1}]
        \leq \left|\E\LRm{\widehat{\Phi}_{\hat{\lambda}(X_{n+1})}(X_{n+1},Y_{n+1}) - \Phi_{\hat{\lambda}(X_{n+1})}(X_{n+1},Y_{n+1}) \mid \gD_n, X_{n+1}}\right|\nonumber\\
        &\quad +\sum_{i=1}^n w_i(X_{n+1})\left|\widehat{\Phi}_{\hat{\lambda}(X_{n+1})}(X_i,Y_i) - \Phi_{\hat{\lambda}(X_{n+1})}(X_i,Y_i)\right|\nonumber\\
        &\quad + \left|\sum_{i=1}^n w_i(X_{n+1}) \Phi_{\hat{\lambda}(X_{n+1})}(X_i,Y_i) - \E\LRm{\Phi_{\hat{\lambda}(X_{n+1})}(X_{n+1},Y_{n+1}) \mid \gD_n, X_{n+1}}\right|\nonumber\\
        &\leq \frac{2L  b_n}{\bar{\mu}} + \sup_{\lambda \in \Lambda} \left|\sum_{i=1}^n w_i(X_{n+1}) \Phi_{\lambda}(X_i,Y_i) - \E[\Phi_{\lambda}(X_{n+1},Y_{n+1})\mid X_{n+1}]\right|\nonumber\\
        &\leq \frac{2 L  b_n}{\bar{\mu}} + \underbrace{\sup_{\lambda\in \Lambda} \left|\sum_{i=1}^n w_i(X_{n+1}) \Big(\Phi_{\lambda}(X_i,Y_i) - \E[\Phi_{\lambda}(X_i,Y_i)\mid X_i]\Big)\right|}_{(\mathrm{I}.1)}\nonumber\\
        &\quad +  \underbrace{\sup_{\lambda\in \Lambda} \sum_{i=1}^n w_i(X_{n+1})\Big|\E[\Phi_{\lambda}(X_{n+1},Y_{n+1})\mid X_{n+1}] - \E[\Phi_{\lambda}(X_i,Y_i)\mid X_i]\Big|}_{(\mathrm{I}.2)}.
    \end{align}
    Applying Lemma \ref{lemma:uniform_weight_concentration_VC}, we have
    \begin{align}
        \sP\LRl{(\mathrm{I}.1) \leq 2 \widehat{\mathfrak{R}}_n(\gG^{\prime}) + \sqrt{\frac{c\log n}{\mathfrak{B}_n(X_j)}} \mid \gX_{n+1}} \geq 1 - n^{-c}.\nonumber
    \end{align}
    Applying Lemma \ref{lemma:uniform_error_bound}, we can get
    \begin{align}
        (\mathrm{I}.2) \leq  \LRs{e h_n \log(h_n^{-d}) + \frac{n h_n\log(h_n^{-d})h_n^d}{\mathfrak{B}_n(X_{n+1})}} \tau.\nonumber
    \end{align}
    Plugging two inequalities into \eqref{eq:error_I} and taking the expectation over $\gD_n$, we have
    \begin{align}\label{eq:term_I_hp_bound}
        &\E[(\mathrm{I}) \mid X_{n+1}] 
        \leq \frac{2 L  b_n}{\bar{\mu}} + \frac{2B}{n^{c}} +2 \E\LRm{\widehat{\mathfrak{R}}_n(\gG^{\prime}) \mid X_{n+1}}  + \sqrt{\frac{c\log n}{\mathfrak{B}_n(X_j)}}\nonumber\\
        &\qquad\qquad + \tau\LRs{e h_n \log(h_n^{-d}) + \frac{n h_n\log(h_n^{-d})h_n^d}{\mathfrak{B}_n(X_{n+1})}} \nonumber\\
        &\qquad = 2\E\LRm{\frac{2L}{\bar{\mu}}\widehat{\mathfrak{R}}_n(\gF)+ \widehat{\mathfrak{R}}_n(\gG^{\prime}) \mid X_{n+1}}\nonumber\\
        &\qquad\qquad+ O\LRl{\frac{B}{n^c} + \LRs{\frac{L}{\bar{\mu}}+1}\sqrt{\frac{\log n}{\rho n h_n^d}} + \LRs{\frac{L }{\bar{\mu}}\bar{\tau} + \tau}\frac{h_n\log(h_n^{-d})}{\rho}},\nonumber
    \end{align}
    where we also used Lemma \ref{lemma:Bn_lower_bound} and $|(\mathrm{I})| \leq 2B$. Notice that $\{\widehat{\Phi}_{\lambda^*(X_{n+1})}(X_i,Y_i)\}_{i=1}^n$ are i.i.d. random variables bounded by $B$ conditioning on $\{X_i\}_{i=1}^{n+1}$, applying weighted Bernstein's inequality, we have with probability at least $1-2n^{-c}$,
    \begin{align}
        \left|\sum_{i=1}^n w_i(X_{n+1}) \Big(\Phi_{\lambda^*(X_{n+1})}(X_i,Y_i) - \E[\Phi_{\lambda^*(X_{n+1})}(X_{i},Y_{i})\mid X_i, X_{n+1}]\Big)\right| \leq \sqrt{\frac{c\log n}{2 \mathfrak{B}_n(X_{n+1})}}.\nonumber
    \end{align}
    Together with Lemma \ref{lemma:Bn_lower_bound}, we have with probability at least $1-3n^{-c}$,
    \begin{align}
        &\E\LRm{(\mathrm{III}) \mid \gD_n, X_{n+1}} \leq \left|\sum_{i=1}^n w_i(X_{n+1}) \Phi_{\lambda^*(X_{n+1})}(X_i,Y_i) - \E[\Phi_{\lambda^*(X_{n+1})}(X_{n+1},Y_{n+1})\mid X_{n+1}]\right|\nonumber\\
        &\qquad\leq \left|\sum_{i=1}^n w_i(X_{n+1}) \Big(\Phi_{\lambda^*(X_{n+1})}(X_i,Y_i) - \E[\Phi_{\lambda^*(X_{n+1})}(X_{i},Y_{i})\mid X_i, X_{n+1}]\Big)\right|\nonumber\\
        &\qquad +\left|\sum_{i=1}^n w_i(X_{n+1})\Big(\E[\Phi_{\lambda^*(X_{n+1})}(X_{i},Y_{i})\mid X_i, X_{n+1}] - \E[\Phi_{\lambda^*(X_{n+1})}(X_{n+1},Y_{n+1})\mid X_{n+1}]\Big)\right|\nonumber\\
        &\qquad \leq \sqrt{\frac{c\log n}{2 \mathfrak{B}_n(X_{n+1})}} + \tau\LRs{e h_n \log(h_n^{-d}) + \frac{n h_n\log(h_n^{-d})h_n^d}{\mathfrak{B}_n(X_{n+1})}}\nonumber\\
        &\qquad = O\LRs{\sqrt{\frac{\log n}{\rho n h_n^d}} + \frac{\tau}{\rho} h_n\log(h_n^{-d})}.\nonumber
    \end{align}
    It follows that
    \begin{align}
        \E\LRm{(\mathrm{III}) \mid X_{n+1}} \leq O\LRl{\frac{B}{n^c} + \sqrt{\frac{\log n}{\rho n h_n^d}} + \frac{\tau}{\rho}h_n\log(h_n^{-d})}.\nonumber
    \end{align}
    By similar arguments in upper bounding $\E[(\mathrm{I}) \mid X_{n+1}]$, we can show the same bound in \eqref{eq:term_I_hp_bound} for $\E[(\mathrm{III}) \mid X_{n+1}]$. For the term $(\mathrm{II})$ in \eqref{eq:intermediate_cond_expectation_upper}, we have
    \begin{align}
        \E[(\mathrm{II}) \mid X_{n+1}] &= \E\LRm{\sum_{i=1}^n w_i(X_{n+1}) \left[\widehat{\Phi}_{\lambda^*(X_{n+1})}(X_i,Y_i) - \Phi_{\lambda^*(X_{n+1})}(X_i,Y_i)\right]\mid X_{n+1}}\nonumber\\
        &\leq \E\LRm{\sum_{i=1}^n w_i(X_{n+1}) \cdot \sup_{\lambda \in \Lambda}\left|\widehat{\Phi}_{\lambda}(X_i,Y_i) - \Phi_{\lambda}(X_i, Y_i)\right| \mid X_{n+1}}\nonumber\\
        &\leq \frac{L  b_n}{\bar{\mu}} + \frac{2B}{n^{c}}.\nonumber
    \end{align}
    Plugging the bounds for $(\mathrm{I})$, $(\mathrm{II})$ and $(\mathrm{III})$ into \eqref{eq:intermediate_cond_expectation_upper}, together with the bounded loss assumption, we can show
    \begin{align}
        &\E\LRm{\widehat{\Phi}_{\hat{\lambda}(X_{n+1})}(X_{n+1},Y_{n+1}) \mid X_{n+1}} - v_{\Lambda}^*(X_{n+1})\nonumber\\
        &\qquad \leq \E\LRm{\frac{4L}{\bar{\mu}}\widehat{\mathfrak{R}}_n(\gF)+ 2\widehat{\mathfrak{R}}_n(\gG^{\prime}) \mid X_{n+1}}\nonumber\\
        &\qquad\qquad+ O\LRl{\frac{B}{n^c} + \LRs{\frac{L}{\bar{\mu}}+1}\sqrt{\frac{\log n}{\rho n h_n^d}} + \frac{1}{\rho}\LRs{\frac{L }{\bar{\mu}}\bar{\tau} + \tau}h_n\log(h_n^{-d})}.\nonumber
    \end{align}
    Combining with the lower bound in \eqref{eq:intermediate_cond_expectation_lower}, we can finish the proof.
\end{proof}

\subsection{Bounds for weighted Rademacher complexities}

\begin{lemma}
    If the VC-dimension of $\{S_{\lambda}(x,y):\lambda \in \Lambda\}$ is $\mathsf{v}$, we almost surely have
    \begin{align*}
        \E\LRm{\widehat{\mathfrak{R}}_n(\gF) \mid X_{n+1}} \leq O\LRs{\sqrt{\frac{\mathsf{v}}{\rho nh_n^d}}}.
    \end{align*}
\end{lemma}

\begin{proof}
    For each $f\in \gF$, we define
    \begin{align*}
        \sG_n(f) = \frac{1}{\sqrt{\mathfrak{B}_n(X_{n+1})}}\sum_{i=1}^n w_i(X_{n+1}) \xi_i f(X_i,Y_i).
    \end{align*}
    For any $f,g\in \gF$, we define the metric
    \begin{align}
        \sup_{f,g\in \gF}\|\sG_n(f) - \sG_n(g)\|_{P_n}^2 &= \frac{\sup_{\substack{\lambda,\lambda^{'} \in \Lambda,\\ q,q^{'}\in \sR}}\sum_{i=1}^n w_i^2(X_j) \LRs{\mathbbm{1}\{S_{\lambda}(X_i,Y_i) \leq q\} - \mathbbm{1}\{S_{\lambda^{\prime}}(X_i,Y_i) \leq q^{\prime}\}}^2}{\mathfrak{B}_n(X_j)}.\nonumber
    \end{align}
    It holds that $\sup_{f,g\in \gF} \|\sG_n(f) - \sG_n(g)\|_{P_n} =\leq \frac{\sum_{i=1}^n w_i^2(X_j)}{\mathfrak{B}_n(X_j)}\leq 1$ by the definition of $\mathfrak{B}_n(X_j)$.
    Taking expectation only on $\xi_1,\ldots,\xi_n$ and using Dudley’s entropy integral bound, we have
    \begin{align*}
        \E\LRm{\sup_{f\in \gF} |\sG_n(f)| \mid \gD_n, X_{n+1}} \leq c\int_{0}^{1} \sqrt{\log N(\epsilon,\gF,\|\cdot\|_{P_n})} d\epsilon \leq c\sqrt{\mathsf{v}+1},
    \end{align*}
    where we used the upper bound \eqref{eq:covering_number_VC}. Taking the expectation conditioning on $X_{n+1}$, together with Lemma \ref{lemma:Bn_lower_bound}, we can finish the proof.
\end{proof}

\begin{lemma}\label{lemma:rademacher_finite_weight}
    If $\Lambda$ is a finite set, we almost surely have
    \begin{align}
        \E\LRm{\widehat{\mathfrak{R}}_n(\gG^{\prime}) \mid X_{n+1}}\leq B\sqrt{\frac{\log |\Lambda|}{\rho nh_n^d}}.
    \end{align}
\end{lemma}
\begin{proof}
    For convenience, let us augment the class $\widetilde{\gG}^{\prime} = \gG^{\prime} \cup -\gG^{\prime}$. Then it holds that
    \begin{align*}
        \widehat{\mathfrak{R}}_n(\gG^{\prime}) \leq \E\LRm{\sup_{g\in \widetilde{\gG}^{\prime}} \frac{1}{n}\sum_{i=1}^n \xi_i g(X_i,Y_i)}.
    \end{align*}
    It implies that for $t\geq 0$, if $\sup_{g\in \gG^{\prime}}|g| \leq B$ we have
    \begin{align*}
        \exp\LRl{t\widehat{\mathfrak{R}}_n(\gG^{\prime})} &\leq \exp\LRl{t \E\LRm{\sup_{g\in \widetilde{\gG}^{\prime}} \sum_{i=1}^n w_i(X_{n+1}) \xi_i f(X_i,Y_i) \mid \gX_{n+1}}}\\
        &\leq \E\LRm{\exp\LRl{t \sup_{g\in \widetilde{\gG}^{\prime}} \sum_{i=1}^n \xi_i w_i(X_{n+1}) g(X_i,Y_i)}}\\
        &\leq \sum_{g\in \widetilde{\gG}^{\prime}}\prod_{i=1}^n\E\LRm{\exp\LRl{t \xi_i w_i(X_{n+1}) g(X_i,Y_i)}\mid \gX_{n+1}}\\
        &= 2|\gG|\cdot \exp\LRs{4t^2 B^2 \sum_{i=1}^n w_i^2(X_{n+1})},
    \end{align*}
    where the last inequality follows the proof of weighted Hoeffding’s inequality. It follows that $\mathfrak{R}_n(\gG^{\prime}) \leq \frac{\log(2|\gG^{\prime}|)}{t} + 4tB^2/\mathfrak{B}_n(X_{n+1})$.
    Choosing $t = \sqrt{\frac{\log(2|\gG^{\prime}|) \mathfrak{B}_n(X_{n+1})}{4B^2 }}$, we can prove that $\mathfrak{R}_n(\gG^{\prime}) \leq 2B\sqrt{\frac{\log(2|\gG^{\prime}|)}{\mathfrak{B}_n(X_{n+1})}}$. Then the conclusion follows from Lemma \ref{lemma:Bn_lower_bound} and $|\gG^{\prime}| = |\Lambda|$.
\end{proof}

\begin{lemma}\label{lemma:rademacher_lipschitz_weight}
    For the continuous index set $\Lambda \subseteq \sR^m$ with bounded radius $R$, if there exists a constant $L_{\Lambda} > 0$ such that  $\sup_{x\in \gX,y\in \gY}\LRabs{\phi(y,z_{\lambda}(x; q_{\lambda}^o(x))) - \phi(y,z_{\lambda^{\prime}}(x;q_{\lambda^{\prime}}^o(x)))} \leq L_{\Lambda} \|\lambda-\lambda^{\prime}\|$ for any $\|\lambda - \lambda^{\prime}\| \leq O(n^{-1})$, then we have $\E[\widehat{\mathfrak{R}}_n(\gG^{\prime})\mid X_{n+1}] \leq O\LRs{B\sqrt{\frac{m \log(R n)}{n}} + \frac{L_{\Lambda}}{n}}$.
\end{lemma}

\begin{proof}
    Let $\{\lambda_{\ell}\}_{\ell=1}^{N_{\epsilon}}$ be an $\epsilon$-covering of $\Lambda \subset \sR^m$ under Euclidean norm $\|\cdot\|$, where $\epsilon \leq n^{-1}$. It holds that $N_{\epsilon} \leq O\{(2R/\epsilon)^m\}$.
    Then for any $\lambda \in \Lambda$, there exists some $\lambda_{\ell}$ such that $\|\lambda - \lambda_{\ell}\| \leq \epsilon$ and $\|\lambda - \lambda_{\ell^{'}}\| > \epsilon$ for $\ell^{'}\neq \ell$. It follows that
    \begin{align}
        &\E\LRm{\widehat{\mathfrak{R}}_n(\gG^{\prime})\mid X_{n+1}} = \E\LRm{\sup_{\lambda \in \Lambda}\mathbbm{1}_{\gE}\LRabs{\sum_{i=1}^n w_i(X_{n+1}) \xi_i\phi(Y_i,z_{\lambda}(X_i;q_{\lambda}^{o}(X_i)))} \mid X_{n+1}}\nonumber\\
        &\leq \E\LRm{\sup_{\lambda \in \Lambda}\sum_{\ell \in [N_{\epsilon}]}\mathbbm{1}\LRl{\|\lambda - \lambda_{\ell}\| \leq \epsilon} \left|\sum_{i=1}^n w_i(X_{n+1}) \xi_i \LRm{\phi(Y_i,z_{\lambda_{\ell}}(X_i;q_{\lambda_{\ell}}^o(X_i))) - \phi(Y_i,z_{\lambda}(X_i;q_{\lambda}^o(X_i)))}\right|\mid X_{n+1}}\nonumber\\
        &\qquad + \E\LRm{\sum_{\ell \in [N_{\epsilon}]}\mathbbm{1}\LRl{\|\lambda - \lambda_{\ell}\| \leq \epsilon}\LRabs{\sum_{i=1}^nw_i(X_{n+1}) \xi_i \phi(Y_i,z_{\lambda_{\ell}}(X_i;q_{\lambda_{\ell}}^o(X_i)))}\mid X_{n+1}}\nonumber\\
        &\leq L_{\Lambda} \epsilon + \E\LRm{\max_{\ell \in [N_{\epsilon}]}\LRabs{\sum_{i=1}^n w_i(X_{n+1}) \xi_i \phi(Y_i,z_{\lambda_{\ell}}(X_i;q_{\lambda_{\ell}}^o(X_i)))}\mid X_{n+1}},\nonumber
    \end{align}
    where the last inequality holds due to locally Lipschitz continuity on $\lambda$ and $\epsilon \leq n^{-1}$. By the same proof of Lemma \ref{lemma:rademacher_finite_weight} with finite index set $[N_{\epsilon}]$, we know
    \begin{align*}
        \E\LRm{\max_{\ell \in [N_{\epsilon}]}\LRabs{\sum_{i=1}^n w_i(X_{n+1}) \xi_i \phi(Y_i,z_{\lambda_{\ell}}(X_i;q_{\lambda_{\ell}}^o(X_i)))} \mid X_{n+1}} \leq B\sqrt{\frac{\log (2N_{\epsilon})}{\rho nh_n^d}}.
    \end{align*}
    Taking $\epsilon = n^{-1}$, together with $N_{\epsilon} \leq O\{(R/\epsilon)^m\}$, we can prove the conclusion.
\end{proof}

\subsection{Proofs of main lemmas}
\subsubsection{Proof of Lemma \ref{lemma:Bn_lower_bound}}
\begin{proof}\label{proof:lemma:Bn_lower_bound}
    Let us divide the support of covariate as $\gX_k = \{x\in \gX: \|x-X_{n+1}\|^2 \in [(k-1)h_n^2, kh_n^2)\}$ for $k \geq 1$. 
    Recall that $H(X_i,X_{n+1}) = e^{-\| X_i-X_{n+1}\|^2/h_n^2} \geq e^{-\| X_i-X_{n+1}\|^2/h_n^2} \mathbbm{1}\{X_i \in \gX_1\} \geq e^{-1}\mathbbm{1}\{X_i \in \gX_1\}$. Conditioning on $X_{n+1}$, it follows that with probability at least $1-n^{-c}$,
    \begin{align}
        e\cdot \mathfrak{B}_n(X_{n+1}) \geq  \sum_{i=1}^n \mathbbm{1}\{X_i \in \gX_1\} &\geq n\sP(X_1 \in \gX_1) - \LRabs{\sum_{i=1}^n \mathbbm{1}\{X_i \in \gX_1\} - \sP(X_i \in \gX_1)}\nonumber\\
        &\Eqmark{i}{\geq} n\sP(X_1 \in \gX_1) - \sqrt{\frac{2C\log n}{n}}\nonumber\\
        &= n\int_{\gX} \mathbbm{1}\{\|x-X\| \in [0, h_n) \}\cdot p(x) dx - \sqrt{\frac{2C\log n}{n}}\nonumber\\
        &\Eqmark{ii}{\geq} \rho V n h_n^d - \sqrt{\frac{2C\log n}{n}}\nonumber\\
        &\Eqmark{iii}{\geq} \frac{\rho V n h_n^d}{2},\nonumber
    \end{align}
    where $(i)$ holds due to Hoeffding's inequality; $(ii)$ holds due to Assumption 4; and $(iii)$ holds due to $\rho V n h_n^d > 2 \sqrt{\frac{2C\log n}{n}}$. 
    
    Regarding the second conclusion, we observe that $\mathfrak{B}_n(X_j) = \sum_{i\neq j}H(X_i,X_j) + 1$. We can then apply similar arguments to complete the proof.
\end{proof}

\subsubsection{Proof of Lemma \ref{lemma:concentration_cover_VC}}

\begin{proof}
    Recall the function class: $\gF=\{\mathbbm{1}\{S_{\lambda}(x,y) \leq q\}: \lambda \in \Lambda, q\in \sR\}$. We define the zero-mean random variable
    \begin{align*}
        M_{\gF} = \sup_{\lambda\in \Lambda, q\in \sR}\left|\sum_{i=1}^n w_i(X_j) \LRs{\mathbbm{1}\{S_{\lambda}(X_i,Y_i) \leq q\} - \sP\LRl{S_{\lambda}(X_i,Y_i) \leq q \mid X_i}}\right|.
    \end{align*}
    Using McDiarmid’s inequality, we have
    \begin{align}
        \sP\LRl{M_{\gF} > \E\left[M_{\gF}\mid \gX_{n+1}\right] + t \mid \gX_{n+1}} \leq 2\exp\LRs{-\frac{t^2}{\sum_{i=1}^n w_i^2(X_j)}}.\nonumber
    \end{align}
    Using the symmetrization technique, we also have
    \begin{align}
        \E\left[M_{\gF}\mid \gX_{n+1}\right] \leq 2 \widehat{\mathfrak{R}}_n(\gF).\nonumber
    \end{align}
    Taking $t = \sqrt{\frac{c\log n}{\mathfrak{B}_n(X_j)}}$, we can have
    \begin{align}
        \sP\LRl{M_{\gF} > 2 \widehat{\mathfrak{R}}_n(\gF) + \sqrt{\frac{c\log n}{\mathfrak{B}_n(X_j)}} \mid \gX_{n+1}} \leq 2n^{-c},\nonumber
    \end{align}
    which proves the conclusion.
\end{proof}

\subsubsection{Proof of Lemma \ref{lemma:cond_quantile_estimation_VC}}
\begin{proof}
    Lemma \ref{lemma:concentration_cover_VC} guarantees that for any $j\in [n+1]$,
    \begin{align}
        \sP\LRs{\sup_{\lambda\in \Lambda,q\in \sR} \LRabs{\sum_{i=1}^n w_i(X_{j})\Big[\mathbbm{1}\LRl{S_{\lambda}(X_i,Y_i) \leq q} - F_{\lambda}(q|X_i)\Big]} > 2 \widehat{\mathfrak{R}}_n(\gF) + \sqrt{\frac{c\log n}{\mathfrak{B}_n(X_{j})}} \mid \gX_{n+1}} \leq 2n^{-c}.\nonumber
    \end{align}
    Due to Lemma \ref{lemma:uniform_cover_bias}, we have
    \begin{align}
        \sup_{\lambda\in \Lambda,q\in \sR}\sum_{i=1}^n w_i(X_{j})|F_{\lambda}(q|X_i) - F_{\lambda}(q|X_{j})| \leq \bar{\tau}\LRs{e h_n \log(h_n^{-d}) + \frac{n h_n\log(h_n^{-d})h_n^d}{\mathfrak{B}_n(X_{j})}}.\nonumber
    \end{align}
    Take $b_n = 2\widehat{\mathfrak{R}}_n(\gF) + \bar{\tau}\LRs{ e h_n \log(h_n^{-d}) + \frac{n h_n\log(h_n^{-d})h_n^d}{\mathfrak{B}_n(X_{j})}} + \sqrt{\frac{c\log n}{\mathfrak{B}_n(X_{j})}}$ and denote $q_{\lambda}^{-}(X_{j}) = F_{\lambda}^{-1}(1-\alpha- b_n| X_{j})$ and $q_{\lambda}^{+}(X_{j}) = F_{\lambda}^{-1}(1-\alpha+ b_n| X_{j})$.
    We can have
    \begin{align}
        \sP\LRl{\forall \lambda \in \Lambda, \hat{q}_{\lambda}(X_{j})\in \LRm{q_{\lambda}^{-}(X_{j}),q_{\lambda}^{+}(X_{j})} \mid \gX_{n+1}} \geq 1-2n^{-c}.\nonumber
    \end{align}
    According to the lower-bounded condition for $f_{\lambda}(s|X_{j})$ in Assumption 4, we know
    \begin{align}
        q_{\lambda}^{+}(X_{j}) - F_{\lambda}^{-1}(1-\alpha|X_{j}) &= F_{\lambda}^{-1}(1-\alpha+b_n|X_{j}) - F_{\lambda}^{-1}(1-\alpha|X_{j})\leq \frac{b_n}{\bar{\mu}}.\nonumber
    \end{align}
    Combining the results above, we can prove the conclusion.
\end{proof}

\subsubsection{Proof of Lemma \ref{lemma:uniform_weight_concentration_VC}}
\begin{proof}
    Recall the function class $\gG^{\prime} = \{\phi(y, z_{\lambda}(x; q_{\lambda}^{o}(x))): \lambda \in \Lambda\}$ and the definition $\Phi_{\lambda}(x,y) = \phi(y, z_{\lambda}(x; q_{\lambda}^{o}(x)))$. We define the random variable
    \begin{align*}
        M_{\gG^{\prime}} = \sup_{\lambda\in \Lambda, q\in \sR}\left|\sum_{i=1}^n w_i(X_{n+1})\Big(\Phi_{\lambda}(X_i,Y_i) - \E[\Phi_{\lambda}(X_i,Y_i)\mid X_i]\Big)\right|.
    \end{align*}
    Using McDiarmid’s inequality, we have
    \begin{align}
        \sP\LRl{M_{\gG^{\prime}} > \E\left[M_{\gG^{\prime}}\mid \gX_{n+1}\right] + t \mid \gX_{n+1}} \leq 2\exp\LRs{-\frac{t^2}{\sum_{i=1}^n w_i^2(X_j)}}.\nonumber
    \end{align}
    Using the symmetrization technique, we also have
    \begin{align}
        \E\left[M_{\gG^{\prime}}\mid \gX_{n+1}\right] & \leq 2 \E\left[\sup_{\lambda\in \Lambda, q\in \sR}\left|\sum_{i=1}^n w_i(X_{n+1})\xi_i\Phi_{\lambda}(X_i,Y_i) \right|\mid \gX_{n+1}\right]\nonumber\\
        &= 2 \widehat{\mathfrak{R}}_n(\gG^{\prime}).\nonumber
    \end{align}
    Taking $t = \sqrt{\frac{c\log n}{\mathfrak{B}_n(X_j)}}$, we can have
    \begin{align}
        \sP\LRl{M_{\gG^{\prime}} > 2 \widehat{\mathfrak{R}}_n(\gG^{\prime}) + \sqrt{\frac{c\log n}{\mathfrak{B}_n(X_j)}} \mid \gX_{n+1}} \leq 2n^{-c},\nonumber
    \end{align}
    which proves the conclusion.
\end{proof}

Let $\xi_1,\ldots,\xi_n$ be i.i.d. Rademacher random variables. For any $f\in \Psi^{'}$, conditioning on $\{X_i\}_{i=1}^{n+1}$, we define the zero-mean random variable
    \begin{align*}
        M_f = \frac{1}{2B\sqrt{\mathfrak{B}_n(X_{n+1})}}\sum_{i=1}^n \xi_i w_i(X_{n+1})\Big(\Phi_{\lambda}(X_i,Y_i) - \E[\Phi_{\lambda}(X_i,Y_i)\mid X_i]\Big).
    \end{align*}
    Then we define the metric
    \begin{align}
        \sup_{f,g\in \Psi^{'}}\|M_{f} - M_{g}\|_{P_n}^2 &= \frac{1}{4B^2\mathfrak{B}_n(X_{n+1})}\sup_{\substack{\lambda,\lambda^{'} \in \Lambda}}\sum_{i=1}^n w_i^2(X_{n+1}) \Big(\Phi_{\lambda}(X_i,Y_i) - \E[\Phi_{\lambda}(X_i,Y_i)\mid X_i]\Big)^2\nonumber\\
        &\leq \frac{\sum_{i=1}^n w_i^2(X_{n+1})}{\mathfrak{B}_n(X_{n+1})} \leq 1,\nonumber
    \end{align}
    where we used $|\Phi_{\lambda}(x,y)| \leq B$.
    Taking expectation only on $\xi_1,\ldots,\xi_n$ and using Dudley’s entropy integral bound, we have
    \begin{align*}
        \E\LRm{\sup_{f\in \Psi^{'}} |M_f| \mid \{(X_i,Y_i)\}_{i=1}^{n+1}} \leq c\int_{0}^{1} \sqrt{\log N(\epsilon,\Psi^{'},\|\cdot\|_{P_n})} d\epsilon \leq c\sqrt{\mathsf{v}(\Psi^{'})}.
    \end{align*}
    Applying McDiarmid’s inequality and the symmetrization, we have
    \begin{align}\label{eq:weighted_uniform_concentration_loss}
        \sP\left\{\frac{2B\sup_{f\in \Psi^{'}} M_f}{\sqrt{\mathfrak{B}_n(X_{n+1})}} \leq 2cB\sqrt{\frac{\mathsf{v}(\Psi^{'})}{\mathfrak{B}_n(X_j)}} + \frac{2B\cdot t}{\sqrt{\mathfrak{B}_n(X_j)}} \mid \gX_{n+1}\right\} &\leq \exp\LRs{\frac{-t^2}{\sum_{i=1}^n w_i^2(X_j)/\mathfrak{B}_n(X_j)}}\nonumber\\
        &\leq \exp\LRs{-t^2},
    \end{align}
    where we used the fact $\sum_{i=1}^n w_i^2(X_{n+1}) \leq \mathfrak{B}_n(X_{n+1})$.
    Taking $t = \sqrt{c\log n}$ in \eqref{eq:weighted_uniform_concentration_loss}, we can prove the conclusion.

\subsubsection{Proof of Lemmas \ref{lemma:uniform_cover_bias} and \ref{lemma:uniform_error_bound}}\label{proof:lemma:uniform_cover_bias}
\begin{proof}
    By Assumption 5, we know
    \begin{align}
        \left|\E\LRm{\phi(Y_i, z_{\lambda}(X_i;q^{o}(X_i)))\mid X_i} - \E\LRm{\phi(Y_{n+1}, z_{\lambda}(X_{n+1};q^{o}(X_{n+1}))))\mid X_{n+1}}\right| \leq \tau\twonorm{X_i - X_{n+1}}.\nonumber
    \end{align}
    Recall that $\gX_k = \{x\in \gX: \|x-X_{n+1}\|^2 \in [(k-1)h_n^2, kh_n^2)\}$ for $k \geq 1$. Then we can write
    \begin{align}
        &\left|\sum_{i=1}^n H(X_i,X_{n+1})\Big(\E\LRm{\phi(Y_i, z_{\lambda}(X_i;q^{o}(X_i)))\mid X_i} - \E\LRm{\phi(Y_{n+1}, z_{\lambda}(X_{n+1};q^{o}(X_{n+1})))\mid X_{n+1}}\Big)\right|\nonumber\\
        &\qquad\leq \sum_{k=1}^{\infty} \sum_{i=1}^n \mathbbm{1}\{X_i\in \gX_k\} H(X_i,X_{n+1}) \tau \twonorm{X_i - X_{n+1}}\nonumber\\
        &\qquad\leq \sum_{k=1}^{\infty} \sum_{i=1}^n \mathbbm{1}\{X_i\in \gX_k\} ke^{-k+1}\tau\cdot h_n\nonumber\\
        &\qquad\Eqmark{i}{\leq} \tau\LRs{e k_0 h_n\sum_{k=1}^{k_0}  \sum_{i=1}^n \mathbbm{1}\{X_i\in \gX_k\}e^{-k} + h_n\sum_{k=k_0+1}^{\infty} k e^{-k+1} \sum_{i=1}^n \mathbbm{1}\{X_i\in \gX_k\}}     \nonumber\\
        &\qquad\Eqmark{ii}{\leq} \LRs{e h_n k_0 \cdot \mathfrak{B}_n(X_{n+1}) + n h_n(k_0+1)e^{-k_0}}\tau,\label{eq:bias_error}
        \end{align}
    where $(i)$ holds for arbitrary $k_0 \geq 1$; and $(ii)$ holds since $\mathbbm{1}\{X_i\in \gX_k\}e^{-k} \leq \exp(-\| X_i-X_{n+1}\|^2/h_n^2)$ for any $k\geq 1$. Taking $k_0 = \lceil \log(h_n^{-d}) \rceil$ in \eqref{eq:bias_error}, we get
    \begin{align}
        |\Xi_n(\lambda)| \leq \LRs{e h_n k_0 + \frac{n h_n(k_0+1)e^{-k_0}}{\mathfrak{B}_n(X_{n+1})}} \tau\leq \LRs{e h_n \log(h_n^{-d}) + \frac{n h_n\log(h_n^{-d})h_n^d}{\mathfrak{B}_n(X_{n+1})}} \tau.\nonumber
    \end{align}
    Taking supreme over $\lambda\in \Lambda$ on the left side, we can prove the conclusion of Lemma \ref{lemma:uniform_error_bound}.

    Now we prove Lemma \ref{lemma:uniform_cover_bias}.
    According to Assumption 4, we know
    \begin{align}
        \sup_{\lambda\in \Lambda,q\in \sR}&\left|\sP\{S_{\lambda}(X_i,Y_i)\leq q\mid X_i\} - \sP\LRl{S_{\lambda}(X_{n+1},Y_{n+1})\leq q\mid X_{n+1}}\right|\nonumber\\
        =& \sup_{q\in \sR}\left|\sP\{S_{\lambda}(X_{n+1},Y_{n+1})\leq q \mid X_i\} - \sP\LRl{S_{\lambda}(X_{n+1},Y_{n+1}) \leq q\mid X_{n+1}}\right| \nonumber\\
        \leq & \bar{\tau}\cdot \|X_i-X_{n+1}\|.\nonumber
    \end{align}
    Similar to the derivation of \eqref{eq:bias_error}, we also have
    \begin{align}
        \sup_{\lambda\in \Lambda,q\in \sR}\sum_{i=1}^n w_i(X_{n+1}) &\left|\sP\{S_{\lambda}(X_i,Y_i)\leq q \mid X_i\} - \sP\LRl{S_{\lambda}(X_{n+1},Y_{n+1}) \leq q\mid X_{n+1}}\right|\nonumber\\
        &\leq \mathfrak{B}_n^{-1}(X_{n+1})\sum_{i=1}^n H(X_i,X_{n+1})\cdot \bar{\tau} \|X_i-X_{n+1}\|\nonumber\\
        &\leq  \mathfrak{B}_n^{-1}(X_{n+1})\LRs{eh_n k_0\cdot \mathfrak{B}_n(X_{n+1}) + nh_n(k_0+1)e^{-k_0}}\cdot \bar{\tau}\nonumber\\
        &\leq \LRs{e h_n \log(h_n^{-d}) + \frac{n h_n\log(h_n^{-d})h_n^d}{\mathfrak{B}_n(X_{n+1})}}\cdot \bar{\tau},\nonumber
    \end{align}
    where the last inequality holds by taking $k_0 = \lceil \log(h_n^{-d}) \rceil$. Notice that for $j\in [n]$,
    \begin{align*}
        &\sup_{\lambda\in \Lambda,q\in \sR}\sum_{i=1}^n w_i(X_{j}) \left|\sP\{S_{\lambda}(X_i,Y_i)\leq q \mid X_i\} - \sP\LRl{S_{\lambda}(X_{j},Y_{j}) \leq q\mid X_{j}}\right|\\
        =& \sup_{\lambda\in \Lambda,q\in \sR}\sum_{i=1,i\neq j}^n w_i(X_{j}) \left|\sP\{S_{\lambda}(X_i,Y_i)\leq q \mid X_i\} - \sP\LRl{S_{\lambda}(X_{j},Y_{j}) \leq q\mid X_{j}}\right|\\
        \leq & \mathfrak{B}_n^{-1}(X_{j}) \sum_{i=1,i\neq j}^n H(X_i,X_{j})\cdot \bar{\tau} \|X_i-X_{j}\|.
    \end{align*}
    Then, using similar arguments can prove the conclusion for $j\in [n]$.
\end{proof}

\section{Additional results for CROiMS}
\subsection{Implementation of CROiMS}
The kernel function can be set as $H(x_1,x_2) = K\LRs{\frac{D^2(x_1,x_2)}{h^2}}$, where $D(x_1,x_2) \geq 0$ is a distance function that measures the dissimilarity between $x_1$ and $x_2$, $K(\cdot)$ is a normalized one-dimensional kernel function, and $h$ is the bandwidth. In the case where $\gX \subseteq \sR^d$, common choices for the kernel function include the Gaussian kernel is $H(x_1,x_2) = \exp\LRs{-\frac{D^2(x_1,x_2)}{h^2}}$ and the box kernel $H(x_1,x_2) = \mathbbm{1}\{D^2(x_1,x_2)\leq h^2\}$. The distance function $D$ can be flexibly chosen according to the specific need for localization. For example, the Euclidean distance for low-dimensional cases and the projection distance for high-dimensional cases \citep{guan2023localized,lei2018distribution}.
In our paper, we use the same kernel function for both the weighted quantile and the weighted ERM problem. In practice, one can choose different kernels to estimate the conditional quantile and risk functions.

\subsection{F-CROiMS for finite-sample marginal robustness}\label{appen:full-CROiMS}

To achieve the finite-sample marginal robustness in Definition 1, this section provides a corrected version of CROiMS by leveraging the full conformal prediction method \citep{vovk2005algorithmic} and swapping technique.

\subsection{Conformal prediction after individual model selection}

Before that, we present a more general procedure to construct a marginally valid prediction set for the individually selected model.
Let $\sA(\{(x_i,y_i)\}_{i\in [n]}, x): (\gX \times \gY)^n \times \gX \to \Lambda$ be a general model selection algorithm, e.g., CROiMS in Algorithm 3. We assume $\sA$ has the following symmetric property.

\begin{assumption}\label{assum:symmetry}
    The algorithm $\sA$ satisfies $\sA(\{(x_{\pi(i)},y_{\pi(i)})\}_{i\in [n]},x) = \sA(\{(x_i,y_i)\}_{i\in [n]}, x)$ holds for any $x\in \gX$, where $\pi$ is arbitrary permutation operator in $[n]$.
\end{assumption}

Now denote $\hat{\lambda}(X_{n+1}) = \sA(\gD_n,X_{n+1})$ be the individually selected model for test data $X_{n+1}$, where $\gD_n = \{(X_i,Y_i)\}_{i\in [n]}$. Given any hypothesized value $y\in \gY$ and $j\in [n+1]$, we define the ``swapped'' dataset $\gD_n^{j,y} = \{(X_i,Y_i)\}_{i\in [n],i\neq j} \cup \{(X_{n+1}, y)\}$. To quantify the uncertainty of $\hat{\lambda}$, we obtain the individually selected models for labeled data through
\begin{align}\label{eq:full_ind_model}
    \hat{\lambda}^y(X_j) = \sA(\gD_n^{j,y}, X_j),\quad j\in [n].
\end{align}
Then the prediction set is defined as
\begin{align}\label{eq:swapped_FCP}
    \widehat{\gU}^{\text{F-CROiMS}}(X_{n+1}) = \LRl{y\in \gY: S_{\hat{\lambda}(X_{n+1})}(X_{n+1},y) \leq Q_{(1-\alpha)(1+n^{-1})}\LRs{\{S_{\hat{\lambda}^y(X_j)}(X_j,Y_j)\}_{j=1}^n}}.
\end{align}
It is worthwhile noticing that if the hypothesized value $y$ imputes the ground truth label $Y_{n+1}$, the individually selected indexes $\{\hat{\lambda}(X_j) \equiv\hat{\lambda}^{Y_{n+1}}(X_j)\}_{j=1}^{n}$ and $\hat{\lambda}$ are exchangeable due to the symmetry property in Assumption \ref{assum:symmetry}. Then we can guarantee the corresponding scores $\{S_{\hat{\lambda}(X_j)}(X_j,Y_j)\}_{j=1}^{n+1}$ are also exchangeable, which leads to the following finite-sample marginal coverage result.

\begin{theorem}\label{thm:cover_swapped}
    Suppose $\{(X_i,Y_i)\}_{i=1}^{n+1}$ are i.i.d., and the selection algorithm $\sA$ satisfies Assumption \ref{assum:symmetry}, then we have
    \begin{align}
        \sP\LRl{Y_{n+1}\in \widehat{\gU}^{\text{F-CROiMS}}(X_{n+1})} \geq 1-\alpha.\nonumber
    \end{align}
\end{theorem}

\begin{remark}
    The swapping technique is very useful in conformal prediction, whose role is to restore the pairwise exchangeability after some data-driven procedure is performed on labeled data and test data. For example, in the selective predictive inference problem \citep{bao2024selective}, the test data selected by some data-dependent procedure is no longer exchangeable with the labeled data. Recently, \citet{bao2024cap} and \citet{jin2024confidence} swapped the test data and labeled data to construct an adaptive pick rule to select labeled data points for further calibration. Here, we utilize this technique for a different purpose, to close the coverage gap after individual model selection.

    The full conformal approach is also employed in \citet{liang2024conformal} to correct the model selection bias in coverage, where the authors proposed a leave-one-out procedure to construct the prediction set. However, they require the algorithm $\sA$ to be symmetric to covariates $\{X_i\}_{i=1}^{n+1}$ (see Definition 2.1 in \citet{liang2024conformal}), which is not satisfied for Algorithm 3. In fact, \citet{liang2024conformal} focused on selecting the model that minimizes the \emph{averaged} loss of prediction sets, e.g., the expectation of size. It is different from our primary concern on individual efficiency.
\end{remark}

\subsection{F-CROiMS with marginal robustness}
Since CROiMS in Algorithm 3 satisfies Assumption \ref{assum:symmetry}. By setting $\sA$ as CROiMS, we can obtain the prediction set in \eqref{eq:swapped_FCP} and make the final decision by solving the CRO problem as follows,
\begin{align}
    \hat{z}^{\text{F-CROiMS}}(X_{n+1}) = \argmin_{z\in \gZ} \max_{c \in \widehat{\gU}^{\text{F-CROiMS}}(X_{n+1})} \phi(c,z).\nonumber 
\end{align}
We call this procedure F-CROiMS and summarize its implementation in Algorithm \ref{alg:full}. 

\begin{algorithm}[H]
	\renewcommand{\algorithmicrequire}{\textbf{Input:}}
	\renewcommand{\algorithmicensure}{\textbf{Output:}}
	\caption{F-CROiMS}
	\label{alg:full}
	\begin{algorithmic}[1]
	\Require Pre-trained models $\{S_{\lambda}:\lambda\in \Lambda\}$, loss function $\phi$, test data $X_{n+1}$, labeled data $\{(X_{i},Y_i)\}_{i=1}^{n}$, kernel function $H$, robustness level $1-\alpha \in (0, 1)$.
        \State Call Algorithm 3 to obtain $\hat{\lambda}(X_{n+1})$.
        \State Initialize $\widehat{\gU}^{\text{F-CROiMS}}(X_{n+1}) \gets \emptyset$.
        \For{$y\in \gY$}
        \For{$j\in [n]$}
        \State Call Algorithm 3 by replacing $(X_j,Y_j)$ with $(X_{n+1},y)$.
        \State $w_i^j(X_j) \gets \frac{H(X_i,X_j)}{\sum_{\ell=1,\ell\neq j}^{n+1} H(X_{\ell},X_j)}$ for $i\in [n+1]$.
        
        \State $\hat{\lambda}^y(X_j) \gets \argmin_{\lambda\in \Lambda} \sum_{i=1,i\neq j}^{n} w_i^j(X_{j})\cdot \phi(Y_i,\hat{z}_{\lambda}^{y,j}(X_i)) + w_{n+1}^j(X_{j})\cdot \phi(y,\hat{z}_{\lambda}^{y,j}(X_{n+1}))$.
        \EndFor
        \If{$S_{\hat{\lambda}(X_{n+1})}(X_{n+1},y) \leq Q_{1-\alpha}\LRs{\{S_{\hat{\lambda}^y(X_j)}(X_j,Y_j)\}_{j=1}^n \cup \{S_{\hat{\lambda}(X_{n+1})}(X_{n+1},y)\}}$}
        \State $\widehat{\gU}^{\text{F-CROiMS}}(X_{n+1}) \gets \widehat{\gU}^{\text{F-CROiMS}}(X_{n+1}) \cup \{y\}$.
        \EndIf
        \EndFor


        \State Solve CRO problem $\hat{z}^{\text{F-CROiMS}}(X_{n+1}) = \argmin_{z\in \gZ} \max_{c \in \widehat{\gU}^{\text{F-CROiMS}}(X_{n+1})} \phi(c,z)$.
        
        \Ensure Decision $\hat{z}^{\text{F-CROiMS}}(X_{n+1})$.
	\end{algorithmic}
\end{algorithm}

As a Corollary of Theorem \ref{thm:cover_swapped}, the decision induced by F-CROiMS has the following robustness guarantee in finite samples.
\begin{corollary}
    Suppose $\{(X_i,Y_i)\}_{i=1}^{n+1}$ are i.i.d., F-CROiMS satisfies the $1-\alpha$ level of marginal robustness.
\end{corollary}

\subsection{Proof of Theorem \ref{thm:cover_swapped}}
\begin{proof}
    Let us define the augmented dataset $\gD_{n+1} = \{(X_i,Y_i)\}_{i=1}^{n+1}$ and its leave-one-out subset $\gD_{n+1}^{-j} = \{(X_i,Y_i)\}_{i=1,i\neq j}^{n+1}$ for $j\in [n+1]$. Then we write $\hat{\lambda}_j = \sA(\gD_{n+1}^{-j},X_j)$ for $j\in [n+1]$. It holds that $\hat{\lambda}_j \equiv \hat{\lambda}_j^{Y_{n+1}}$ and $\hat{\lambda}_{n+1} \equiv \hat{\lambda}(X_{n+1})$.
Recalling the definition of \eqref{eq:swapped_FCP}, it follows that
\begin{align}\label{eq:cover_equality}
    &\sP\LRl{Y_{n+1}\in \widehat{\gU}^{\text{F-CROiMS}}(X_{n+1})}\nonumber\\
    &\qquad= \sP\LRl{S_{\hat{\lambda}_{n+1}}(X_{n+1},Y_{n+1}) \leq Q_{1-\alpha}\LRs{\{S_{\hat{\lambda}_j^{Y_{n+1}}}(X_j,Y_j)\}_{j=1}^n \cup \{S_{\hat{\lambda}_{n+1}}(X_{n+1},Y_{n+1})\}}}\nonumber\\
    &\qquad= \sP\LRl{S_{\hat{\lambda}_{n+1}}(X_{n+1},Y_{n+1}) \leq Q_{1-\alpha}\LRs{\{S_{\hat{\lambda}_j}(X_j,Y_j)\}_{j=1}^n \cup \{S_{\hat{\lambda}_{n+1}}(X_{n+1},Y_{n+1})\}}}.
\end{align}
Let $Z_i = (X_i,Y_i)$ and $\mZ = [Z_1,\ldots,Z_{n+1}]$ be the unordered set of augmented dataset. For any $\vz = (z_1,\ldots,z_{n+1}) \in (\gX \times \gY)^{n+1}$ with $z_j = (x_j,y_j)$, denote the event $\gE = \{\mZ = \vz\}$. According to Assumption \ref{assum:symmetry}, given $\gE$, we know $\hat{\lambda}_j$ depends only on the value of $Z_j$. It means that the unordered set of $\{S_{\hat{\lambda}_j}(X_j,Y_j)\}_{j=1}^{n+1}$ is fixed given the event $\gE$, which is $\left[S_{\lambda_1}(z_1),\ldots,S_{\lambda_{n+1}}(z_{n+1})\right]$,
where $\lambda_j = \sA(\vz^{-j},x_j)$ for $j\in [n+1]$. 
Define the set of strange points
\begin{align}
    \gS(\vz)=\LRl{z_i: S_{\lambda_i}(z_i) > Q_{1-\alpha}\LRs{\{S_{\lambda_j}(z_j)\}_{j=1}^{n+1}},i\in [n+1]}.\nonumber
\end{align}
In particular, $\gS(\vz)$ is invariant to the permutation of $\{z_i\}_{i=1}^{n+1}$.
By the definition of quantile, we know $\frac{1}{n+1}\sum_{i=1}^{n+1}\mathbbm{1}\{z_i\in \gS(\vz)\} \leq \alpha.$
Notice that under the event $\gE$, it holds that
$$
Q_{1-\alpha}\LRs{\{S_{\hat{\lambda}_j}(Z_j)\}_{j=1}^{n+1}}\mid \gE = Q_{1-\alpha}\LRs{\{S_{{\lambda}_j}(z_j)\}_{j=1}^{n+1}}.
$$ 
Invoking the exchangeability of $\{Z_i\}_{i=1}^{n+1}$, we have
\begin{align}
    \sP\LRl{S_{\hat{\lambda}}(X_{n+1},Y_{n+1}) \leq \hat{Q}_{n+1} \mid \gE} &= 1 - \sP\LRl{Z_{n+1} \in \gS(\vz) \mid \gE}\nonumber\\
    &= 1 - \frac{1}{n+1}\sum_{i=1}^{n+1} \sP\LRl{Z_i \in \gS(\vz) \mid \gE}\nonumber\\
    &= 1 - \E\LRm{\frac{1}{n+1}\sum_{i=1}^{n+1} \mathbbm{1}\{Z_i \in \gS(\vz)\}\mid \gE}\nonumber\\
    &= 1 - \E\LRm{\frac{1}{n+1}\sum_{i=1}^{n+1} \mathbbm{1}\{z_i \in \gS(\vz)\}\mid \gE} \geq 1 - \alpha&.\nonumber
\end{align}
Marginalizing over $\gE$, together with \eqref{eq:cover_equality}, we can show the conclusion. 
\end{proof}

\section{Additional experiment results and deferred settings}\label{appen:simulation_setting}

\subsection{Model selection in the averaged case}\label{appen:additional_simulation}

\subsubsection{Classification task: models are trained with different features}

The loss function and data generation setting are identical to those in Section 5.1.1.
The nonconformity score function used here is $S_{\lambda}(x,y) = 1 - f_{\lambda}^{y}(x)$, where $f_{\lambda}:\mathcal{X} \rightarrow [0,1]^{|\mathcal{Y}|}$ is the softmax layer of a classifier. Candidate models $\{S_\lambda,\lambda \in \Lambda\}$ are trained with different features, where $f_{\lambda}$ is fitted by the Gradient Boosting algorithm. Specifically, $S_{\lambda}$ is trained with the target variable $Y$ and the covariate $\tilde{X}=(\tilde{X}_1,\tilde{X}_2,\tilde{X}_3,\tilde{X}_4,\tilde{X}_5)^\top$, where $\tilde{X}_1,\tilde{X}_2, \tilde{X}_3$ are uniformly selected from categorical variables $X_1,X_2,X_3, X_4$ and $\tilde{X}_4,\tilde{X}_5$ are uniformly selected from continuous variables $X_5, X_6,X_7$.

To illustrate the robustness and efficiency of different methods, we consider the effect of the size of labeled samples $n$, the size of candidate models $|\Lambda|$, and the nominal level $\alpha$ on decision performance. The respective simulation results are summarized in Figure \ref{fig: ACFSLP}. We observe that the experimental results are similar to those in Section 5.1. \texttt{F-CROMS} maintains the coverage guarantee while achieving a lower average decision loss. \texttt{E-CROMS} always achieves the minimal averaged loss while controlling the marginal misrobustness below the level $\alpha$. 

\begin{figure}[H]
    \centering
    \begin{subfigure}[t]{\textwidth}
        \includegraphics[width=\textwidth]{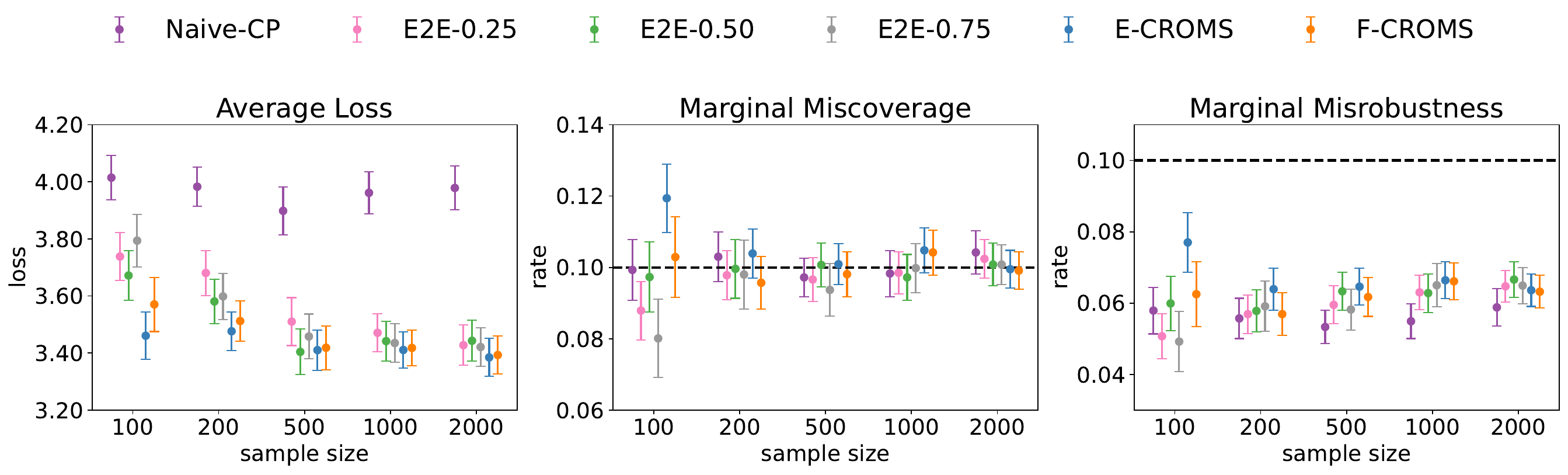}
        \caption{Varying sample size $n$ with $|\Lambda| = 15$ and $\alpha = 0.1$.}
    \end{subfigure}

    \begin{subfigure}[t]{\textwidth}
        \includegraphics[width=\textwidth]{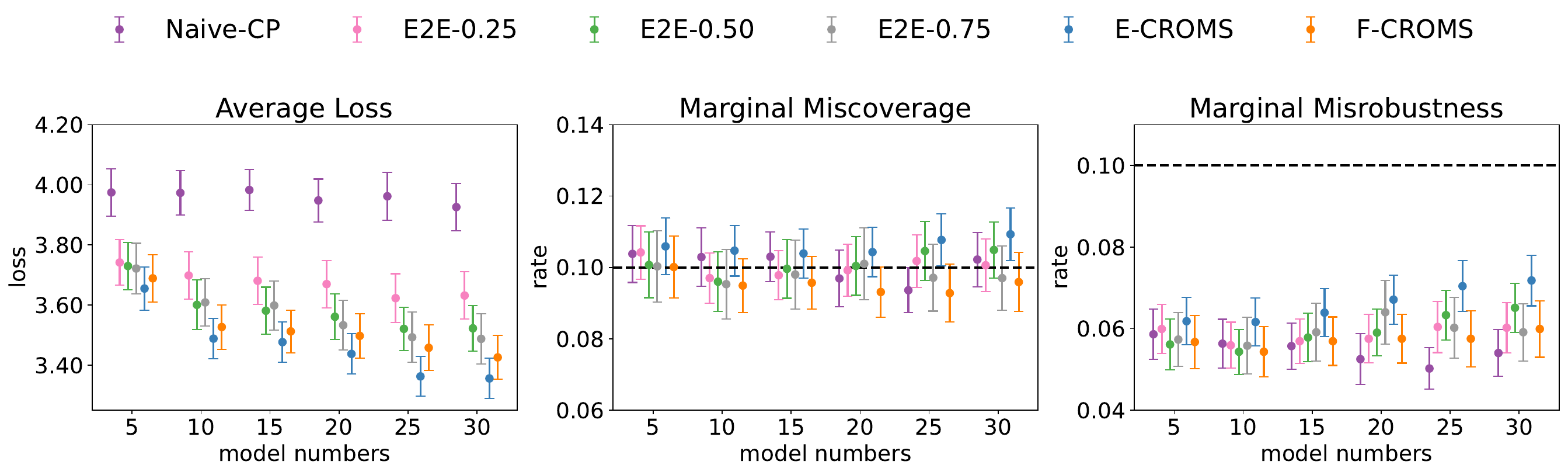}
        \caption{Varying candidate model numbers $|\Lambda|$ with $n = 200$ and $\alpha = 0.1$.}
    \end{subfigure}

    \begin{subfigure}[t]{1\textwidth}
        \includegraphics[width=\textwidth]{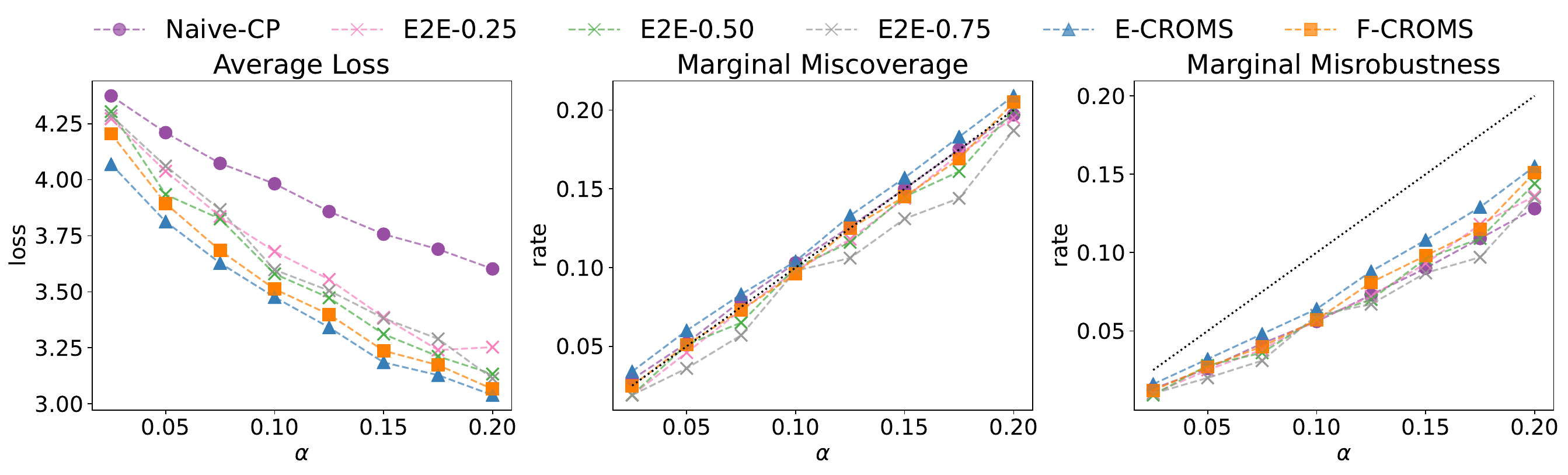}
        \caption{Varying nominal level $\alpha$ with $n = 200$ and $|\Lambda| = 15$.}
    \end{subfigure}
    \caption{The average loss, marginal coverage, and robustness in the classification task, where the candidate models are trained with different features.}
    \label{fig: ACFSLP}
\end{figure}






\subsubsection{Regression task: models are trained from different datasets}\label{appen:gen_regression}
In this regression task, we define the loss function as $\phi(y,z) = -y^{\top}z$, where $\mathcal{Y} = \mathbb{R}^{2}$ and $\mathcal{Z} = \{z \in [0,1]^2:\|z\|_1=1, z \geq 0\}$. The labeled and test data are generated by: $Y_1 = -X_1 - X_2^2 + \epsilon_1,\ Y_2 = -X_1^2 - X_2 + \epsilon_2$.
Let $Y = (Y_1,Y_2)^\top$, and $\epsilon = (\epsilon_1,\epsilon_2)^\top$. The covariate $X = (X_1,X_2)^\top$ follows the distribution $N((1,1)^{\top},2.25\cdot\mathrm{I}_2)$. The noise $\epsilon\sim N(0,\Sigma)$ is independent of $X$. The noise $\epsilon = (\epsilon_1, \epsilon_2)$ follows the multivariate normal distribution $N(0,\Sigma)$ and the corresponding covariance matrix is $\Sigma = 0.25LL^{\top}$, where $L=\begin{pmatrix}
  1.0 & 0.5  \\
  0.5 & 4.0 
\end{pmatrix}.$

\begin{figure}[H] 
\centering
\includegraphics[width=0.8\textwidth]{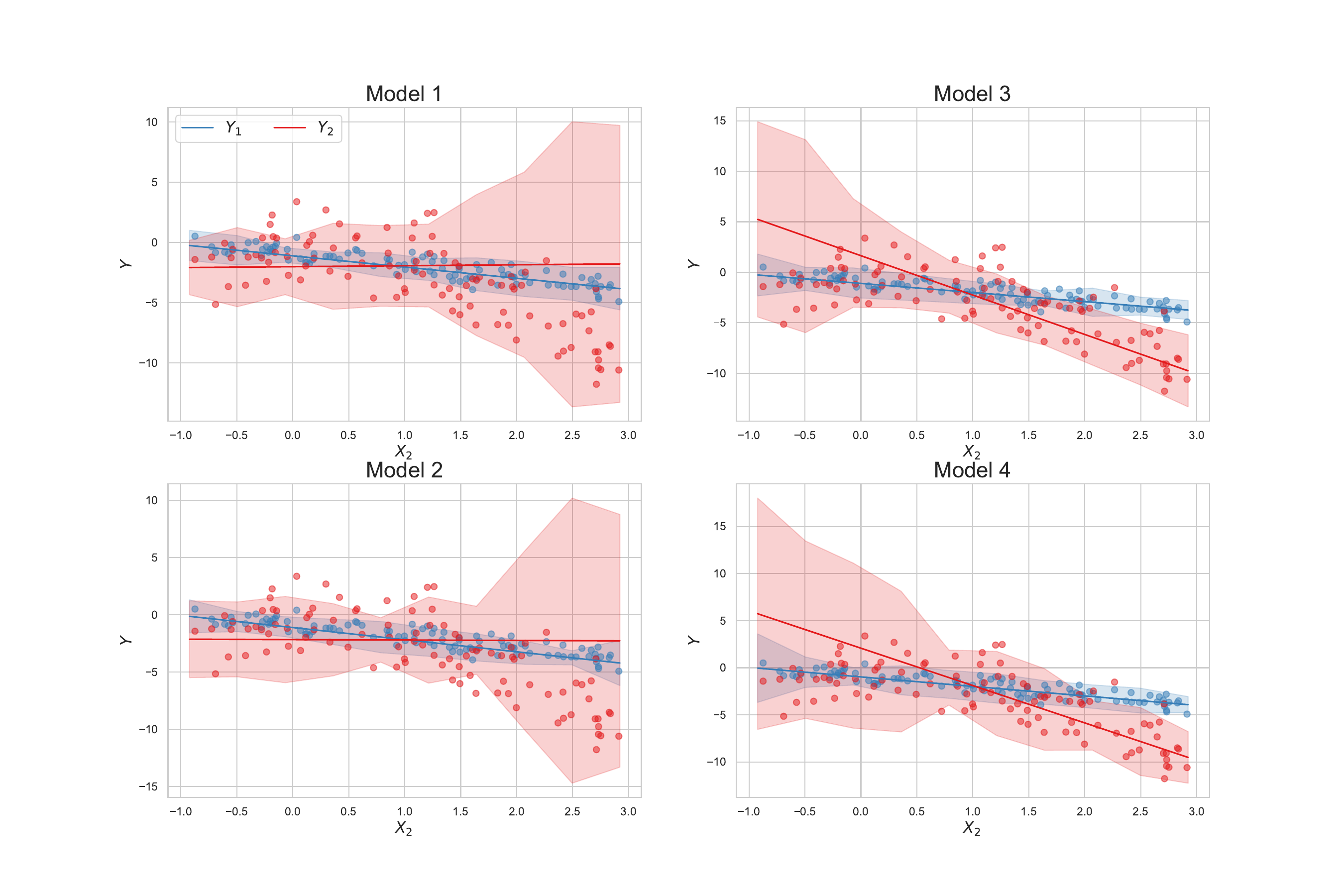}
\caption{Prediction sets with fixing $X_1 = 1$ constructed by four models trained in different datasets exhibiting covariate shift. The blue and red shaded regions represent the prediction sets for $Y_1$ and $Y_2$, respectively. The red points represent the test data.}
\label{fig:IRCS-interval}
\end{figure}

We consider the ellipsoid score $S_{\lambda}(x,y) = (y - \hat{\mu}_{\lambda}(x))^{\top}\hat{\Sigma}_{\lambda}^{-1}(x)(y - \hat{\mu}_{\lambda}(x))$, where $\hat{\mu}_{\lambda},\hat{\Sigma}_{\lambda}$ are pre-trained mean function and covariance function, respectively. Candidate models $\{S_{\lambda}:\lambda \in \Lambda\}$ are trained on four different datasets $\mathcal{D}_{\rm train}^\lambda$ for $\lambda \in \{1,...,4\}$, which are sampled from different distributions. Specifically, let $\mathcal{D}_{\rm train}^1 = \{(X,Y):X\sim N((0,0)^\top,\mathrm{I}_2)\}$, $\mathcal{D}_{\rm train}^2 = \{(X,Y):X\sim N((2,0)^\top,\mathrm{I}_2)\}$, $\mathcal{D}_{\rm train}^3 = \{(X,Y):X\sim N((0,2)^\top,\mathrm{I}_2)\}$ and $\mathcal{D}_{\rm train}^4 = \{(X,Y):X\sim N((2,2)^\top,\mathrm{I}_2)\}$. For each model, 
$\hat{\mu}_{\lambda}(\cdot)$ is fitted with least square algorithm and $\hat{\Sigma}_{\lambda}$ is fitted with random forest algorithm. The prediction sets constructed by four different models are shown in Figure \ref{fig:IRCS-interval}. In general, the models tend to fit better in regions where the training data are concentrated. For example, Model 1 has training data distributed on the left side of $X_2 = 1$, resulting in narrower prediction sets on that side. In contrast, Model 3 has training data concentrated on the right side of $X_2=1$, leading to more effective prediction sets on the right.

We examine the effect of the sample size $n$ and the nominal level $\alpha$ on decision performance,  with results shown in Figures \ref{fig:ARCSLP_samples} and \ref{fig:ARCSLP_alphas}. Both \texttt{E-CROMS} and \texttt{F-CROMS} achieve lower average decision loss than other baseline methods. The difference between the simulation results here and those in classification is that \texttt{F-CROMS} may achieve lower decision loss than \texttt{E-CROMS}. This occurs because we apply a discretization adjustment to \texttt{F-CROMS} to avoid excessive computations. It should be emphasized that, despite discretization, \texttt{F-CROMS} still maintains the coverage guarantee according to Theorem 3.6.

\begin{figure}[h] 
\centering
\includegraphics[width=\textwidth]{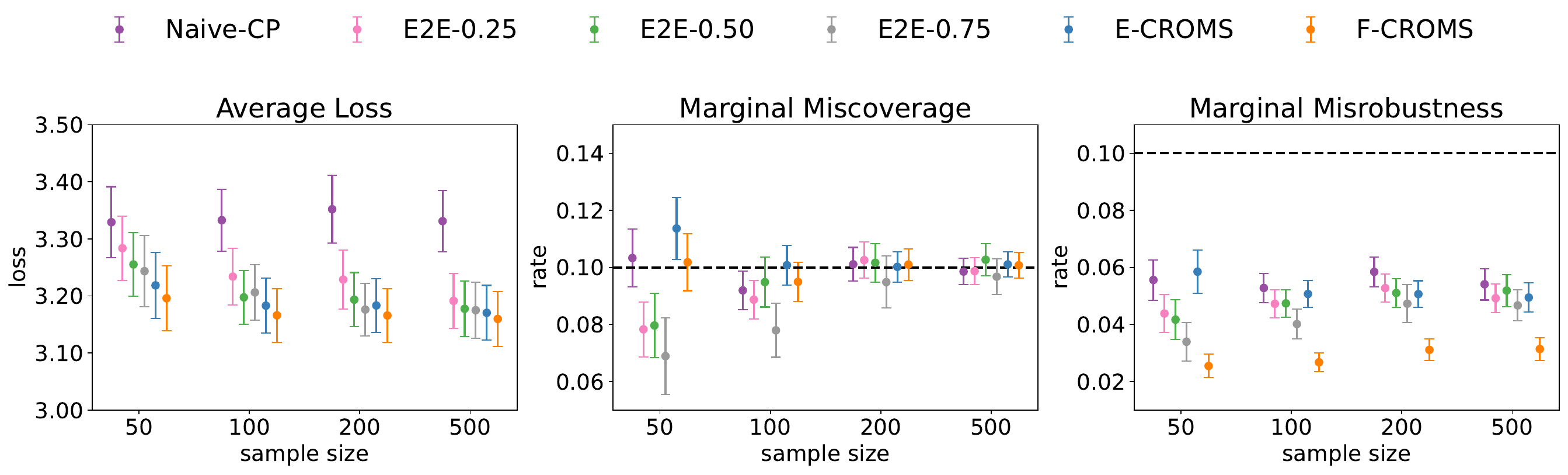}
\caption{The experimental results under $n \in \{50,100,200,500\}$, $|\Lambda| = 4$, $\alpha = 0.1$.}
\label{fig:ARCSLP_samples}
\end{figure}

\begin{figure}[h] 
\centering
\includegraphics[width=\textwidth]{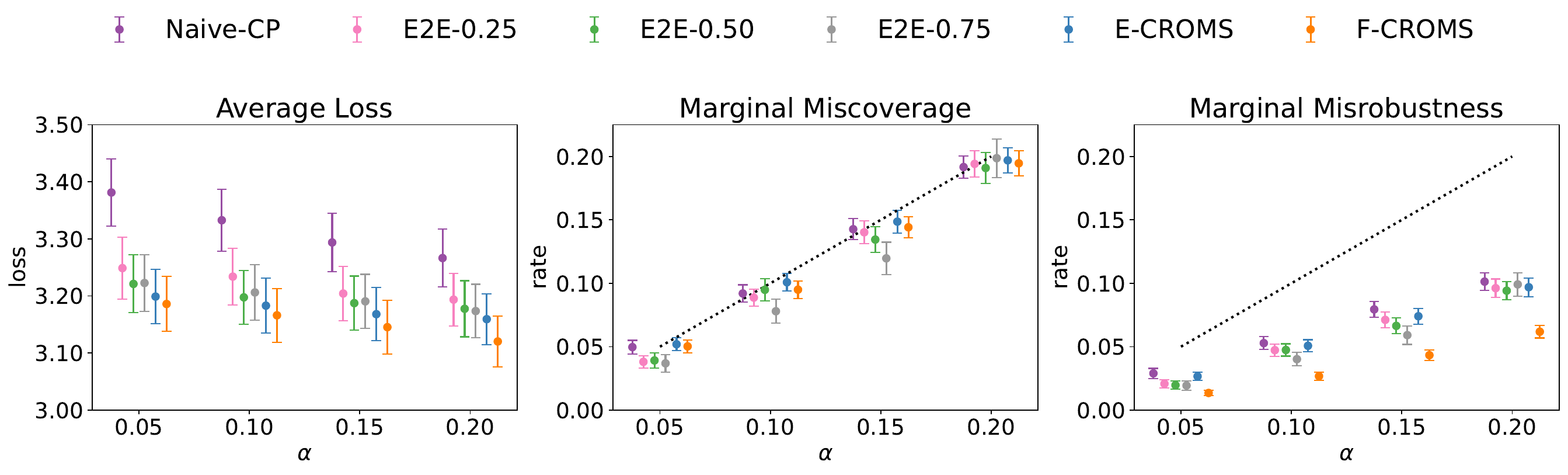}
\caption{The experimental results under $n = 100$, $|\Lambda| = 4$, $\alpha \in \{0.05,0.10,0.15,0.20  \}$.}
\label{fig:ARCSLP_alphas}
\end{figure}

\subsubsection{Additional results in Section 5.1.2}
In Section 5.1.2, we define the loss function as $\phi(y,z) = -y^{\top}z$, where $\mathcal{Y} = \mathbb{R}^2$ and $\mathcal{Z} = \{z \in [0,1]^2 : \|z\|_1 = 1, z \geq 0\}$. The labeled and test data are generated as follows:
$Y_1 = \sum_{k=1}^{50}\beta_{k1}X_k + \epsilon_1, Y_2 =\sum_{k=1}^{50}\beta_{k2}X_k + \epsilon_2$, where $\beta_{kl} = \mathbbm{1}\{(k+l) \text{ mod } 10 = 0\}$ for each $k\in[50], l\in[2]$. 
The features $\{X_k\}_{k=1}^{50}$ are independently and identically drawn from a t-distribution with 3 degrees of freedom and truncated at $|x| = 3$. That is, letting $\widetilde{X}$ follows an independent and identically distributed t-distribution, the truncated variable $X$ is defined as:

$$
X = 
\begin{cases}
    -3 &\text{if } \widetilde{X} < -3,\\
    \widetilde{X} & \text{if } -3 \leq \widetilde{X} \leq 3, \\
    3 & \text{if } \widetilde{X} > 3. \\
\end{cases}
$$
Similarly, the independent noise terms $\epsilon_1$ and $\epsilon_2$ follow a standard normal distribution truncated at $1.8$. Consequently, $Y$ is a bounded random vector, with each component having lower and upper bounds of $-16.80$ and $16.80$, respectively. Therefore, when applying grid-based algorithms (such as the F-CROMS algorithm or J-CROMS algorithm), we partition the region $[-16.80, 16.80]\times [-16.80, 16.80]$ into grid blocks. For this setup we set the number of grid points for each component to $N_{\text{grid}} = 3\sqrt{n}$, where $n$ is the number of labeled data points, and fix the number of test points per experiment at $m=100$.

We consider two different settings for candidate models. In the first setting, the candidate models are different box scores $S_{\lambda}(x,y) = \|\left(y-\hat{\mu}_{\lambda}(x)\right)/\hat{\sigma}_{\lambda}(x)\|_{\infty}$ for $\lambda \in \Lambda$, where $\hat{\mu}_{\lambda}$ and $\hat{\sigma}_{\lambda}$ are pre-trained mean and standard deviation functions, respectively. These candidate models are generated following a procedure similar to that in \citet{liang2024conformal}. Specifically, $\hat{\mu}_{\lambda}, \hat{\sigma}_{\lambda}$ are obtained by first uniformly selecting 20\% features at random, then fitting the mean and standard deviation functions on the projected data, and finally embedding the fitted functions back into the original 50-dimensional space. Thus each $\lambda \in \Lambda$ corresponds to a distinct subset of features. The sample size used for training the mean and standard deviation functions is 
$n_{\text{train}} = 300$. Following \citet{liang2024conformal}, the mean function is fitted via ridge regression (with penalty $\eta = 0.1$), and the standard deviation function is fitted via random forest algorithm. Figure \ref{fig:ACFSLP} presents the experimental results when the candidate models are box score functions.
\begin{figure}[H]
    \centering
    \begin{subfigure}[t]{\textwidth}
        \includegraphics[width=\textwidth]{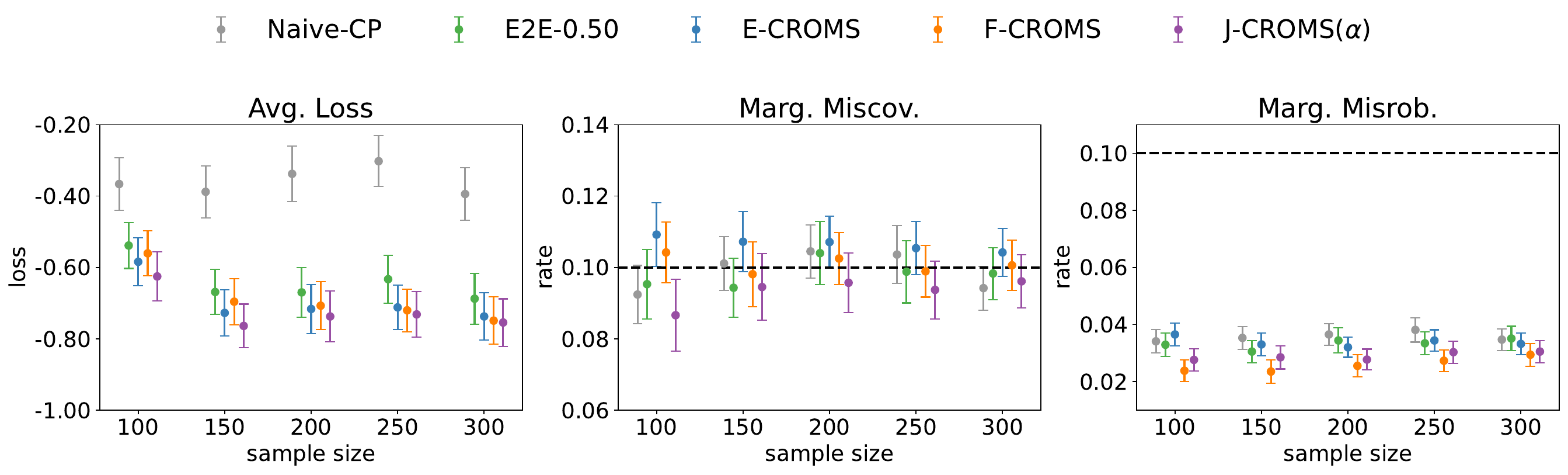}
        \caption{Varying sample size $n$ with $|\Lambda| = 30$ and $\alpha = 0.1$.}
    \end{subfigure}

    \begin{subfigure}[t]{\textwidth}
        \includegraphics[width=\textwidth]{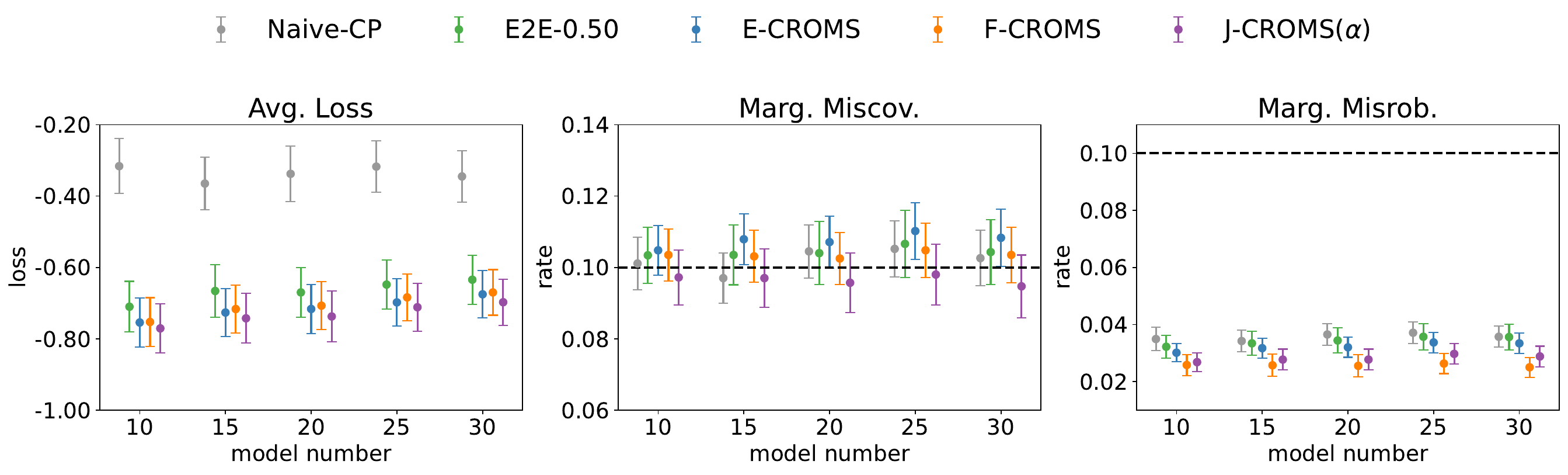}
        \caption{Varying index set size $|\Lambda|$ with $n = 100$ and $\alpha = 0.1$.}
    \end{subfigure}
    \caption{The evaluation metrics with confidence intervals in the regression task, where each candidate model corresponds to a distinct feature subset.}
    \label{fig:ACFSLP}
\end{figure}

In the second setting, the candidate models are different ellipsoid scores $S_{\lambda}(x,y) = \left(y-\hat{\mu}_{\lambda}(x)\right)^{\top} \widehat{\Sigma}_{\lambda}(X)\left( y- \hat{\mu}_{\lambda}(x) \right)$ for $\lambda \in \Lambda$, where $\hat{\mu}_{\lambda}$ and $\hat{\sigma}_{\lambda}$ are pre-trained mean and covariance functions, respectively. The procedure for generating these candidate models remains essentially the same as described earlier, except that the training objective shifts from estimating the mean and standard deviation functions to estimating the mean and covariance functions. All other aspects of the setup are unchanged. Table \ref{table:regression_ellipsoid_performance} presents the experimental results when the candidate models are ellipsoid score functions.

\begin{table}[H]
    \centering
    \small
    \setlength{\tabcolsep}{4pt}
    \caption{Evaluation metrics and runtimes (seconds) on 100 test points, with the radius of 95\% confidence intervals (in parentheses), are reported for the {\it regression} task using the {\it ellipsoid score} function in Section 5.1.2, under the scenario $n=150$, $|\Lambda|=25$.}
    \label{table:regression_ellipsoid_performance}
    \resizebox{\textwidth}{!}{
    \begin{tabular}{@{}clccccc@{}}
    \toprule
    $\alpha$ & \textbf{Method} & \textbf{Avg. Loss} & \textbf{Marg. Miscov.} & \textbf{Marg. Misrob.} & \textbf{Time} \\
    \midrule
    \multirow{8}{*}{0.10} 
    & Naive-CP & -0.172 (0.050) & 0.101 (0.008) & 0.016 (0.003) & 1.294 (0.019) \\
    \cmidrule(lr){2-6}
    & E2E-0.25 & -0.316 (0.052) & 0.105 (0.009) & 0.018 (0.003) & 9.936 (0.021) \\
    & E2E-0.50 & -0.325 (0.049) & 0.097 (0.010) & 0.015 (0.003) & 18.467 (0.039) \\
    & E2E-0.75 & -0.307 (0.051) & 0.078 (0.011) & 0.012 (0.003) & 26.812 (0.054) \\
    \cmidrule(lr){2-6}
    & LOO & -0.325 (0.048) & 0.102 (0.008) & 0.015 (0.003) & 73.510 (0.866) \\
    \cmidrule(lr){2-6}
    & E-CROMS & -0.332 (0.048) & 0.112 (0.008) & 0.017 (0.003) & 35.045 (0.204) \\
    & F-CROMS & -0.385 (0.054) & 0.101 (0.008) & 0.013 (0.002) & 5467.818 (236.380) \\
    & J-CROMS($\alpha$) & \textbf{-0.410} (0.051) & 0.096 (0.009) & 0.014 (0.002) & 130.505 (0.561) \\
    & J-CROMS($\alpha/2$) & -0.345 (0.051) & 0.045 (0.007) & 0.006 (0.002) & 131.828 (0.628) \\
    & CV-CROMS($K=5$) & \textbf{-0.405} (0.054) & 0.067 (0.008) & 0.008 (0.002) & 220.741 (1.804) \\
    & CV-CROMS($K=10$) & \textbf{-0.402} (0.052) & 0.080 (0.009) & 0.010 (0.002) & 380.393 (3.718) \\
    \cmidrule(lr){1-6}
    \multirow{8}{*}{0.20} 
    & Naive-CP & -0.204 (0.052) & 0.205 (0.009) & 0.036 (0.004) & 1.397 (0.048) \\
    \cmidrule(lr){2-6}
    & E2E-0.25 & -0.374 (0.052) & 0.196 (0.011) & 0.037 (0.004) & 10.018 (0.024) \\
    & E2E-0.50 & -0.395 (0.052) & 0.201 (0.012) & 0.036 (0.004) & 18.648 (0.047) \\
    & E2E-0.75 & -0.374 (0.050) & 0.187 (0.015) & 0.034 (0.005) & 27.063 (0.061) \\
    \cmidrule(lr){2-6}
    & LOO & -0.412 (0.049) & 0.208 (0.011) & 0.038 (0.004) & 75.674 (0.862) \\
    \cmidrule(lr){2-6}
    & E-CROMS & -0.418 (0.049) & 0.219 (0.011) & 0.041 (0.004) & 35.428 (0.114) \\
    & F-CROMS & -0.488 (0.054) & 0.209 (0.010) & 0.035 (0.004) & 2563.318 (117.992) \\
    & J-CROMS($\alpha$) & \textbf{-0.496} (0.056) & 0.207 (0.011) & 0.036 (0.004) & 131.979 (0.553) \\
    & J-CROMS($\alpha/2$) & -0.410 (0.051) & 0.096 (0.009) & 0.014 (0.002) & 133.461 (0.578) \\
    & CV-CROMS($K=5$) & \textbf{-0.517} (0.056) & 0.166 (0.012) & 0.024 (0.004) & 224.200 (1.780) \\
    & CV-CROMS($K=10$) & \textbf{-0.511} (0.054) & 0.180 (0.012) & 0.029 (0.004) & 387.371 (3.694) \\
    \bottomrule
    \end{tabular}}
    \end{table}

\subsection{Model selection in the individualized case} \label{appen:ind_regression_covariate_shit}

\subsubsection{Classification task}

In Section 5.2, the covariate $X = (X_1,X_2,X_3)$ follows a multivariate normal distribution $N(0,\Sigma)$ and the corresponding covariance matrix is $\Sigma = L L^{\top}$, where $L$ is set as follows:
$$
L =
\begin{pmatrix}
  1.5 & 0.1 & -0.2 \\
  0.1 & 2.0 & 0.4  \\
 -0.2 & 0.4 & 3.0
\end{pmatrix}.
$$
The loss function is $\phi(y,z) = M_{yz}$ where the loss matrix is 
$$
M=
\begin{pmatrix}
  0 & 4 & 10 \\
  2 & 0 & 9  \\
  7 & 6 & 0
\end{pmatrix}.
$$
In this setting, there are three candidate models, namely:
\begin{itemize}
    \item $S_1$ : trained with the target variable $Y$ and the covariate $(X_1,X_2)$.
    
    \item $S_2$ : trained with the target variable $Y$ and the covariate $(X_1,X_3)$.
    
    \item $S_3$ : trained with the target variable $Y$ and the covariate $(X_2,X_3)$.
\end{itemize}
The training sample size is $400$. The test sample size is $m=1000$. A larger test sample size is adopted to ensure that the empirical estimates of the conditional metrics can better approximate their expectation values. Every ball $B \in \mathcal{B}$ contains $10\%$ test samples, and $|\mathcal{B}| = 20$.
Additionally, we consider the coverage gap metric ``CovGap'' from \citet{kaur2025conformal}. First, we define a new partition $\mathcal{G} = \{G_1,G_2,G_3,G_4\}$, where $G_1 = \{X \in \mathbb{R}^{3}:X_1 \geq 1.2\}, G_2 = \{X \in \mathbb{R}^3:0 \leq X_1 < 1.2\}, G_3 = \{X \in \mathbb{R}^3: -1.2 \leq X_1 < 0\}, G_4 = \{X \in \mathbb{R}^3: X_1 < -1.2\}$. Then the coverage gap is defined as: $\textit{CovGap} = \sum_{g=1}^{4} |\hat{c}_g - (1-\alpha)|$,
where $\hat{c}_g = \sum_{j=n+1}^{n+m} \mathbbm{1}\{X_j \in G_g, Y_j \in \widehat{\mathcal{U}}(X_j)\}/\sum_{j=n+1}^{n+m} \mathbbm{1} \{X_j \in G_g\}$ denotes the empirical coverage in $G_g$ for $g \in [4]$. Similarly, the robustness gap metric is defined as $\textit{RobGap} = \sum_{g=1}^{4} |\hat{r}_g - (1-\alpha)|$,
where $\hat{r}_g = \sum_{j=n+1}^{n+m} \mathbbm{1}\{X_j \in G_g, \phi(Y_{j}, \hat{z}(X_{j})) \leq \max_{c\in \widehat{\gU}(X_{j})}\phi(c,\hat{z}(X_{j}))\}/\sum_{j=n+1}^{n+m} \mathbbm{1} \{X_j \in G_g\}$ denotes the empirical robustness in $G_g$ for $g \in [4]$.

As shown in Figure \ref{fig:ICFSM-coveragegap}, \texttt{Naive-LCP} and \texttt{CROiMS} both maintain coverage rates around $1 - \alpha$ across different regions. 
As shown in Figure \ref{fig: ICFSM} (a), the averaged model selection methods tend to be overly conservative in some regions ($G_2,G_3$) while failing to control the miscoverage rate under the nominal level $\alpha$ in others ($G_1,G_4$). Finally, Figure \ref{fig: ICFSM} (b) demonstrates that \texttt{CROiMS} achieves lower decision loss across all regions.



\begin{figure}[H] 
\centering
\includegraphics[width=\textwidth]{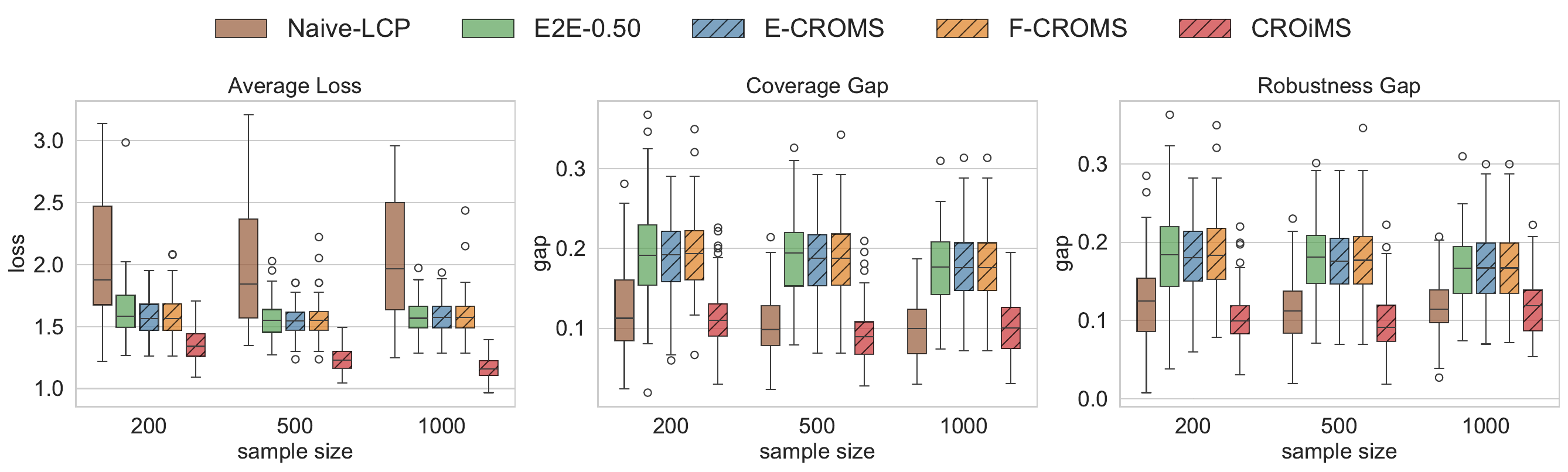}
\caption{The evaluation metrics for varying sample size $n$ with $|\Lambda| = 3$.}
\label{fig:ICFSM-coveragegap}
\end{figure}

\begin{figure}[H]
    \begin{subfigure}[t]{\textwidth}
        \centering
        \includegraphics[width=.7\textwidth]{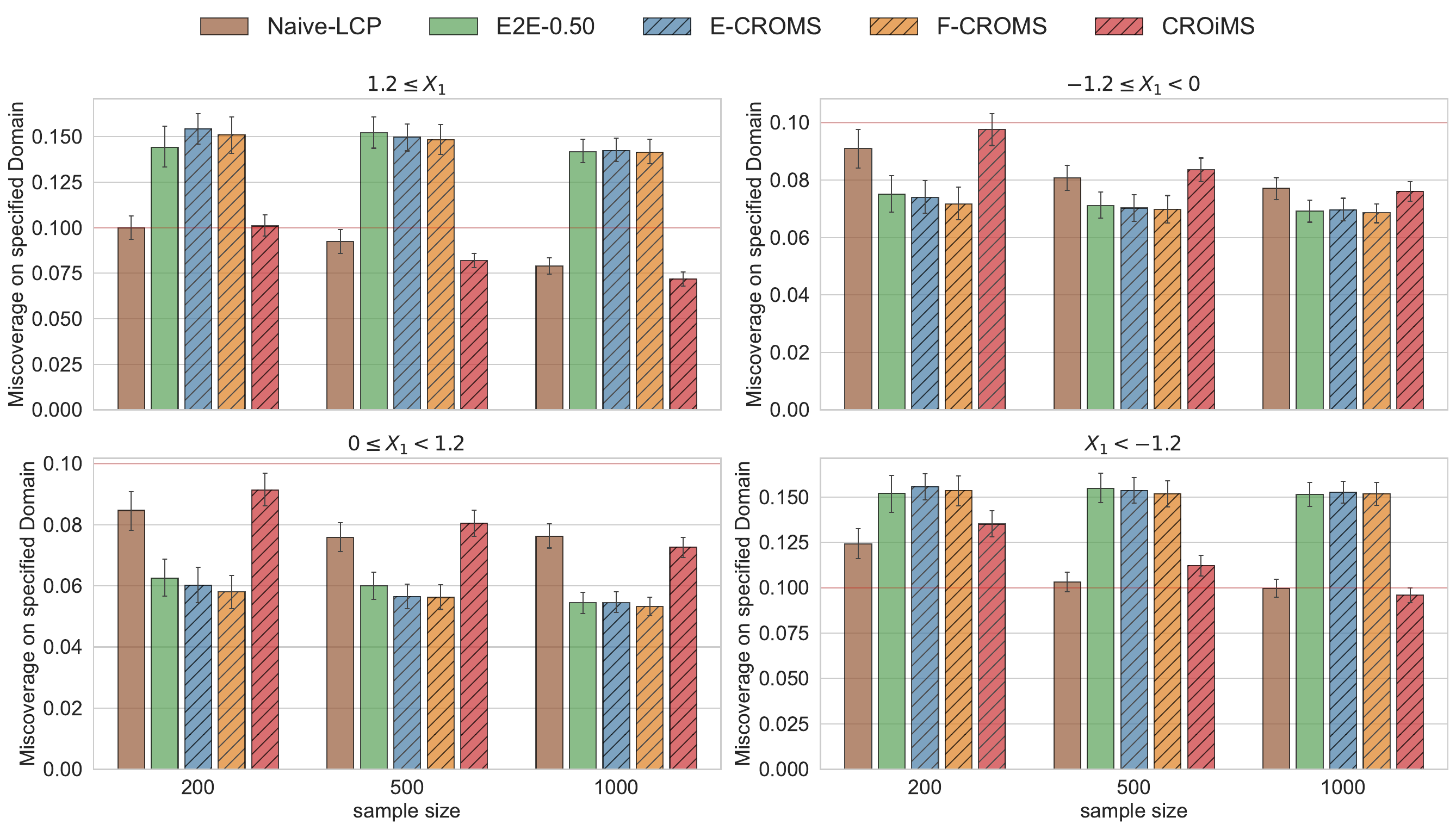}
        \caption{Conditional miscoverage}
    \end{subfigure}

    \begin{subfigure}[t]{\textwidth}
    \centering
        \includegraphics[width=.7\textwidth]{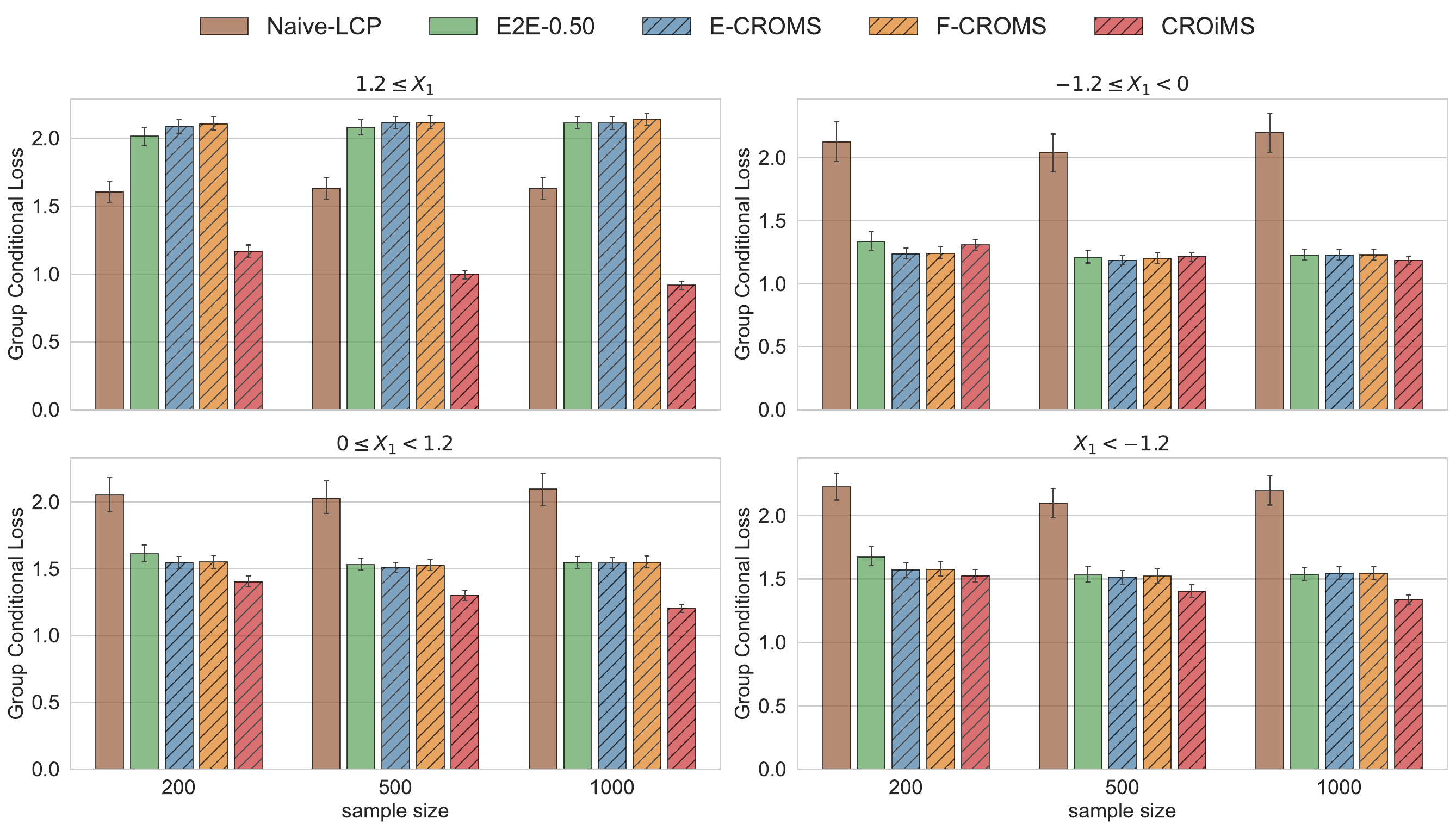}
        \caption{Conditional loss}
    \end{subfigure}
    \caption{The simulation results of the conditional miscoverage and loss on specific domains with varying sample size $n$ and fixed $|\Lambda|=3$ and $\alpha = 0.1$.}
    \label{fig: ICFSM}
\end{figure}

\begin{table}[H]
\centering
\small
\setlength{\tabcolsep}{3pt}
\caption{Performance of CROiMS with varying constant \(c\) in bandwidth \(h_n = c n^{-1/5}\) under classification task in Section 5.2.}
\label{tab: sensitivity_bandwidth_c}
\resizebox{.8\textwidth}{!}{
\begin{tabular}{ccccccccc}
\toprule
\multirow{2}{*}{Metrics} & \multirow{2}{*}{Sample size} & \multirow{2}{*}{Nominal level} & \multicolumn{5}{c}{Constant \(c\)} \\
\cmidrule(lr){4-8}
 & & & 4 & 5 & 6 & 7 & 8 \\
\midrule
\multirow{9}{*}{Avg. Loss}
& \multirow{3}{*}{\(n=500\)}
& $\alpha=0.10$ & 1.306 & 1.263 & 1.234 & 1.233 & 1.242 \\
& & $\alpha=0.15$ & 1.310 & 1.226 & 1.189 & 1.182 & 1.199 \\
& & $\alpha= 0.20$ & 1.351 & 1.250 & 1.209 & 1.202 & 1.211 \\
\arrayrulecolor{black!40}\cmidrule(lr){4-8}
& \multirow{3}{*}{\(n=1000\)}
& $\alpha=0.10$ & 1.226 & 1.174 & 1.163 & 1.170 & 1.183 \\
& & $\alpha=0.15$ & 1.233 & 1.155 & 1.124 & 1.120 & 1.140 \\
& & $\alpha=0.20$ & 1.264 & 1.173 & 1.142 & 1.137 & 1.156 \\
\cmidrule(lr){4-8}
& \multirow{3}{*}{\(n=1500\)}
& $\alpha=0.10$ & 1.190 & 1.149 & 1.141 & 1.139 & 1.160 \\
& & $\alpha=0.15$ & 1.184 & 1.118 & 1.096 & 1.096 & 1.114 \\
& & $\alpha=0.20$ & 1.220 & 1.136 & 1.113 & 1.107 & 1.120 \\
\midrule
\multirow{9}{*}{Worst Cond. Miscov.}
& \multirow{3}{*}{\(n=500\)}
& $\alpha=0.10$ & 0.117 & 0.119 & 0.122 & 0.131 & 0.140 \\
& & $\alpha=0.15$ & 0.168 & 0.170 & 0.175 & 0.182 & 0.193 \\
& & $\alpha=0.20$ & 0.217 & 0.218 & 0.232 & 0.239 & 0.247 \\
\cmidrule(lr){4-8}
& \multirow{3}{*}{\(n=1000\)}
& $\alpha=0.10$ & 0.102 & 0.103 & 0.111 & 0.116 & 0.129 \\
& & $\alpha=0.15$ & 0.150 & 0.154 & 0.161 & 0.165 & 0.173 \\
& & $\alpha=0.20$ & 0.198 & 0.205 & 0.211 & 0.218 & 0.227 \\
\cmidrule(lr){4-8}
& \multirow{3}{*}{\(n=1500\)}
& $\alpha=0.10$ & 0.109 & 0.110 & 0.112 & 0.116 & 0.118 \\
& & $\alpha=0.15$ & 0.156 & 0.159 & 0.162 & 0.166 & 0.172 \\
& & $\alpha=0.20$ & 0.200 & 0.208 & 0.214 & 0.219 & 0.225 \\
\midrule
\multirow{9}{*}{Worst Cond. Misrob.}
& \multirow{3}{*}{\(n=500\)}
& $\alpha=0.10$ & 0.112 & 0.112 & 0.115 & 0.125 & 0.133 \\
& & $\alpha=0.15$ & 0.157 & 0.162 & 0.165 & 0.174 & 0.187 \\
& & $\alpha=0.20$ & 0.207 & 0.209 & 0.223 & 0.229 & 0.237 \\
\cmidrule(lr){4-8}
& \multirow{3}{*}{\(n=1000\)}
& $\alpha=0.10$ & 0.098 & 0.096 & 0.103 & 0.108 & 0.121 \\
& & $\alpha=0.15$ & 0.138 & 0.141 & 0.149 & 0.156 & 0.166 \\
& & $\alpha=0.20$ & 0.185 & 0.191 & 0.196 & 0.205 & 0.214 \\
\cmidrule(lr){4-8}
& \multirow{3}{*}{\(n=1500\)}
& $\alpha=0.10$ & 0.102 & 0.104 & 0.106 & 0.108 & 0.111 \\
& & $\alpha=0.15$ & 0.144 & 0.149 & 0.152 & 0.155 & 0.161 \\
& & $\alpha=0.20$ & 0.188 & 0.194 & 0.199 & 0.207 & 0.214 \\
\arrayrulecolor{black}\bottomrule
\end{tabular}}
\end{table}

The bandwidth $h = 6.06 n^{-1/5}$. Under this bandwidth, the corresponding effective sample size is $n_{\text{eff}} \approx 50$ when the sample size is $n=200$. The effect of bandwidth choice is summarized in Table \ref{tab: sensitivity_bandwidth_c}.

\subsubsection{Regression task}
In this experiment, we used the same setting in Section \ref{appen:gen_regression}.
The labeled sample size is $n=500$, and the test sample size is $m=200$. Every ball $B \in \mathcal{B}$ contains $20\%$ test samples, and $|\mathcal{B}| = 20$. The simulation results for varying labeled sample sizes are shown in Figure \ref{fig:IRCSM-loss}.
Similar to the results in the classification task, only \texttt{Naive-LCP} and \texttt{CROiMS} achieve worst-case conditional miscoverage close to the nominal level $\alpha=0.1$ as $n$ increases, especially at $n=500$. Moreover, we find that the individualized model selection method \texttt{CROiMS} outperforms all other methods and results in the lowest average loss. Additionally, we observe that all methods yield worst-case conditional misrobustness much below the nominal level, implying that there exists a gap between coverage and robustness level. How to bridge this gap to obtain better robustness warrants further study.  Let $\mathcal{G} = \{G_1,G_2,G_3,G_4\}$, where $G_1 = \{X \in \mathbb{R}^2: X_1 \geq 1\}$, $G_2 = \{X \in \mathbb{R}^2: X_1 < 1\}$, $G_3 = \{X \in \mathbb{R}^2: X_2 \geq 1\}$, $G_4 = \{X \in \mathbb{R}^2: X_2 < 1\}$, i.e., using the same partitioning scheme as in the main text.

\begin{figure}[H] 
\centering
\includegraphics[width=1.0\textwidth]{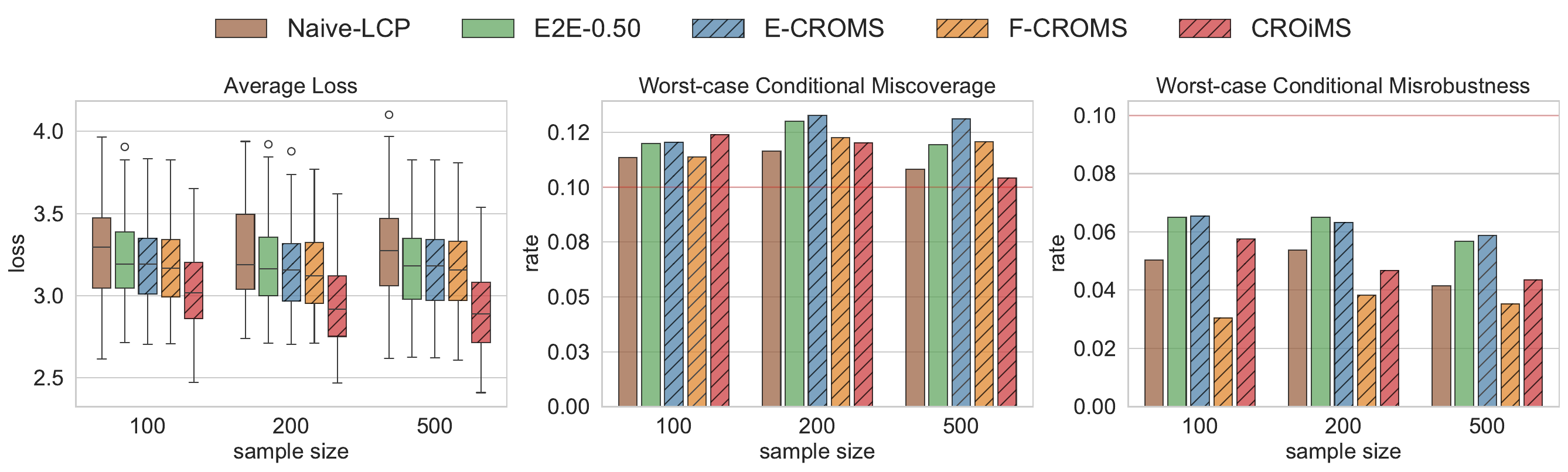}
\caption{The evaluation metrics with varying sample size $n$ in the regression task and $|\Lambda| = 4$, $\alpha = 0.1$.}
\label{fig:IRCSM-loss}
\end{figure}

As shown in Figures \ref{fig:IRCSM-coveragegap} and \ref{fig: IRCSM}, \texttt{CROiMS} performs better than the other methods, achieving stable coverage guarantees across all regions and yielding lower decision loss. Moreover, \texttt{CROiMS} may exhibit a larger robustness gap compared to other methods. This occurs because the robustness rate is always higher than the coverage rate in this regression.

\begin{figure}[H] 
\centering
\includegraphics[width=.9\textwidth]{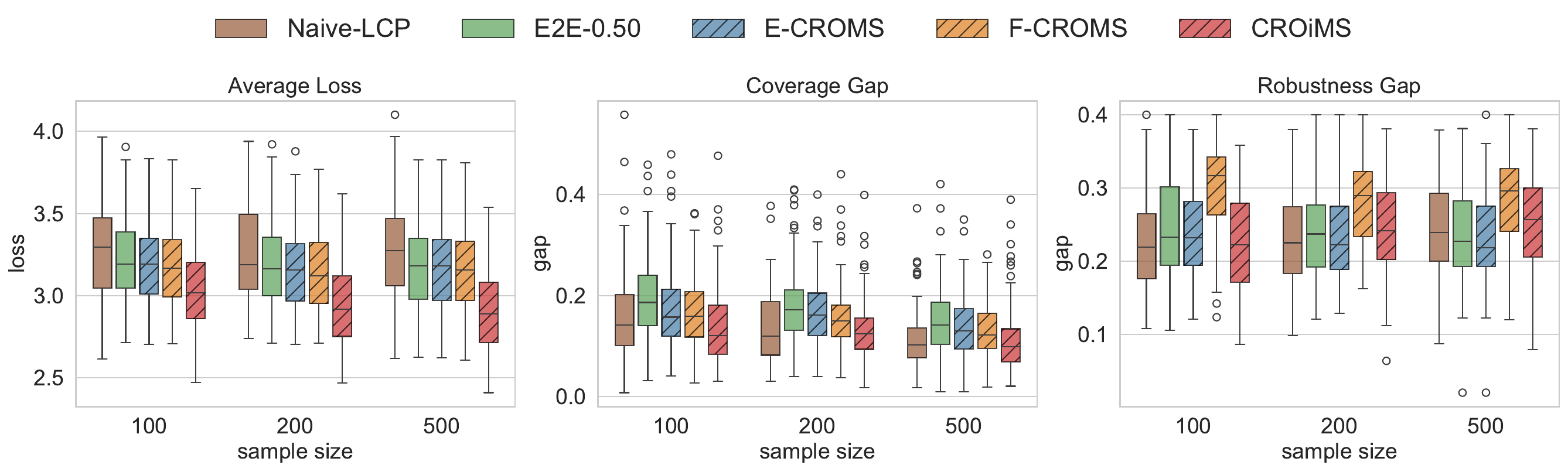}
\caption{The evaluation metrics with varying sample size $n$ and fixed $|\Lambda| = 4$.}
\label{fig:IRCSM-coveragegap}
\end{figure}



 \begin{figure}[H]
    
    \begin{subfigure}[t]{\textwidth}
        \centering
        \includegraphics[width=.75\textwidth]{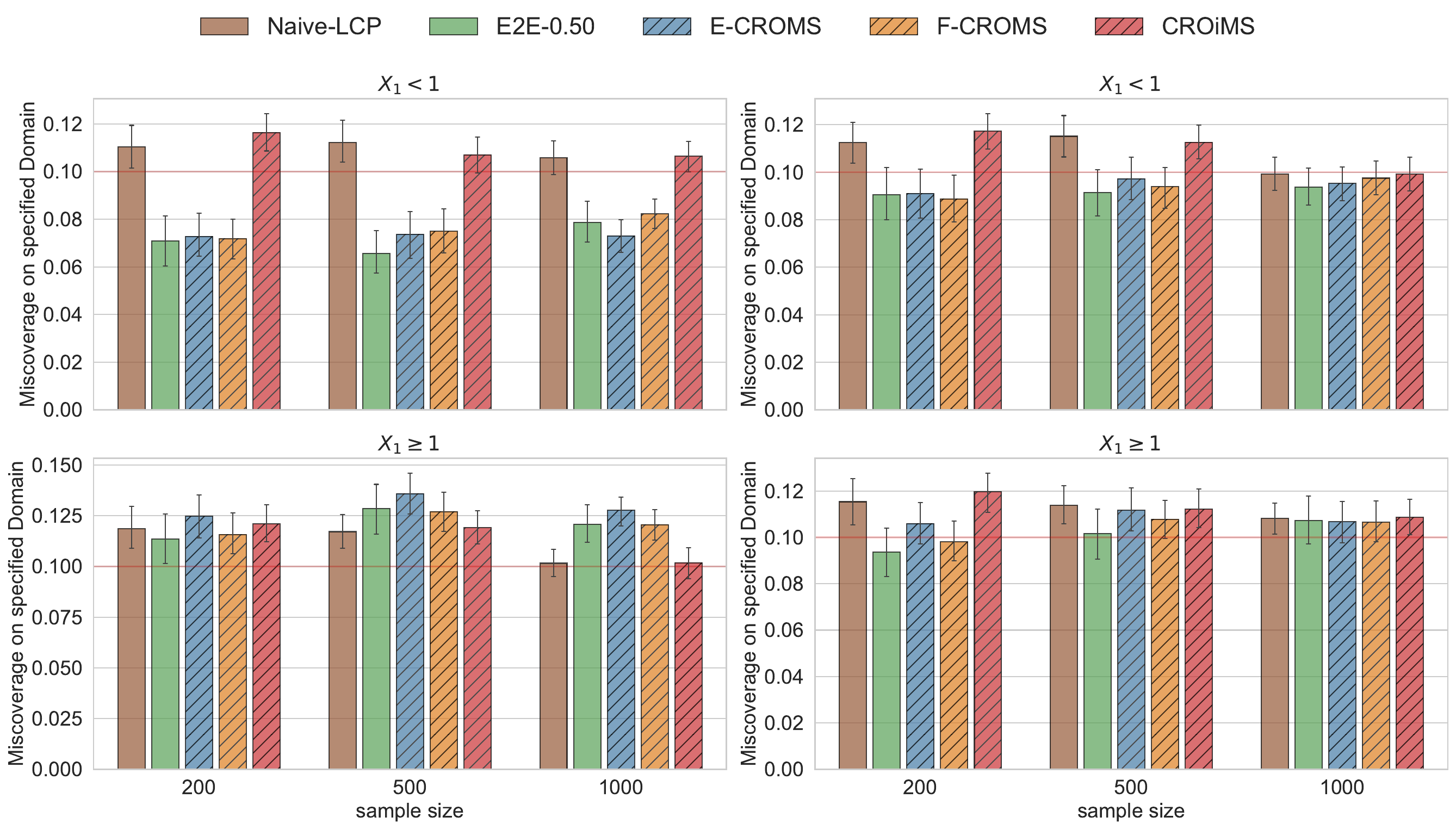}
        \caption{Conditional miscoverage}
    \end{subfigure}

    \begin{subfigure}[t]{\textwidth}
    \centering
        \includegraphics[width=.75\textwidth]{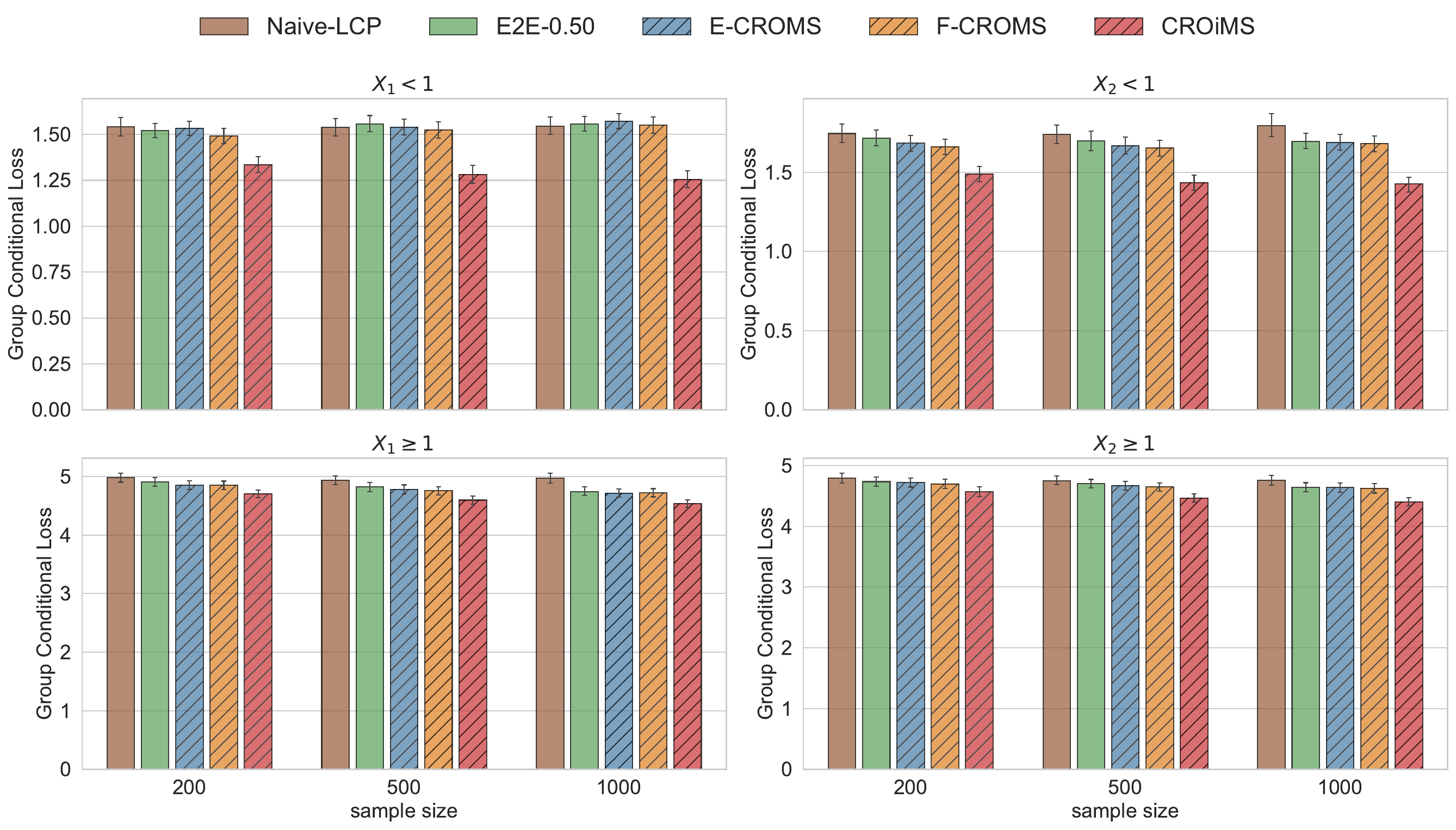}
        \caption{Conditional loss}
    \end{subfigure}
    \caption{The simulation results of conditional miscoverage and loss on specific domains with varying sample size $n$ and fixed $|\Lambda|=4$ and $\alpha = 0.1$.}
    \label{fig: IRCSM}
\end{figure}

The bandwidth for \texttt{CROiMS} is $h = 5.38 n^{-1/4}$ by the rule of effective sample size. Under this bandwidth, the corresponding effective sample size is $\hat{n}_{\text{eff}} \approx 50$ when $n=100$. The effect of bandwidth choice is given in Table \ref{tab: sensitivity_bandwidth_r}.
\begin{table}[H]
\centering
\small
\setlength{\tabcolsep}{3pt}
\caption{Performance of CROiMS with varying constant \(c\) in bandwidth \(h_n = c n^{-1/4}\) under regression task in Appendix G.2.2.}
\label{tab: sensitivity_bandwidth_r}
\resizebox{.8\textwidth}{!}{
\begin{tabular}{lcccccccc}
\toprule
\multirow{2}{*}{Metrics} & \multirow{2}{*}{Sample size} & \multirow{2}{*}{Nominal level} & \multicolumn{5}{c}{Constant \(c\)} \\
\cmidrule(lr){4-8}
 & & & 4 & 5 & 6 & 7 & 8 \\
\midrule
\multirow{9}{*}{Avg. Loss}
& \multirow{3}{*}{\(n=200\)}
& $\alpha=0.10$ & 2.897 & 2.913 & 2.932 & 2.954 & 2.973 \\
& & $\alpha=0.15$ & 2.890 & 2.903 & 2.922 & 2.941 & 2.960 \\
& & $\alpha=0.20$ & 2.880 & 2.898 & 2.916 & 2.933 & 2.947 \\
\arrayrulecolor{black!40}\cmidrule(lr){4-8}
& \multirow{3}{*}{\(n=500\)}
& $\alpha=0.10$ & 2.868 & 2.888 & 2.906 & 2.920 & 2.938 \\
& & $\alpha=0.15$ & 2.862 & 2.879 & 2.896 & 2.913 & 2.927 \\
& & $\alpha=0.20$ & 2.850 & 2.868 & 2.888 & 2.904 & 2.917 \\
\cmidrule(lr){4-8}
& \multirow{3}{*}{\(n=1000\)}
& $\alpha=0.10$ & 2.842 & 2.855 & 2.875 & 2.891 & 2.906 \\
& & $\alpha=0.15$ & 2.834 & 2.849 & 2.863 & 2.882 & 2.897 \\
& & $\alpha=0.20$ & 2.825 & 2.837 & 2.853 & 2.872 & 2.889 \\
\midrule
\multirow{9}{*}{Worst Cond. Miscov.}
& \multirow{3}{*}{\(n=200\)}
& $\alpha=0.10$ & 0.135 & 0.125 & 0.120 & 0.119 & 0.118 \\
& & $\alpha=0.15$ & 0.182 & 0.172 & 0.176 & 0.177 & 0.176 \\
& & $\alpha=0.20$ & 0.231 & 0.225 & 0.221 & 0.228 & 0.230 \\
\cmidrule(lr){4-8}
& \multirow{3}{*}{\(n=500\)}
& $\alpha=0.10$ & 0.119 & 0.107 & 0.107 & 0.107 & 0.112 \\
& & $\alpha=0.15$ & 0.168 & 0.158 & 0.158 & 0.163 & 0.168 \\
& & $\alpha=0.20$ & 0.217 & 0.211 & 0.218 & 0.221 & 0.229 \\
\cmidrule(lr){4-8}
& \multirow{3}{*}{\(n=1000\)}
& $\alpha=0.10$ & 0.116 & 0.114 & 0.107 & 0.107 & 0.101 \\
& & $\alpha=0.15$ & 0.162 & 0.165 & 0.163 & 0.164 & 0.159 \\
& & $\alpha=0.20$ & 0.215 & 0.218 & 0.214 & 0.210 & 0.212 \\
\midrule
\multirow{9}{*}{Worst Cond. Miscob.}
& \multirow{3}{*}{\(n=200\)}
& $\alpha=0.10$ & 0.055 & 0.052 & 0.048 & 0.047 & 0.050 \\
& & $\alpha=0.15$ & 0.079 & 0.071 & 0.072 & 0.073 & 0.075 \\
& & $\alpha=0.20$ & 0.100 & 0.098 & 0.096 & 0.097 & 0.100 \\
\cmidrule(lr){4-8}
& \multirow{3}{*}{\(n=500\)}
& $\alpha=0.10$ & 0.049 & 0.043 & 0.044 & 0.046 & 0.046 \\
& & $\alpha=0.15$ & 0.066 & 0.066 & 0.065 & 0.070 & 0.070 \\
& & $\alpha=0.20$ & 0.091 & 0.090 & 0.089 & 0.092 & 0.094 \\
\cmidrule(lr){4-8}
& \multirow{3}{*}{\(n=1000\)}
& $\alpha=0.10$ & 0.041 & 0.040 & 0.040 & 0.044 & 0.044 \\
& & $\alpha=0.15$ & 0.064 & 0.061 & 0.061 & 0.065 & 0.066 \\
& & $\alpha=0.20$ & 0.090 & 0.088 & 0.088 & 0.090 & 0.094 \\
\arrayrulecolor{black}\bottomrule
\end{tabular}}
\end{table}

\subsection{Deferred experiment settings}
\subsubsection{Deferred settings and results in Section 5.1.1}\label{appen:setting_simulation_GCP}

We begin by outlining the score function introduced in \citet{cortes2024decision}. Suppose that there is a label-penalty function $\ell: \mathcal{Y} \rightarrow \mathbb{R}^{+}$. This function assigns varying penalties to different labels, ensuring cautious decision-making in sensitive scenarios. For instance, in initial medical screenings, misclassifying a patient as having a serious condition carries greater consequences. Thus, the penalty for labeling a case as ``critical illness'' is set higher, nudging the model toward safer intermediate outcomes like ``recommend additional tests''. 

Suppose that $f:\mathcal{X} \rightarrow [0,1]^{|\mathcal{Y}|}$ is the softmax layer of a classifier, where $\mathcal{Y} = \{1,...,K\}$ is the label set. Given $X = x$, let $\{\sigma_1(x),...,\sigma_K(x)\}$ be the permutation that orders the probabilities $f^{1}(x),...,f^{K}(x)$ from greatest to lowest. Define $\rho(x,y) = \Sigma_{i=1}^{r} f^{\sigma_i(x)}(x)$ where $\sigma_r(x) = y$ and define $L(y)$ as $\Sigma_{i=1}^{r} \ell(\sigma_i(x))$. Then the score function is defined as
$$
S_{\lambda}(x,y) :=  \rho(x,y) + \lambda L(y),
$$
where $\lambda \in \mathbb{R}^{+}$ is the pre-specified score penalty.
Here, we simply set $\ell(y) = y$ for $y \in \mathcal{Y} = \{1,2,3,4,5\}$. The classifier $f$ is trained with Gradient Boosting algorithm. Candidate models $\{S_{\lambda}, \lambda \in \Lambda\}$ are established with different score penalties $\lambda \in [0,0.2]$. The score penalty $\lambda \in \Lambda$ is always obtained from a uniform grid over the interval $[0,0.2]$. For example, if we set $|\Lambda| = 6$, then $\lambda \in \Lambda=\{0.00,0.04,0.08,0.12,0.16,0.20\}$. Finally, the training sample size is $400$, the test sample size is $m=100$. The loss matrix is
$$
M=
\begin{pmatrix}
  0 & 3 & 5 & 7 & 10\\
  2 & 0 & 4 & 6 & 9 \\
  2.5 & 4.5 & 0 & 7 & 8\\
  3 & 5 & 6 & 0 & 7 \\
  3.5 & 6 & 8 & 10 & 0
\end{pmatrix}.
$$
The functions $v_k(\cdot),k=1,...,5$ in the data generation have the following form:
$$
v_k(X) = A_{k1} + A_{k2}X_1 + (A_{k3} + A_{k4}X_2)X_5 + (A_{k5} + A_{k6}X_3)X_6 + (A_{k7} + A_{k8}X_4)X_7,
$$
for $k=1,...,5$, where the coefficient matrix $A$ is set as follows:
$$
A =
\begin{pmatrix}
  0 & 1 & 0 & 0 & 2 & 3 & 3 & 3 \\
  0 & 1 & 1 & 4 & 0 & 0 & 2 & 5  \\
  0 & 1 & 6 & -4 & 6 & -5 & 7 & -4 \\
  1 & -1 & 0 & 3 & 1 & 5 & 4 & 1 \\
  1 & -1 & 1 & 6 & 0 & 3 & 2 & 4 
\end{pmatrix}.
$$
We examine the impact of the nominal level $\alpha$ on decision performance, with corresponding results shown in Figure \ref{fig:ACSSLP-2}. Clearly, \texttt{E-CROMS} and \texttt{F-CROMS} outperform other baselines.

\begin{figure}[H] 
\centering
\includegraphics[width=\textwidth]{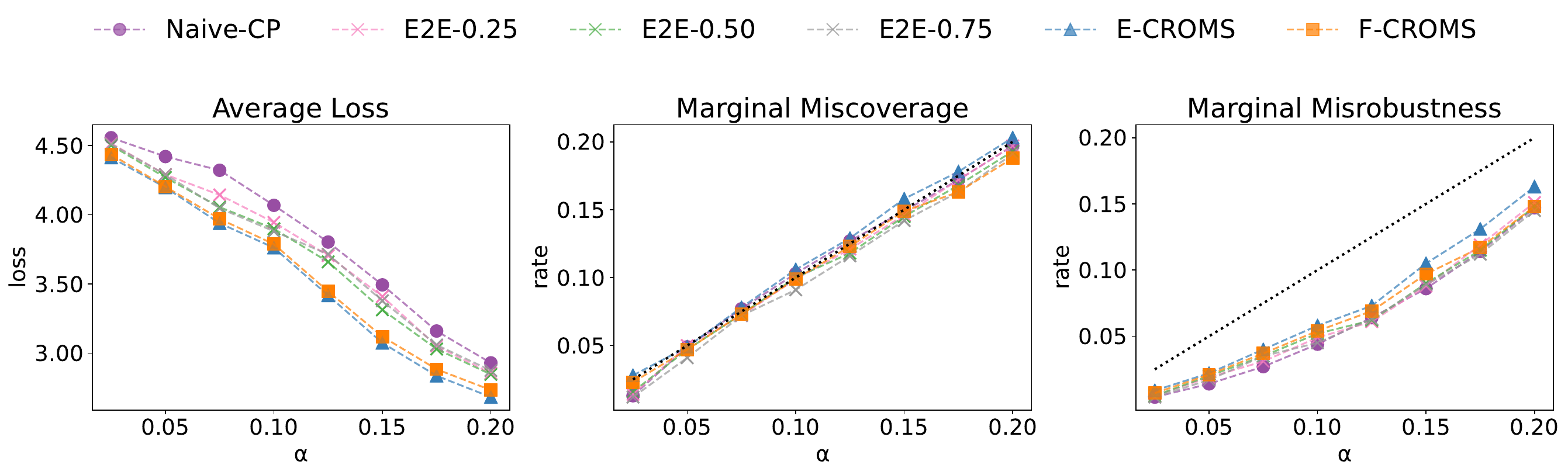}
\caption{The average loss, marginal coverage, and robustness in the classification task. Varying nominal level $\alpha$ with $n = 600$ and $|\Lambda| = 10$.}
\label{fig:ACSSLP-2}
\end{figure}




\subsubsection{Deferred settings and results in Section 6}\label{appen:real_data_COVID-19}

In Section 6, the loss matrix is given in the Table \ref{COVID-19cost}. The score function is defined as $S_{\lambda}(x,y) = 1-f_{\lambda}^{y}(x)$, where the classifier $f_{\lambda}:\mathcal{X} \rightarrow [0,1]^{4}$ is trained with the convolutional neural network (CNN). The CNN architecture consists of two convolutional layers, two pooling layers, two linear layers, three activation layers, and one softmax layer. The softmax layer outputs a 4-dimensional vector representing the probability of an input belonging to each of the four categories. We divided the original dataset into two parts. The first part is used for model training, and the second part is used for sampling both labeled data and test data. We train candidate models $\{S_{\lambda}: \lambda \in [4] \}$ on different datasets sampled from the former part, where these datasets have varying label distributions. Specifically,
\begin{itemize}
    \item $S_1$: Trained on dataset $\mathcal{D}_1$, which is sampled from the first dataset with a label distribution of $[0.15,0.35,0.35,0.15]$. The size of dataset is $|\mathcal{D}_1| = 1000$.
    \item $S_2$: Trained on dataset $\mathcal{D}_2$, with a label distribution $[0.35,0.35,0.15,0.15]$, $|\mathcal{D}_2|=1000$.
    \item $S_3$: Trained on dataset $\mathcal{D}_3$, with a label distribution $[0.20,0.30,0.20,0.30]$, $|\mathcal{D}_3|= 1000$.
    \item $S_4$: Trained on dataset $\mathcal{D}_4$, with a label distribution $[0.20,0.20,0.30,0.30]$, $|\mathcal{D}_4| = 1000$.
\end{itemize}

\begin{table}[H]
    \centering\footnotesize
    \setlength{\tabcolsep}{3pt}
    \renewcommand{\arraystretch}{0.8}
    \caption{The loss matrix of COVID-19 diagnosis in \citet{kiyani2025decision}.}
    \label{COVID-19cost}
    \resizebox{0.8\textwidth}{!}{
    \begin{tabular}{ccccc}
    \toprule
    \diagbox[dir=NW]{\textbf{Label ($y$)}}{\textbf{Decision ($z$)}}
    & {No Action} & {Antibiotics} & {Quarantine} & {Additional Testing} \\ 
    \midrule
    {Normal}           & 0                 & 8                     & 8                    & 6                           \\ 
    {COVID-19}         & 10                  & 7                     & 0                   & 2                           \\ 
    {Pneumonia}        & 10                  & 0                    & 7                    & 3                           \\ 
    {Lung Opacity}     & 9                  & 6                     & 6                    & 0                          \\
    \bottomrule
    \end{tabular}
    }
\end{table}

For similarity measurement, we integrate the softmax layer outputs from four models as a feature extractor $f_{\text{ex}}(x) = (f_{1}(x), f_2(x), f_3(x),f_4(x))$, which maps high-dimensional images into 16-dimensional feature vectors. The similarity between individuals is then computed using the Euclidean distance between these vectors.
The labeled and test data are both sampled from the second dataset with uniform label distribution $[0.25,0.25,0.25,0.25]$, with sample sizes set to $n = m = 300$. The bandwidth $h$ is set to 1.80, yielding an effective sample size $n_{\text{eff}} \approx  150$. Each ball $B \in \mathcal{B}$ contains $20\%$ test samples, with $|\mathcal{B}| = 50$.

Here, we present some relevant results. Let $\mathcal{G} = \{G_1,G_2,G_3,G_4\}$, where $G_1 = \{X \in \mathcal{X}: \hat{Y} = \text{COVID-19}\}$, $G_2 = \{X \in \mathcal{X}: \hat{Y} = \text{Lung-Opacity}\}$, $G_3 = \{X \in \mathcal{X}: \hat{Y} = \text{Normal}\}$, $G_4 = \{X \in \mathcal{X}: \hat{Y} = \text{Pneumonia}\}$. Prediction $\hat{Y}$ is obtained by statistically aggregating the outputs of four models for input $X$, where $\hat{Y}$ corresponds to the label with the ``highest vote count.'' For example, if Models 1, 2, and 3 output label ``COVID-19'', while Model 4 outputs label ``Normal'', then $\hat{Y}$ is determined to be ``COVID-19''. As shown in Figures \ref{fig:Medical diagnosis-CovGap}, \ref{fig:Medical diagnosis}, compared to other methods, \texttt{CROiMS} maintains conditional coverage consistently at the $1-\alpha$ level while achieving lower group conditional loss across all regions.

\begin{figure}[H] 
    \centering
    \includegraphics[width=.9\textwidth]{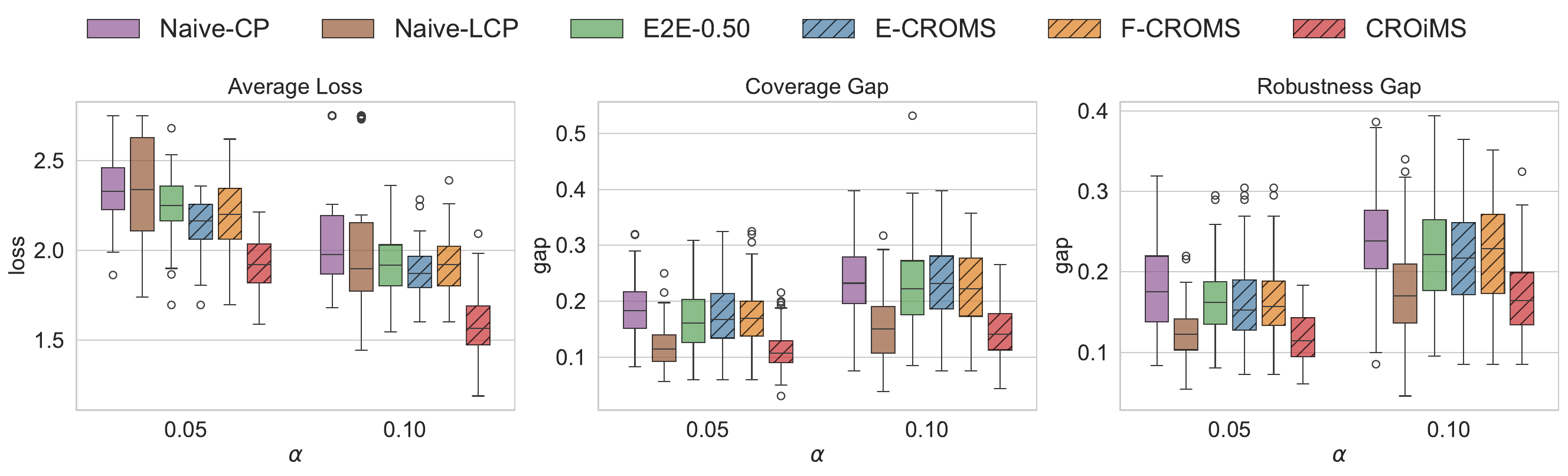}
    \caption{The evaluation metrics for the different nominal level $\alpha \in \{0.05,0.10\}$.}
    \label{fig:Medical diagnosis-CovGap}
\end{figure}

\begin{figure}[H]
    \begin{subfigure}[t]{\textwidth}
        \centering
        \includegraphics[width=.76\textwidth]{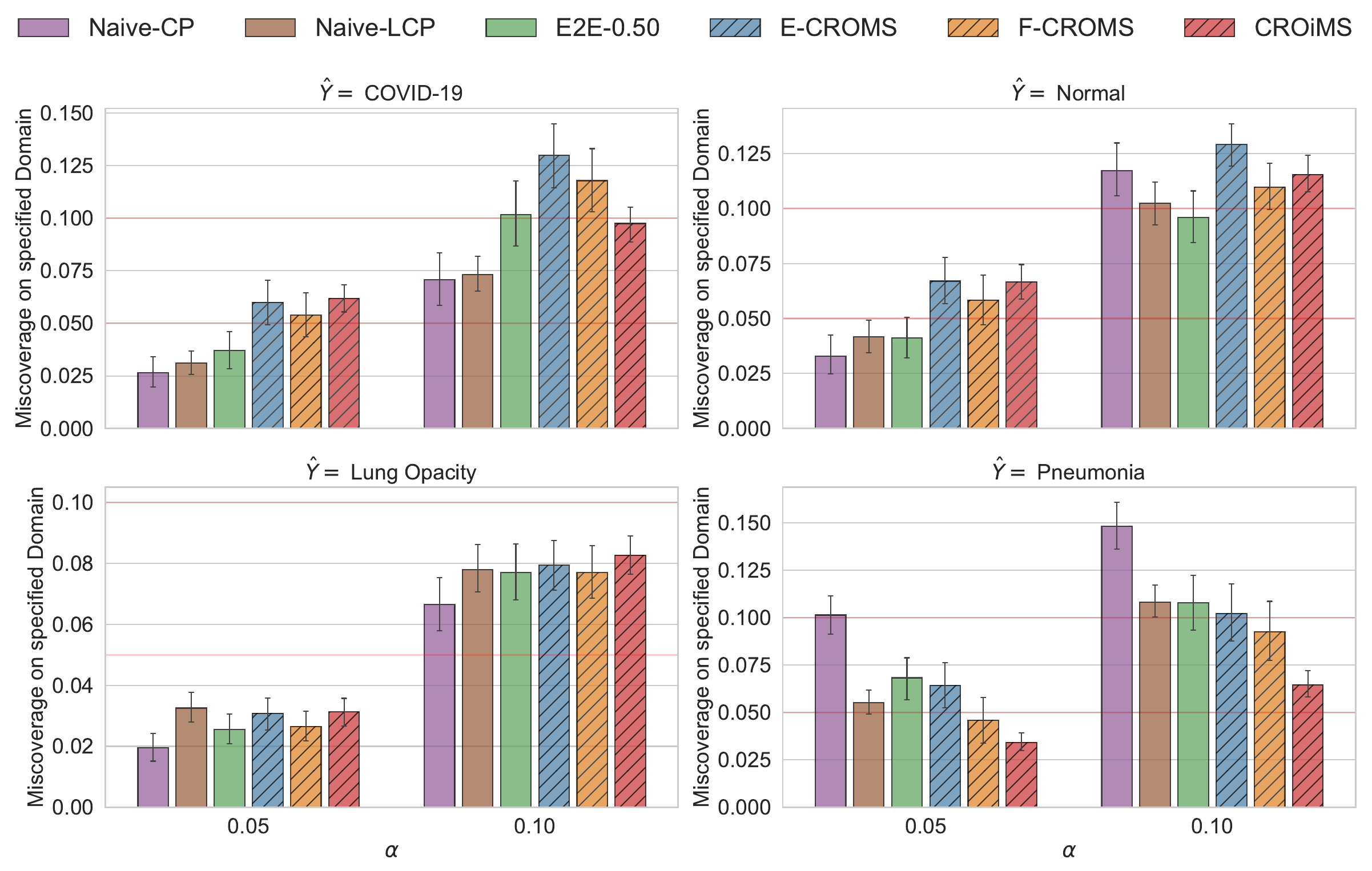}
        \caption{Conditional miscoverage}
    \end{subfigure}

    \begin{subfigure}[t]{\textwidth}
    \centering
        \includegraphics[width=.76\textwidth]{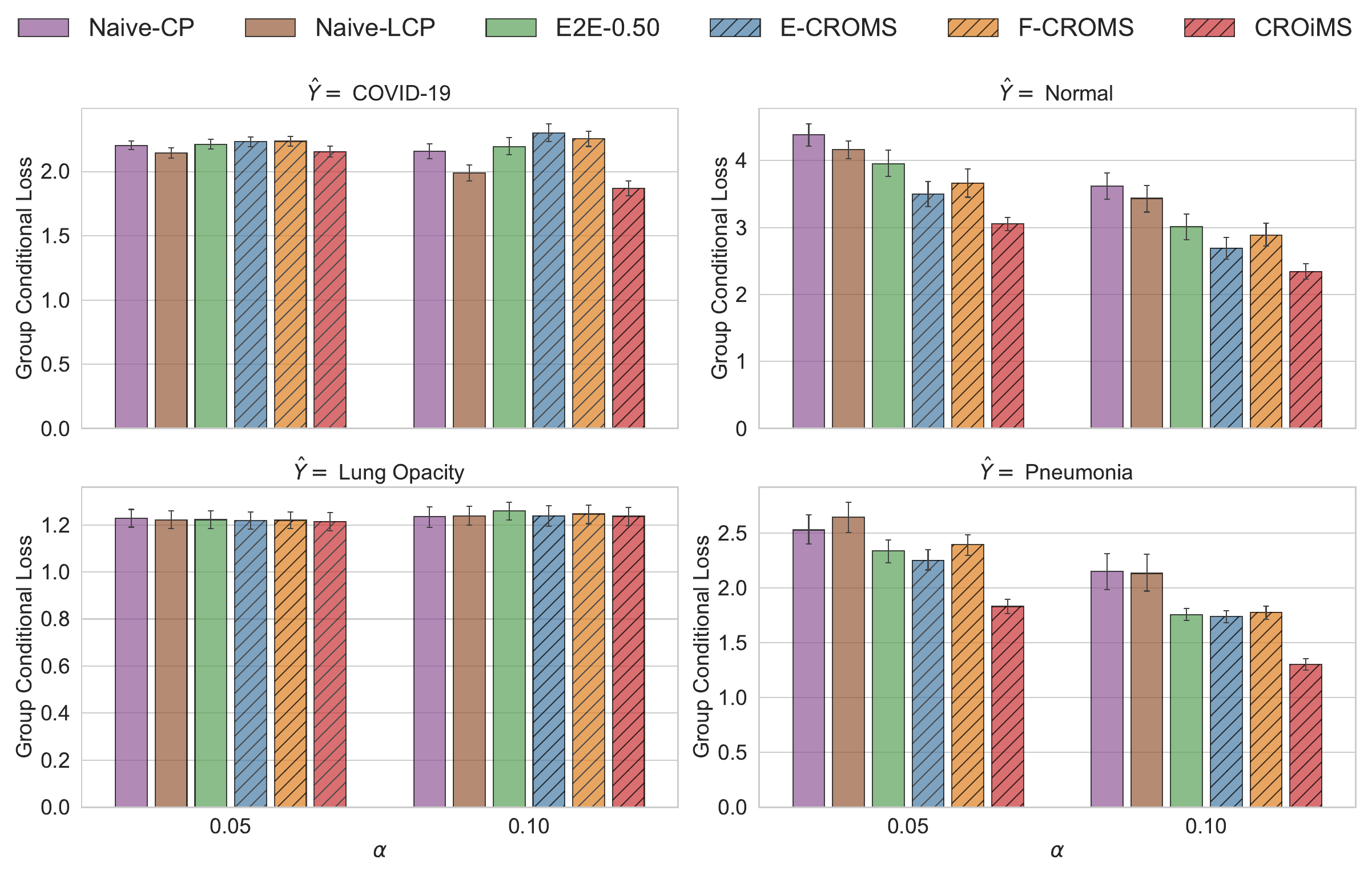}
        \caption{Conditional loss}
    \end{subfigure}
    \caption{The group conditional miscoverage and loss on COVID-19 Radiography Database under the different nominal level $\alpha \in \{0.05, 0.10\}$}
    \label{fig:Medical diagnosis}
\end{figure}



\subsection{Application on dermoscopic diagnosis}\label{Realdata: HAM10000}
The HAM10000 dataset \citep{tschandl2018ham10000} is a widely used collection of dermoscopic images for skin lesion classification. It contains 10,015 images of skin lesions from a diverse group of individuals. The dataset also includes each individual's basic demographic information (such as age) and dermatological diagnoses, covering 7 diagnostic categories. For our analysis, we treat the image as covariates ($X$) and the diagnostic results as two categories ($Y$): ``melanocytic nevus'' and ``others'' (including ``benign keratosis-like lesions'' and other malignant conditions). The loss matrix is established in Table \ref{HAM10000cost} to simulate decision costs in medical diagnosis. We take age as the group feature ($X_g$) and aim to investigate the optimal decision for each group.

\begin{table}[H]
    \centering\footnotesize
    \setlength{\tabcolsep}{3pt}
    \caption{The loss matrix of dermoscopic diagnosis.}\label{HAM10000cost}
    \begin{tabular}{ccccc}
    \toprule
    \diagbox[dir=NW]{\textbf{Label} $y$}{\textbf{Decision} $z$}
    & {No Action} & {Additional test} & {Disease}  \\ 
    \midrule
    {Melanocytic nevus}           & 0                 & 3                     & 6                \\ 
    {Others}         & 8                  & 4                     & 0       \\ 
    \bottomrule
    \end{tabular}
\end{table}

We divide the HAM10000 dataset into two parts: one for training different models and the other for sampling the labeled data and the test data. The former part are partitioned into four distinct training subsets based on individual ages: $\mathcal{D}_1 = \{(X,Y):0 \leq X_g<40\}$, $\mathcal{D}_2 = \{(X,Y):40 \leq X_g<55\}$, $\mathcal{D}_3 = \{(X,Y):55 \leq X_g<70\}$, and $\mathcal{D}_4 = \{(X,Y):70 \leq X_g\leq 85\}$. Then four separate models $\{S_{\lambda}: \lambda \in \Lambda\}$ with $\Lambda=4$ are trained on these subsets. In this experiment, the score function is $S_{\lambda}(x,y) = 1 - f_{\lambda}^{y}(x)$, where $x \in \mathbb{R}^{3\times224\times224}$ represents the image and the classifier $f_{\lambda}: \mathcal{X} \rightarrow [0,1]^{2}$ is trained using a CNN with DenseNet-121 architecture. Each replication randomly samples 500 labeled and test data, respectively. For the implementation of CROiMS, similarity between individuals is measured based on age differences, formulated as $H(X_g, X_g^{\prime}) = \exp\left(-\|X_g - X_g^{\prime}\|^2/h^2\right)$ with bandwidth $h=10$. Under this bandwidth, the corresponding effective sample size $n_{\text{eff}} \approx 200$. Every ball $B \in \mathcal{B}$ contains $10\%$ test samples,
and $|\mathcal{B}| = 10$.

The experimental results under different nominal levels $\alpha$ are displayed in Figure \ref{fig:HAM10000}. We see that \texttt{CROiMS} significantly outperforms all other methods, achieving the lowest average loss while maintaining marginal misrobustness and worst-case conditional misrobustness around the nominal levels. In Figure \ref{fig:HAM10000-CovLoss}, we further illustrate the group conditional losses across various age groups. The advantages of \texttt{CROiMS} are even more pronounced, with a nearly 41\% reduction in group conditional loss compared to other methods under $\alpha=0.1$ for $\text{age}\in[0,40]$. This demonstrates that adaptively selecting different models for different groups often leads to superior decision-making outcomes.


\begin{figure}[H] 
    \centering
    \includegraphics[width=1.0\textwidth]{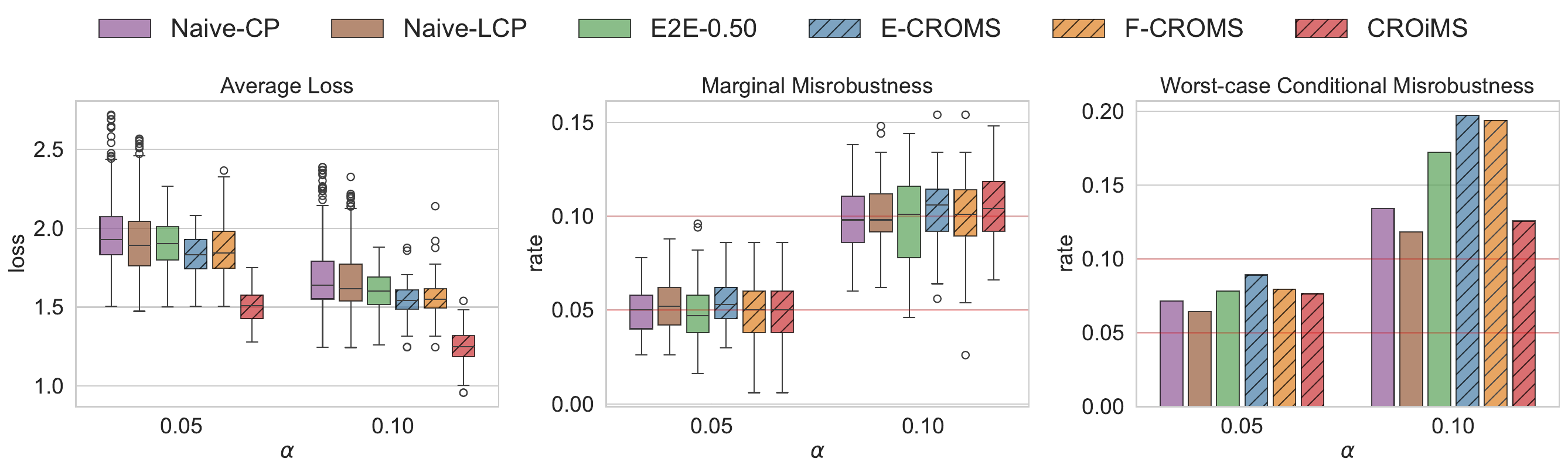}
    \caption{The average loss, marginal misrobustness, and worst-case conditional misrobustness on HAM10000 Dataset under the nominal level $\alpha = 0.05,0.10$. The candidate models are trained from four datasets with different ages.}
    \label{fig:HAM10000}
\end{figure}

\begin{figure}[H] 
    \centering
    \includegraphics[width=0.9\textwidth]{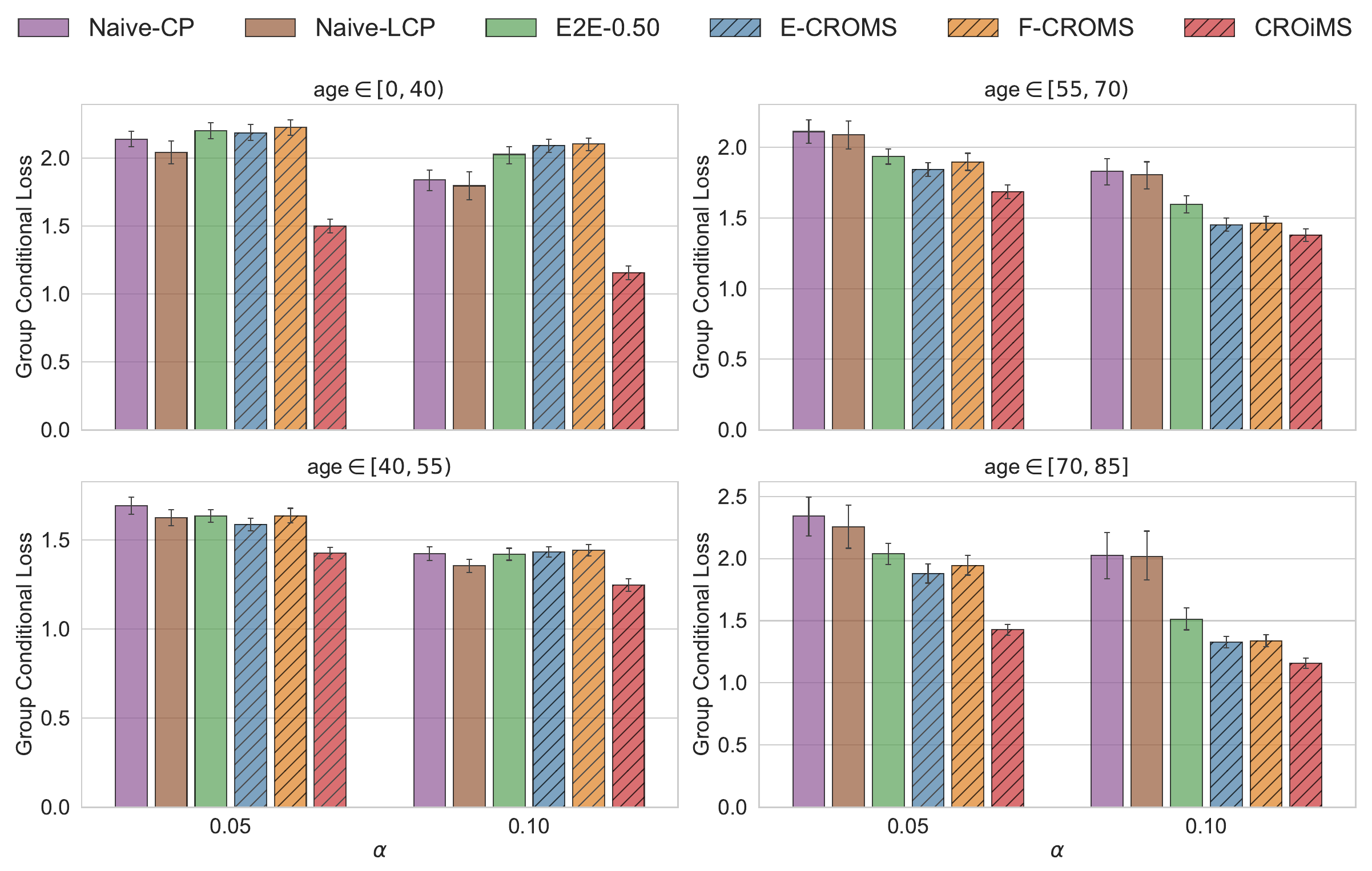}
    \caption{The group conditional loss on HAM10000 Dataset under the nominal level $\alpha \in \{0.05,0.1\}$. The candidate models are trained from four datasets with different ages.}
    \label{fig:HAM10000-CovLoss}
\end{figure}

\subsection{Comparison with the model selection methods for minimizing the prediction set size}

Some works select prediction models by minimizing the size of the prediction set, such as the EFCP and VFCP methods proposed by \citet{yang2024selection}, and the ModSel-CP method introduced by \citet{liang2024conformal}. However, it should be noted that the prediction set with the smallest size does not necessarily perform best in downstream tasks. 

In this section, we compare these three methods that minimize the size of the prediction set with our proposed new methods, E-CROMS and F-CROMS, in simulated decision-making problems. The results demonstrate that our proposed methods, which directly minimize the downstream task loss, perform more effectively in decision-making problems. Now we compare these methods with our proposals under the same experimental setup as in Section 5.1, including the data generation process, model training, and other configurations. The results are reported in Figure \ref{fig: EVM}, where both E-CROMS and F-CROMS achieve lower averaged decision loss than these baselines. It
also confirms that size efficiency does not necessarily imply decision efficiency.

\begin{figure}[H]
    \centering
    \begin{subfigure}[t]{\textwidth}
        \includegraphics[width=\textwidth]{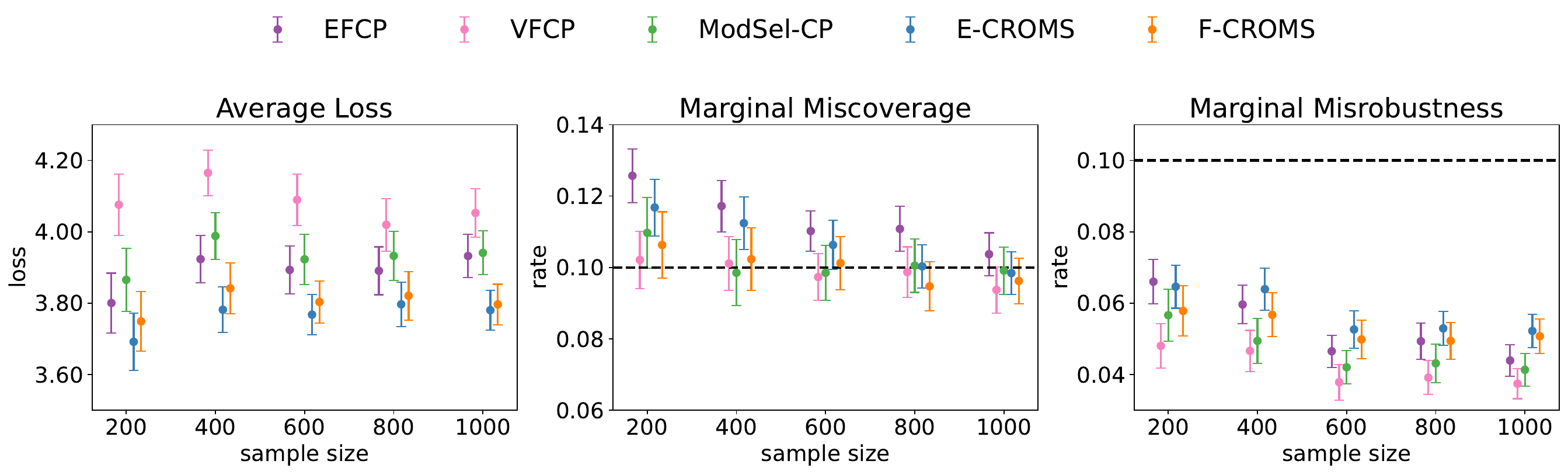}
        \caption{Varying sample size $n$ with $|\Lambda| = 10$ and $\alpha = 0.1$.}
    \end{subfigure}

    \begin{subfigure}[t]{\textwidth}
        \includegraphics[width=\textwidth]{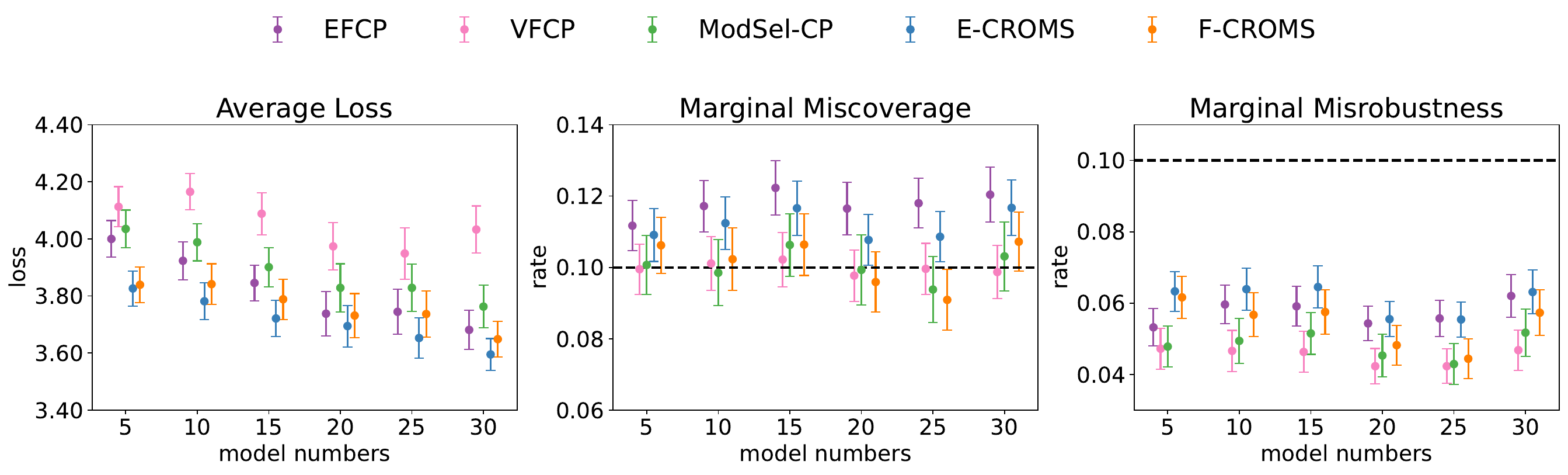}
        \caption{Varying candidate model numbers $|\Lambda|$ with $n = 400$ and $\alpha = 0.1$.}
    \end{subfigure}

    \begin{subfigure}[t]{1\textwidth}
        \includegraphics[width=\textwidth]{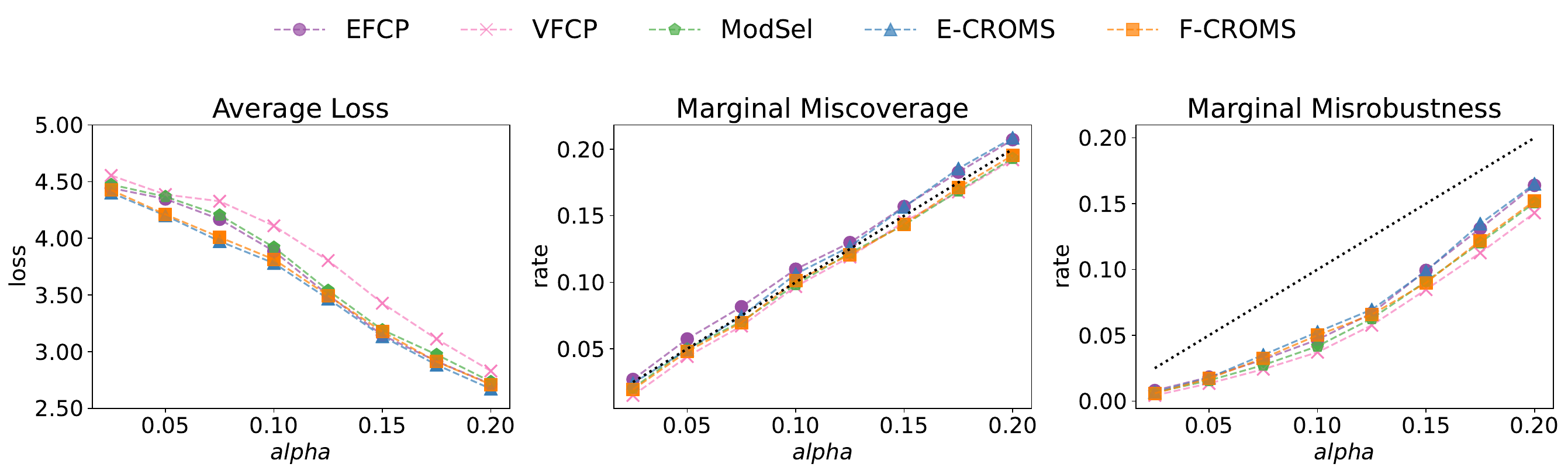}
        \caption{Varying nominal level $\alpha$ with $n = 600$ and $|\Lambda| = 10$.}
    \end{subfigure}
    \caption{The average loss, marginal coverage, and robustness in the classification task, where the candidate models are built on the same score function with different penalties.}
    \label{fig: EVM}
\end{figure}

\end{document}